%% file: neurips_2026_preprint.tex
\PassOptionsToPackage{dvipsnames,svgnames,x11names,table}{xcolor}
\documentclass{article}

\usepackage[main,preprint,nonatbib]{neurips_2026}

\usepackage[utf8]{inputenc}
\usepackage[T1]{fontenc}
\usepackage{hyperref}
\usepackage{url}
\usepackage{booktabs}
\usepackage{amsfonts}
\usepackage{nicefrac}
\usepackage{microtype}
\usepackage{xcolor}
\usepackage{bm}
\usepackage{array}
\usepackage{tikz}
\usetikzlibrary{arrows.meta,positioning,fit,calc,backgrounds,decorations.pathreplacing,shapes.geometric}
\usepackage{amsmath}
\usepackage{mathtools}
\usepackage{amsthm}
\usepackage{amssymb}
\usepackage{subcaption}
\usepackage{makecell}
\usepackage{multicol}
\usepackage{multirow} 
\usepackage{graphicx}
\usepackage[separate-uncertainty=true]{siunitx}
\usepackage{cleveref}
\usepackage{csquotes}
\usepackage{xfrac}
\usepackage{adjustbox}
\usepackage[most]{tcolorbox}
\usepackage{pifont}
\usepackage[framemethod=default]{mdframed}
\usepackage{todonotes}
\usepackage{ifthen}
\usepackage{etoc}
\usepackage[inline]{enumitem}
\theoremstyle{definition}
\newtheorem{definition}{Definition}

\usepackage{float} 
\theoremstyle{plain}
\newtheorem{theorem}{Theorem}

\newtheorem{lemma}{Lemma}
\newtheorem{proposition}{Proposition}
\newtheorem{proof-sketch}{Proof Sketch}

\usepackage[
  backend=biber,
  style=numeric-comp,
  backref=false,
  url=false,
  natbib=true,
  giveninits=true,
  maxbibnames=99,
  sorting=none,
]{biblatex}
\definecolor{Top1}{HTML}{0072B2} 
\definecolor{Top2}{HTML}{E69F00} 
\definecolor{Top3}{HTML}{CC79A7}

\newcommand{\pmstd}[2]{#1$\pm$#2}

\newcommand{\placeholder}[1]{\textit{#1}} 
\newcommand{\good}[1]{\textcolor{Top3}{#1}}
\newcommand{\bad}[1]{\textcolor{Top2}{#1}}

\newcommand{\paramcells}[6]{%
  \multicolumn{1}{r}{\scriptsize\itshape \#params} &
  \scriptsize #1 & \scriptsize #2 & \scriptsize #3 &
  \scriptsize #4 & \scriptsize #5 & \scriptsize #6\\}
  
\newcommand{\paramcellsimp}[7]{%
  \multicolumn{1}{r}{\scriptsize\itshape \#params} &
  \scriptsize #1 & \scriptsize #2 & \scriptsize #3 &
  \scriptsize #4 & \scriptsize #5 & \scriptsize #6 & \scriptsize #7\\}

\input{custom.tex}
\title{Polynomial Neural Sheaf Diffusion: A Spectral Filtering Approach on Cellular Sheaves}

\author{%
\small
\begin{tabular}{>{\centering\arraybackslash}m{0.28\textwidth}
                >{\centering\arraybackslash}m{0.28\textwidth}
                >{\centering\arraybackslash}m{0.28\textwidth}}
Alessio Borgi$^{1,2}$\thanks{Correspondence to: Alessio Borgi \texttt{<alessio.borgi@uniroma1.it, ab3352@cam.ac.uk>}}
&
Fabrizio Silvestri$^{2}$
&
Pietro Liò$^{1}$
\\[0.8em]
\multicolumn{3}{c}{%
\footnotesize
\begin{minipage}{0.75\textwidth}
\centering
$^{1}$Department of Computer Science and Technology,
University of Cambridge,\\ United Kingdom\\[-0.05em]
$^{2}$Department of Computer, Control and Management Engineering,
Sapienza University, Italy
\end{minipage}
}
\end{tabular}%
}

\begin{document}

\maketitle

\input{chapters/0_Abstract}

\input{chapters/1_Introduction}
\input{chapters/2_RelatedWorks}

\input{chapters/3_Methodology}

\input{chapters/4_Experiments}

\input{chapters/5_Conclusions}
\printbibliography
\newpage
\clearpage
\onecolumn
\input{chapters/A_Appendix}

\input{chapters/B_Appendix}
\input{chapters/C_Appendix}


\end{document}

%% file: custom.tex
\usepackage{wrapfig}
\usepackage{import}

\usepackage{bm}
\usepackage{mathtools}
\usepackage{amsthm}
\usepackage{amssymb}
\usepackage{xfrac}

\usepackage{subcaption}
\usepackage{makecell}
\usepackage{multicol}
\usepackage{multirow}
\usepackage{adjustbox}
\usepackage[separate-uncertainty=true]{siunitx}
\usepackage[inline]{enumitem}
\usepackage{etoc}

\usepackage{tikz}
\usetikzlibrary{arrows.meta,positioning,fit,calc,backgrounds,decorations.pathreplacing,shapes.geometric}

\usepackage{cleveref}
\usepackage{csquotes}

\usepackage[most]{tcolorbox}
\usepackage{pifont}
\usepackage{mdframed}

\newcommand{\cmark}{\textcolor{ForestGreen}{\ding{51}}}
\newcommand{\xmark}{\textcolor{BrickRed}{\ding{55}}}

\setuptodonotes{fancyline}

\makeatletter

\definecolor{Top1}{HTML}{0072B2}
\definecolor{Top2}{HTML}{E69F00}
\definecolor{Top3}{HTML}{CC79A7}

\definecolor{ufill}{HTML}{E1D5E7}
\definecolor{uborder}{HTML}{9673A6}
\definecolor{vfill}{HTML}{D5E8D4}
\definecolor{vborder}{HTML}{82B366}

\theoremstyle{definition}

%% file: chapters/0_Abstract.tex
\begin{abstract}
On heterophilic graphs, neighbouring nodes may carry incompatible features, so isotropic message passing can mix information in the wrong coordinate system and destroy class-discriminative signals. Sheaf Neural Networks (SheafNNs) overcome this by assigning local vector spaces to nodes and edges and by learning restriction maps that transport features before comparison. However, current SheafNNs still rely on first-order spatial diffusion: reaching distant nodes requires depth, while limited first-order propagation is often compensated by parameter-heavy stalks and dense per-edge maps. We propose \textbf{PolyNSD} (\textbf{Poly}nomial \textbf{N}eural \textbf{S}heaf \textbf{D}iffusion), a sheaf model-agnostic spectral diffusion operator that enhances first-order sheaf propagation with a learnable degree-$K$ polynomial filter of the normalised sheaf Laplacian, enabling explicit multi-hop, frequency-selective transport in a single layer. Across standard, filtered, malignant, and newly proposed heterophily benchmarks, PolyNSD improves over first-order NSD and reaches state-of-the-art or highly competitive performance, often with diagonal maps and small stalk dimensions. We further validate its benefits in federated causal graph learning and long-range transductive benchmarks, while diagnostics on depth, influence decay, spectral responses, and restriction-map geometry show that PolyNSD provides stable, interpretable, and efficient higher-order sheaf diffusion.
\end{abstract}

%% file: chapters/1_Introduction.tex
\section{Introduction}
Graph Neural Networks (GNNs) \citep{goller1996learning, gori2005new, scarselli2008graph, bruna2013spectral, defferrard2016convolutional, velickovic2017graph, gilmer2017neural} have become a standard tool for learning on relational data. However, they often underperform on \emph{heterophilic} graphs, where connected nodes may have different labels or incompatible features \citep{zhu2020beyond}, and suffer from \emph{oversmoothing}, where node representations become increasingly indistinguishable as depth grows \citep{nt2019revisiting, rusch2023survey}. A possible solution provided by \cite{graph_rewiring_for_heterophily} is to modify the graph by rewiring it with homophilic edges and pruning heterophilic ones. A more principled way to avoid that, modelling heterophily in the graph’s underlying topology, is via (cellular) \emph{sheaves} \cite{hansen2020sheaf, bodnar2022neural}: each node/edge carries a local feature space (a stalk) and edges carry linear restriction maps that specify how to align and compare features across incidences. The resulting sheaf Laplacian implements \emph{transport-aware} diffusion that can better accommodate heterophily than conventional, isotropic graph filters. However, existing neural sheaf diffusion layers still inherit a first-order spatial propagation mechanism. They are therefore (i) effectively one-step propagators, so reaching distant nodes requires stacking many layers; (ii) often reliant on dense, per-edge restriction maps to increase expressivity; and (iii) highly sensitive to the choice of stalk dimension, with stronger performance requiring larger local feature spaces. Together, these limitations tie expressivity to architectural scale: gains in accuracy often come only through deeper stacks, denser transports, or larger stalks, which increase parameter count and runtime and can make optimisation increasingly fragile as depth and stalk dimensionality grow.

We propose \emph{Polynomial Neural Sheaf Diffusion (PolyNSD)}, a model-agnostic sheaf spectral diffusion operator that addresses this propagation bottleneck directly and can be integrated into any sheaf-based architecture. We argue that SheafNNs are effective because they learn \emph{how} features should be aligned and compared across edges, but their diffusion should not be limited to one-hop, first-order propagation. Rather than repeatedly applying a spatial operator of the form $aI+bL$, where $L$ is the sheaf Laplacian, PolyNSD replaces this update with a learnable degree-$K$ polynomial filter of the normalised sheaf Laplacian. This yields explicit multi-hop, transport-aware propagation in a single layer while preserving the geometric structure learned by the sheaf. Indeed, rather than being limited to repeated local smoothing, PolyNSD can learn low-pass, high-pass, or band-pass responses over sheaf Fourier modes, allowing it to preserve smooth aligned signals, retain class-discriminative disagreement patterns, and mix information across multiple structural scales. This spectral view also improves efficiency: higher-order propagation is controlled by only $K{+}1$ scalar coefficients per layer and evaluated through stable recurrence relations, without requiring an explicit eigendecomposition. Because the same learned restriction maps are reused throughout the polynomial recurrence, PolyNSD decouples the receptive field of each layer from both repeated one-hop diffusion and increasingly complex transports. Importantly, this does not prevent depth: multiple PolyNSD layers can still be stacked, each with its own polynomial recurrence and learned spectral response. As a result, strong performance can be achieved with fewer layers, small stalk dimensions, and simple diagonal restriction maps, reducing parameters, memory, and runtime.

\noindent\textbf{Contributions.} Our main contributions are as follows:
\begin{enumerate}[leftmargin=*]
    \item We propose \textsc{PolyNSD}, a model-agnostic spectral diffusion operator for Sheaf Neural Networks that replaces first-order spatial propagation with learnable orthogonal-polynomial filters of the normalised sheaf Laplacian, enabling higher-order sheaf diffusion.

    \item We formulate sheaf propagation as a \emph{frequency-selective multi-hop filtering} problem, allowing each layer to obtain an explicit $K$-hop receptive field and to learn low-, high-, or band-pass responses over sheaf Fourier modes, instead of relying on repeated local smoothing.

    \item We show that polynomial spectral filtering improves the SheafNNs accuracy--efficiency trade-off, improving over first-order NSD across homophilic and heterophilic benchmarks, often using only diagonal restriction maps and small stalk dimensions.

    \item We validate \textsc{PolyNSD} beyond standard node classification through extensive ablations and diagnostics, including alternative split protocols, polynomial-order and spectral-scaling studies, depth and oversmoothing analyses. We also extend PolyNSD to a continuous-time variant and apply it to federated causal graph learning and long-range transductive benchmarks.
\end{enumerate}

%% file: chapters/2_RelatedWorks.tex
\section{Background and Related Works}
\label{sec:related_work}

\noindent\textbf{Sheaf Neural Networks, Heterophily and Oversmoothing.}
GNNs have evolved from early message–passing formulations to a family of architectures that trade off locality, expressivity, and efficiency. Canonical baselines include spectral and spatial convolutions \cite{defferrard2016convolutional, bruna2013spectral, kipf2016semi}, attention mechanisms \cite{velickovic2017graph}, principled aggregates \cite{gilmer2020} and transformer-based versions \cite{yun2019graph, dwivedi2020generalization}. Despite this progress, two persistent pathologies limit standard GNNs. The former is \emph{Oversmoothing}, which arises as layers deepen, since repeated low-pass propagation collapses node features toward a near-constant signal and yielding accuracy drop-offs beyond shallow depth \cite{nt2019revisiting, rusch2023survey}. \emph{Heterophily} further stresses isotropic message passing: when adjacent nodes belong to different classes or carry contrasting attributes, the averaging operation blurs the high-frequency signals that separate classes, and performance degrades as homophily decreases \cite{zhu2020beyond}. These phenomena are tightly linked in practice and motivate transport-aware architectures that decouple \emph{who} communicates from \emph{how} features are compared.   

\noindent\textbf{Sheaf Neural Networks.}
Cellular sheaf theory \cite{shepard1985cellular, curry2014sheaves} equips graphs with local feature spaces (stalks) and a linear map (restriction map) for each incident node-edge pair, enabling transport-aware diffusion that can better handle heterophily than isotropic message passing. SheafNNs, originally introduced using a hand-crafted sheaf with a single dimensionality in  \cite{hansen2020sheaf} and further improved by learning the sheaf through a parametric function \cite{bodnar2022neural}, demonstrated strong performance under heterophily and against oversmoothing, by instantiating diffusion on the sheaf Laplacian. Subsequent works explored attention on sheaves \cite{barbero2022sheaf_attentional}, learning the graph connection Laplacian directly from data and at preprocessing time \cite{barbero2022sheaf}, sheaves-based positional encoding \cite{he2023sheaf}, introducing non-linearities in the process \cite{zaghen2024nonlinear}, handling graph heterogeneity \cite{braithwaite2024heterogeneous}, sheaf hypergraphs \cite{duta2023sheaf, mule2025directional} and directional extensions \cite{fiorini2025sheavesreloadeddirectionalawakening}, applications to recommendation systems \cite{purificato2023sheaf4rec} and federated learning settings \cite{nguyen2024sheaf}, and more general frameworks such as Copresheafs \cite{hajij2025copresheaftopologicalneuralnetworks}. 

\begin{figure*}[t]
  \centering
  \includegraphics[width=0.99\linewidth]{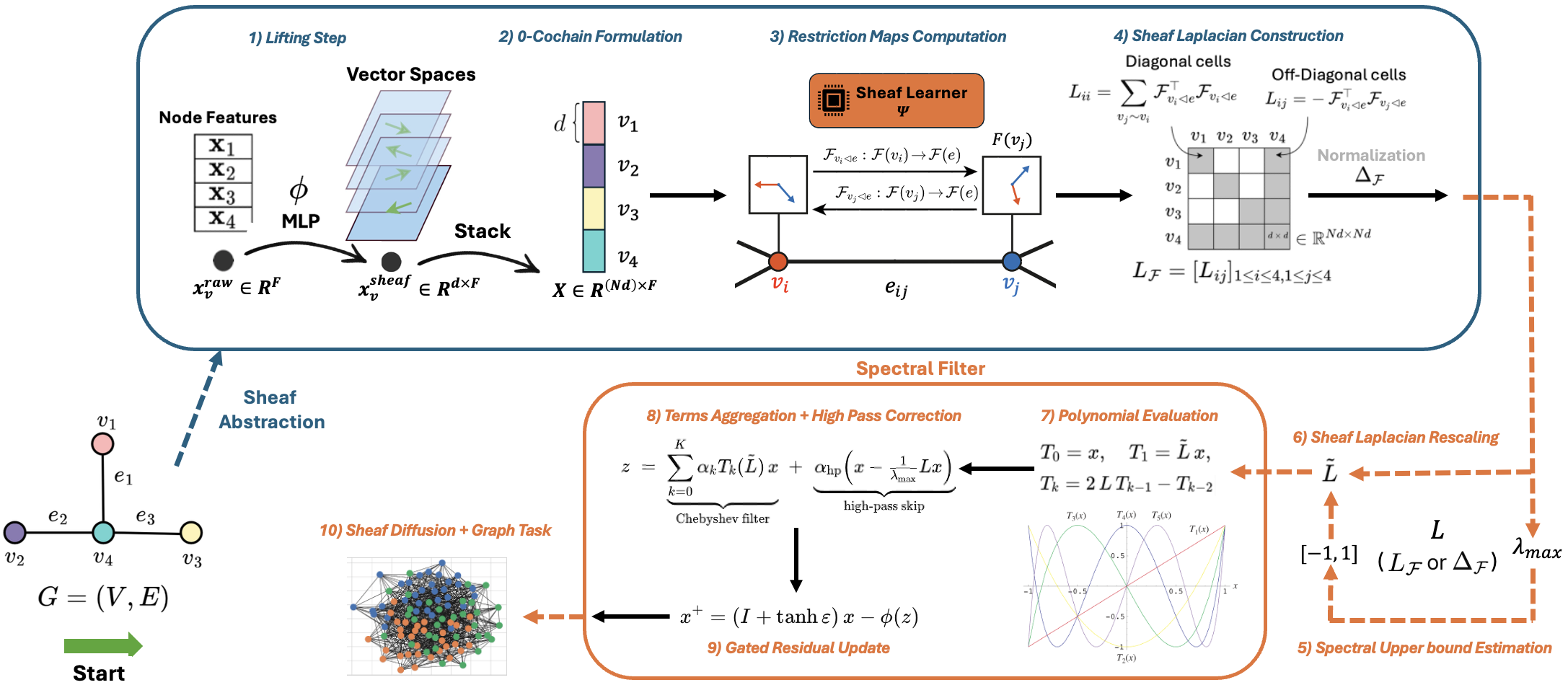}
  \caption{\emph{Polynomial Sheaf Neural Network Architecture.} Graph's raw node features are lifted to \(d\)-dimensional stalk signals and stacked into a \(0\)-cochain. A sheaf learner predicts edge-wise restriction maps and defines the vertex sheaf Laplacian \(L\). PolyNSD then estimates \(\lambda_{\max}\), rescales \(L\) to \(\widetilde L\in[-1,1]\), and applies a degree-\(K\) spectral filter \(p_\theta(\widetilde L)\) via a stable three-term recurrence, followed by a lightweight high-pass correction and a gated residual update, finally enabling \(K\)-hop mixing within a single diffusion layer.}
  \label{fig:polynsd-layer}
\end{figure*}

\noindent\textbf{Cellular Sheaves on Graphs.} Formally, let $\mathcal G=(\mathcal V,\mathcal E)$ be a finite undirected graph with an arbitrary orientation on edges. A \emph{Cellular Sheaf} $\mathcal F$ on $\mathcal G$ assigns a finite-dimensional inner-product space (the \emph{stalk}) $\mathcal F(v)$ to each vertex $v\in\mathcal V$ and $\mathcal F(e)$ to each edge $e\in\mathcal E$, together with linear restriction map $\mathcal F_{v\unlhd e}:\ \mathcal F(v)\to \mathcal F(e), \forall\text{ incident pair }v\unlhd e$. By stacking the vertex and edge stalks, we can obtain, respectively, the $0$- and $1$-cochain spaces $C^0(\mathcal G;\mathcal F)\ :=\ \bigoplus_{v\in\mathcal V}\mathcal F(v),   
C^1(\mathcal G;\mathcal F)\ :=\ \bigoplus_{e\in\mathcal E}\mathcal F(e)$. 
Given the chosen orientation (e.g., $e=(u\!\to\!v)$), we can use them to construct the \emph{sheaf coboundary} $\delta:C^0(\mathcal G; \mathcal F)\to C^1(\mathcal G; \mathcal F)$, which acts edgewise as $\delta (x)_e\ =\ \mathcal F_{v\unlhd e}\,x_v\ -\ \mathcal F_{u\unlhd e}\,x_u$. Using the co-boundary, one can define the \emph{Sheaf Laplacian} as $L_{\mathcal F}\ :=\ \delta^T \delta\ $, or node-wise as $L_{\mathcal F} (x)_u \;=\; \sum_{u,v \unlhd e} \, \mathcal F_{u\unlhd e}^{\top}\!\big(\mathcal F_{u\unlhd e}\,x_u \;-\; \mathcal F_{v\unlhd e}\,x_v\big)$. Its normalised version is obtained by reweighting the sheaf’s inner products as: $\Delta_{\mathcal F} \;:=\; D^{-1/2} \, L_{\mathcal F} \, D^{-1/2}$, where $D$ is the block diagonal of $L_{\mathcal F}$.


\noindent\textbf{Neural Sheaf Diffusion (NSD).}
\label{subsec:nsd}
Starting from the same undirected graph $\mathcal G$, with each node $u\in\mathcal V$ carrying a $d$-dimensional feature vector $x_u\in\mathcal F(u)$, we can stack node features (after a projection to the stalk space), into the column $x\in C^0(\mathcal G;\mathcal F)$, and by allowing up to $f$ feature channels, we can collect them in a matrix $X\in\mathbb R^{nd\times f}$. 
Following \cite{bodnar2022neural}, the edge-wise restriction/transport maps that define the sheaf at layer (or time) $t$ are computed from incident node features via a learnable function. For $e=(u,v)$, we have $\mathcal F_{u\unlhd e}^{(t)} \;=\; \Phi^{(t)}(x_u,x_v),
\mathcal F_{v\unlhd e}^{(t)} \;=\; \Phi^{(t)}(x_v,x_u) $, where, typically, $F_{u\unlhd e}^{(t)} =\Phi^{(t)}(x_u,x_v)=\text{MLP}(x_u||x_v)$ and similarly for $F_{v\unlhd e}^{(t)}$. These can produce matrices of the appropriate shape (diagonal, bundle, or general), that can be used to build the sheaf Laplacian ($L_{\mathcal F(t)}$ or $\Delta_{\mathcal F(t)}$) and perform a sheaf-aware diffusion step on $X$ as: $X^{(t+1)} = X^{(t)} - \sigma\!\left(\Delta_{\mathcal F(t)} (I_{nd}\!\otimes\! W_1^{(t)})\, X^{(t)} W_2^{(t)}\right)$, where $W_1^{(t)}$ and $W_2^{(t)}$ are trainable weight matrices, $I_{nd}$ is the identity, $\otimes$ denotes the Kronecker product, and $\sigma(\cdot)$ is a nonlinearity. Sheaf Diffusion replaces standard graph diffusion by measuring disagreements \emph{after} transporting node features into each edge’s discourse space, and then aggregating them back to nodes via the sheaf Laplacian.

\noindent\textbf{Spectral and Polynomial Graph Filters.}
A large body of GNN methods designs filters in the spectrum of the graph Laplacian, from early spectral CNNs and graph wavelets \cite{bruna2013spectral,hammond2011wavelets,shuman2013emerging} to localized polynomial approximations that avoid explicit eigen-decompositions \cite{defferrard2016convolutional,kipf2016semi}. Chebyshev-style recurrences and related rational/polynomial filters offer stable, scalable propagation and interpretable frequency responses \cite{levie2018cayleynets,bianchi2021graph}. Further developments include Lanczos-based personalized PageRank diffusion and precomputed multi-hop schemes \cite{liao2019lanczosnet,klicpera2019predict,wu2019simplifying,rossi2020sign,chien2021adaptive}. These ideas established that multi-hop context and frequency selectivity can be achieved with sparse matrix–vector primitives instead of decompositions. Our work adopts this philosophy, but lifts it from the graph Laplacian to the \emph{sheaf} Laplacian.


%% file: chapters/3_Methodology.tex
\section{Polynomial Neural Sheaf Diffusion}
\label{sec:polysd}
\noindent\textbf{Problem Setting \& Architectural Overview.}
We consider a graph $G=(V,E)$ with $N=|V|$ nodes and raw node features $\{x_v^{\mathrm{raw}}\in\mathbb R^{F}\}_{v\in V}$. As illustrated in \autoref{fig:polynsd-layer}, PolyNSD first lifts each raw feature vector into a $d$-dimensional stalk space through a map $\phi$, obtaining $x_v=\phi(x_v^{\mathrm{raw}})\in\mathbb R^{d\times F}$. Stacking the lifted node features over all vertices gives a $0$-cochain $x\in C^{0}(G;\mathcal F)\cong\mathbb R^{(Nd)\times F}$ (steps 1--3). A parametric sheaf learner $\Psi$ then predicts the edge-wise restriction maps $\{\mathcal F_{v\unlhd e}\}$, which may be diagonal, bundle-valued, or general linear maps (step 4). These restriction maps define the vertex sheaf Laplacian $L\in\mathbb R^{Nd\times Nd}$, either in its unnormalised form $L_{\mathcal F}$ or its degree-normalised form $\Delta_{\mathcal F}$. In Neural Sheaf Diffusion (NSD) \citep{bodnar2022neural}, propagation is based on a first-order diffusion step of the form $x\mapsto(aI+bL)x$, applied repeatedly across layers. PolyNSD replaces this repeated local diffusion with a learnable degree-$K$ polynomial spectral filter of the sheaf Laplacian (steps 5--10). For $K=1$, this recovers the first-order NSD operator up to normalisation, while for $K>1$, it yields higher-order polynomials in $L$, enabling explicit $K$-hop transport-aware mixing in a single layer. A degree-$K$ PolyNSD layer requires $K$ sparse--dense multiplications of $\widetilde L$, with cost $O(K\,\mathrm{nnz}(L)\,C)$, matching the asymptotic cost of stacking $K$ NSD layers but avoiding repeated sheaf prediction and Laplacian re-assembly. The filtered signal is then stabilised with a lightweight spectral correction and a gated residual update before being passed to the task-specific prediction head. (Further details on each step are provided in \autoref{app:polynsd_layer}, with an extended NSD comparison in Appendix \autoref{app:polynsd-vs-nsd}).

\noindent\textbf{From Spatial Diffusion to Polynomial Sheaf Spectral Filtering.}
\label{subsec:polysd-poly-filter}
Although first-order sheaf diffusion is effective, it inherits two key limitations of local message passing. First, long-range communication requires stacking many diffusion layers, which increases computational cost. Second, the behaviour of the diffusion operator is spectrally constrained: repeated applications of $(aI+bL)$ mostly induce a fixed smoothing pattern, offering limited control over which sheaf-frequency components are preserved, suppressed, or amplified. PolyNSD addresses these limitations by reformulating sheaf diffusion as a \emph{learnable spectral filtering} problem: rather than repeatedly applying a local first-order operator, we parameterise propagation directly as a function of the sheaf Laplacian spectrum. This gives each layer a controllable frequency response while preserving the transport-aware geometry encoded by the learned restriction maps.

\begin{definition}[Polynomial learnable sheaf spectral filter]
\label{def:polynomial-learnable-sheaf-spectral-filter}
\emph{Let $L\in\mathbb R^{Nd\times Nd}$ be a symmetric positive semidefinite sheaf Laplacian with eigendecomposition $L=U\Lambda U^\top$, where $\Lambda=\operatorname{diag}(\lambda_1,\ldots,\lambda_{Nd})$. A \emph{learnable sheaf spectral filter} is an operator of the form:
\begin{equation}
\label{eq:learnable-sheaf-spectral-filter}
    p_\theta(L)
    =
    U p_\theta(\Lambda) U^\top,
    \qquad
    p_\theta(\Lambda)
    =
    \operatorname{diag}\big(p_\theta(\lambda_1),\ldots,p_\theta(\lambda_{Nd})\big)
\end{equation}
where $p_\theta(\lambda)$ is a learnable scalar response over sheaf frequencies. In PolyNSD, this response is parameterised as a degree-$K$ polynomial:
\begin{equation}
\label{eq:polysd-poly-filter}
    p_\theta(\lambda)
    =
    \sum_{k=0}^{K} c_k \lambda^k,
    \qquad
    y
    =
    p_\theta(L)x
    =
    \sum_{k=0}^{K} c_k L^k x
\end{equation}
with learnable scalar coefficients $\theta=\{c_k\}_{k=0}^{K}$ shared across nodes and stalk coordinates.}
\end{definition}

Since $U$ defines an orthonormal basis of \emph{sheaf Fourier modes}, $p_\theta(L)$ acts diagonally in this basis: the $i$-th mode is scaled by the multiplier $p_\theta(\lambda_i)$. Learning the polynomial coefficients therefore directly controls the layer's frequency response, allowing PolyNSD to learn low-, band-, or high-pass sheaf filters. In view of the Dirichlet-energy identity $\langle x,Lx\rangle=\sum_i\lambda_i\widehat{x}_i^2$, decreasing responses emphasise smooth, globally aligned sheaf signals, while band-pass or high-pass responses retain or amplify transported disagreement patterns. Moreover, because $L^k$ propagates information along paths of length up to $k$, the degree-$K$ polynomial gives each layer an explicit $K$-hop receptive field without stacking $K$ first-order diffusion steps. This extends the spectral theory of cellular sheaves \citep{hansen2019toward} to learnable filters and connects with opinion-dynamics interpretations in which $p_\theta(L)$ explicitly shapes how agreement and disagreement propagate across the sheaf. Further details and proofs on the underlying \emph{Sheaf Spectral Theory} are provided in Appendix~\autoref{app:spectral-selfadjoint}--\autoref{app:spectral-smoothness}.

\noindent\textbf{K-hop Locality.}
Although PolyNSD is defined spectrally, its polynomial form also has a clear spatial interpretation. Each multiplication by the sheaf Laplacian $L$ propagates information across one graph edge, while $L^k$ mixes information along paths of length at most $k$. Therefore, a degree-$K$ filter has an explicit $K$-hop receptive field, allowing a single PolyNSD layer to perform multi-hop transport-aware propagation without stacking $K$ first-order message-passing layers.
\begin{proposition}
\label{prop:polysd-khop}
Let $L$ be a sparse block sheaf Laplacian on a graph $G=(V,E)$, whose off-diagonal block $(v,u)$ can be nonzero only if $(v,u)\in E$, while diagonal blocks are arbitrary positive semidefinite matrices. Let $p(L)$ be the degree-$K$ polynomial filter in \autoref{eq:polysd-poly-filter}. Then:
\begin{equation}
    \big(p(L)\big)_{vu}=0
    \qquad
    \text{whenever}
    \qquad
    \operatorname{dist}_G(v,u)>K 
\end{equation}
(Proof provided in Appendix~\autoref{app:poly-proofs}).
\end{proposition}

\noindent\textbf{Commutation and Dirichlet-Energy Control.}
Polynomial filters of the same sheaf Laplacian satisfy two useful structural properties. First, they share the same eigenbasis and therefore commute: composing two filters $p(L)$ and $q(L)$ is equivalent to applying the single polynomial filter $(pq)(L)$. Hence, in the linear case, stacking PolyNSD layers with a fixed Laplacian does not introduce a different class of operators; it simply increases the effective polynomial degree. Second, bounded spectral multipliers provide direct control over the sheaf Dirichlet energy. If $0\leq p(\lambda)\leq 1$ over the spectrum of $L$, then $p(L)$ cannot increase transported disagreement: it can only damp sheaf-frequency components according to their learned spectral response.
\begin{proposition}
\label{prop:polysd-commute-energy}
Let $L$ be a symmetric positive semidefinite sheaf Laplacian and let $p,q$ be real polynomials. Then:
\begin{equation}
    p(L)q(L)=q(L)p(L)=(pq)(L)
\end{equation}
Moreover, if $L=U\Lambda U^\top$, $x=U\widehat{x}$, and $p(\lambda)$ satisfies $0\leq p(\lambda)\leq 1$ for every $\lambda\in\sigma(L)$, then:
\begin{equation}
\label{eq:polysd-energy-monotone}
    \big\langle p(L)x,\ Lp(L)x\big\rangle
    =
    \sum_{i=1}^{Nd}\lambda_i p(\lambda_i)^2 \widehat{x}_i^2
    \leq
    \sum_{i=1}^{Nd}\lambda_i \widehat{x}_i^2
    =
    \langle x,Lx\rangle 
\end{equation}
(Proof provided in Appendix~\autoref{app:poly-proofs}).
\end{proposition}

\noindent\textbf{Choice of Basis and Chebyshev Parameterisation.}
Learning the monomial coefficients in \autoref{eq:polysd-poly-filter} is expressive, but numerically unstable at moderate or large degrees, since monomial bases are poorly conditioned on bounded intervals. We therefore parameterise the filter in an orthogonal polynomial basis, after rescaling the spectrum of the sheaf Laplacian to $[-1,1]$. Given an upper bound $\lambda_{\max}$ such that $\sigma(L)\subset[0,\lambda_{\max}]$, we define:
\begin{equation}
\label{eq:polysd-rescale}
    \widetilde L
    =
    \frac{2}{\lambda_{\max}}L-I,
    \qquad
    \sigma(\widetilde L)\subset[-1,1]
\end{equation}
For the degree-normalised sheaf Laplacian $\Delta_{\mathcal F}$, we use the canonical bound $\lambda_{\max}=2$, while for the unnormalised $L_{\mathcal F}$ we use either a Gershgorin-type analytic bound or a short power iteration (see Appendix~\autoref{subsec:lambda-max} for further details). We then write the spectral response as:
\begin{equation}
\label{eq:polysd-basis-expansion}
    p_\theta(\lambda)
    =
    \sum_{k=0}^{K}\theta_k B_k\!\left(\xi(\lambda)\right),
    \qquad
    \xi(\lambda)=\frac{2}{\lambda_{\max}}\lambda-1
\end{equation}
where $\{B_k\}_{k=0}^{K}$ is an orthogonal polynomial basis on $[-1,1]$. In our default implementation, we use first-kind Chebyshev polynomials, $B_k(\xi)=T_k(\xi)$, which satisfy $|T_k(\xi)|\leq 1$ for $\xi\in[-1,1]$. This boundedness ensures that Chebyshev filters remain stable after spectral rescaling, while convex coefficients $\theta=\mathrm{softmax}(\eta)$, $\eta\in\mathbb R^{K+1}$, yield uniformly controlled spectral multipliers. The implementation remains \emph{basis-agnostic} (alternative orthogonal polynomial formulation used can be found in Appendix~\autoref{app:cheb-bound}).

\noindent\textbf{High-Pass Skip and Gated Residual.}
To counteract the low-pass bias of diffusion, we augment the polynomial filter with a lightweight high-pass correction and a gated residual path. Given $L$ and an upper bound $\lambda_{\max}$, we define:
\begin{equation}
\label{eq:polynsd-core-update-prev}
    h_{\mathrm{hp}}
    =
    x-\lambda_{\max}^{-1}Lx,
    \qquad
    z
    =
    p_\theta(\widetilde L)x
    +
    \alpha_{\mathrm{hp}}h_{\mathrm{hp}}
\end{equation}
where $\alpha_{\mathrm{hp}}\in\mathbb R$ is learnable. The layer output is then:$
    x^{+}
    =
    \big(I+\tanh\varepsilon\big)x
    -
    \phi(z)$, with a $1$-Lipschitz nonlinearity $\phi$ and a diagonal residual gate $\varepsilon$.
Since $\widetilde L$ is an affine function of $L$, the polynomial term and the high-pass correction share the same eigenbasis. For $L=U\Lambda U^\top$, the pre-nonlinear map acts on each sheaf Fourier mode with multiplier:
\begin{equation}
\label{eq:polynsd-spectral-multiplier}
    m(\lambda)
    =
    p_\theta\!\left(\frac{2\lambda}{\lambda_{\max}}-1\right)
    +
    \alpha_{\mathrm{hp}}
    \left(1-\frac{\lambda}{\lambda_{\max}}\right),
    \qquad
    \lambda\in\sigma(L)
\end{equation}
Thus, $\alpha_{\mathrm{hp}}$ provides a simple, interpretable way to deform the learned spectral response, adjusting the balance between low- and high-frequency components. If $p_\theta(\xi)\geq 0$ on $[-1,1]$ and $\alpha_{\mathrm{hp}}>0$, then $m(\lambda)>0$ for all $\lambda\in[0,\lambda_{\max})$, preventing non-harmonic modes from being accidentally cancelled. The gated residual controls the deviation from the identity, while the boundedness of $p_\theta(\widetilde L)$ and $I-\lambda_{\max}^{-1}L$ yields an explicit Lipschitz control for the update. (Further details and proofs are provided in Appendix~\autoref{app:hp-gate}).

%% file: chapters/4_Experiments.tex
\section{Experimental Evaluation}
\label{sec:experiments}

\input{tables/big_leaderboard_polynsd} 
We evaluate \emph{Polynomial Neural Sheaf Diffusion (PolyNSD)} on a broad set of graph learning settings, including standard real-world and synthetic node-classification benchmarks, filtered and newly proposed heterophily datasets, federated causal graph learning, and long-range transductive tasks. Beyond accuracy, we analyse the behaviour of PolyNSD through depth and Dirichlet-energy diagnostics, long-range influence decay, learned spectral responses, restriction-map geometry, and a continuous-time Neural Sheaf ODE extension. These experiments test four main claims: (i) PolyNSD improves over first-order NSD across homophilic and heterophilic regimes; (ii) higher-order polynomial filtering ($K>1$) provides better long-range, frequency-selective propagation than repeated first-order diffusion; (iii) strong performance can often be achieved with diagonal maps and small stalk dimensions, improving the accuracy--efficiency trade-off of SheafNNs; and (iv) the learned spectral and geometric diagnostics reveal stable, interpretable, and task-adaptive sheaf diffusion.

\input{tables/platonov_main_paper}

\noindent\textbf{Datasets, Splits, Metrics and Models.}
We evaluate PolyNSD on standard real-world node-classification benchmarks from the sheaf-learning literature \citep{sen2008collective,tang2009social,namata2012query,pei2020geom,rozemberczki2021multi}, using the fixed per-class split protocol of \citet{bodnar2022neural}: 48\%/32\%/20\% train/validation/test, with test accuracy reported as mean $\pm$ std over 10 splits. We further test on filtered \textsc{Chameleon}/\textsc{Squirrel} and the new heterophily benchmark suite of \citet{platonov2023critical}, as well as malignant heterophily splits following the 60\%/20\%/20\% protocol \cite{luan2024re}. For synthetic experiments, we use the controlled benchmark of \citet{caralt2024joint}, which varies feature noise, heterophily level $het$, and graph scale $(N,K,n_c)$ to decouple feature complexity from graph connectivity. (Further details about experimental settings, dataset statistics, and hyperparameters are provided in Appendix\autoref{app:hyperparams} and \autoref{app:datasets}). We evaluate our three PolyNSD proposed models, and to isolate the effect of polynomial spectral filtering from learned sheaf transports, we also introduce \emph{PolySpectralGNN}, a matched non-sheaf baseline that applies the same Chebyshev filtering mechanism to the scalar normalised graph Laplacian. Equivalently, PolySpectralGNN can be seen as PolyNSD with stalk dimension $d=1$ and identity transports. We then compare against classical GNNs, including GCN \citep{kipf2016semi}, GAT \citep{velickovic2017graph}, and GraphSAGE \citep{hamilton2017representation}; heterophily-oriented models such as GGCN \citep{yan2022two}, Geom-GCN \citep{pei2020geom}, H2GCN \citep{zhu2020beyond}, GPRGNN \citep{chien2020adaptive}, FAGCN \citep{bo2021beyond}, and MixHop \citep{abu2019mixhop}; oversmoothing remedies such as GCNII \citep{chien2020adaptive} and PairNorm \citep{chen2020simple}; spectral baselines including ChebNet \citep{defferrard2016convolutional}, ChebNetII, and BernNet; and sheaf-based models including NSD \citep{bodnar2022neural}, SAN/ANSD \citep{barbero2022sheaf_attentional}, Conn-NSD \citep{barbero2022sheaf}, RiSNN/JdSNN \citep{caralt2024joint}.

\noindent\textbf{Node Classification under Heterophily: Real-World, New and Synthetic Benchmarks with Malignant and Continuous Variants.}
\label{subsec:hetero-real-main}
We first compare PolyNSD against sheaf and non-sheaf baselines on the nine standard real-world benchmarks in \autoref{tab:big-leaderboard}. PolyNSD matches or improves prior sheaf models, reaching top-3 performance on almost all datasets across both heterophilic and homophilic regimes. The gains over first-order NSD show that learning a degree-$K$ spectral response $p_K(\widetilde L)$ is beneficial beyond local diffusion. At the same time, the gap between PolyNSD and PolySpectralGNN confirms that the sheaf structure itself is fundamental and that polynomial filtering alone is not sufficient. A key empirical finding is that PolyNSD often reaches strong results even with \emph{diagonal} restriction maps and small stalk dimensions. This contrasts with the common NSD preference for denser bundle or general maps, suggesting that higher-order spectral control can substitute for the expressivity usually obtained through expensive transports. Stalk-dimension sweeps in Appendix~\autoref{app:stalk-sweep-real} confirm that these gains do not rely on inflating the local feature space, improving parameter, memory, and runtime efficiency. We further ablate different polynomial bases in Appendix~\autoref{app:orthogonal-bases}, finding stable performance across Chebyshev, Legendre, Gegenbauer, Jacobi. Under the malignant heterophily protocol of \citet{luan2024re}, we separately evaluate \textsc{Texas}, \textsc{Wisconsin}, \textsc{Film}, and \textsc{Cornell}. PolyNSD consistently improves over the corresponding first-order NSD variants and remains competitive with strong MLP baselines. This suggests that polynomial sheaf diffusion can also mitigate harmful local graph structure by learning a higher-order spectral response that selectively combines local and non-local information (full results are reported in Appendix~\autoref{app:malignant_heterophily_60_20_20}). We then extend the evaluation to stricter and more recent heterophily settings in \autoref{tab:filtered_new_heterophily_results}. On the filtered \textsc{Chameleon} and \textsc{Squirrel} datasets of \citet{platonov2023critical}, where duplicate nodes are removed to reduce leakage, PolyNSD remains the strongest model, improving over NSD-General by roughly $+3.2$ and $+5.1$ points, respectively. On the new heterophily benchmark suite \cite{platonov2023critical}, PolyNSD-General achieves the best result on all five datasets. Controlled synthetic experiments confirm the same trends under isolated stress factors: as heterophily increases, while homophily-biased GNNs degrade toward MLP-level performance, PolyNSD variants remain near the top of the accuracy curves. When scaling the number of nodes and graph degree, PolyNSD maintains high accuracy across larger and denser graphs. Under increasing feature noise, PolyNSD degrades more slowly than classical GNNs and remains among the most robust sheaf models, supporting the interpretation that learned transports align local features while polynomial filtering controls the propagation of noisy frequency components. (Full synthetic results are provided in Appendix~\autoref{app:synthetic}). Finally, we also instantiate a continuous-depth variant of PolyNSD through Neural Sheaf ODEs, where we replace the affine NSD generator with a learnable polynomial generator: $q_\theta(\Delta_{\mathcal F(t)})$. More in particular, the diffusion changes to: $\dot{X}(t)
  =
  -\,\sigma\!\Big(
    q_\theta(\Delta_{\mathcal{F}(t)}) \, (I_n \otimes W_1)\, X(t)\, W_2
  \Big)$, where $q_\theta$ is implemented through a stable Chebyshev recurrence. For fixed sheaf structure and linear dynamics, each spectral mode with eigenvalue $\lambda$ is scaled by $\exp(-Tq_\theta(\lambda))$, generalising continuous NSD from affine to polynomial spectral generators. PolyNSD variants, in continuous time, frequently reach top-three performance, and the diagonal variant serves as a strong default. (Full formulation and results are given in Appendix~\autoref{app:ode-polynsd}).

\input{tables/spectral_response_diagnostic}
\noindent\textbf{Depth Robustness, Polynomial Order, and Accuracy--Efficiency.}
We study depth robustness by sweeping \(L\in\{2,4,8,16,32\}\) on representative homophilic and heterophilic benchmarks, reporting full results in Appendix~\autoref{app:oversmoothing}. Classical GNNs often degrade rapidly as depth increases, especially under heterophily, while PolyNSD remains stable and competitive up to \(L=32\), often more robust than first-order NSD variants. To analyse this beyond accuracy, we track the channel-averaged normalised Dirichlet energy
\(E_{\mathrm{norm}}(x_\ell)=\frac{\langle x_\ell,Lx_\ell\rangle}{\langle x_\ell,x_\ell\rangle}\),
\(\ell=1,\dots,L\), aggregated across seeds, stalk dimensions, depths, and transport classes. NSD exhibits increasing energy with depth, indicating amplification of transported disagreement, whereas PolyNSD maintains lower and more stable trajectories, suggesting better-conditioned depth-wise dynamics. We further ablate the Chebyshev order \(K\in\{1,2,4,8,12,16\}\) at fixed depth and width, finding that the best PolyNSD configuration always uses \(K>1\), hence strictly improving over the NSD-equivalent first-order case (Appendix~\autoref{app:chebK-sweep}). Moderate orders \(K\approx4\text{--}8\) are typically sufficient on homophilic graphs, while strongly heterophilic graphs benefit from larger orders \(K\approx8\text{--}16\), consistent with their need for longer-range, multi-frequency mixing. Finally, depth and width sweeps against spatial NSD show that PolyNSD matches or improves NSD with substantially fewer resources: it remains competitive at comparable or smaller parameter counts on homophilic graphs and yields gains up to \(\sim\!+20\%\) on heterophilic graphs at equal propagation power, despite using fewer layers or much smaller hidden dimensions (Appendix~\autoref{app:polynsd-vs-nsd-exp}).

\begin{table*}[t]
\centering
\scriptsize
\caption{\textit{Transductive long-range results on \textsc{CityNetworks}.}
We report accuracy (\%) on \textsc{Paris}, \textsc{Shanghai},
\textsc{Los Angeles}, and \textsc{London}. Baseline results are the best
reported configurations at depth $L=16$ for all models, except for the
\emph{Sheaf} models, which are reported at depth $L=6$. Higher is better.
The top three models for each dataset are coloured by
\textcolor{Top1}{\textbf{First}},
\textcolor{Top2}{\textbf{Second}} and
\textcolor{Top3}{\textbf{Third}}, respectively.}
\label{tab:citynetworks_transductive_results}
\setlength{\tabcolsep}{4.8pt}
\renewcommand{\arraystretch}{1.08}
\begin{adjustbox}{max width=0.7\textwidth}
\begin{tabular}{llccccc}
\toprule
\textbf{Family} & \textbf{Method}
& \textbf{Paris}
& \textbf{Shanghai}
& \textbf{Los Angeles}
& \textbf{London}
& \textbf{Avg.} \\
\midrule
\multirow{3}{*}{MPNNs}
& MLP
& $25.50 \pm 0.40$
& $28.40 \pm 0.60$
& $24.10 \pm 0.50$
& $27.90 \pm 0.10$
& $26.48$ \\

& GCN
& $53.20 \pm 0.30$
& $62.10 \pm 0.20$
& $58.30 \pm 0.30$
& $50.10 \pm 0.70$
& $55.93$ \\

& GraphSAGE
& \textcolor{Top3}{\textbf{$54.60 \pm 0.20$}}
& \textcolor{Top3}{\textbf{$68.30 \pm 0.50$}}
& \textcolor{Top3}{\textbf{$61.40 \pm 0.30$}}
& $55.40 \pm 0.20$
& \textcolor{Top2}{\textbf{$59.93$}} \\
\midrule
\multirow{1}{*}{GTs}
& Exphormer
& \textcolor{Top2}{\textbf{$55.10 \pm 0.80$}}
& \textcolor{Top1}{\textbf{$70.20 \pm 0.40$}}
& \textcolor{Top1}{\textbf{$63.80 \pm 0.60$}}
& $49.50 \pm 0.40$
& \textcolor{Top3}{\textbf{$59.65$}} \\
\midrule
\multirow{2}{*}{Spectral}
& ChebNetII
& $49.89 \pm 2.33$
& $66.83 \pm 0.00$
& $58.34 \pm 4.72$
& \textcolor{Top2}{\textbf{$57.37 \pm 5.45$}}
& $58.11$ \\

& BernNet
& $47.43 \pm 7.23$
& $65.23 \pm 0.00$
& $57.77 \pm 4.45$
& \textcolor{Top3}{\textbf{$56.75 \pm 4.91$}}
& $56.80$ \\
\midrule
\multirow{3}{*}{Sheaf}
& NSD-Diag
& $34.55 \pm 4.76$
& $43.85 \pm 0.78$
& $34.54 \pm 3.22$
& $38.14 \pm 2.67$
& $37.77$ \\

& PolySpectralGNN
& $45.72 \pm 3.36$
& $57.78 \pm 0.00$
& $55.67 \pm 3.45$
& $51.25 \pm 2.04$
& $52.61$ \\

& \textbf{PolyNSD-Diag}
& \textcolor{Top1}{\textbf{$57.43 \pm 2.28$}}
& \textcolor{Top2}{\textbf{$68.57 \pm 0.34$}}
& \textcolor{Top2}{\textbf{$61.82 \pm 4.23$}}
& \textcolor{Top1}{\textbf{$58.72 \pm 3.33$}}
& \textcolor{Top1}{\textbf{$61.64$}} \\
\bottomrule
\end{tabular}
\end{adjustbox}
\end{table*}

\noindent\textbf{Long-Range, Spectral, and Geometric Diagnostics.}
We complement accuracy with diagnostics focused on whether PolyNSD preserves long-range influence, how its learned spectral filters behave, and how restriction maps organise transport across layers. We first measure long-range influence through gradients. For a target node \(v\), source node \(u\), and score \(s_v(x)\), we define \(G_{uv}=\|\partial s_v/\partial x_u\|_2\) and aggregate by hop distance $I(d)
=
\mathbb{E}_{v\in\mathcal{T}}
\left[
\frac{1}{|\mathcal{N}_d(v)|}
\sum_{u\in\mathcal{N}_d(v)}G_{uv}
\right],
\qquad
\widetilde I(d)=\frac{I(d)}{I(0)}$. On \textsc{Minesweeper}, \textsc{Roman-Empire}, and \textsc{Amazon-Ratings}, we compare NSD with \(10\) layers against PolyNSD with \(6\) layers and polynomial degree \(K=10\), across all transport classes. We show in Appendix \autoref{app:long_range_influence} that NSD influence typically decays by several orders of magnitude with distance, whereas PolyNSD preserves substantially stronger medium- and long-range influence, often remaining orders of magnitude above NSD at large \(d\). This shows that polynomial sheaf filtering improves sensitivity to distant nodes by combining learned transports with explicit multi-hop spectral propagation. We further validate this long-range behaviour on \textsc{CityNetworks}~\citep{liang2026quantifyinglongrangeinteractionsgraph}, a transductive benchmark on city-scale spatial graphs. PolyNSD-Diag achieves the best average accuracy across \textsc{Paris}, \textsc{Shanghai}, \textsc{Los Angeles}, and \textsc{London}, reaching \(61.64\%\), compared to \(37.77\%\) for NSD-Diag, and is competitive with or stronger than deeper MPNN, transformer, and spectral baselines. Runtime-normalised results in Appendix \autoref{app:citynetworks_transductive} also show that PolyNSD is faster per epoch than NSD in this setting. We then inspect the \emph{learned spectral response}. For a sheaf Laplacian eigenpair \(Lu=\lambda u\), the polynomial filter acts as \(p(L)u=p(\lambda)u\); including the high-pass reinjection, the effective multiplier is $m(\lambda)
=
p(\xi(\lambda))
+
\alpha_{\mathrm{hp}}
\left(1-\frac{\lambda}{\lambda_{\max}}\right),
\qquad
\xi(\lambda)=\frac{2\lambda}{\lambda_{\max}}-1$. 
We summarise \(m(\lambda)\) through low/high gains, \(\Delta G=G_{\mathrm{high}}-G_{\mathrm{low}}\), and non-monotonicity. Across datasets, homophilic graphs learn stronger low--high separation and larger high-pass reinjection, while heterophilic graphs more often learn non-monotone, band-pass-like responses, consistent with the need to combine information across multiple radii rather than only smooth signals (see Appendix~\autoref{app:spectral-response-diagnostics}). Finally, we analyse the \emph{learned restriction-map geometry}. UMAP projections in Appendix \autoref{app:restriction-map-geometry} show that early layers learn broad transport manifolds, while deeper layers form more separated, filament-like structures with larger transport norms. Heatmaps reveal increasing dimension-wise specialisation, with later layers learning stronger activation and suppression patterns across edges and stalk dimensions. 

\noindent\textbf{Federated causal sheaf learning.}
We further test whether PolyNSD remains effective when training is distributed across clients by integrating it into the FedATH causal federated graph-learning pipeline~\cite{fu2025less} on \textsc{Roman-Empire}. This setting is challenging because the graph is sparse, weakly homophilic, and contains long-range syntactic dependencies. We partition the graph into $K=10$ clients and compare conventional federated baselines, FedATH graph backbones, first-order sheaf backbones, and polynomial sheaf variants. Replacing standard graph backbones with sheaf models already yields a large improvement, with FedCausalSheaf-General reaching $72.31\%$ compared to $48.18\%$ for FedATH-GCN. Adding polynomial sheaf propagation further improves performance, with FedCausalSheaf-GeneralPoly reaching $76.19\%$. Finally, when learned restriction maps are reused as edge-level signals for causal masking, performance increases further: LoadRMaps-FedCausalSheaf-GeneralPoly achieves the best result, $80.24\%$. Full details are provided in Appendix~\ref{app:federated-causal-sheaf}.

%% file: tables/big_leaderboard_polynsd.tex
\begin{table*}[t]
  \centering
  \caption{\textit{PolyNSD discrete node-classification benchmark results.}
  Results are reported as mean $\pm$ std over 10 splits. The top three models for each dataset are coloured by
  \textcolor{Top1}{\textbf{First}},
  \textcolor{Top2}{\textbf{Second}} and
  \textcolor{Top3}{\textbf{Third}}, respectively.}
  \label{tab:big-leaderboard}
  \setlength{\tabcolsep}{4pt}
  \renewcommand{\arraystretch}{1.10}
  \begin{adjustbox}{max width=0.9\textwidth}
  \begin{tabular}{lccccccccc}
    \toprule
     & \textbf{Texas} & \textbf{Wisconsin} & \textbf{Film} & \textbf{Squirrel} & \textbf{Chameleon} & \textbf{Cornell} & \textbf{Citeseer} & \textbf{Pubmed} & \textbf{Cora}\\
    \midrule
    \textit{Homophily} & \textbf{0.11} & \textbf{0.21} & \textbf{0.22} & \textbf{0.22} & \textbf{0.23} & \textbf{0.30} & \textbf{0.74} & \textbf{0.80} & \textbf{0.81} \\
    \#Nodes  & 183 & 251 & 7{,}600 & 5{,}201 & 2{,}277 & 183 & 3{,}327 & 18{,}717 & 2{,}708 \\
    \#Edges  & 295 & 466 & 26{,}752 & 198{,}493 & 31{,}421 & 280 & 4{,}676 & 44{,}327 & 5{,}278 \\
    \#Classes& 5 & 5 & 5 & 5 & 5 & 5 & 7 & 3 & 6 \\
    \midrule

    \textbf{DiagPolySD}
    & \textcolor{Top1}{\textbf{\pmstd{90.00}{4.68}}}
    & \pmstd{88.63}{3.59}
    & \pmstd{37.31}{0.98}
    & \textcolor{Top1}{\textbf{\pmstd{56.61}{2.06}}}
    & \textcolor{Top1}{\textbf{\pmstd{71.45}{2.03}}}
    & \textcolor{Top2}{\textbf{\pmstd{86.49}{5.54}}}
    & \textcolor{Top2}{\textbf{\pmstd{77.74}{1.26}}}
    & \pmstd{89.70}{0.32}
    & \textcolor{Top1}{\textbf{\pmstd{88.79}{1.13}}} \\

    \textbf{BundlePolySD}
    & \textcolor{Top2}{\textbf{\pmstd{89.74}{5.32}}}
    & \textcolor{Top1}{\textbf{\pmstd{89.41}{4.04}}}
    & \pmstd{37.47}{0.86}
    & \pmstd{55.76}{2.02}
    & \textcolor{Top2}{\textbf{\pmstd{71.18}{1.46}}}
    & \textcolor{Top1}{\textbf{\pmstd{86.76}{4.90}}}
    & \textcolor{Top3}{\textbf{\pmstd{77.57}{1.55}}}
    & \textcolor{Top3}{\textbf{\pmstd{89.75}{0.34}}}
    & \pmstd{88.33}{1.34} \\

    \textbf{GeneralPolySD}
    & \textcolor{Top3}{\textbf{\pmstd{89.21}{5.05}}}
    & \pmstd{88.82}{4.89}
    & \pmstd{37.34}{1.13}
    & \textcolor{Top3}{\textbf{\pmstd{55.79}{2.52}}}
    & \pmstd{69.62}{1.85}
    & \textcolor{Top3}{\textbf{\pmstd{86.49}{5.80}}}
    & \pmstd{77.21}{1.58}
    & \pmstd{89.73}{0.41}
    & \textcolor{Top2}{\textbf{\pmstd{88.47}{1.19}}} \\

    \textbf{PolySpectralGNN}
    & \pmstd{64.59}{6.40}
    & \pmstd{58.62}{6.04}
    & \pmstd{25.50}{0.85}
    & \pmstd{51.72}{1.76}
    & \pmstd{62.65}{3.03}
    & \pmstd{54.59}{6.26}
    & \pmstd{60.40}{0.89}
    & \pmstd{78.70}{0.56}
    & \pmstd{76.20}{0.67} \\

    \midrule
    ChebNetII
    & \pmstd{68.42}{6.96}
    & \pmstd{64.76}{6.93}
    & \pmstd{32.39}{4.45}
    & \pmstd{47.06}{1.74}
    & \pmstd{63.23}{5.83}
    & \pmstd{80.40}{5.56}
    & \pmstd{72.00}{1.66}
    & \pmstd{85.87}{3.34}
    & \pmstd{85.35}{6.75} \\

    BernNet
    & \pmstd{73.68}{6.00}
    & \pmstd{63.65}{6.76}
    & \pmstd{30.23}{6.34}
    & \pmstd{46.47}{1.47}
    & \pmstd{61.00}{2.34}
    & \pmstd{76.87}{5.26}
    & \pmstd{65.29}{5.55}
    & \pmstd{83.94}{3.83}
    & \pmstd{84.76}{4.75} \\

    \midrule
    RiSNN
    & \pmstd{86.84}{3.72}
    & \pmstd{87.84}{2.60}
    & \placeholder{N/A}
    & \pmstd{53.30}{3.30}
    & \pmstd{65.15}{2.40}
    & \pmstd{85.95}{6.14}
    & \pmstd{76.23}{1.81}
    & \pmstd{88.00}{0.42}
    & \pmstd{85.27}{1.11} \\

    JdSNN
    & \pmstd{87.37}{5.10}
    & \textcolor{Top3}{\textbf{\pmstd{89.22}{3.42}}}
    & \placeholder{N/A}
    & \pmstd{49.89}{1.71}
    & \pmstd{66.40}{2.33}
    & \pmstd{85.41}{4.55}
    & \pmstd{73.27}{1.86}
    & \pmstd{88.19}{0.55}
    & \pmstd{85.43}{1.73} \\

    Conn-NSD
    & \pmstd{86.16}{2.24}
    & \pmstd{88.73}{4.47}
    & \textcolor{Top1}{\textbf{\pmstd{37.91}{1.28}}}
    & \pmstd{45.19}{1.57}
    & \pmstd{65.21}{2.04}
    & \pmstd{85.95}{7.72}
    & \pmstd{75.61}{1.93}
    & \pmstd{89.28}{0.38}
    & \pmstd{83.74}{2.19} \\

    Diag-NSD
    & \pmstd{85.67}{6.95}
    & \pmstd{88.63}{2.75}
    & \pmstd{37.79}{1.01}
    & \pmstd{54.78}{1.81}
    & \pmstd{68.68}{1.73}
    & \pmstd{86.49}{7.35}
    & \pmstd{77.14}{1.85}
    & \pmstd{89.42}{0.43}
    & \pmstd{87.14}{1.06} \\

    O($d$)-NSD
    & \pmstd{85.95}{5.51}
    & \textcolor{Top2}{\textbf{\pmstd{89.41}{4.74}}}
    & \textcolor{Top2}{\textbf{\pmstd{37.81}{1.15}}}
    & \textcolor{Top2}{\textbf{\pmstd{56.34}{1.32}}}
    & \pmstd{68.04}{1.58}
    & \pmstd{84.86}{4.71}
    & \pmstd{76.70}{1.57}
    & \pmstd{89.49}{0.40}
    & \pmstd{86.90}{1.13} \\

    Gen-NSD
    & \pmstd{82.97}{5.13}
    & \pmstd{89.21}{3.84}
    & \textcolor{Top3}{\textbf{\pmstd{37.80}{1.22}}}
    & \pmstd{53.17}{1.31}
    & \pmstd{67.93}{1.58}
    & \pmstd{85.68}{6.51}
    & \pmstd{76.32}{1.65}
    & \pmstd{89.33}{0.35}
    & \pmstd{87.30}{1.15} \\

    \midrule
    GGCN
    & \pmstd{84.86}{4.55}
    & \pmstd{86.86}{3.29}
    & \pmstd{37.54}{1.56}
    & \pmstd{55.17}{1.58}
    & \textcolor{Top3}{\textbf{\pmstd{71.14}{1.84}}}
    & \pmstd{85.68}{6.63}
    & \pmstd{77.14}{1.45}
    & \pmstd{89.15}{0.37}
    & \pmstd{87.95}{1.05} \\

    H2GCN
    & \pmstd{84.86}{7.23}
    & \pmstd{87.65}{4.98}
    & \pmstd{35.70}{1.00}
    & \pmstd{36.48}{1.86}
    & \pmstd{60.11}{2.15}
    & \pmstd{82.70}{5.28}
    & \pmstd{77.11}{1.57}
    & \pmstd{89.49}{0.38}
    & \pmstd{87.87}{1.20} \\

    GCNII
    & \pmstd{77.57}{3.83}
    & \pmstd{80.39}{3.40}
    & \pmstd{37.44}{1.30}
    & \pmstd{38.47}{1.58}
    & \pmstd{63.86}{3.04}
    & \pmstd{77.86}{3.79}
    & \pmstd{77.33}{1.48}
    & \textcolor{Top1}{\textbf{\pmstd{90.15}{0.43}}}
    & \textcolor{Top3}{\textbf{\pmstd{88.37}{1.25}}} \\

    \midrule
    GraphSAGE
    & \pmstd{82.43}{6.14}
    & \pmstd{81.18}{5.56}
    & \pmstd{34.23}{0.99}
    & \pmstd{41.61}{0.74}
    & \pmstd{58.73}{1.68}
    & \pmstd{75.95}{5.01}
    & \pmstd{76.04}{1.30}
    & \pmstd{88.45}{0.50}
    & \pmstd{86.90}{1.04} \\

    GCN
    & \pmstd{55.14}{5.16}
    & \pmstd{51.76}{3.06}
    & \pmstd{27.32}{1.10}
    & \pmstd{53.43}{2.01}
    & \pmstd{64.82}{2.24}
    & \pmstd{60.54}{5.30}
    & \pmstd{76.50}{1.36}
    & \pmstd{88.42}{0.50}
    & \pmstd{86.98}{1.27} \\

    GAT
    & \pmstd{52.16}{6.63}
    & \pmstd{49.41}{4.09}
    & \pmstd{27.44}{0.89}
    & \pmstd{40.72}{1.55}
    & \pmstd{60.26}{2.50}
    & \pmstd{61.89}{5.05}
    & \pmstd{76.55}{1.23}
    & \pmstd{87.30}{1.10}
    & \pmstd{86.33}{0.48} \\

    MLP
    & \pmstd{80.81}{4.75}
    & \pmstd{85.29}{3.31}
    & \pmstd{36.53}{0.70}
    & \pmstd{28.77}{1.56}
    & \pmstd{46.21}{2.99}
    & \pmstd{81.89}{6.40}
    & \pmstd{74.02}{1.90}
    & \pmstd{75.69}{2.00}
    & \pmstd{87.16}{0.37} \\

    \bottomrule
  \end{tabular}
  \end{adjustbox}
\end{table*}

%% file: tables/platonov_main_paper.tex
\begin{table*}[t]
\centering
\scriptsize
\setlength{\tabcolsep}{3.2pt}
\renewcommand{\arraystretch}{0.92}
\caption{Results on filtered and new heterophily benchmarks from Platonov et al.~\citep{platonov2023critical}. Accuracy is reported for \textsc{Chameleon-Filtered}, \textsc{Squirrel-Filtered}, \textsc{Roman-Empire}, and \textsc{Amazon-Ratings}; ROC-AUC is reported for \textsc{Minesweeper}, \textsc{Tolokers}, and \textsc{Questions}. Higher is better. The top three models are coloured by \textcolor{Top1}{\textbf{First}}, \textcolor{Top2}{\textbf{Second}}, and \textcolor{Top3}{\textbf{Third}}.}
\label{tab:filtered_new_heterophily_results}
\resizebox{\textwidth}{!}{%
\begin{tabular}{lccccccc}
\toprule
\textbf{Method}
& \texttt{Cham.-Filt.}
& \texttt{Squi.-Filt.}
& \texttt{Roman}
& \texttt{Amazon}
& \texttt{Mine.}
& \texttt{Tolok.}
& \texttt{Quest.} \\
\midrule
ResNet
& $36.73 \pm 4.71$
& $36.55 \pm 1.82$
& $65.88 \pm 0.38$
& $45.90 \pm 0.52$
& $50.89 \pm 1.39$
& $72.95 \pm 1.06$
& $70.34 \pm 0.76$ \\

GCN
& $40.89 \pm 4.12$
& $39.47 \pm 1.47$
& $73.69 \pm 0.74$
& $48.70 \pm 0.63$
& $89.75 \pm 0.52$
& $83.64 \pm 0.67$
& $76.09 \pm 1.27$ \\

SAGE
& $37.77 \pm 4.14$
& $36.09 \pm 1.99$
& $85.74 \pm 0.67$
& \textcolor{Top3}{$\mathbf{53.63 \pm 0.39}$}
& $93.51 \pm 0.57$
& $82.43 \pm 0.44$
& $76.44 \pm 0.62$ \\

GT-sep
& $40.31 \pm 3.01$
& $36.66 \pm 1.63$
& $87.32 \pm 0.39$
& $52.18 \pm 0.80$
& $92.29 \pm 0.47$
& $82.52 \pm 0.92$
& \textcolor{Top3}{$\mathbf{78.05 \pm 0.93}$} \\

FAGCN
& $41.90 \pm 2.72$
& $41.08 \pm 2.27$
& $65.22 \pm 0.56$
& $44.12 \pm 0.30$
& $88.17 \pm 0.73$
& $77.75 \pm 1.05$
& $77.24 \pm 1.26$ \\

FSGNN
& $40.61 \pm 2.97$
& $35.92 \pm 1.32$
& $79.92 \pm 0.56$
& $52.74 \pm 0.83$
& $90.08 \pm 0.70$
& $82.76 \pm 0.61$
& \textcolor{Top2}{$\mathbf{78.86 \pm 0.92}$} \\

\midrule
ChebNetII
& $40.82 \pm 4.96$
& $41.09 \pm 2.41$
& $63.86 \pm 0.43$
& $38.82 \pm 0.07$
& $74.73 \pm 1.51$
& $75.09 \pm 0.87$
& $68.50 \pm 1.55$ \\

BernNet
& $37.30 \pm 4.17$
& $42.41 \pm 1.97$
& $67.46 \pm 4.38$
& $36.79 \pm 0.05$
& $80.94 \pm 1.16$
& $76.19 \pm 0.85$
& $66.35 \pm 1.23$ \\

PolySpectralGNN
& $38.25 \pm 5.33$
& $41.24 \pm 1.87$
& $59.16 \pm 1.64$
& $36.98 \pm 0.64$
& $86.59 \pm 0.60$
& $78.20 \pm 0.61$
& $69.40 \pm 1.74$ \\

\midrule
NSD-Diag
& $41.05 \pm 4.01$
& $43.11 \pm 2.09$
& $76.82 \pm 2.49$
& $38.79 \pm 0.39$
& $86.34 \pm 0.69$
& $75.31 \pm 1.14$
& $69.94 \pm 2.04$ \\

NSD-Bundle
& $40.56 \pm 2.67$
& $42.43 \pm 3.23$
& $78.75 \pm 2.23$
& $40.23 \pm 0.23$
& $87.43 \pm 0.96$
& $78.65 \pm 2.94$
& $66.23 \pm 1.49$ \\

NSD-General
& $42.72 \pm 2.06$
& $43.71 \pm 2.06$
& $80.41 \pm 0.72$
& $42.76 \pm 0.54$
& $92.15 \pm 0.84$
& $80.31 \pm 0.76$
& $69.69 \pm 1.46$ \\

\textbf{PolyNSD-Diag}
& \textcolor{Top2}{$\mathbf{45.88 \pm 3.60}$}
& \textcolor{Top2}{$\mathbf{48.60 \pm 1.85}$}
& $87.49 \pm 2.21$
& $52.87 \pm 0.24$
& \textcolor{Top3}{$\mathbf{96.98 \pm 0.55}$}
& $82.27 \pm 0.61$
& $75.37 \pm 1.08$ \\

\textbf{PolyNSD-Bundle}
& \textcolor{Top3}{$\mathbf{44.64 \pm 3.78}$}
& \textcolor{Top3}{$\mathbf{47.86 \pm 1.79}$}
& \textcolor{Top3}{$\mathbf{88.54 \pm 1.65}$}
& \textcolor{Top2}{$\mathbf{53.87 \pm 0.28}$}
& \textcolor{Top2}{$\mathbf{97.23 \pm 0.55}$}
& \textcolor{Top3}{$\mathbf{83.92 \pm 0.61}$}
& $76.37 \pm 1.08$ \\

\textbf{PolyNSD-General}
& \textcolor{Top1}{$\mathbf{45.94 \pm 3.79}$}
& \textcolor{Top1}{$\mathbf{48.82 \pm 1.60}$}
& \textcolor{Top1}{$\mathbf{89.87 \pm 1.36}$}
& \textcolor{Top1}{$\mathbf{54.34 \pm 0.34}$}
& \textcolor{Top1}{$\mathbf{98.86 \pm 0.55}$}
& \textcolor{Top1}{$\mathbf{84.75 \pm 0.61}$}
& \textcolor{Top1}{$\mathbf{79.37 \pm 1.08}$} \\
\bottomrule
\end{tabular}%
}
\end{table*}

%% file: chapters/5_Conclusions.tex
\section{Conclusions}
\label{sec:conclusions}
We introduced \emph{Polynomial Neural Sheaf Diffusion} (PolyNSD), a model-agnostic spectral extension of Neural Sheaf Diffusion that replaces first-order spatial propagation with learnable degree-$K$ orthogonal-polynomial filters of the normalised sheaf Laplacian. PolyNSD gives each layer an explicit multi-hop receptive field and a learnable frequency response, while reusing the learned transports through stable recurrence relations. Across standard, filtered, malignant, synthetic, and new heterophily benchmarks, as well as homophilic graphs, PolyNSD consistently improves over first-order NSD and often reaches state-of-the-art or highly competitive results with diagonal maps and small stalk dimensions. Ablations confirm the benefit of higher-order filtering ($K>1$), while PolySpectralGNN shows that the gains require both polynomial spectral control and sheaf-based transports. Diagnostics further show improved depth stability, stronger long-range influence, interpretable spectral responses, and structured restriction-map geometry. Extensions to continuous-time Sheaf ODEs, federated causal graph learning, and CityNetworks indicate that polynomial sheaf filtering is a general mechanism for stable, efficient, and interpretable higher-order diffusion on graphs.

\section{Limitations}
\label{sec:limitations}
While \textsc{PolyNSD} provides an efficient spectral route to higher-order sheaf diffusion, it still inherits the computational structure of Sheaf Laplacian-based models. A degree-$K$ layer requires $K$ sparse applications of the learned sheaf Laplacian, so runtime can grow on very large graphs, at high polynomial orders, or with large stalk dimensions, especially when the sheaf Laplacian is recomputed across layers. Moreover, although diagonal restriction maps are often sufficient in our experiments, some domains may require more expressive bundle or general transports, whose memory and parameter costs scale with the stalk dimension. Finally, our evaluation focuses primarily on transductive and semi-supervised graph learning; extending PolyNSD to larger-scale inductive and heterogeneous settings remains an important direction for future work.

%% file: chapters/A_Appendix.tex
\appendix

\providecolor{Green}{RGB}{0,128,0}
\providecolor{Red}{RGB}{200,0,0}
\providecommand{\cmark}{\textcolor{Green}{\ding{51}}}
\providecommand{\xmark}{\textcolor{Red}{\ding{55}}}

\sisetup{
    mode=text,
    table-alignment-mode=none,
    reset-text-family = false,
    reset-text-series = false,
    reset-text-shape = false,
    table-number-alignment = center,
    table-align-comparator = false,
    table-align-text-pre = false,
    table-align-text-post = false,
    round-mode=uncertainty,
    round-precision=3
}

\definecolor{proofbar}{RGB}{225,225,225}
\definecolor{theorembar}{RGB}{128,0,128}
\definecolor{propositionbar}{RGB}{0,100,0}
\definecolor{definitionbar}{RGB}{0,0,200}
\definecolor{lemmabar}{RGB}{255,140,0}
\definecolor{corollarybar}{RGB}{204,153,255}
\definecolor{remarkbar}{RGB}{0,139,139}
\definecolor{assumptionbar}{RGB}{178,34,34}

\makeatletter
\@ifundefined{theorem}{\newtheorem{theorem}{Theorem}}{}
\@ifundefined{proposition}{\newtheorem{proposition}{Proposition}}{}
\@ifundefined{remark}{\newtheorem{remark}{Remark}}{}
\@ifundefined{definition}{\newtheorem{definition}{Definition}}{}
\@ifundefined{lemma}{\newtheorem{lemma}{Lemma}}{}
\@ifundefined{corollary}{\newtheorem{corollary}{Corollary}}{}
\@ifundefined{assumption}{\newtheorem{assumption}{Assumption}}{}
\makeatother

\surroundwithmdframed[
  hidealllines=true,
  leftline=true,
  linecolor=proofbar,
  linewidth=2pt,
  innerleftmargin=6pt,
  innerrightmargin=0pt,
  innertopmargin=2pt,
  innerbottommargin=2pt,
  splittopskip=\topskip,
  splitbottomskip=0pt,
  skipabove=0.5\baselineskip,
  skipbelow=0.5\baselineskip
]{proof}

\surroundwithmdframed[
  hidealllines=true,leftline=true,linecolor=theorembar,linewidth=2pt,
  innerleftmargin=6pt,innerrightmargin=0pt,innertopmargin=2pt,innerbottommargin=2pt
]{theorem}
\surroundwithmdframed[
  hidealllines=true,leftline=true,linecolor=propositionbar,linewidth=2pt,
  innerleftmargin=6pt,innerrightmargin=0pt,innertopmargin=2pt,innerbottommargin=2pt
]{proposition}
\surroundwithmdframed[
  hidealllines=true,leftline=true,linecolor=remarkbar,linewidth=2pt,
  innerleftmargin=6pt,innerrightmargin=0pt,innertopmargin=2pt,innerbottommargin=2pt
]{remark}
\surroundwithmdframed[
  hidealllines=true,leftline=true,linecolor=definitionbar,linewidth=2pt,
  innerleftmargin=6pt,innerrightmargin=0pt,innertopmargin=2pt,innerbottommargin=2pt
]{definition}
\surroundwithmdframed[
  hidealllines=true,leftline=true,linecolor=lemmabar,linewidth=2pt,
  innerleftmargin=6pt,innerrightmargin=0pt,innertopmargin=2pt,innerbottommargin=2pt
]{lemma}
\surroundwithmdframed[
  hidealllines=true,leftline=true,linecolor=corollarybar,linewidth=2pt,
  innerleftmargin=6pt,innerrightmargin=0pt,innertopmargin=2pt,innerbottommargin=2pt
]{corollary}
\surroundwithmdframed[
  hidealllines=true,leftline=true,linecolor=assumptionbar,linewidth=2pt,
  innerleftmargin=6pt,innerrightmargin=0pt,innertopmargin=2pt,innerbottommargin=2pt
]{assumption}


\newenvironment{restatedtheorem}[2][]{%
  \begin{mdframed}[
    hidealllines=true,leftline=true,linecolor=theorembar,linewidth=2pt,
    innerleftmargin=6pt,innerrightmargin=0pt,
    innertopmargin=2pt,innerbottommargin=2pt
  ]
  \noindent\textbf{Theorem~\ref{#2}}%
  \if\relax\detokenize{#1}\relax
  .%
  \else
  \textbf{ (#1).}%
  \fi
  \itshape
}{%
  \end{mdframed}
}

\newenvironment{restatedproposition}[2][]{%
  \begin{mdframed}[
    hidealllines=true,leftline=true,linecolor=propositionbar,linewidth=2pt,
    innerleftmargin=6pt,innerrightmargin=0pt,
    innertopmargin=2pt,innerbottommargin=2pt
  ]
  \noindent\textbf{Proposition~\ref{#2}}%
  \if\relax\detokenize{#1}\relax
  .%
  \else
  \textbf{ (#1).}%
  \fi
  \itshape
}{%
  \end{mdframed}
}

\newenvironment{restatedlemma}[2][]{%
  \begin{mdframed}[
    hidealllines=true,leftline=true,linecolor=lemmabar,linewidth=2pt,
    innerleftmargin=6pt,innerrightmargin=0pt,
    innertopmargin=2pt,innerbottommargin=2pt
  ]
  \noindent\textbf{Lemma~\ref{#2}}%
  \if\relax\detokenize{#1}\relax
  .%
  \else
  \textbf{ (#1).}%
  \fi
  \itshape
}{%
  \end{mdframed}
}

\newenvironment{restatedcorollary}[2][]{%
  \begin{mdframed}[
    hidealllines=true,leftline=true,linecolor=corollarybar,linewidth=2pt,
    innerleftmargin=6pt,innerrightmargin=0pt,
    innertopmargin=2pt,innerbottommargin=2pt
  ]
  \noindent\textbf{Corollary~\ref{#2}}%
  \if\relax\detokenize{#1}\relax
  .%
  \else
  \textbf{ (#1).}%
  \fi
  \itshape
}{%
  \end{mdframed}
}

\providecommand*{\theoremautorefname}{Theorem}
\providecommand*{\lemmaautorefname}{Lemma}
\providecommand*{\propositionautorefname}{Proposition}
\providecommand*{\definitionautorefname}{Definition}
\providecommand*{\corollaryautorefname}{Corollary}
\providecommand*{\remarkautorefname}{Remark}
\providecommand*{\exampleautorefname}{Example}
\providecommand*{\assumptionautorefname}{Assumption}

\makeatletter
\@ifundefined{ifshowcomments}{\newboolean{showcomments}}{}
\makeatother
\setboolean{showcomments}{true}
\providecommand{\fr}[1]{\textcolor{red}{\textbf{[FR: #1]}}}

\definecolor{Top1}{HTML}{0072B2}
\definecolor{Top2}{HTML}{E69F00}
\definecolor{Top3}{HTML}{CC79A7}

\usetikzlibrary{calc,positioning}

\definecolor{ufill}{HTML}{E1D5E7}
\definecolor{uborder}{HTML}{9673A6}
\definecolor{vfill}{HTML}{D5E8D4}
\definecolor{vborder}{HTML}{82B366}


\begin{center}
    {\LARGE \bfseries Polynomial Neural Sheaf Diffusion: A Spectral Filtering Approach on Cellular Sheaves\par}
    \vspace{0.4em}
    {\Large Supplementary Material\par}
\end{center}

\vspace{1em}

\addcontentsline{toc}{part}{Supplementary Material}
\etocsetnexttocdepth{3}
\localtableofcontents


\newpage
\section{Spectral Theory applied to Sheaf Laplacians}
\label{app:spectral-sheaf}

This appendix is meant to better explain and justify the spectral viewpoint adopted in \autoref{sec:polysd}. We make explicit the role of the sheaf Laplacian as a real symmetric positive semidefinite operator on $C^0(\mathcal G;\mathcal F)$, introduce the associated \emph{sheaf Fourier transform}, and explain how spectral multipliers $f(L_{\mathcal F})$ implement frequency-selective filters on sheaf signals. We also clarify why small eigenvalues correspond to smooth, globally aligned sheaf signals, while large eigenvalues capture oscillatory disagreement patterns, connecting to the opinion-dynamics interpretation of \cite{hansen2020sheaf} and the spectral theory of cellular sheaves developed in Definition \cite{hansen2019toward}. Throughout, $\mathcal G=(\mathcal V,\mathcal E)$ is a finite undirected graph
equipped with a cellular sheaf $\mathcal F$, and $L_{\mathcal F}$ denotes the (vertex) sheaf Laplacian.

\subsection{Self-adjointness and Positive Semidefiniteness}
\label{app:spectral-selfadjoint}

We first collect the basic operator-theoretic properties of the vertex sheaf
Laplacian. Throughout, $C^0(\mathcal G;\mathcal F)\cong\mathbb R^{nd}$ and
$C^1(\mathcal G;\mathcal F)\cong\mathbb R^{|\mathcal E|d_e}$ are equipped with
the canonical Euclidean inner product induced by the inner products on the
stalks.

\begin{lemma}[\textbf{Sheaf Coboundary and Laplacian}]
\label{lem:spectral-delta}
Let $\delta : C^0(\mathcal G;\mathcal F) \to C^1(\mathcal G;\mathcal F)$ be the
sheaf coboundary associated with $\mathcal F$ and the chosen orientation, and
let $L_{\mathcal F} := \delta^\top\delta$. Then:
\begin{enumerate}
    \item $L_{\mathcal F}$ is \emph{symmetric} (self-adjoint) with respect to the 
    canonical inner product on $C^0(\mathcal G;\mathcal F)$.
    \item $L_{\mathcal F}$ is \emph{positive semidefinite} (PSD), i.e.\ 
    $\langle x, L_{\mathcal F} x\rangle \ge 0$ for all $x\in C^0(\mathcal G;\mathcal F)$.
\end{enumerate}
\end{lemma}

\begin{proof}
For the first point, by definition, $\delta^\top$ denotes the adjoint of $\delta$ with respect to the inner products on $C^0$ and $C^1$. Hence $L_{\mathcal F}=\delta^\top\delta$ satisfies: 
for all $x,y\in C^0$:
\begin{equation}
\langle x, L_{\mathcal F} y\rangle
=
\langle x, \delta^\top\delta y\rangle
=
\langle \delta x, \delta y\rangle
=
\langle \delta^\top\delta x, y\rangle
=
\langle L_{\mathcal F} x, y\rangle,
\end{equation}
which is exactly self-adjointness. For the second point, for any $x\in C^0(\mathcal G;\mathcal F)$, $L_{\mathcal F}$ is PSD because:
\begin{equation}
    \langle x, L_{\mathcal F} x\rangle
=
\langle x, \delta^\top\delta x\rangle
=
\langle \delta x, \delta x\rangle
=
\|\delta x\|^2
\;\ge\; 0.
\end{equation}
\end{proof}

\noindent
\autoref{lem:spectral-delta} shows that the sheaf Laplacian
is the canonical “energy” operator associated with the coboundary: it is
self-adjoint because it is built as an adjoint composition, and it is
positive semidefinite because it always measures a squared norm of
edge-wise discrepancies. Since $C^0(\mathcal G;\mathcal F)$ is finite-dimensional and $L_{\mathcal F}$ is real symmetric PSD, the \emph{Spectral Theorem} applies.

\begin{proposition}[\textbf{Spectral Decomposition of $L_{\mathcal F}$}]
\label{prop:spectral-decomp}
There exists an orthonormal basis of $C^0(\mathcal G;\mathcal F)$ consisting of
eigenvectors of $L_{\mathcal F}$. Equivalently, there is an orthogonal matrix
$U\in\mathbb R^{nd\times nd}$ and a diagonal matrix
$\Lambda=\mathrm{diag}(\lambda_1,\dots,\lambda_{nd})$ with
$\lambda_i\ge 0$ s.t. $L_{\mathcal F} \;=\; U \Lambda U^\top$. We call the columns of $U$ the \emph{sheaf Fourier modes}.
\end{proposition}

\paragraph{Sheaf-Graph Laplacians Relation.}
\autoref{prop:spectral-decomp} is the sheaf analogue of the usual
spectral decomposition of the graph Laplacian. In the trivial-sheaf case
($d{=}1$ and scalar weights on edges), $L_{\mathcal F}$ reduces to the standard
(combinatorial) Laplacian and $U$ is the familiar basis of graph Fourier modes.
In the general sheaf case, stalks can have dimension $d>1$ and edge restriction
maps can perform anisotropic transports between local feature spaces, but the
same spectral picture survives.

\subsection{Sheaf Fourier Transform, Plancherel Identity and Dirichlet Energy in the Spectral Domain}
\label{app:spectral-fourier}

Once the spectral decomposition $L_{\mathcal F} = U \Lambda U^\top$ is available
(\autoref{prop:spectral-decomp}), it is natural to regard the columns of $U$ as an orthonormal “frequency” basis for sheaf signals. In complete analogy with
the classical graph Laplacian, any $x\in C^0(\mathcal G;\mathcal F)$ can be
expanded in this eigenbasis, and the corresponding coefficients play the role of
\emph{sheaf Fourier coefficients}. The change of coordinates
$x \mapsto \widehat x = U^\top x$ is therefore the sheaf analogue of the graph
Fourier transform: it expresses $x$ as a linear combination of sheaf Fourier
modes ordered by their associated eigenvalues. This spectral viewpoint is convenient for two reasons:
(i) the Euclidean norm is preserved by this change of basis (Plancherel
identity), so energy computations can be carried out either in the original or
in the spectral domain; and (ii) the Dirichlet energy $\langle x, L_{\mathcal F}
x\rangle$ becomes a simple weighted sum of squared Fourier coefficients, where
the weights are precisely the eigenvalues $\lambda_i$. This makes explicit how
each mode contributes to the total “disagreement” on the sheaf and underpins the
interpretation of small eigenvalues as smooth directions. We now formalise the \emph{sheaf Fourier transform} and the representation of
Dirichlet energy in the spectral domain.

\begin{definition}[\textbf{Sheaf Fourier Transform}]
\label{def:sheaf-fourier}
Let $L_{\mathcal F}=U\Lambda U^\top$ be as in
\autoref{prop:spectral-decomp}. For a sheaf signal
$x\in C^0(\mathcal G;\mathcal F)\cong\mathbb R^{nd}$ we define $\widehat x \;:=\; U^\top x \in \mathbb R^{nd}$ as the \emph{sheaf Fourier transform} of $x$. The inverse transform is then defined as $x = U \widehat x$.
\end{definition}

\begin{lemma}[\textbf{Parseval/Plancherel Identity}]
\label{lem:plancherel}
For all $x,y\in C^0(\mathcal G;\mathcal F)$ (with $\|x\|_2=\|\widehat x\|_2$), we have: $\langle x, y\rangle \;=\; \langle \widehat x, \widehat y\rangle$.
\end{lemma}

\begin{proof}
We have $x=U\widehat x$, $y=U\widehat y$, with $\widehat x=U^\top x$,
$\widehat y=U^\top y$. Since $U$ is orthogonal, we have:
\begin{equation}
    \langle x,y\rangle
=
\langle U\widehat x, U\widehat y\rangle
=
\widehat x^\top U^\top U \widehat y
=
\widehat x^\top \widehat y
=
\langle \widehat x, \widehat y\rangle
\end{equation}
Taking $y=x$ gives $\|x\|_2^2 = \|\widehat x\|_2^2$.
\end{proof}

\begin{proposition}[\textbf{Dirichlet Energy in the Spectral Domain}]
\label{prop:dirichlet-spectral}
Let $x\in C^0(\mathcal G;\mathcal F)$ and $\widehat x=U^\top x$. Then the
Dirichlet Energy of $x$ with respect to $L_{\mathcal F}$ can be written as:
\begin{equation}
\label{eq:dirichlet-spectral}
\mathcal E(x)
\;:=\;
\langle x,L_{\mathcal F}x\rangle
\;=\;
\sum_{i=1}^{nd} \lambda_i\,\widehat x_i^2.
\end{equation}
\end{proposition}

\begin{proof}
Using $L_{\mathcal F}=U\Lambda U^\top$ and $x=U\widehat x$,
\[
\langle x,L_{\mathcal F}x\rangle
=
x^\top L_{\mathcal F} x
=
\widehat x^\top U^\top U \Lambda U^\top U \widehat x
=
\widehat x^\top \Lambda \widehat x
=
\sum_{i=1}^{nd} \lambda_i \widehat x_i^2.
\]
\end{proof}

\noindent
\autoref{eq:dirichlet-spectral} can now be interpreted in two complementary
ways: edge-wise and spectral perspective.

\paragraph{Edge-wise Viewpoint: Energy as Disagreement.}
Recall that, by definition of the Sheaf Laplacian, we have:
\begin{equation}
    \mathcal E(x)
=
\langle x,L_{\mathcal F}x\rangle
=
\|\delta x\|^2
=
\sum_{e=(u,v)\in\mathcal E}
\big\|
  \mathcal F_{u\unlhd e}x_u
  -
  \mathcal F_{v\unlhd e}x_v
\big\|^2
\end{equation}
Thus, $\mathcal E(x)$ is literally the sum of squared \emph{disagreements}
between the incident opinions transported into each edge’s discourse space. If
$x$ is \emph{perfectly consistent} with the sheaf (every incident pair agrees
after transport), then $\delta x=0$ and $\mathcal E(x)=0$.

\paragraph{Spectral Viewpoint: Eigenvalues as Smoothness Penalties.}
In the eigenbasis of $L_{\mathcal F}$, from \autoref{prop:dirichlet-spectral}, we have that each sheaf Fourier mode (eigenvector) contributes to the total energy according to its eigenvalue $\lambda_i$. From this, depending on the value that the eigenvalue assumes, we have:
\begin{itemize}
    \item If $\lambda_i=0$, then any component along the $i$-th eigenvector
    contributes with \emph{no} energy. These eigenvectors span $\ker(L_{\mathcal F})$
    and coincide with discrete \emph{harmonic} sheaf signals (i.e., the global sections). They satisfy $\delta x=0$ and are perfectly aligned across edges.

    \item For $\lambda_i>0$, the contribution of mode $i$ is
    $\lambda_i \widehat x_i^2$: to achieve a large coefficient $\widehat x_i$
    without incurring a large energy cost, the eigenvalue $\lambda_i$ must be
    small. Conversely, modes with large $\lambda_i$ are heavily penalised and
    can only appear with substantial energy if they encode strong edge-wise
    disagreements.
\end{itemize}

This leads to the usual notion of \emph{smoothness} on the sheaf: eigenvectors
with small eigenvalues are those directions along which one can vary the signal
$x$ without creating much disagreement on edges (small $\|\delta x\|$), so they
represent slowly varying, globally aligned opinion profiles. Eigenvectors with
large eigenvalues, instead, necessarily create large contributions to
$\|\delta x\|^2$ and thus correspond to highly \emph{oscillatory} patterns of
disagreement across incidences. In other words, the spectrum of $L_{\mathcal
F}$ orders sheaf Fourier modes from smooth, globally consistent patterns
(low $\lambda$) to rapidly varying, strongly conflicting ones (high
$\lambda$), exactly mirroring the role of the graph Laplacian spectrum in
classical spectral graph theory but now in the richer, transport-aware sheaf
setting.

\subsection{Spectral Multipliers and Functional Calculus}
\label{app:spectral-multipliers}

Once the spectral decomposition $L_{\mathcal F} = U \Lambda U^\top$ is available,
it becomes natural to ask how to apply a \emph{function} $f$ to the operator
$L_{\mathcal F}$ itself. Intuitively, since $U$ diagonalises $L_{\mathcal F}$,
we can move to the eigenbasis, modify each eigenmode by a scalar factor
$f(\lambda_i)$, and then move back. In other words, $f$ acts as a
\emph{frequency response} that scales each sheaf Fourier mode according to its
eigenvalue, exactly as in classical spectral graph theory. For polynomials $p$, this idea coincides with the $p(L_{\mathcal F}) = \sum_k c_k L_{\mathcal F}^k$ definition. The spectral theorem guarantees that this polynomial
action is equivalent to applying $p$ entrywise on the eigenvalues:
$p(L_{\mathcal F}) = U\,p(\Lambda)\,U^\top$. More generally, by uniform
approximation of continuous functions with polynomials, one can extend this
construction from polynomials to arbitrary continuous $f$ on the spectrum of
$L_{\mathcal F}$, leading to the standard \emph{continuous functional calculus}.
In this framework, diffusion semigroups, heat kernels, and our polynomial sheaf
filters all appear as particular choices of spectral multipliers $f(\lambda)$. 

\begin{lemma}[\textbf{Polynomial Functional Calculus}]
\label{lem:poly-functional}
Let $p(\lambda)=\sum_{k=0}^{K}c_k\lambda^k$ be a real polynomial and
$L_{\mathcal F}=U\Lambda U^\top$ as above. Let's define $p(L_{\mathcal F}) \;:=\; \sum_{k=0}^K c_k L_{\mathcal F}^k$  . Then:
\begin{equation}
\label{eq:poly-functional-calc}
p(L_{\mathcal F}) \;=\; U\,p(\Lambda)\,U^\top,
\qquad
p(\Lambda)=\mathrm{diag}\!\big(p(\lambda_1),\dots,p(\lambda_{nd})\big).
\end{equation}
\end{lemma}

\begin{proof}
From $L_{\mathcal F}=U\Lambda U^\top$ we obtain by induction $L_{\mathcal F}^k
=U\Lambda^k U^\top$ for all $k\ge 0$. Indeed, for $k=0$ we have
$L_{\mathcal F}^0 = I = U I U^\top = U \Lambda^0 U^\top$. Assume the identity now holds for some $k\ge 0$, then:
\begin{equation}
    L_{\mathcal F}^{k+1}
=
L_{\mathcal F} L_{\mathcal F}^{k}
=
(U\Lambda U^\top)(U\Lambda^k U^\top)
=
U\Lambda (U^\top U)\Lambda^k U^\top
=
U\Lambda^{k+1} U^\top
\end{equation}
By induction the claim holds for all $k$. Thus:
\begin{equation}
    p(L_{\mathcal F})
=
\sum_{k=0}^K c_k L_{\mathcal F}^k
=
\sum_{k=0}^K c_k U\Lambda^k U^\top
=
U\Big(\sum_{k=0}^K c_k \Lambda^k\Big)U^\top
\end{equation}
Since $\Lambda$ is diagonal, $\Lambda^k$ is also diagonal with entries
$\lambda_i^k$, and therefore, we have:
\begin{equation}
    \sum_{k=0}^K c_k \Lambda^k
=
\mathrm{diag}\!\Big(\sum_{k=0}^K c_k \lambda_1^k,\dots,
                    \sum_{k=0}^K c_k \lambda_{nd}^k\Big)
=
\mathrm{diag}\!\big(p(\lambda_1),\dots,p(\lambda_{nd})\big)
=
p(\Lambda)
\end{equation}
Substituting back yields \eqref{eq:poly-functional-calc}.
\end{proof}

\begin{proposition}[\textbf{Continuous functional Calculus and Operator Norm}]
\label{prop:cont-functional}
Let $f:[0,\infty)\to\mathbb R$ be continuous on an interval containing
$\sigma(L_{\mathcal F})$. Then there exists an operator $f(L_{\mathcal F})$
defined by:
\begin{equation}
f(L_{\mathcal F}) := U\,f(\Lambda)\,U^\top,
\qquad
f(\Lambda)=\mathrm{diag}\!\big(f(\lambda_1),\dots,f(\lambda_{nd})\big)
\end{equation}
which is the uniform limit (in operator norm) of polynomial approximants
$p^{(m)}(L_{\mathcal F})$. Moreover:
\begin{equation}
\|f(L_{\mathcal F})\|_2
=
\max_{\lambda\in\sigma(L_{\mathcal F})} |f(\lambda)|.
\end{equation}
\end{proposition}

\begin{proof}
We split the proof into several steps.

\emph{Step 1: Polynomial approximation of $f$.}
Let $\lambda_{\max} := \max \sigma(L_{\mathcal F})$. Since
$\sigma(L_{\mathcal F}) \subset [0,\lambda_{\max}]$ and $f$ is continuous on an
interval containing this set, the restriction of $f$ to $[0,\lambda_{\max}]$
is continuous. By the Weierstrass approximation theorem, there exists a
sequence of real polynomials $(p^{(m)})_{m\in\mathbb N}$ such that:
\begin{equation}
\label{eq:weierstrass}
\sup_{\lambda \in [0,\lambda_{\max}]}
\big|p^{(m)}(\lambda) - f(\lambda)\big|
\;\xrightarrow[m\to\infty]{}\; 0.
\end{equation}

\emph{Step 2: Induced convergence in operator norm.}
For each $m$, define $p^{(m)}(L_{\mathcal F})$ via \autoref{lem:poly-functional}:
\[
p^{(m)}(L_{\mathcal F})
=
U\,p^{(m)}(\Lambda)\,U^\top,
\quad
p^{(m)}(\Lambda)
=
\mathrm{diag}\!\big(p^{(m)}(\lambda_1),\dots,p^{(m)}(\lambda_{nd})\big).
\]
We first show that the family $(p^{(m)}(L_{\mathcal F}))_m$ is Cauchy in
operator norm. Let $m,n\in\mathbb N$. Then:
\begin{equation}
    p^{(m)}(L_{\mathcal F}) - p^{(n)}(L_{\mathcal F})
=
U\big(p^{(m)}(\Lambda) - p^{(n)}(\Lambda)\big) U^\top.
\end{equation}
Since $U$ is orthogonal, conjugation by $U$ preserves the spectral (operator)
norm:
\begin{equation}
\label{eq:unitary-invariance}
\big\|p^{(m)}(L_{\mathcal F}) - p^{(n)}(L_{\mathcal F})\big\|_2
=
\big\|p^{(m)}(\Lambda) - p^{(n)}(\Lambda)\big\|_2.
\end{equation}
Now $p^{(m)}(\Lambda) - p^{(n)}(\Lambda)$ is diagonal, with $i$-th diagonal
entry $p^{(m)}(\lambda_i) - p^{(n)}(\lambda_i)$. The operator norm of a real
diagonal matrix is the maximum of the absolute values of its diagonal entries,
hence:
\begin{equation}
\label{eq:diag-norm}
\big\|p^{(m)}(\Lambda) - p^{(n)}(\Lambda)\big\|_2
=
\max_{1\le i\le nd}
\big|p^{(m)}(\lambda_i) - p^{(n)}(\lambda_i)\big|.
\end{equation}
Combining \eqref{eq:unitary-invariance} and \eqref{eq:diag-norm} we obtain:
\begin{equation}
\label{eq:opnorm-poly-diff}
\big\|p^{(m)}(L_{\mathcal F}) - p^{(n)}(L_{\mathcal F})\big\|_2
=
\max_{\lambda\in\sigma(L_{\mathcal F})}
\big|p^{(m)}(\lambda) - p^{(n)}(\lambda)\big|.
\end{equation}
Since $\sigma(L_{\mathcal F})\subset [0,\lambda_{\max}]$, the right-hand side
is bounded by the uniform difference on the whole interval:
\[
\max_{\lambda\in\sigma(L_{\mathcal F})}
\big|p^{(m)}(\lambda) - p^{(n)}(\lambda)\big|
\le
\sup_{\lambda\in[0,\lambda_{\max}]}
\big|p^{(m)}(\lambda) - p^{(n)}(\lambda)\big|.
\]
By \eqref{eq:weierstrass}, $(p^{(m)})_m$ is Cauchy in the uniform norm on
$[0,\lambda_{\max}]$, so the right-hand side tends to zero as $m,n\to\infty$.
Hence $(p^{(m)}(L_{\mathcal F}))_m$ is Cauchy in operator norm and therefore
convergent in the finite-dimensional space of operators on $C^0(\mathcal G;\mathcal F)$.

\emph{Step 3: Identification of the limit with $U f(\Lambda) U^\top$.}
Let's define the diagonal matrix $f(\Lambda)=\mathrm{diag}\!\big(f(\lambda_1),\dots,f(\lambda_{nd})\big)$ 
and the operator $f(L_{\mathcal F}) := U f(\Lambda) U^\top$. We claim that
$p^{(m)}(L_{\mathcal F}) \to f(L_{\mathcal F})$ in operator norm. Indeed, similarly to \eqref{eq:opnorm-poly-diff}, we have:
\begin{align}
\big\|p^{(m)}(L_{\mathcal F}) - f(L_{\mathcal F})\big\|_2
=
\big\|U\big(p^{(m)}(\Lambda) - f(\Lambda)\big)U^\top\big\|_2
&=
\big\|p^{(m)}(\Lambda) - f(\Lambda)\big\|_2\\
&=
\max_{1\le i\le nd}
\big|p^{(m)}(\lambda_i) - f(\lambda_i)\big|
\end{align}
Again using $\sigma(L_{\mathcal F})\subset [0,\lambda_{\max}]$ and
\eqref{eq:weierstrass}, we bound $\max_{1\le i\le nd}
\big|p^{(m)}(\lambda_i) - f(\lambda_i)\big|
\le
\sup_{\lambda\in[0,\lambda_{\max}]}
\big|p^{(m)}(\lambda) - f(\lambda)\big|
\xrightarrow[m\to\infty]{} 0$. Hence, we get:
\begin{equation}
    p^{(m)}(L_{\mathcal F}) \;\xrightarrow[m\to\infty]{\|\cdot\|_2}\; f(L_{\mathcal F})
\end{equation}
This shows both that $f(L_{\mathcal F})$ is the operator-theoretic limit of
$p^{(m)}(L_{\mathcal F})$ and that the operator is independent of the
particular approximating sequence $(p^{(m)})_m$, since any two such sequences
will converge to the same diagonal $f(\Lambda)$ entrywise on the eigenvalues.
\end{proof}
\begin{proof}
\emph{Step 4: Operator norm identity.}
We now prove that:
\begin{equation}
    \|f(L_{\mathcal F})\|_2
=
\max_{\lambda\in\sigma(L_{\mathcal F})}|f(\lambda)|.
\end{equation}
First, using the orthogonal invariance of the spectral norm, we have:
\begin{equation}
\|f(L_{\mathcal F})\|_2
=
\|U f(\Lambda) U^\top\|_2
=
\|f(\Lambda)\|_2.
\end{equation}
Thus it suffices to compute the operator norm of the diagonal matrix
$f(\Lambda)$. Let $\mu_i := f(\lambda_i)$ denote its diagonal entries. For any
vector $y\in\mathbb R^{nd}$ with $\|y\|_2=1$ we have:
\begin{equation}
\|f(\Lambda) y\|_2^2
=
\sum_{i=1}^{nd} \mu_i^2 y_i^2
\le
\Big(\max_i |\mu_i|^2\Big)\sum_{i=1}^{nd} y_i^2
=
\Big(\max_i |\mu_i|\Big)^2.
\end{equation}
Taking square roots and the supremum over all $\|y\|_2=1$ yields $\|f(\Lambda)\|_2
\le
\max_i |\mu_i|
=
\max_i |f(\lambda_i)|$. To see that equality holds, pick an index $j$ such that
$|\mu_j| = \max_i |\mu_i|$, and let $e_j$ be the $j$-th standard basis vector
in $\mathbb R^{nd}$ (which has norm $1$). , $\|f(\Lambda)e_j\|_2 =
|\mu_j|
=
\max_i |f(\lambda_i)|$. Therefore, we have: 
\begin{equation}
\|f(\Lambda)\|_2
=
\sup_{\|y\|_2=1} \|f(\Lambda)y\|_2
\ge
\|f(\Lambda)e_j\|_2
=
\max_i |f(\lambda_i)|.
\end{equation}
Combining the upper and lower bounds, we conclude that $\|f(\Lambda)\|_2
=
\max_i |f(\lambda_i)|.$ Hence, we can prove the claimed norm identity: 
\begin{equation}
    \|f(L_{\mathcal F})\|_2
=
\|f(\Lambda)\|_2
=
\max_{1\le i\le nd} |f(\lambda_i)|
=
\max_{\lambda\in\sigma(L_{\mathcal F})} |f(\lambda)|
\end{equation}
\end{proof}

\noindent
In particular, taking $f(\lambda)=p(\lambda)$ as a polynomial recovers \autoref{lem:poly-functional}. Taking $f(\lambda)=e^{-t\lambda}$ gives the
\emph{sheaf heat kernel} $e^{-tL_{\mathcal F}}$, i.e.\ the diffusion semigroup
at time $t$, which contracts each sheaf Fourier mode by $e^{-t\lambda_i}$.

\subsection{Smoothness, Global Sections and Opinion Dynamics}
\label{app:spectral-smoothness}

The spectral decomposition of $L_{\mathcal F}$, in addition to provide a convenient basis, it organises sheaf signals according to how \emph{globally consistent} or \emph{disagreeing} they are with respect to the underlying transports. At one extreme, vectors in the kernel of $L_{\mathcal F}$ generate perfectly aligned sheaf sections that incur no edge-wise disagreement. At the other, eigenvectors 
with large eigenvalues necessarily create strong local conflicts after transport.
Between these two extremes lies the full spectrum of ``almost harmonic'' modes,
whose small but positive eigenvalues encode smooth, slowly varying patterns of
opinion. This viewpoint is made precise by the Dirichlet energy identity
(\autoref{prop:dirichlet-spectral}) and matches the linear opinion-dynamics
models where diffusion driven by $L_{\mathcal F}$ relaxes arbitrary initial
profiles toward globally consistent sections. We now connect the spectrum of $L_{\mathcal F}$ with smoothness and global consistency of sheaf signals, as well as with the opinion-dynamics viewpoint.

\paragraph{Zero Modes and Global Sections.}
The kernel of $L_{\mathcal F}$ encodes \emph{discrete harmonic} $0$-cochains
for the sheaf: by \autoref{lem:spectral-delta}, we have $L_{\mathcal F}x=0
\Longleftrightarrow\delta x=0$, i.e., $x$ is a $0$-cocycle (i.e., a 0-cochain whose coboundary is zero). When the sheaf is suitably regular,
\cite{hansen2019toward} shows that $\ker L_{\mathcal F}$ is isomorphic to the
space of \emph{global sections} (or $0$-th cohomology) $H^0(\mathcal G;\mathcal
F)$, meaning that $\dim\ker L_{\mathcal F} \;=\; \dim H^0(\mathcal G;\mathcal F)$. Intuitively, eigenvectors with $\lambda_i=0$ are perfectly consistent sheaf
signals that incur zero disagreement on every edge, and therefore correspond to
globally aligned opinion profiles across the network.

\paragraph{Opinion Dynamics and Diffusion.}
In the linearised opinion-dynamics model of \cite{hansen2021opinion}, node
opinions $x_t\in C^0(\mathcal G;\mathcal F)$ evolve according to a (discrete or
continuous) diffusion equation driven by $L_{\mathcal F}$, such as $ x_{t+1} = x_t - \eta L_{\mathcal F} x_t,
\quad\text{or}\quad
\dot x(t) = -L_{\mathcal F} x(t)$, for some step size $\eta>0$. In the eigenbasis of $L_{\mathcal F}$, each mode evolves independently as:
\begin{equation}
    \widehat x_{t+1}(\lambda_i)
=
(1-\eta\lambda_i)\,\widehat x_t(\lambda_i)
\quad\text{or}\quad
\widehat x(t,\lambda_i)
=
e^{-\lambda_i t}\,\widehat x(0,\lambda_i)    
\end{equation}
so low-frequency modes (small $\lambda_i$) decay slowly and persist, while
high-frequency (large-$\lambda_i$) disagreement patterns decay quickly. In
steady state, only the kernel modes survive, corresponding to globally
consistent sections. This directly supports the interpretation used in the main
text: small eigenvalues correspond to smooth, globally aligned sheaf signals,
and large eigenvalues to oscillatory disagreement patterns.

\paragraph{Spectral Filters as Opinion Shapers.}
Given a spectral multiplier $p(L_{\mathcal F})$ as in
\autoref{lem:poly-functional}, the response of the $i$-th sheaf Fourier mode is
scaled by $p(\lambda_i)$. Choosing $p(\lambda)$ decreasing in $\lambda$
emphasises smooth, consensus-like profiles (low-pass); band-pass choices
emphasise intermediate disagreement patterns; and choices increasing in
$\lambda$ can highlight highly localised conflicts. In this sense, the learned
polynomial filters $p(L_{\mathcal F})$ in PolyNSD explicitly shape how opinion
profiles are smoothed, sharpened or selectively amplified across the sheaf
spectrum, while remaining decomposition-free and implementable via sparse
matrix--vector products.

\subsection{Polynomial Filters: K-hop Locality and Energy Monotonicity}
\label{app:poly-proofs}

In this subsection we justify the structural properties of polynomial sheaf
filters used in \autoref{sec:polysd}, namely \emph{K-hop Locality}
(\autoref{prop:polysd-khop}) and a \emph{non-increasing Dirichlet Energy} for
diffusion-like multipliers (\autoref{prop:polysd-commute-energy}). Throughout,
$L$ denotes a symmetric positive semidefinite (vertex) sheaf Laplacian on a
finite graph $G=(V,E)$, with the usual block sparsity pattern induced by
edges, and $p,q$ denote real polynomials.

\paragraph{K-hop Locality.}In order to prove K-hop locality principle, we first make precise how the sparsity pattern of $L^k$ reflects walks in $G$:

\begin{lemma}[\textbf{Support of Powers and Walks}]
\label{lem:power-walks}
Let $L$ be a block matrix indexed by $V$ such that its off-diagonal block
$(v,u)$ is nonzero only when $(v,u)\in E$, and diagonal blocks $(v,v)$ are
arbitrary. Then for any $k\ge 1$ and any pair $(v,u)\in V\times V$: $
    \big(L^k\big)_{vu}\neq 0
\quad\Rightarrow\quad
\exists\text{ a walk } v=v_0\sim v_1\sim\cdots\sim v_\ell=u
\text{ in }G \text{ of length } \ell\le k$. In particular, if $\mathrm{dist}_G(v,u)>k$ (no walk of length $\le k$), then
$\big(L^k\big)_{vu}=0$.
\end{lemma}

\begin{proof}
We proceed by induction on $k$.

\emph{Base case $k=1$.}
By assumption on $L$, for $v\neq u$ we have $L_{vu}\neq 0$ only if $(v,u)\in E$,
which is a walk of length $1$. For $v=u$ the block $L_{vv}$ may be nonzero,
corresponding to the trivial walk of length $0\le 1$. Hence the claim holds.

\emph{Induction step.}
Assume the statement holds for some $k\ge 1$. Consider $L^{k+1}=L^k L$. Its
$(v,u)$ block satisfies: $\big(L^{k+1}\big)_{vu}
=
\sum_{w\in V} \big(L^k\big)_{vw}\,L_{wu}$. If $\big(L^{k+1}\big)_{vu}\neq 0$, then there exists some $w$ such that
$\big(L^k\big)_{vw}\neq 0$ and $L_{wu}\neq 0$. By the induction hypothesis,
$\big(L^k\big)_{vw}\neq 0$ implies the existence of a walk
$v=v_0\sim v_1\sim\cdots\sim v_\ell=w$ of length $\ell\le k$. The condition
$L_{wu}\neq 0$ means either $w=u$ (diagonal block) or $(w,u)\in E$ (off-diagonal
edge block). In particular:
\begin{itemize}
    \item If $w=u$, we obtain a walk from $v$ to $u$ of length $\ell\le k<k+1$ by appending a trivial step at $u$.
    \item If $(w,u)\in E$, we can append this edge to the walk $v\leadsto w$ and obtain a walk $v=v_0\sim\cdots\sim v_\ell=w\sim u$ of length $\ell+1\le k+1$.
\end{itemize}
In both cases, there exists a walk from $v$ to $u$ of length at most $k+1$. This
proves the implication in the lemma for $k+1$, closing the induction. For the final statement, if $\mathrm{dist}_G(v,u)>k$ then by definition there is
no walk from $v$ to $u$ of length $\le k$, so the implication above forces
$\big(L^k\big)_{vu}=0$.
\end{proof}

We can now prove the \emph{K-hop locality of Polynomial Filters}.

\begin{proof}[Proof of \autoref{prop:polysd-khop}]
Recall that, as defined in \autoref{eq:polysd-poly-filter}, we have: $p(L)=\sum_{k=0}^K c_k L^k$, and more specifically, for $v,u\in V$ we have $
    \big(p(L)\big)_{vu}
=
\sum_{k=0}^K c_k\,\big(L^k\big)_{vu}$. If $\mathrm{dist}_G(v,u)>K$, then $\mathrm{dist}_G(v,u)>k$ for all $0\le k\le K$. For $k=0$ we have $L^0=I$, so $(L^0)_{vu}=0$ whenever $v\neq u$. For $k\ge 1$, \autoref{lem:power-walks} implies that $\big(L^k\big)_{vu}=0$ whenever $\mathrm{dist}_G(v,u)>k$. Therefore, each term in the sum vanishes, and we obtain $\big(p(L)\big)_{vu}=0$ whenever $\mathrm{dist}_G(v,u)>K$, as claimed.
\end{proof}

\paragraph{Commutation and Dirichlet Energy Monotonicity.}

We now justify \autoref{prop:polysd-commute-energy}.

\begin{proof}[Proof of \autoref{prop:polysd-commute-energy}]
Let $L$ be symmetric PSD and let $L=U\Lambda U^\top$ be its spectral
decomposition, with $\Lambda=\mathrm{diag}(\lambda_1,\dots,\lambda_{nd})$ and
$U$ orthogonal. For any real polynomial $r$ we have, by
\autoref{lem:poly-functional}, we have $ r(L)=U\,r(\Lambda)\,U^\top$, where $r(\Lambda)$ is diagonal with entries $r(\lambda_i)$.

\emph{Commutation and Polynomial Composition.}
Let $p,q$ be real polynomials. Then:
\begin{equation}
    p(L)q(L)
=
\big(U p(\Lambda)U^\top\big)\big(U q(\Lambda)U^\top\big)
=
U\,p(\Lambda)\,(U^\top U)\,q(\Lambda)\,U^\top
=
U\,p(\Lambda)q(\Lambda)\,U^\top
\end{equation}
Since $p(\Lambda)$ and $q(\Lambda)$ are diagonal matrices, they commute and
their product is diagonal with entries $p(\lambda_i)q(\lambda_i)$. In other
words, we have: $p(\Lambda)q(\Lambda) = q(\Lambda)p(\Lambda) = (pq)(\Lambda)$ and therefore, $p(L)q(L) = U\,(pq)(\Lambda)\,U^\top = (pq)(L)$. By symmetry of the argument, $q(L)p(L)=(pq)(L)$ as well. This proves the commutation and composition identity.

\emph{Dirichlet Energy under a Diffusion Multiplier.}
Let $x\in C^0(\mathcal G;\mathcal F)$ and write $x=U\hat x$, so that
$\hat x=U^\top x$ are the sheaf Fourier coefficients. Using the spectral
functional calculus, we have $p(L)x = U p(\Lambda) U^\top x = U p(\Lambda) \hat x
$. The Dirichlet energy of $p(L)x$ with respect to $L$ is:
\begin{align}
\big\langle p(L)x,\ L\,p(L)x\big\rangle
=
\big\langle U p(\Lambda)\hat x,\ U\Lambda p(\Lambda)\hat x\big\rangle
=
\hat x^\top p(\Lambda)^\top \Lambda p(\Lambda)\,\hat x
\end{align}
where we used orthogonality of $U$ and symmetry of $L$. Since $p(\Lambda)$ and
$\Lambda$ are diagonal and real, $p(\Lambda)^\top=p(\Lambda)$ and these matrices
all commute. Writing $\mu_i:=p(\lambda_i)$, we obtain, $p(\Lambda) \Lambda p(\Lambda) = \mathrm{diag}\!\big(\lambda_1 \mu_1^2,\dots,\lambda_{nd}\mu_{nd}^2\big)$. We therefore have then: 
\begin{equation}
  \big\langle p(L)x,\ L\,p(L)x\big\rangle
=
\sum_{i=1}^{nd} \lambda_i\,\mu_i^2\,\hat x_i^2
=
\sum_{i=1}^{nd} \lambda_i\,p(\lambda_i)^2\,\hat x_i^2  
\end{equation}
By contrast, the original Dirichlet energy of $x$ is, by
\autoref{prop:dirichlet-spectral}, defined as: $\langle x,Lx\rangle
=
\sum_{i=1}^{nd} \lambda_i\,\hat x_i^2
$.
Assume now that $0\le p(\lambda)\le 1$ on $\sigma(L)$. Then for every
eigenvalue $\lambda_i$ we have $0\le p(\lambda_i)\le 1$ and therefore
$0\le p(\lambda_i)^2\le 1$. Since each term in the sums is nonnegative
($\lambda_i\ge 0$ and $\hat x_i^2\ge 0$), we obtain $\lambda_i\,p(\lambda_i)^2\,\hat x_i^2
\;\le\;
\lambda_i\,\hat x_i^2
\quad\text{for all }i$. Summing over $i$ yields:
\begin{equation}
    \big\langle p(L)x,\ L\,p(L)x\big\rangle
=
\sum_{i=1}^{nd} \lambda_i\,p(\lambda_i)^2\,\hat x_i^2
\;\le\;
\sum_{i=1}^{nd} \lambda_i\,\hat x_i^2
=
\langle x,Lx\rangle
\end{equation}
which is precisely \eqref{eq:polysd-energy-monotone}. This shows that any
spectral multiplier with $0\le p(\lambda)\le 1$ on the spectrum acts as a
diffusion-like contraction of Dirichlet energy.
\end{proof}

\noindent
Together, \autoref{prop:polysd-khop} and \autoref{prop:polysd-commute-energy}
formalise the two key structural properties of PolyNSD layers: \emph{locality}
(a single layer implements exactly $K$-hop mixing) and \emph{stability}
(diffusion-like multipliers cannot increase disagreement energy).

\subsection{Chebyshev Filters: Spectral Rescaling, Norm Control and Energy Monotonicity}
\label{app:cheb-bound}
In this subsection we make precise the claims stated in \autoref{subsec:polysd-poly-filter}. In particular, we justify why the affine rescaling in \autoref{eq:polysd-rescale} is a \emph{structural} requirement for
Chebyshev parameterisations, and prove that, once the spectrum of $L$ is mapped
to $[-1,1]$, convex mixtures of first-kind Chebyshev polynomials define
nonexpansive spectral multipliers $p(\widetilde L)$ with non-increasing
Dirichlet energy. This provides the technical backbone for the design choices
in the Chebyshev-PolyNSD layer used in the main text.

\paragraph{Spectral Rescaling to \texorpdfstring{$[-1,1]$}{[-1,1]}.}
The affine rescaling that occurs inside the \emph{Polynomial Neural Sheaf Diffusion Layer}, as described in \autoref{eq:polysd-rescale}, is what makes the PolyNSD layer both \emph{numerically stable} and \emph{theoretically controlled}. In particular:
\begin{enumerate}
\item \emph{Extremal Property and Uniform Bound.}
First-kind Chebyshev polynomials satisfy $T_k(\xi) = \cos(k\arccos\xi), |\xi|\le 1$, and hence:
\begin{equation}
    |T_k(\xi)| \;\le\; 1 \quad\text{for all}\quad \xi\in[-1,1],\ k\ge 0
\end{equation}
This extremal/minimax property \emph{only} holds on $[-1,1]$. For $|\xi|>1$,
Chebyshev polynomials grow exponentially:
\begin{equation}
    T_k(\xi) = \cosh\!\big(k\,\mathrm{arcosh}|\xi|\big),\qquad |\xi|>1
\end{equation}
so $|T_k(\xi)|$ behaves like $c(\xi)\,\rho(\xi)^k$ with $\rho(\xi)>1$. If we
were to apply $T_k(L)$ directly to a Laplacian whose eigenvalues satisfy
$\lambda_i>1$, high-degree terms would produce operator norms that grow
exponentially in $k$. In other words, the very same sequence that is
beautifully bounded on $[-1,1]$ is catastrophically unbounded as soon as
eigenvalues leave that interval. By forcing the spectrum of the \emph{rescaled} operator $\widetilde L$ to lie in $[-1,1]$, we ensure that every eigenvalue $\tilde\lambda_i$ sits exactly in the regime where Chebyshev polynomials are uniformly bounded by $1$. This is what makes convex mixtures $\sum_k \theta_k T_k(\widetilde L)$ nonexpansive
in operator norm.

\item \emph{Stability of the Recurrence.}
The matrix recurrence:
\begin{equation}
    T_0 = x,\qquad T_1 = \widetilde L x,\qquad
T_{k+1} = 2\widetilde L T_k - T_{k-1}
\end{equation}
is a direct lifting of the scalar recurrence $T_{k+1}(\xi)=2\xi T_k(\xi)-T_{k-1}(\xi)$. When $|\xi|\le 1$, this recurrence is numerically stable: rounding errors do not get amplified exponentially and $T_k(\xi)$ remains $\mathcal O(1)$ in $k$. If instead some eigenvalues $\lambda_i$ satisfy $|\lambda_i|>1$, the same recurrence \emph{amplifies}
errors and the computed $T_k(\lambda_i)$ will quickly diverge in magnitude,
even for moderate degrees $K$. Rescaling to $[-1,1]$ by design avoids this
pathological regime and allows us to evaluate $p(\widetilde L)$ reliably with
$K$ sparse matrix–vector multiplies.

\item \emph{Transferring Scalar Approximation Bounds to Operators.}
All classical Chebyshev approximation results, are stated for functions on $[-1,1]$.
What we want in PolyNSD is an \emph{operator} approximation of a spectral
multiplier $f(L)$, with $\sigma(L)\subset[0,\lambda_{\max}]$. The affine map $\xi(\lambda) = \frac{2}{\lambda_{\max}}\lambda - 1$ identifies $[0,\lambda_{\max}]$ with $[-1,1]$. We approximate $\tilde f(\xi):=f(\tfrac{\lambda_{\max}}{2}(\xi+1))$ by a Chebyshev polynomial
$p^{(K)}(\xi)$ on $[-1,1]$, and then lift the scalar bound $\sup_{\xi\in[-1,1]}|\tilde f(\xi)-p^{(K)}(\xi)|$ to an operator bound:
$ \|f(L)-p^{(K)}(\widetilde L)\|_2 = 
\max_{\lambda_i\in\sigma(L)}
\big|f(\lambda_i)-p^{(K)}(\xi(\lambda_i))\big|$.
\end{enumerate}

We now specialise the general energy argument of \autoref{prop:polysd-commute-energy} to the Chebyshev parametrisation used in PolyNSD. This is then easily provable to be the same also for different basis parameterisations. The key point is that convex mixtures of first-kind Chebyshev polynomials on $[-1,1]$ are uniformly bounded by $1$, which immediately implies nonexpansiveness and non-increasing Dirichlet energy after spectral rescaling.


\begin{proposition}[\textbf{Boundedness and Energy Monotonicity of Convex Chebyshev Mixtures}]
\label{prop:cheb-bound-appendix}
If $\sigma(\widetilde L)\subset[-1,1]$ and $\theta\in\Delta^{K}$, then $\label{eq:cheb-operator-norm}
\big\|p_\theta(\widetilde L)\big\|_2 \;\le\; 1$, and for any symmetric PSD $L$ with $\widetilde L$ as above and any $x\in C^0(\mathcal G;\mathcal F)$, we have: $\big\langle p_\theta(\widetilde L)x,\ L\,p_\theta(\widetilde L)x\big\rangle
\;\le\;
\langle x,Lx\rangle$.
\end{proposition}

\begin{proof}
By the spectral theorem, there exists an orthogonal $U$ and diagonal
$\widetilde\Lambda=\mathrm{diag}(\tilde\lambda_i)$ with
$\tilde\lambda_i\in[-1,1]$ such that $\widetilde L=U\widetilde\Lambda U^\top$.
By polynomial functional calculus,
$p_\theta(\widetilde L)=U\,p_\theta(\widetilde\Lambda)\,U^\top$ with
$p_\theta(\widetilde\Lambda)=\mathrm{diag}(p_\theta(\tilde\lambda_i))$.
Since $|T_k(\xi)|\le 1$ for all $\xi\in[-1,1]$ and $\theta\in\Delta^{K}$,
for any $\tilde\lambda\in[-1,1]$ we have:
\begin{equation}
    \bigl|p_\theta(\tilde\lambda)\bigr|
=
\Bigl|\sum_{k}\theta_k T_k(\tilde\lambda)\Bigr|
\le
\sum_k \theta_k |T_k(\tilde\lambda)|
\le
\sum_k\theta_k
=1
\end{equation}
Thus $|p_\theta(\tilde\lambda_i)|\le 1$ for all $i$, and the spectral norm of
$p_\theta(\widetilde L)$ is $ \big\|p_\theta(\widetilde L)\big\|_2
=
\big\|p_\theta(\widetilde\Lambda)\big\|_2
=
\max_i |p_\theta(\tilde\lambda_i)|
\le 1$ which proves \eqref{eq:cheb-operator-norm}. For the energy claim, write $L=U\Lambda U^\top$ and $x=U\hat x$ as in \autoref{prop:dirichlet-spectral}. Since $\widetilde L$ is an affine function of $L$, the two are simultaneously diagonalisable by $U$, and by \autoref{prop:polysd-commute-energy} we obtain $
    \big\langle p_\theta(\widetilde L)x,\ L\,p_\theta(\widetilde L)x\big\rangle
=
\sum_{i}\lambda_i\,p_\theta(\tilde\lambda_i)^2\,\hat x_i^2$. We have $0\le p_\theta(\tilde\lambda_i)^2\le 1$ for all $i$ by the previous bound, so each term satisfies $\lambda_i p_\theta(\tilde\lambda_i)^2\hat x_i^2\le \lambda_i\hat x_i^2$. Since all quantities are nonnegative, summing over $i$ yields, as claimed: 
\begin{equation}
    \big\langle p_\theta(\widetilde L)x,\ L\,p_\theta(\widetilde L)x\big\rangle
\le
\sum_i \lambda_i\hat x_i^2
=
\langle x,Lx\rangle
\end{equation}
\end{proof}
So, Proposition~\ref{prop:cheb-bound-appendix} shows that, once the spectrum is
rescaled to $[-1,1]$, the Chebyshev-PolyNSD layer inherits a \emph{hard}
$L^2$-stability guarantee: the linear core has operator norm at most $1$ and
cannot increase the sheaf Dirichlet energy. This means that
PolyNSD behaves as a spectral regulariser whose frequency response is learned
within a nonexpansive envelope. It is worth noting that the proof strategy is not specific to Chebyshev polynomials: any orthogonal basis $\{B_k\}$ that is uniformly
bounded on $[-1,1]$, combined with a convex parametrisation of the coefficients,
admits an analogous argument. Chebyshev polynomials are simply the canonical
instance that simultaneously provide (i) this stability, (ii) near-minimax
approximation properties, and (iii) an efficient three-term recurrence.

%% file: chapters/B_Appendix.tex
\newpage
\section{PolyNSD Layer}
\label{app:polynsd_layer}

This section makes precise the layer-level construction sketched in
\autoref{fig:polynsd-layer}, and provides detailed operator-theoretic
and approximation-theoretic guarantees. We first introduce the Chebyshev-PolyNSD
block, then describe the full architecture, discuss approximation of diffusion
kernels, justify our choices for estimating the spectral scale, and finally
analyse the high-pass skip and residual gating.


\subsection{Chebyshev-PolyNSD Layer}
\label{app:cheb-polynsd-detailed}

Let $L$ be either the unnormalised sheaf
Laplacian $L_{\mathcal F}$ or its degree-normalised variant
$\Delta_{\mathcal F}$, and let $\lambda_{\max}$ be an upper bound on
$\sigma(L)$ (equal to 2 in the normalised case, set analytically or estimated by power iteration in the unnormalised case (see \autoref{subsec:lambda-max})). We first rescale to $[-1,1]$ as in \autoref{eq:polysd-rescale}, i.e., $\widetilde L \;=\; \frac{2}{\lambda_{\max}}L - I \sigma(\widetilde L)\subset[-1,1]$. \emph{Chebyshev Polynomials} of the first kind satisfy the following: 
\begin{equation} 
    T_0(\xi)=1,\qquad T_1(\xi)=\xi,\qquad
T_{k+1}(\xi)=2\xi T_k(\xi)-T_{k-1}(\xi)
\end{equation}
which we lift to operators via the three-term recurrence:
\begin{equation}
\label{eq:polysd-cheb-recurrence-appendix}
T_0 = x,\qquad T_1 = \widetilde L\,x,\qquad
T_{k+1} = 2\,\widetilde L\,T_k - T_{k-1},\quad k\ge 1
\end{equation}
A degree-$K$ Chebyshev-PolyNSD filter with trainable logits
$\eta\in\mathbb R^{K+1}$ and $\theta=\mathrm{softmax}(\eta)\in\Delta^K$ is
then:
\begin{equation}
\label{eq:polysd-cheb-filter-appendix}
p_\theta(\widetilde L)\,x \;=\; \sum_{k=0}^{K}\theta_k\,T_k
\end{equation}
By construction $|T_k(\xi)|\le 1$ for all $\xi\in[-1,1]$ and
$\theta$ is a convex combination, so $|p_\theta(\xi)|\le 1$ on $[-1,1]$.
Therefore $\|p_\theta(\widetilde L)\|_2\le 1$, and by
\autoref{prop:polysd-commute-energy} and
\autoref{prop:cheb-bound-appendix} Chebyshev-PolyNSD layers are linear
\emph{nonexpansive} maps that cannot increase the Dirichlet energy
$\langle x,Lx\rangle$.

\subsection{Full Polynomial Neural Sheaf Diffusion Architecture}
\label{app:polynsd-architecture}

\autoref{fig:polynsd-layer} provides a detailed view of the full
PolyNSD architecture, from raw node features to task outputs. We summarise each step, but in subsequent subsections, you can find extensive details for each of them. 

\emph{(1) Feature Lifting to Stalks.}
Given a graph $G=(V,E)$ with raw node features
$x_v^{\mathrm{raw}}\in\mathbb R^{F_{\mathrm{in}}}$, an input MLP
$\phi:\mathbb R^{F_{\mathrm{in}}}\to\mathbb R^d$ produces stalk features
$x_v=\phi(x_v^{\mathrm{raw}})\in\mathbb R^d$. Stacking across nodes yields a
$0$-cochain $x\in C^0(\mathcal G;\mathcal F)\cong\mathbb R^{Nd\times C}$, where
$C$ denotes the number of feature channels.

\emph{(2) Sheaf Learner and Restriction Maps.}
A sheaf learner $\Psi$ takes as input local edge neighbourhoods (and optionally
edge attributes) and outputs per-incidence restriction maps
$\mathcal F_{v\unlhd e}:\mathbb R^d\to\mathbb R^d$. We support three families:
\emph{diagonal} maps $\mathcal F_{v\unlhd e}=\mathrm{diag}(t_{v\unlhd e})$ with
$O(|E|d)$ parameters, \emph{bundle/orthogonal} maps
$\mathcal F_{v\unlhd e}\in O(d)$ acting as parallel transports, and
\emph{general} linear maps $\mathcal F_{v\unlhd e}\in\mathrm{GL}(d)$ with
maximal expressivity. Optional scalar edge weights $w_e$ further modulate the
assembled operator.

\emph{(3) Laplacian Assembly and Spectral Scale.}
From $\{\mathcal F_{v\unlhd e}\}$ we assemble the vertex sheaf Laplacian
$L_{\mathcal F}=\delta_{\mathcal F}^\top\delta_{\mathcal F}$ (and, optionally,
its degree-normalised variant $\Delta_{\mathcal F}$). For the normalised Laplacian we use the
standard spectral bound $\lambda_{\max}=2$, while for the unnormalised case we
either use a Gershgorin-type analytic bound or a short power iteration, as
detailed in \autoref{subsec:lambda-max}.

\emph{(4) PolyNSD + High-pass Skip and Residual Gate.}
Given $L$ and $\lambda_{\max}$, we form $\widetilde L$ and evaluate a
degree-$K$ Chebyshev-PolyNSD filter $p_\theta(\widetilde L)$ via the recurrence
\autoref{eq:polysd-cheb-recurrence-appendix}. To compensate for the intrinsic
low-pass bias of diffusion we add the high-pass skip
$h_{\mathrm{hp}} = x - \lambda_{\max}^{-1} Lx$, scaled by a learned scalar
$\alpha_{\mathrm{hp}}$, and define the pre-nonlinearity
$z = p_\theta(\widetilde L)x + \alpha_{\mathrm{hp}} h_{\mathrm{hp}}$.
A 1-Lipschitz nonlinearity $\phi$ and a diagonal residual gate $\varepsilon$
then produce the update $x^{+} = (I+\tanh\varepsilon)x - \phi(z)$. Its spectral form and global Lipschitz bound are analysed in \autoref{app:hp-gate}.

\emph{(5) Readout and depth.}
A linear readout head maps the final stalk features to logits or regression
targets. Multiple PolyNSD blocks can be stacked. In practice, we recompute
restriction maps at each depth, but keep the within-block recurrence cheap by
reusing the same sparse Laplacian and storing only two work buffers
$(T_{k-1},T_k)$, so the extra memory overhead is $O(NdC)$ and independent of
$K$.

\subsection{Chebyshev Approximation of Diffusion Kernels}
\label{subsec:cheb-approx}

This subsection formalises the intuition that Chebyshev-PolyNSD can approximate
diffusion semigroups and other smooth spectral responses with exponentially
small error in the polynomial degree $K$, provided the target response admits
an analytic extension to a neighbourhood of the spectrum. Let $f:[0,\lambda_{\max}]\to\mathbb R$ be a continuous spectral response and
consider its affine rescaling to $[-1,1]$ via $\xi(\lambda)=\tfrac{2}{\lambda_{\max}}\lambda-1$. We recall the standard notion of a Bernstein ellipse $E_\rho$ with foci at $\pm 1$ and parameter $\rho>1$, and use it to quantify the analyticity region of $f$ (see \autoref{fig:bernstein-ellipse}).
\begin{figure}[h!]
  \centering
  \includegraphics[width=0.35\linewidth]{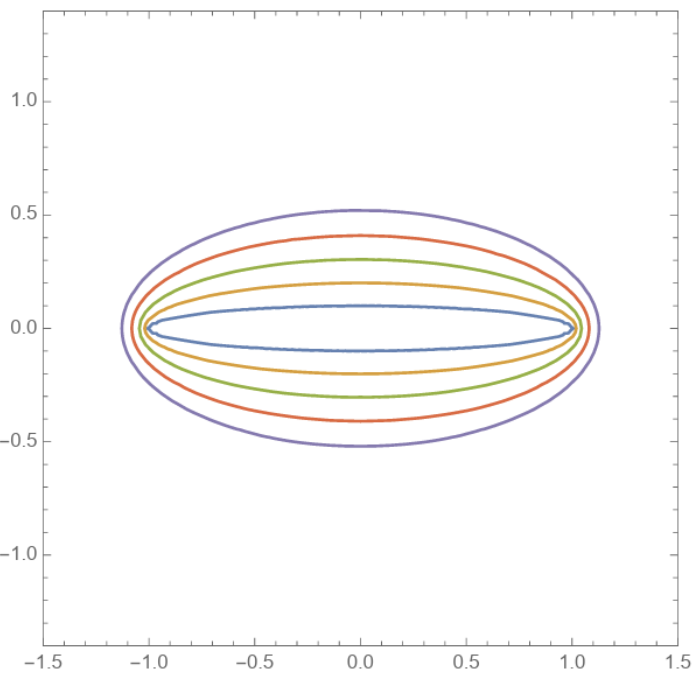}
  \caption{\textit{Bernstein ellipses $E_\rho$ for increasing $\rho$.}
  Each ellipse has foci at $\pm 1$. The rescaled response
  $\tilde f$ on $E_\rho$ implies an exponential-in-$K$ Chebyshev approximation
  rate with factor $\rho^{-K}$. }
  \label{fig:bernstein-ellipse}
\end{figure}
Our first step is to control the Chebyshev coefficients of a function that is
analytic and bounded on a Bernstein ellipse.

\begin{lemma}[\textbf{Chebyshev Coefficients under Bernstein Analyticity}]
\label{lem:cheb-coeff-bernstein}
Let $\tilde f:[-1,1]\to\mathbb R$ be continuous, and assume that $\tilde f$
admits an analytic continuation to the interior of $E_\rho$ and is bounded by
$M$ there, i.e., $ |\tilde f(z)| \le M
\quad\text{for all } z\in E_\rho $. Let $\tilde f(\xi) = \sum_{k=0}^{\infty} a_k T_k(\xi),   \xi\in[-1,1]$ be its (pointwise convergent) Chebyshev series expansion on $[-1,1]$, where $T_k$ are the Chebyshev polynomials of the first kind. Then, for all $k\ge 1$, we have:
\begin{equation}
\label{eq:cheb-coeff-bound}
|a_k| \;\le\; \frac{2M}{\rho^k}.
\end{equation}
\end{lemma}

\begin{proof}
We proceed in three steps: (i) relate $\tilde f$ on $[-1,1]$ to a function on a
circle in the complex plane, (ii) express the Chebyshev coefficients $a_k$ as
contour integrals, and (iii) bound those integrals by Cauchy’s estimate.

\emph{Step 1: Mapping the Circle to the Ellipse.}
Recall that the Bernstein ellipse $E_\rho$ is the image of the circle
$\{w\in\mathbb C:|w|=\rho\}$ under the Joukowski map:
$\Phi(w)
\;:=\;
\frac{1}{2}\Bigl(w + \frac{1}{w}\Bigr)
$. More precisely, for $w=\rho e^{i\theta}$ we have $\Phi(w)\in E_\rho$, and as $\theta$ ranges over $[0,2\pi)$, $\Phi(w)$ traces the ellipse $E_\rho$ once. For $|w|=1$ the same map $\Phi$ parametrises the interval $[-1,1]$ via $\Phi(e^{i\theta})=\cos\theta$.

Define the function $g$ on the annulus $\{w\in\mathbb C : \rho^{-1} < |w| < \rho\}$ by $g(w) := \tilde f(\Phi(w)) = \tilde f\!\Bigl(\frac{1}{2}\bigl(w+w^{-1}\bigr)\Bigr)$. By assumption, $\tilde f$ is analytic on the interior of $E_\rho$, and the Joukowski map is analytic and maps $\{w:|w|\le\rho\}$ onto the closed ellipse and its interior. Hence $g$ is analytic for $|w|<\rho$ (away from $w=0$ the inverse map is well defined, and at $w=0$ analyticity follows by removable singularity, since $\Phi$ is bounded near $0$ and $\tilde f$ is analytic in a
neighbourhood of $\Phi(0)$). On the circle $|w|=\rho$ we have $|\Phi(w)|\in E_\rho$ so $|\tilde f(\Phi(w))|\le M$, thus:
\begin{equation}
\label{eq:g-bound}
|g(w)|\le M \quad\text{for all } |w|=\rho
\end{equation}
\end{proof}
\begin{proof}
\emph{Step 2: Chebyshev Expansion and Fourier Series of $g$.}
On the unit circle $|w|=1$, we can write $w=e^{i\theta}$ and
$\Phi(w)=\cos\theta$. The Chebyshev polynomials satisfy $T_k(\cos\theta) = \cos(k\theta) \text{for all } k\ge 0$, so the Chebyshev series expansion of $\tilde f$ on $[-1,1]$ becomes a cosine series in the angular variable:
\begin{equation}
\label{eq:cheb-series-on-unit-circle}
\tilde f(\cos\theta)
=
\sum_{k=0}^{\infty} a_k T_k(\cos\theta)
=
\sum_{k=0}^{\infty} a_k \cos(k\theta),
\qquad \theta\in\mathbb R
\end{equation}
On the other hand, using $w=e^{i\theta}$, we have $g(e^{i\theta}) = \tilde f\bigl(\Phi(e^{i\theta})\bigr) = \tilde f(\cos\theta)$, so, $g$ is restricted to the unit circle has the same real cosine expansion. We now express this cosine expansion in terms of the complex Fourier coefficients of $g$. Recall that $\cos(k\theta) = \tfrac{1}{2}(e^{ik\theta}+e^{-ik\theta})$, so \autoref{eq:cheb-series-on-unit-circle} becomes:
\begin{equation}
\label{eq:g-cos-fourier}
g(e^{i\theta})
=
\tilde f(\cos\theta)
=
\sum_{k=0}^{\infty} a_k \cos(k\theta)
=
\frac{a_0}{2}
+
\sum_{k=1}^{\infty} \frac{a_k}{2}\bigl(e^{ik\theta}+e^{-ik\theta}\bigr)
\end{equation}
Comparing with the standard Fourier series $g(e^{i\theta}) = \sum_{m=-\infty}^{\infty} c_m e^{im\theta}$, we identify:
\begin{equation}
\label{eq:fourier-coeffs}
c_0 = \frac{a_0}{2},\qquad
c_k = \frac{a_k}{2},\qquad
c_{-k} = \frac{a_k}{2}
\quad\text{for }k\ge 1
\end{equation}

\emph{Step 3: Cauchy Integral Formula and Coefficient Bounds.}
By analyticity of $g$ in $|w|<\rho$, each $c_m$ can be written via Cauchy’s
integral formula on the circle $|w|=\rho$:$c_m
=
\frac{1}{2\pi i} \int_{|w|=\rho} \frac{g(w)}{w^{m+1}}\,dw, m\in\mathbb Z$. Taking absolute values and using \autoref{eq:g-bound} and $|w|=\rho$, we obtain:
\begin{equation}
|c_m|
\le
\frac{1}{2\pi} \int_{|w|=\rho} \frac{|g(w)|}{|w|^{|m|+1}}\,|dw|
\le
\frac{1}{2\pi} \int_{|w|=\rho} \frac{M}{\rho^{|m|+1}}\,|dw|
=
\frac{M}{2\pi\rho^{|m|+1}} \cdot 2\pi\rho
=
\frac{M}{\rho^{|m|}}
\end{equation}
In particular, for $k\ge 1$ we have $|c_k| \le \frac{M}{\rho^k}
\quad\text{and}\quad
|c_{-k}| \le \frac{M}{\rho^k}$. By \autoref{eq:fourier-coeffs}, $a_k = 2c_k = 2c_{-k}$ for $k\ge 1$, we obtain exactly \autoref{eq:cheb-coeff-bound} as:
\begin{equation}
    |a_k|
=
2|c_k|
\le
\frac{2M}{\rho^k}
\end{equation}
\end{proof}

\noindent
We now use this coefficient bound to control the truncation error of the
Chebyshev series, and then lift this scalar bound to an operator bound for
matrix-valued spectral filters.

\begin{theorem}[\textbf{Exponential Chebyshev Convergence for Analytic Spectral Responses}]
\label{thm:cheb-analytic-appendix}
Assume that the rescaled response
$\tilde f(\xi):=f\big(\tfrac{\lambda_{\max}}{2}(\xi+1)\big)$ admits an analytic
extension to the Bernstein ellipse $E_\rho$ for some $\rho>1$, and let
$M:=\max_{z\in E_\rho}|\tilde f(z)|$. Let $p^{(K)}$ be the degree-$K$
Chebyshev truncation of $\tilde f$ on $[-1,1]$. Then:
\begin{equation}
\label{eq:bernstein-scalar}
\sup_{\xi\in[-1,1]}
\big|\tilde f(\xi)-p^{(K)}(\xi)\big|
\;\le\;
\frac{2M}{\rho^K(\rho-1)}
\end{equation}
Consequently, for any symmetric PSD $L$ with spectrum $\sigma(L)\subset[0,\lambda_{\max}]$ and $\widetilde L$ as in \autoref{eq:polysd-rescale}:
\begin{equation}
\label{eq:bernstein-operator}
\big\|f(L)-p^{(K)}(\widetilde L)\big\|_2
\;\le\;
\frac{2M}{\rho^K(\rho-1)}
\end{equation}
In particular, the heat kernel $f(\lambda)=e^{-t\lambda}$ admits exponentially
convergent Chebyshev approximations on $[0,\lambda_{\max}]$ for any fixed
$t>0$.
\end{theorem}

\begin{proof}
We split the proof into two parts: first we bound the scalar approximation
error of $\tilde f$ by its degree-$K$ Chebyshev truncation on $[-1,1]$, then we
transfer the bound to the operator setting via the spectral decomposition of
$L$ and the affine rescaling $\lambda\leftrightarrow\xi$.

\emph{Step 1: Scalar Chebyshev Truncation Error on $[-1,1]$.}
Let $\tilde f(\xi) = \sum_{k=0}^{\infty} a_k T_k(\xi),
 \xi\in[-1,1]$ be the Chebyshev series expansion of $\tilde f$ on $[-1,1]$, which converges uniformly on $[-1,1]$ under our analyticity assumption. Denote by $p^{(K)}$ the degree-$K$ truncation $p^{(K)}(\xi)
:=\sum_{k=0}^{K} a_k T_k(\xi)$. The truncation error (the tail of the series) is then:
\begin{equation}
    R_K(\xi)
:=
\tilde f(\xi) - p^{(K)}(\xi)
=
\sum_{k=K+1}^{\infty} a_k T_k(\xi)
\end{equation}
Using the uniform bound $|T_k(\xi)|\le 1$ for all $\xi\in[-1,1]$ and
\autoref{eq:cheb-coeff-bound}, we obtain:
\begin{equation}
\label{eq:tail-bound-raw}
|R_K(\xi)|
\le
\sum_{k=K+1}^{\infty} |a_k|\,|T_k(\xi)|
\le
\sum_{k=K+1}^{\infty} |a_k|
\le
\sum_{k=K+1}^{\infty} \frac{2M}{\rho^k}
\quad\text{for all }\xi\in[-1,1]
\end{equation}
The remaining sum is geometric $\sum_{k=K+1}^{\infty} \frac{1}{\rho^k}
=
\frac{\rho^{-(K+1)}}{1-\rho^{-1}}
=
\frac{\rho^{-K}}{\rho(\,1-\rho^{-1}\,)}
=
\frac{\rho^{-K}}{\rho-1}$. Plugging this into \autoref{eq:tail-bound-raw} yields: \begin{equation}
    |R_K(\xi)|
\le
2M\cdot \frac{\rho^{-K}}{\rho-1}
=
\frac{2M}{\rho^K(\rho-1)}
\quad\text{for all }\xi\in[-1,1]
\end{equation}
Taking the supremum over $\xi\in[-1,1]$ gives exactly \autoref{eq:bernstein-scalar}:
\begin{equation}
    \sup_{\xi\in[-1,1]}
\big|\tilde f(\xi)-p^{(K)}(\xi)\big|
=
\sup_{\xi\in[-1,1]} |R_K(\xi)|
\le
\frac{2M}{\rho^K(\rho-1)}
\end{equation}

\emph{Step 2: Lifting the Scalar Bound to Operators.}
Let $L$ be a real symmetric PSD matrix with eigenvalues
$\lambda_1,\dots,\lambda_{nd}$ contained in $[0,\lambda_{\max}]$, and
spectral decomposition $ L = U\Lambda U^\top, \Lambda = \mathrm{diag}(\lambda_1,\dots,\lambda_{nd})$ with $U$ orthogonal. Define the affine rescaling as in \autoref{eq:polysd-rescale}. By functional calculus, we have $f(L) = U f(\Lambda) U^\top, p^{(K)}(\widetilde L) = U\,p^{(K)}(\widetilde\Lambda)\,U^\top$ where $f(\Lambda)$ is diagonal with entries $f(\lambda_i)$, and $p^{(K)}(\widetilde\Lambda)$ is diagonal with entries $p^{(K)}(\tilde\lambda_i)=p^{(K)}(\xi(\lambda_i))$. Hence, $f(L)-p^{(K)}(\widetilde L) = U\big(f(\Lambda)-p^{(K)}(\widetilde\Lambda)\big)U^\top$. The spectral (operator) norm is invariant under orthogonal conjugation, so:
\begin{equation}
\label{eq:operator-norm-diag}
\big\|f(L)-p^{(K)}(\widetilde L)\big\|_2
=
\big\|f(\Lambda)-p^{(K)}(\widetilde\Lambda)\big\|_2
\end{equation}
The right-hand side is the operator norm of a diagonal matrix whose diagonal entries are $d_i := f(\lambda_i) - p^{(K)}(\xi(\lambda_i))$. The spectral norm of a diagonal matrix is just the maximum absolute value of its diagonal entries, thus:
\begin{equation}
\label{eq:diag-norm-max}
\big\|f(\Lambda)-p^{(K)}(\widetilde\Lambda)\big\|_2
=
\max_{1\le i\le nd} |d_i|
=
\max_{\lambda_i\in\sigma(L)}
\big|f(\lambda_i)-p^{(K)}(\xi(\lambda_i))\big|
\end{equation}
By definition of $\tilde f$, we have $\tilde f(\xi(\lambda))
=
f\Bigl(\frac{\lambda_{\max}}{2}(\xi(\lambda)+1)\Bigr)
=
f(\lambda)$, so $f(\lambda_i) = \tilde f(\xi(\lambda_i))$. Therefore for each eigenvalue $\lambda_i$ we can rewrite $\big|f(\lambda_i)-p^{(K)}(\xi(\lambda_i))\big| =  \big|\tilde f(\xi(\lambda_i)) - p^{(K)}(\xi(\lambda_i))\big|$. Since $\xi(\lambda_i)\in[-1,1]$, \autoref{eq:bernstein-scalar} gives $\big|\tilde f(\xi(\lambda_i)) - p^{(K)}(\xi(\lambda_i))\big|
\le
\frac{2M}{\rho^K(\rho-1)}$. Taking the maximum over all $i$ and combining \autoref{eq:operator-norm-diag}–\autoref{eq:diag-norm-max} yields precisely \autoref{eq:bernstein-operator}.

\emph{Step 3: Application to Heat Kernels.}
For the heat kernel $f(\lambda)=e^{-t\lambda}$ with any fixed $t>0$, the
rescaled function: $\tilde f(\xi)=f\Bigl(\tfrac{\lambda_{\max}}{2}(\xi+1)\Bigr)
=
\exp\!\Bigl(-t\,\tfrac{\lambda_{\max}}{2}(\xi+1)\Bigr)$ is an entire function of $\xi\in\mathbb C$, hence analytic on every Bernstein ellipse $E_\rho$ and bounded by some finite $M_\rho$. Thus the assumptions of the theorem hold for all $\rho>1$, and we conclude that the Chebyshev truncations $p^{(K)}$ approximate $f(L)$ with an error that decays
exponentially in $K$ as in \autoref{eq:bernstein-operator}.
\end{proof}

\noindent
\autoref{thm:cheb-analytic-appendix} formalises the intuition behind the
Chebyshev-PolyNSD design: once the Laplacian spectrum is rescaled to $[-1,1]$, the quality of the approximation $p^{(K)}(\widetilde L)\approx f(L)$ is controlled by \emph{how far} the target response $f$ extends analytically into the complex plane, as quantified by the largest Bernstein ellipse $E_\rho$ on which $\tilde f$ remains analytic and bounded. The larger this ellipse (i.e., the larger $\rho$), the faster the error decreases geometrically in the degree $K$ with factor $\rho^{-K}$. The theorem makes the \emph{spectral sharpness} and \emph{analyticity} trade-off explicit: sharper spectral filters require either a higher polynomial degree $K$ or a relaxation of the desired uniform accuracy, whereas smoother filters (such as diffusions) can be captured extremely well with modest $K$.

In addition to that, the scalar bound in \autoref{eq:bernstein-scalar} directly controls the operator approximation error in \autoref{eq:bernstein-operator} through the spectral calculus. Combined with the $K$-hop locality result for polynomial filters (see \autoref{app:poly-proofs}), this yields a clean picture: a single Chebyshev-PolyNSD layer with degree $K$ implements a \emph{strictly $K$-hop
local} operator whose action approximates a global diffusion semigroup
$e^{-tL}$ (or, more generally, any analytic response $f(L)$) up to an error
decaying like $\rho^{-K}$. PolyNSD trades off the depth of the
GNN with the polynomial degree $K$, while retaining both locality and a
spectral approximation guarantee.

Finally, the argument is completely agnostic to the particular choice of $f$:
any spectral response that is analytic on a Bernstein ellipse containing the
rescaled spectrum can be uniformly approximated by Chebyshev polynomials with
an explicit geometric rate. This covers not only diffusion semigroups but also
families of low-pass, band-pass, and high-pass filters used in graph signal
processing, provided they are regularised to admit an analytic extension on a
sufficiently large $E_\rho$. So, \autoref{thm:cheb-analytic-appendix} gives us the approximation-theoretic backbone for Chebyshev-PolyNSD: it guarantees that PolyNSD layers can approximate a wide class of smooth spectral filters (in particular diffusion semigroups), with exponentially small error in the polynomial degree, while retaining strict $K$-hop locality and the stability properties established in \autoref{app:cheb-bound}.

\subsection{Estimating the Spectral Scale: Analytic Bound and Power Iteration}
\label{subsec:lambda-max}

Chebyshev rescaling requires an upper spectral bound $\lambda_{\max}$ for the chosen Laplacian $L$, so that the affine map in \autoref{eq:polysd-rescale}, has spectrum in $[-1,1]$.
In PolyNSD we use two strategies to obtain such a bound:
(i) \emph{closed-form spectral enclosures} for normalised sheaf Laplacians and
for unnormalised Laplacians via Gershgorin-type arguments, and
(ii) a \emph{short power iteration} to refine (or replace) the analytic bound
on unnormalised operators. Both approaches are cheap for sparse sheaf
Laplacians and provide the structural precondition needed by
\autoref{subsec:cheb-approx} and \autoref{app:cheb-bound}.

\paragraph{Normalised Sheaf Laplacian.}

For the degree-normalised sheaf Laplacian
$\Delta_{\mathcal F}=D^{-1/2}L_{\mathcal F}D^{-1/2}$, the situation is
directly analogous to the scalar graph case: all eigenvalues lie in $[0,2]$.

\begin{proposition}[\textbf{Spectral Enclosure for the Normalised Sheaf Laplacian}]
\label{prop:norm-sheaf-spectrum}
Let $\Delta_{\mathcal F}$ be the degree-normalised (vertex) sheaf Laplacian on
a finite sheaf $\mathcal F$ over $G=(V,E)$. Then $\Delta_{\mathcal F}$ is
symmetric positive semidefinite and its spectrum satisfies
$\sigma(\Delta_{\mathcal F}) \subset [0,2]$.
\end{proposition}

\begin{proof}
Symmetry and positive semidefiniteness follow from the usual sheaf Laplacian
construction: $L_{\mathcal F}$ is symmetric PSD as a discrete Hodge Laplacian
and $D^{-1/2}$ is symmetric and invertible, so
$\Delta_{\mathcal F}=D^{-1/2}L_{\mathcal F}D^{-1/2}$ is symmetric and
PSD as well. For the upper bound, we use a Rayleigh quotient argument. For any nonzero
$x\in C^0(\mathcal G;\mathcal F)$, the Rayleigh quotient of $\Delta_{\mathcal F}$
is:
\begin{equation}
  R(x)
  :=
  \frac{\langle x,\Delta_{\mathcal F}x\rangle}{\|x\|_2^2}
  =
  \frac{\langle D^{-1/2}x,\ L_{\mathcal F}D^{-1/2}x\rangle}
       {\|x\|_2^2} =
  \frac{\langle y, L_{\mathcal F}y\rangle}{\langle D^{1/2}y, D^{1/2}y\rangle}
\end{equation}
where the last equality holds if we have $y=D^{-1/2}x$. The numerator has the usual sheaf Dirichlet form representation $\langle y,L_{\mathcal F}y\rangle
=
\frac12\sum_{(u,v)\in E}\big\|T_{vu}y_u-y_v\big\|^2_{\mathcal F_v}\;\ge\;0$, while the denominator is a weighted norm
$\langle D^{1/2}y, D^{1/2}y\rangle=\sum_v\langle D_v y_v, y_v\rangle$.
The same computation as in the scalar normalised Laplacian case shows that $0 \;\le\; R(x) \;\le\; 2, \forall x\neq 0$ because the local contributions at each edge are controlled by the incident degrees. Since the eigenvalues of a symmetric matrix are exactly the extremal values of its Rayleigh quotient, this implies $\sigma(\Delta_{\mathcal F})\subset[0,2]$.
\end{proof}

In practice, for normalised sheaf Laplacians we simply set $\lambda_{\max}=2$,
which guarantees $\sigma(\widetilde L)\subset[-1,1]$ without any numerical
estimation.

\paragraph{Unnormalised Sheaf Laplacian: Analytic Bound via Gershgorin Discs.}

When working with an unnormalised sheaf Laplacian
$L=L_{\mathcal F}=D-A$, we can obtain a cheap but safe upper bound on the
spectral radius by combining Gershgorin’s theorem with the structural
constraints on $D$ and $A$.

\begin{proposition}[\textbf{Gershgorin-Type Bound for Unnormalised Sheaf Laplacians}]
\label{prop:gershgorin-sheaf}
Let $L_{\mathcal F}=D-A$ be an unnormalised sheaf Laplacian, with
$D=\mathrm{blkdiag}(D_v)\succeq 0$ block-diagonal over vertices and
$A$ supported on off-diagonal edge blocks (so that the diagonal of $A$ is
zero). Assume that for each vertex $v$ the diagonal block $D_v$ dominates
the outgoing sheaf couplings in the sense that:
\begin{equation}
\label{eq:block-domination}
  \sum_{u\neq v} \|A_{vu}\|_2 \;\le\; \|D_v\|_2
\end{equation}
Then $L_{\mathcal F}$ is symmetric PSD and its largest eigenvalue satisfies:
\begin{equation}
  \lambda_{\max}(L_{\mathcal F}) \;\le\; 2\max_{v\in V}\|D_v\|_2
\end{equation}
In the scalar trivial-sheaf case ($\dim\mathcal F_v=1$, $D_v=\deg(v)$) this
reduces to the familiar bound $\lambda_{\max}\le 2\max_v\deg(v)$.
\end{proposition}

\begin{proof}
Positive semidefiniteness follows from the usual sheaf Laplacian construction
(as for $L_{\mathcal F}$ above), so all eigenvalues are real and nonnegative. We now apply Gershgorin’s disc theorem to the full matrix $L_{\mathcal F}$, viewed as an $N\times N$ real symmetric matrix with $N=\sum_v\dim\mathcal F_v$. Let $L_{ij}$ denote its scalar entries, and let $\lambda$ be any eigenvalue. Gershgorin’s theorem guarantees that there exists a row index $i$ such that $|L_{ii}-\lambda| \;\le\; \sum_{j\neq i}|L_{ij}|$. The row index $i$ belongs to some vertex fibre $\mathcal F_v$, so the diagonal entry $L_{ii}$ is one of the diagonal entries of the block $D_v$. Hence $|L_{ii}|\le\|D_v\|_2$. On the other hand, the off-diagonal entries $L_{ij}$ in that row come from the off-diagonal sheaf blocks $\{-A_{vu}\}_{u\neq v}$, so by the triangle inequality and the definition of the operator norm, we will have $\sum_{j\neq i}|L_{ij}|   \;\le\;   \sum_{u\neq v}\|A_{vu}\|_2   \;\le\;  \|D_v\|_2$ by the domination assumption \autoref{eq:block-domination}. Therefore, the corresponding Gershgorin disc for row $i$ is contained in the real interval:
\begin{equation}
  [L_{ii}-R_i,\ L_{ii}+R_i]
  \subseteq
  [-\|D_v\|_2,\ 2\|D_v\|_2],
\qquad
R_i:=\sum_{j\neq i}|L_{ij}|
\end{equation}
Since $L_{\mathcal F}$ is symmetric, all its eigenvalues are real and must lie
in the union of these intervals over all rows $i$, and hence over all vertices
$v$. Intersecting with $[0,\infty)$ (because $L_{\mathcal F}$ is PSD), we obtain $0 \;\le\; \lambda \;\le\; 2\max_{v\in V}\|D_v\|_2$ for every eigenvalue $\lambda$ of $L_{\mathcal F}$, which proves the claim.
In the scalar case, $D_v=\deg(v)$ and the off-diagonal entries in row $v$
are $-1$ along each incident edge, so $R_i=\deg(v)$ and the usual scalar
Gershgorin discs $[0,2\deg(v)]$ are recovered.
\end{proof}

In PolyNSD, the Gershgorin-type bound of
\autoref{prop:gershgorin-sheaf} provides a very fast, purely local estimate
of $\lambda_{\max}$ that is guaranteed to be safe for Chebyshev rescaling,
even though it can be somewhat conservative.

\paragraph{Unnormalised Sheaf Laplacian: Power Iteration.}
To obtain a tighter estimate of the largest eigenvalue of $L_{\mathcal F}$, we can refine (or replace) the analytic bound via a short power iteration. For a symmetric matrix $L$ with eigenvalues $\lambda_1\ge\lambda_2\ge\cdots\ge 0$,
the classical power method repeatedly applies $L$ and normalises: $x^{(t+1)}
  \;\leftarrow\;
  \frac{Lx^{(t)}}{\|Lx^{(t)}\|_2},
  \|x^{(0)}\|_2=1$, and uses the Rayleigh quotient $\widehat\lambda^{(t)}
  :=
  \frac{\langle x^{(t)},Lx^{(t)}\rangle}{\|x^{(t)}\|_2^2}
  =
  \langle x^{(t)},Lx^{(t)}\rangle$ as an approximation to $\lambda_1$. The next proposition states a standard error bound.

\begin{proposition}[\textbf{Power-Iteration Error}]
\label{prop:power-error-appendix}
Let $L=U\Lambda U^\top$ be real symmetric with eigenvalues
$\lambda_1>\lambda_2\ge\cdots\ge\lambda_n\ge 0$ and orthonormal eigenvectors
$\{u_i\}_{i=1}^n$. Let $x^{(0)}=\sum_{i=1}^n\alpha_i u_i$ with $\alpha_1\neq 0$
and define $x^{(t)}$ and $\widehat\lambda^{(t)}$ as above. Set
$r:=\lambda_2/\lambda_1\in(0,1)$. Then for every $t\ge 0$:
\begin{equation}
  0\;\le\; \lambda_1-\widehat\lambda^{(t)}
  \;\le\;
  (\lambda_1-\lambda_2)\,r^{2t}\,
  \frac{\|x^{(0)}-\langle x^{(0)},u_1\rangle u_1\|_2^2}
       {\langle x^{(0)},u_1\rangle^2}
\end{equation}
In particular, the Rayleigh quotient converges to $\lambda_1$ at a geometric
rate $r^{2t}$ as $t$ increases.
\end{proposition}

\begin{proof}
We make the dependence on the spectral decomposition explicit. Since
$L=U\Lambda U^\top$ and the eigenvectors form an orthonormal basis, we can
write: $x^{(0)}=\sum_{i=1}^n\alpha_i u_i, \alpha_i=\langle ^{(0)},u_i\rangle$. After $t$ steps of the (unnormalised) power method, we have $L^t x^{(0)} = \sum_{i=1}^n \alpha_i \lambda_i^t u_i$. Normalising, then yields:
\begin{equation}
  x^{(t)}
  =
  \frac{L^t x^{(0)}}{\|L^t x^{(0)}\|_2}
  =
  \frac{\sum_{i=1}^n \alpha_i \lambda_i^t u_i}
       {\Big(\sum_{i=1}^n \alpha_i^2\lambda_i^{2t}\Big)^{1/2}}
\end{equation}
so the coordinates of $x^{(t)}$ in the eigenbasis are: $\beta_i^{(t)}
  :=
  \langle x^{(t)},u_i\rangle
  =
  \frac{\alpha_i \lambda_i^t}
       {\Big(\sum_{j=1}^n \alpha_j^2\lambda_j^{2t}\Big)^{1/2}}$. The Rayleigh quotient at step $t$ can then be expressed as: 
\begin{equation}
  \widehat\lambda^{(t)}
  =
  \langle x^{(t)},Lx^{(t)}\rangle
  =
  \sum_{i=1}^n \lambda_i \big(\beta_i^{(t)}\big)^2
  =
  \frac{\sum_{i=1}^n \alpha_i^2\lambda_i^{2t+1}}
       {\sum_{i=1}^n \alpha_i^2\lambda_i^{2t}}
\end{equation}

\emph{Nonnegativity of the Error.}
We first show that $\widehat\lambda^{(t)}\le\lambda_1$. Since
$\lambda_1\ge\lambda_i$ for all $i$, we have $\lambda_1\sum_{i=1}^n \alpha_i^2\lambda_i^{2t} \;\ge\; \sum_{i=1}^n \alpha_i^2\lambda_i^{2t+1}$ and hence:
\begin{equation}
  \lambda_1 - \widehat\lambda^{(t)}
  =
  \frac{\lambda_1\sum_i \alpha_i^2\lambda_i^{2t}
        -\sum_i \alpha_i^2\lambda_i^{2t+1}}
       {\sum_i \alpha_i^2\lambda_i^{2t}}
  \;\ge\; 0
\end{equation}
Thus $0\le\widehat\lambda^{(t)}\le\lambda_1$ for all $t$.

\emph{Closed-Form Expression for the Error.}
We can rewrite the numerator more transparently by factoring $\lambda_1$:
\begin{equation}
  \lambda_1\sum_i \alpha_i^2\lambda_i^{2t}
  -\sum_i \alpha_i^2\lambda_i^{2t+1}
  =
  \sum_{i=1}^n \alpha_i^2(\lambda_1-\lambda_i)\lambda_i^{2t}
\end{equation}
The $i=1$ term vanishes because $\lambda_1-\lambda_1=0$, so: $\lambda_1-\widehat\lambda^{(t)}
  =
  \frac{\sum_{i=2}^n \alpha_i^2(\lambda_1-\lambda_i)\lambda_i^{2t}}
       {\sum_{i=1}^n \alpha_i^2\lambda_i^{2t}}$.

\emph{Upper Bound via the Spectral Gap.}
For $i\ge 2$ we have $\lambda_i\le\lambda_2$ and
$\lambda_1-\lambda_i\le\lambda_1-\lambda_2$. Therefore:
\begin{equation}
  \sum_{i=2}^n \alpha_i^2(\lambda_1-\lambda_i)\lambda_i^{2t}
  \;\le\;
  (\lambda_1-\lambda_2)\sum_{i=2}^n \alpha_i^2\lambda_i^{2t}
  \;\le\;
  (\lambda_1-\lambda_2)\lambda_2^{2t}\sum_{i=2}^n \alpha_i^2
\end{equation}
On the other hand, the denominator is bounded from below by the $i=1$ term $\sum_{i=1}^n \alpha_i^2\lambda_i^{2t}
  \;\ge\;
  \alpha_1^2\lambda_1^{2t}$. Combining these two inequalities yields:
\begin{equation}
  \lambda_1-\widehat\lambda^{(t)}
  \;\le\;
  (\lambda_1-\lambda_2)\,
  \frac{\lambda_2^{2t}\sum_{i=2}^n \alpha_i^2}
       {\alpha_1^2\lambda_1^{2t}}
  =
  (\lambda_1-\lambda_2)\,
  \Big(\frac{\lambda_2}{\lambda_1}\Big)^{2t}\,
  \frac{\sum_{i=2}^n\alpha_i^2}{\alpha_1^2}
\end{equation}
Recalling that $r=\lambda_2/\lambda_1\in(0,1)$, we identify:
\begin{equation}
  \sum_{i=2}^n\alpha_i^2
  =
  \|x^{(0)}-\alpha_1 u_1\|_2^2
  =
  \|x^{(0)}-\langle x^{(0)},u_1\rangle u_1\|_2^2,
  \qquad
  \alpha_1=\langle x^{(0)},u_1\rangle
\end{equation}
so that we have the following equation that we can substitute in the previous inequality:
\begin{equation}
  \frac{\sum_{i=2}^n\alpha_i^2}{\alpha_1^2}
  =
  \frac{\|x^{(0)}-\langle x^{(0)},u_1\rangle u_1\|_2^2}
       {\langle x^{(0)},u_1\rangle^2},
\end{equation}
\begin{equation}
  \lambda_1-\widehat\lambda^{(t)}
  \;\le\;
  (\lambda_1-\lambda_2)\,r^{2t}\,
  \frac{\|x^{(0)}-\langle x^{(0)},u_1\rangle u_1\|_2^2}
       {\langle x^{(0)},u_1\rangle^2}
\end{equation}
Since $0<r<1$, the prefactor $r^{2t}$ ensures geometric
convergence of $\widehat\lambda^{(t)}$ to $\lambda_1$.
\end{proof}

In PolyNSD we typically combine these ingredients as follows. For normalised
sheaf Laplacians we fix $\lambda_{\max}=2$ by
\autoref{prop:norm-sheaf-spectrum}. For unnormalised sheaf Laplacians we
initialise $\lambda_{\max}$ using the Gershgorin-type analytic bound of
\autoref{prop:gershgorin-sheaf}, which requires only local degree information,
and (optionally) refine it with $5$--$10$ steps of power iteration as in
\autoref{prop:power-error-appendix}. This yields a spectrally safe rescaling
$\widetilde L$ for the Chebyshev-PolyNSD layer, while keeping the overhead
negligible compared to the overall training cost.

\subsection{High-pass Skip and Residual Gating}
\label{app:hp-gate}

We conclude the analysis of PolyNSD by formalising the spectral effect and
stability properties of the \emph{high-pass correction} and \emph{gated
residual} used in the main text, and by explaining why they are structurally
useful in the design of our PolyNSD layers. A single PolyNSD layer applies, before the pointwise nonlinearity $\phi$, the linear transformation:
\begin{align}
h_{\mathrm{hp}}
&:= x - \tfrac{1}{\lambda_{\max}}Lx
\\
z
&:= p(\widetilde L)x + \alpha_{\mathrm{hp}}\, h_{\mathrm{hp}}
\\
x^{+}
&:= (I+\tanh\varepsilon)x - \phi(z)
\end{align}
where $\widetilde L$ is defined as \autoref{eq:polysd-rescale} and it is the spectrally rescaled Laplacian, $p$ is a polynomial filter (typically a convex Chebyshev mixture as in \autoref{app:cheb-bound}), $\alpha_{\mathrm{hp}}\in\mathbb R$ is a learnable scalar, $\varepsilon$ is a learnable diagonal parameter (so $I+\tanh\varepsilon$ is a diagonal residual gate), and $\phi$ is a component-wise nonlinearity with Lipschitz constant $\mathrm{Lip}(\phi)\le 1$. The term $p(\widetilde L)x$ plays the role of a learned, spectrally controlled diffusion-like filter, whose response is bounded and well approximated as in \autoref{subsec:cheb-approx}. The additional term $h_{\mathrm{hp}}=x-\tfrac{1}{\lambda_{\max}}Lx$ is a simple linear spectral correction: it is affine in $L$, cheap to compute, and has a closed-form frequency response that can be combined with $p$ to compensate for excessive low-pass bias. The residual gate $(I+\tanh\varepsilon)$ then allows us to tune the deviation from the identity map while keeping a global Lipschitz control. 
\begin{proposition}[\textbf{Spectral Shape and Global Lipschitz Bound}]
\label{prop:hp-spectral-appendix}
Let $L=U\Lambda U^\top$ be symmetric positive semidefinite, with eigenvalues
$\{\lambda_i\}_{i=1}^{nd}$ and orthonormal eigenvectors $\{u_i\}$, and let
$\widetilde L=\tfrac{2}{\lambda_{\max}}L-I$ with
$\sigma(L)\subset[0,\lambda_{\max}]$. Let $p$ be any real polynomial such that
$\|p(\widetilde L)\|_2\le 1$.
Then:
\begin{enumerate}
\item In the eigenbasis of $L$, the linear operator
      $x\mapsto p(\widetilde L)x+\alpha_{\mathrm{hp}}h_{\mathrm{hp}}$
      has per-eigenvalue multiplier:
      \begin{equation}
      \label{eq:hp-multiplier}
      m(\lambda)
      =
      p\!\Big(\tfrac{2\lambda}{\lambda_{\max}}-1\Big)
      +
      \alpha_{\mathrm{hp}}\Big(1-\tfrac{\lambda}{\lambda_{\max}}\Big),
      \qquad \lambda\in\sigma(L)
      \end{equation}
      In particular, if $\alpha_{\mathrm{hp}}>0$ and
      $p(\xi)\ge 0$ for all $\xi\in[-1,1]$, then $m(\lambda)>0$ for all
      $\lambda\in[0,\lambda_{\max})$, so no non-harmonic mode
      ($\lambda>0$ with $\lambda<\lambda_{\max}$) can be annihilated.
\item The full mapping $T:x\mapsto x^{+}$ satisfies the global Lipschitz bound:
      \begin{equation}
      \label{eq:hp-lip-constant}
      \|T(x)-T(y)\|_2
      \;\le\;
      \Bigl[
        \bigl(1+\|\tanh\varepsilon\|_\infty\bigr)
        +
        \mathrm{Lip}(\phi)\,\bigl(1+2|\alpha_{\mathrm{hp}}|\bigr)
      \Bigr]\,
      \|x-y\|_2,
      \end{equation}
      i.e.\ $T$ is Lipschitz with constant at most: $L_T
      \;\le\;
      \bigl(1+\|\tanh\varepsilon\|_\infty\bigr)
      +
      \mathrm{Lip}(\phi)\,\bigl(1+2|\alpha_{\mathrm{hp}}|\bigr)$.
\end{enumerate}
\end{proposition}

\begin{proof}
We treat the two claims separately.

\emph{(1) Spectral Multiplier of the High-pass Skip.}

Since $\widetilde L$ is an affine function of $L$, we have $L\widetilde L=\widetilde L L$ and the two operators are
simultaneously diagonalisable by the same orthogonal basis $U$. Writing $L = U\Lambda U^\top, \Lambda=\mathrm{diag}(\lambda_1,\dots,\lambda_{nd})$, we obtain:
\begin{equation}
    \widetilde L
=
\tfrac{2}{\lambda_{\max}}U\Lambda U^\top - I
=
U\Big(\tfrac{2}{\lambda_{\max}}\Lambda-I\Big)U^\top
=
U\widetilde\Lambda U^\top
\end{equation}
with $\widetilde\Lambda=\mathrm{diag}(\tilde\lambda_i)$ and $\tilde\lambda_i
=
\tfrac{2}{\lambda_{\max}}\lambda_i-1
\;\in\;[-1,1] \text{for all }i$, 
by $\sigma(L)\subset[0,\lambda_{\max}]$. Let $x\in C^0(\mathcal G;\mathcal F)$ and write $x=U\hat x$ with $\hat x=(\hat x_i)_i$. By polynomial functional calculus, we have: $p(\widetilde L)x = U\,p(\widetilde\Lambda)\hat x = U\big(\mathrm{diag}(p(\tilde\lambda_i))\hat x\big)$ so the $i$-th Fourier coefficient of $p(\widetilde L)x$ is: 
$p(\tilde\lambda_i)\hat x_i =
p(\tfrac{2\lambda_i}{\lambda_{\max}}-1)\,\hat x_i$. Similarly, for the
high-pass term we have:
\begin{equation}
    h_{\mathrm{hp}}
=
\Big(I-\tfrac{1}{\lambda_{\max}}L\Big)x
=
U\Big(I-\tfrac{1}{\lambda_{\max}}\Lambda\Big)\hat x
\end{equation}
so its $i$-th Fourier coefficient is then $\Big(1-\tfrac{\lambda_i}{\lambda_{\max}}\Big)\hat x_i$. Putting these together, the pre-nonlinearity linear combination $ x\mapsto p(\widetilde L)x + \alpha_{\mathrm{hp}} h_{\mathrm{hp}} $ acts diagonally in the eigenbasis as:
\begin{equation}
    \hat x_i \;\longmapsto\;
\Big[
  p\!\Big(\tfrac{2\lambda_i}{\lambda_{\max}}-1\Big)
  +
  \alpha_{\mathrm{hp}}\Big(1-\tfrac{\lambda_i}{\lambda_{\max}}\Big)
\Big]\hat x_i
\end{equation}
which is exactly the multiplier $m(\lambda_i)$ in
\autoref{eq:hp-multiplier}. For the positivity statement, assume
$\alpha_{\mathrm{hp}}>0$ and $p(\xi)\ge 0$ on $[-1,1]$. For any eigenvalue
$\lambda\in[0,\lambda_{\max})$ we have $1-\lambda/\lambda_{\max}>0$, hence, $\alpha_{\mathrm{hp}}\Big(1-\tfrac{\lambda}{\lambda_{\max}}\Big) > 0$ and the additional term $p(\tfrac{2\lambda}{\lambda_{\max}}-1)\ge 0$ by hypothesis. Therefore $m(\lambda)>0$ for all $\lambda\in[0,\lambda_{\max})$, so no non-harmonic
eigenmode can be annihilated. The only potential zero of $m$ on the spectrum
can appear at $\lambda=\lambda_{\max}$, e.g.\ if $p(1)=0$ and
$\alpha_{\mathrm{hp}}=0$.
\emph{(2) Global Lipschitz Bound for the Full Update.}

Let $x,y$ be two arbitrary inputs and set: $\Delta x := x-y, \Delta z := z(x)-z(y)$. Then, $T(x)-T(y) = (I+\tanh\varepsilon)\Delta x
-
\bigl(\phi\big(z(x)\big)-\phi\big(z(y)\big)\bigr)$
So by the triangle inequality, we have:
\begin{equation}
\label{eq:hp-lip-split}
\|T(x)-T(y)\|_2
\le
\|(I+\tanh\varepsilon)\Delta x\|_2
+
\|\phi(z(x))-\phi(z(y))\|_2
\end{equation}
We first bound the residual gate term. Since $I+\tanh\varepsilon$ is diagonal, its spectral norm equals the maximum absolute value of its diagonal entries:
\begin{equation}
    \|I+\tanh\varepsilon\|_2
=
\max_j |1+\tanh\varepsilon_j|
\le
1+\max_j |\tanh\varepsilon_j|
=
1+\|\tanh\varepsilon\|_\infty
\end{equation}
Hence, 
\begin{equation}
\label{eq:hp-residual-bound}
\|(I+\tanh\varepsilon)\Delta x\|_2
\le
\bigl(1+\|\tanh\varepsilon\|_\infty\bigr)\,\|\Delta x\|_2
\end{equation}

Next we control the nonlinear term. By Lipschitzness of $\phi$ with constant
$\mathrm{Lip}(\phi)$, we have:
\begin{equation}
\label{eq:hp-phi-lip}
\|\phi(z(x))-\phi(z(y))\|_2
\le
\mathrm{Lip}(\phi)\,\|\Delta z\|_2
\end{equation}
The increment $\Delta z$ can then be written as $\Delta z = p(\widetilde L)\Delta x + \alpha_{\mathrm{hp}}\Big(I-\tfrac{1}{\lambda_{\max}}L\Big)\Delta x$, so by the triangle inequality and submultiplicativity of the operator norm:
\begin{equation}
\label{eq:hp-delta-z}
\|\Delta z\|_2
\le
\|p(\widetilde L)\|_2\,\|\Delta x\|_2
+
|\alpha_{\mathrm{hp}}|\,\Big\|I-\tfrac{1}{\lambda_{\max}}L\Big\|_2\,\|\Delta x\|_2
\end{equation}
By hypothesis $\|p(\widetilde L)\|_2\le 1$. Moreover, since $L$ is symmetric
PSD with eigenvalues $\lambda_i\in[0,\lambda_{\max}]$, the eigenvalues of
$I-\tfrac{1}{\lambda_{\max}}L$ are $\mu_i=1-\lambda_i/\lambda_{\max}\in[0,1]$,
so in fact $\big\|I-\tfrac{1}{\lambda_{\max}}L\big\|_2\le 1$. Using instead
the looser bound $\big\|I-\tfrac{1}{\lambda_{\max}}L\big\|_2\le 2$ (which
remains valid even if $\lambda_{\max}$ is only an upper bound on the spectrum), we obtain from \autoref{eq:hp-delta-z}:
\begin{equation}
\label{eq:hp-delta-z-final}
\|\Delta z\|_2
\le
\bigl(1+2|\alpha_{\mathrm{hp}}|\bigr)\,\\|\Delta x\|_2
\end{equation}
\end{proof}
\begin{proof}
Combining \autoref{eq:hp-phi-lip} and \autoref{eq:hp-delta-z-final} yields:
\begin{equation}
\|\phi(z(x))-\phi(z(y))\|_2
\le
\mathrm{Lip}(\phi)\,\bigl(1+2|\alpha_{\mathrm{hp}}|\bigr)\,\|\Delta x\|_2
\end{equation}
Substituting together with \autoref{eq:hp-residual-bound} into \autoref{eq:hp-lip-split}, we obtain then:
\begin{equation}
\|T(x)-T(y)\|_2
\le
\Bigl[
  \bigl(1+\|\tanh\varepsilon\|_\infty\bigr)
  +
  \mathrm{Lip}(\phi)\,\bigl(1+2|\alpha_{\mathrm{hp}}|\bigr)
\Bigr]
\|\Delta x\|_2
\end{equation}
which is exactly the claimed Lipschitz bound:
\autoref{eq:hp-lip-constant}.
\end{proof}

\noindent
\paragraph{High-pass Skip and Residual Gating.} \autoref{eq:hp-multiplier} shows that the frequency response $m(\lambda)$ of the linear core is a \emph{sum of two simple terms}: a learned polynomial
$p(\tfrac{2\lambda}{\lambda_{\max}}-1)$ and a linear profile
$\alpha_{\mathrm{hp}}(1-\lambda/\lambda_{\max})$. If $p$ is chosen to
approximate a diffusion semigroup $e^{-t\lambda}$ (or any analytic low-pass
filter as in \autoref{subsec:cheb-approx}), then $p$ is typically strongly
decreasing in $\lambda$, i.e.\ $p(0)\approx 1$ and $p(\lambda_{\max})\approx 0$.
In isolation, this tends to \emph{aggressively suppress} high-frequency modes, and repeated application of such a filter can readily lead to
oversmoothing. The correction $ \lambda\;\longmapsto\;\alpha_{\mathrm{hp}}\Big(1-\tfrac{\lambda}{\lambda_{\max}}\Big)$ is a simple knob that deforms this pure diffusion profile. For instance, indeed:
\begin{itemize}
    \item If $\alpha_{\mathrm{hp}}<0$ with small magnitude, then $m(\lambda)=p(\cdot)+\alpha_{\mathrm{hp}}(1-\lambda/\lambda_{\max})$ \emph{subtracts more} from low frequencies (where $1-\lambda/\lambda_{\max}\approx 1$) than from high ones (where $1-\lambda/\lambda_{\max}\approx 0$), relatively flattening the low-pass behaviour and partially reweighting intermediate and high modes. 
    \item Conversely, if $\alpha_{\mathrm{hp}}>0$, the same term can be used to ensure $m(\lambda)$ remains strictly positive on $[0,\lambda_{\max})$ when $p\ge 0$, as captured in item (1) of \autoref{prop:hp-spectral-appendix}, so that no non-harmonic mode is accidentally annihilated by the learned polynomial. 
\end{itemize}
In both cases, the effect is controlled by a \emph{single scalar} and a linear function of the spectrum, making the impact on $m(\lambda)$ easy to analyse and to regularise. Second, the Lipschitz estimate in \autoref{eq:hp-lip-constant} quantifies the stability cost of inserting both the high-pass correction and the gated residual. The term $1+\|\tanh\varepsilon\|_\infty$ bounds how far the residual branch deviates from the identity: when $\varepsilon$ is small, we have $\tanh\varepsilon\approx 0$ and the residual behaves nearly isometrically. The contribution $\mathrm{Lip}(\phi)\,(1+2|\alpha_{\mathrm{hp}}|)$ comes from the composition of a 1-Lipschitz nonlinearity with the pre-activation linear map $z(x)$: here $1$ corresponds to the nonexpansive polynomial $p(\widetilde L)$ (by design of the Chebyshev-PolyNSD core), while $2|\alpha_{\mathrm{hp}}|$ bounds the additional amplification introduced by the high-pass term $I-\tfrac{1}{\lambda_{\max}}L$. This makes transparent the trade-off: larger $|\alpha_{\mathrm{hp}}|$ allows more aggressive reshaping of the spectrum but increases the worst-case Lipschitz constant linearly and explicitly, which can be controlled during optimisation (e.g.\ via weight regularisation or explicit constraints on $\alpha_{\mathrm{hp}}$ and
$\varepsilon$). Together with the approximation results of
\autoref{subsec:cheb-approx} and the nonexpansive properties of Chebyshev
mixtures in \autoref{app:cheb-bound},
\autoref{prop:hp-spectral-appendix} therefore explains why PolyNSD augments
its diffusion-like polynomial core with exactly (i) a linear spectral
correction of the form $I-\tfrac{1}{\lambda_{\max}}L$ and (ii) a diagonal
residual gate: they provide a low-complexity, analytically tractable way to
combat oversmoothing and to preserve informative high-frequency content,
while keeping the overall layer within a globally controlled Lipschitz
envelope.

\subsection{PolyNSD vs Neural Sheaf Diffusion: Operator Class}
\label{app:polynsd-vs-nsd}

In this subsection we make precise the claim that Chebyshev-PolyNSD with
degree \(K{=}1\) recovers the same diffusion operator class as Neural Sheaf
Diffusion (NSD) \cite{bodnar2022neural}, while for \(K{>}1\) it strictly
generalises it to higher-order polynomials in the sheaf Laplacian. In order to do that, we focus on the \emph{diffusion core}, i.e. the linear operator that acts on sheaf
0-cochains between pointwise nonlinearities and feature-mixing MLPs.

\paragraph{Canonical NSD Diffusion Core.}
In Neural Sheaf Diffusion, once a sheaf Laplacian \(L\) has been assembled
from learned restriction maps, the diffusion step acting on a 0-cochain
\(x\in C^0(\mathcal G;\mathcal F)\) can be written abstractly as:
\begin{equation}
\label{eq:nsd-core}
x^{+}
=
A x - B L x,
\end{equation}
where \(A\) and \(B\) are feature-wise scaling operators, encoding the step size and residual weighting. For the purposes of spectral
analysis, the key point is that, \emph{for fixed \(L\)}, NSD’s diffusion core
always lies in the two-dimensional linear span \(\{I,L\}\): it is a first-order
polynomial \(aI + bL\), with \(a,b\) determined by the learnable scalings in
\autoref{eq:nsd-core} (possibly different per feature channel). We now show that Chebyshev-PolyNSD with \(K{=}1\) recovers exactly this operator class, and that allowing \(K{>}1\) extends it to higher-order polynomials in \(L\) with explicit \(K\)-hop locality. 

\subsubsection{Chebyshev PolyNSD with \texorpdfstring{$K{=}1$}{K=1} Recovers NSD}
\label{app:cheb1-nsd-equivalence}

Let's start from the layer update used in PolyNSD (\autoref{eq:polynsd-core-update-prev}, where \(p_\theta(\widetilde L)\) is a degree-\(K\) polynomial in \(\widetilde L\)
(e.g. a convex Chebyshev mixture), \(\alpha_{\mathrm{hp}}\in\mathbb R\) is a
scalar high-pass weight, \(\varepsilon\) is a diagonal residual gate, and
\(\phi\) is a 1-Lipschitz nonlinearity. For the operator-class comparison we
temporarily set \(\phi\) to the identity and focus on the linear map
\(x\mapsto x^{+}\).

\paragraph{Degree-\(1\) Chebyshev Parameterisation.}
For \(K{=}1\), the Chebyshev recurrence yields:
\begin{equation}
T_0(\widetilde L)=I,
\qquad
T_1(\widetilde L)=\widetilde L
=
\tfrac{2}{\lambda_{\max}}L - I
\end{equation}
Let \(\theta=(\theta_0,\theta_1)\in\Delta^1\) be the (convex) coefficients of
the Chebyshev mixture, e.g. \(\theta=\mathrm{softmax}(\eta)\) for some logits
\(\eta\in\mathbb R^2\). The degree-\(1\) filter is then $p_\theta(\widetilde L)
=
\theta_0 T_0(\widetilde L) + \theta_1 T_1(\widetilde L)
=
\theta_0 I + \theta_1\Big(\tfrac{2}{\lambda_{\max}}L - I\Big)$. Acting on an input \(x\), this gives $x_{\mathrm{cheb}}:=p_\theta(\widetilde L)x=(\theta_0-\theta_1)\,x+\theta_1\tfrac{2}{\lambda_{\max}}Lx$. The high-pass correction is $h_{\mathrm{hp}}=x - \tfrac{1}{\lambda_{\max}}Lx$ and the pre-nonlinearity activation becomes: 
\begin{equation}
\label{eq:z-k1}
z
=
x_{\mathrm{cheb}}
+
\alpha_{\mathrm{hp}} h_{\mathrm{hp}}
=
(\theta_0-\theta_1 + \alpha_{\mathrm{hp}})\,x
+
\Big(\theta_1\tfrac{2}{\lambda_{\max}}
     - \alpha_{\mathrm{hp}}\tfrac{1}{\lambda_{\max}}\Big)Lx
\end{equation}
Finally, setting \(\phi\) to the identity\footnote{In practice \(\phi\) is a
1-Lipschitz nonlinearity; here we isolate the linear diffusion core. The
nonlinearity can be placed before or after this core in both NSD and PolyNSD, and does not change the operator class of the linear part.}, we obtain:
\begin{equation}
\label{eq:polynsd-k1-update}
x^{+}
=
(I+\tanh\varepsilon)x - z
\end{equation}

\begin{proposition}[\textbf{Chebyshev-PolyNSD with $K{=}1$ Induces a First-Order Polynomial in $L$}]
\label{prop:cheb-k1-linear}
Under the setup above with \(K{=}1\) and \(\phi=\mathrm{id}\), the PolyNSD
update in \autoref{eq:polynsd-k1-update} can be written as:
\begin{equation}
\label{eq:cheb-k1-ab-form}
x^{+}
=
a\,x + b\,Lx
\end{equation}
where the coefficients \(a,b\) (per feature channel) are given by:
\begin{equation}
\label{eq:cheb-k1-a-b}
a
=
(1+\tanh\varepsilon) - \big(\theta_0-\theta_1 + \alpha_{\mathrm{hp}}\big),
\qquad
b
=
-\Big(\theta_1\tfrac{2}{\lambda_{\max}}
     - \alpha_{\mathrm{hp}}\tfrac{1}{\lambda_{\max}}\Big)
\end{equation}
In particular, for any fixed \(L\) and \(\lambda_{\max}\), the degree-\(1\)
Chebyshev-PolyNSD core spans the same operator class \(\{aI + bL\}\) as NSD’s diffusion core in \autoref{eq:nsd-core}. 
\end{proposition}

\begin{proof}
Substituting the expression for \(z\) in \autoref{eq:z-k1} into
\autoref{eq:polynsd-k1-update} yields:
\begin{align}
x^{+}
&=
(I+\tanh\varepsilon)x
-
\Big[
  (\theta_0-\theta_1 + \alpha_{\mathrm{hp}})\,x
  +
  \Big(\theta_1\tfrac{2}{\lambda_{\max}}
       - \alpha_{\mathrm{hp}}\tfrac{1}{\lambda_{\max}}\Big)Lx
\Big]
\\
&=
\Big[
  (1+\tanh\varepsilon)
  - (\theta_0-\theta_1 + \alpha_{\mathrm{hp}})
\Big]x
-
\Big[
  \theta_1\tfrac{2}{\lambda_{\max}}
  - \alpha_{\mathrm{hp}}\tfrac{1}{\lambda_{\max}}
\Big]Lx
\end{align}
Defining \(a,b\) as in \autoref{eq:cheb-k1-a-b} gives the claimed form in 
\autoref{eq:cheb-k1-ab-form}. Since \(a,b\in\mathbb R\) (or feature-wise scalars
when \(\varepsilon\) is diagonal), this coincides with the NSD form
\autoref{eq:nsd-core}, i.e.\ a first-order polynomial \(aI + bL\).
\end{proof}

\paragraph{Finding PolyNSD Parameters from NSD Diffusion Step.}
Given a target NSD diffusion step: $ x^{+}_{\mathrm{NSD}} = A x - B L x$, we can always find PolyNSD parameters \((\theta_0,\theta_1,\alpha_{\mathrm{hp}},\varepsilon)\) realising the same
operator, up to per-channel scaling, provided \(\lambda_{\max}>0\). For
example, in the degree-normalised setting where \(\lambda_{\max}=2\) and we
ignore the high-pass term (\(\alpha_{\mathrm{hp}}=0\)), choosing $\theta_1 = \frac{B}{2}, \theta_0 = 1-\theta_1$ gives \(b=-B\) in \autoref{eq:cheb-k1-a-b}, and then we can enforce \(a=A\) by tuning \(\varepsilon\), i.e. setting the residual gate so that $1+\tanh\varepsilon = A + (\theta_0-\theta_1)$. Thus, after fixing \(\lambda_{\max}\), the map from PolyNSD parameters to the pair \((a,b)\) is surjective onto the NSD operator class \(\{aI+bL\}\).

\subsubsection{Strict Generalisation for \texorpdfstring{$K{>}1$}{K>1}}
\label{app:polynsd-strict}

We now show that allowing higher polynomial degrees \(K{>}1\) extends the
diffusion operator class from \(\{aI+bL\}\) to all polynomials in \(L\) of
degree at most \(K\). Under mild spectral assumptions on \(L\), this
generalises NSD strictly, meaning that there exist PolyNSD operators that cannot be represented by any single NSD diffusion step.

\paragraph{PolyNSD Operator Class for General $K$.}
For a degree-\(K\) Chebyshev-PolyNSD core, the filter \(p_\theta\) can be
written as $p_\theta(\widetilde L)
=
\sum_{k=0}^{K} \theta_k T_k(\widetilde L),
\theta\in\Delta^K$. Each \(T_k(\widetilde L)\) is itself a degree-\(k\) polynomial in \(L\) (since \(\widetilde L\) is affine in \(L\)), so \(p_\theta(\widetilde L)\) is a polynomial in \(L\) of degree at most \(K\). The same calculation as in the \(K{=}1\) case shows that the full linear core: $x\mapsto (I+\tanh\varepsilon)x - \big(p_\theta(\widetilde L)x + \alpha_{\mathrm{hp}} h_{\mathrm{hp}}\big)$ remains a polynomial in \(L\) of degree at most \(K\). 

\begin{proposition}[\textbf{PolyNSD Core as a Polynomial in $L$}]
\label{prop:polynsd-poly-class}
For fixed \(L\) and \(\lambda_{\max}>0\), any Chebyshev-PolyNSD diffusion core
with degree \(K\) (and \(\phi=\mathrm{id}\)) can be written as $x^{+} = q(L)\,x$, where \(q\) is a polynomial in one variable of degree at most \(K\), whose coefficients depend linearly on \(\theta\), \(\alpha_{\mathrm{hp}}\) and
\(I+\tanh\varepsilon\). In particular, taking \(K{=}1\) recovers the class
\(\{aI+bL\}\) of NSD, while \(K{>}1\) yields higher-order polynomials.
\end{proposition}

\begin{proof}
By definition, \(T_k(\widetilde L)\) is a degree-\(k\) polynomial in
\(\widetilde L\). Since \(\widetilde L\) is an
affine function of \(L\), each \(T_k(\widetilde L)\) is a degree-\(k\)
polynomial in \(L\). Therefore, $p_\theta(\widetilde L)
=
\sum_{k=0}^K \theta_k T_k(\widetilde L)
$ is a polynomial in \(L\) of degree at most \(K\). The high-pass term
\(h_{\mathrm{hp}}=x-\lambda_{\max}^{-1}Lx\) is a first-order polynomial in
\(L\), and the residual term \((I+\tanh\varepsilon)x\) is a degree-0
polynomial (the identity scaled). Subtracting a linear combination of these
three contributions preserves polynomial structure and does not increase the
degree. Hence the overall linear map \(x\mapsto x^{+}\) is of the form
\(x^{+}=q(L)x\) with \(\deg q\le K\).
\end{proof}

\paragraph{Strictness Under Mild Spectral Assumptions.}
To argue that PolyNSD with \(K{>}1\) is not merely a reparameterisation of
NSD, we show that, under mild conditions on the spectrum of \(L\), there exist
polynomials \(q\) of degree at least \(2\) that cannot be represented as
\(aI+bL\).

\begin{proposition}[\textbf{Strict Generalisation for $K>1$}]
\label{prop:polynsd-strict}
Assume \(L\) has at least three distinct eigenvalues
\(\lambda_{i_1},\lambda_{i_2},\lambda_{i_3}\in\sigma(L)\). Let \(q\) be a
real polynomial of degree \(d\ge 2\), and consider the operator \(q(L)\). Then, there is no pair \(a,b\in\mathbb R\) such that $q(L) = aI + bL$, unless \(q\) coincides with an affine polynomial on the set \(\{\lambda_{i_1},\lambda_{i_2},\lambda_{i_3}\}\). In particular, for generic
choices of \(\theta\) that induce a polynomial of degree at least \(2\),
the resulting PolyNSD operator cannot be realised by any NSD core of the form \autoref{eq:nsd-core}.
\end{proposition}

\begin{proof}
Write the spectral decomposition \(L=U\Lambda U^\top\) with
\(\Lambda=\mathrm{diag}(\lambda_1,\dots,\lambda_{nd})\). Then, $q(L) = U\,q(\Lambda)\,U^\top, aI + bL = U(aI + b\Lambda)U^\top$. If \(q(L)=aI+bL\), then conjugating by \(U^\top\) gives \(q(\Lambda)=aI + b\Lambda\), i.e., $q(\lambda_i) = a + b\lambda_i, \text{for all }i$. In particular, for the three distinct eigenvalues \(\lambda_{i_1},\lambda_{i_2},\lambda_{i_3}\), we must have $q(\lambda_{i_j}) = a + b\lambda_{i_j}, j=1,2,3$. 
But an affine function \(\lambda\mapsto a + b\lambda\) is determined uniquely
by its values on any two distinct points. If a degree-\(d\) polynomial
\(q\) with \(d\ge 2\) coincides with an affine function on three distinct
points, then its second finite difference on these points must vanish, forcing the quadratic and higher-order terms of \(q\) to cancel there. For a
generic choice of \(q\) and eigenvalues, this does not occur. Concretely, if
we choose coefficients so that the highest-degree coefficient is nonzero and
the polynomial is not exactly affine on the three eigenvalues, then no such
\(a,b\) can exist. This shows that, whenever \(q\) is genuinely of degree at
least 2 on three distinct eigenvalues, the operator \(q(L)\) cannot be
compressed into the NSD form \(aI+bL\).
\end{proof}

Thus, as soon as \(L\) exhibits a reasonably rich spectrum (three or more
distinct eigenvalues) and the learned Chebyshev coefficients
\(\theta\in\Delta^K\) activate higher-order terms (\(\deg q\ge 2\)), PolyNSD
realises operators outside the NSD class.

\paragraph{Depth vs Polynomial Degree and Computational Cost.}
Let's now also compare PolyNSD and NSD from a \emph{depth} and
\emph{complexity} perspective. For fixed \(L\) and linear activations, stacking
\(T\) NSD diffusion cores of the form \(a_t I + b_t L\) yields $x^{(T)}
=
\Big(\prod_{t=1}^T (a_t I + b_t L)\Big) x^{(0)}$, which is a polynomial in \(L\) of degree at most \(T\). Conversely, any polynomial in \(L\) with real roots can be factorised (over \(\mathbb C\)) as a product of linear factors, which can in principle be mapped to NSD-style layers. In practice, however, this requires \(T\) layers, each with its own sheaf prediction and Laplacian assembly. By contrast, a single degree-\(K\) PolyNSD layer computes \(p_\theta(\widetilde L)x\) via a three-term recurrence using \(K\) sparse matrix–vector products with \(\widetilde L\), reusing the same sheaf Laplacian within the block. The resulting complexity is $\mathcal O\big(K\,\mathrm{nnz}(L)\,C\big)$, for feature dimension \(C\), which matches the asymptotic cost of stacking \(K\) NSD layers with \emph{fixed} sheaf structure, \emph{but avoids repeated edge-wise sheaf prediction and Laplacian construction}. Combined with the approximation guarantees in \autoref{subsec:cheb-approx} and the stability
results in \autoref{app:cheb-bound} and \autoref{app:hp-gate}, this shows
that PolyNSD provides a spectrally controlled, computationally efficient
generalisation of NSD: degree \(K{=}1\) recovers NSD’s diffusion core, while
\(K{>}1\) trades depth for polynomial degree, enlarging the operator class
from first-order to higher-order polynomials in the sheaf Laplacian with
explicit \(K\)-hop locality.

%% file: chapters/C_Appendix.tex
\newpage
\section{Extensive and Additional Experiments}
\label{app:exp-setup}

This appendix collects the full experimental details for the results reported
in \autoref{sec:experiments}. We first describe the datasets used in our
evaluation, then provide a consolidated view of model variants,
hyperparameters, ablations, interpretability diagnostics, and training protocols for PolyNSD and all baselines. We also provide extended quantitative results that complement the main trends reported in \autoref{sec:experiments}.

\subsection{Hyper-Parameters}
\label{app:hyperparams}

This section collects the full hyper-parameter search space used in the
experiments in \autoref{sec:experiments}. The final reported test-set accuracy for each run is chosen from the checkpoint that achieved the highest validation accuracy. All experiments were run on a single \emph{NVIDIA A100–SXM4–80GB} GPU. For what concerns \emph{real datasets}, 
\autoref{tab:hyper-table} lists the hyperparameters and their searchable values/ranges used in the experiments, reflecting the same choices of \cite{bodnar2022neural}. For \emph{synthetic datasets}, we evaluate three families of settings mirroring the experiments in \cite{caralt2024joint}: \emph{Heterophily}, \emph{Feature Noise}, and \emph{Amount of Data}, each instantiated in two columns (\textsc{Diff} on the right and \textsc{RiSNN} on the left), which differ only in feature dimensionality and a few graph knobs. \autoref{tab:synthetic-hparams-common} summarises the common training setup and \autoref{tab:synthetic-hparams-panels} the panel–specific grids.

\subsection{Hardware and Software Setup}
\label{app:hardware_setup}
All experiments were conducted on an AWS EC2 g6.xlarge instance running Ubuntu 24.04 LTS, equipped with an NVIDIA L4 GPU (CUDA 12.8, cuDNN 9.1). The codebase is implemented in Python 3.13.2 using PyTorch 2.8 and PyTorch Geometric 2.7. Hyperparameter searches were performed via Weights \& Biases sweeps. These choices were kept fixed across models to ensure that empirical differences reflect modelling choices rather than software-stack variation.

\subsection{Datasets}
\label{app:datasets}

\paragraph{Real-World Node-Classification Benchmarks.}
For real-world experiments we use the standard node-classification benchmarks
commonly adopted in the sheaf-learning and heterophily literature:
\textsc{Texas}, \textsc{Wisconsin}, \textsc{Film}, \textsc{Squirrel},
\textsc{Chameleon}, \textsc{Cornell}, \textsc{Citeseer}, \textsc{PubMed}, and
\textsc{Cora}. Below we briefly recall their construction.

\emph{WebKB (Cornell, Texas, Wisconsin).}
Nodes correspond to webpages from the computer science departments of the
respective universities, and edges are hyperlinks between pages.
Node features are bag-of-words representations of page content.  
The node labels distinguish different page types (e.g., student, project,
course, staff, faculty).
\begin{itemize}
    \item \emph{WikipediaNetwork (Chameleon, Squirrel).}
Nodes are Wikipedia pages and edges denote hyperlinks, while the node features are derived from page content, and labels denote page categories/roles.
    \item \emph{Film.} In the \textsc{Film} graph, nodes correspond to actors and edges connect actors co-occurring on the same Wikipedia page or film. Node features are keyword-based descriptors, while the task is multi-class node classification.
    \item \emph{Citation graphs (Cora, Citeseer, PubMed).}
Nodes represent scientific articles, edges encode citation relations.
Node features are bag-of-words or standard pre-computed attributes, and the task is to classify each article into a subject category.
\end{itemize}

\emph{Filtered Wikipedia Heterophily Benchmarks and New Benchmark.}
Platonov et al.~\cite{platonov2023critical} proposed filtered versions of the
Wikipedia \textsc{Chameleon} and \textsc{Squirrel} datasets, after identifying some duplicate nodes in the original graphs, could induce train--test leakage and alter model rankings. The same work also introduces
a new benchmark suite of larger and more diverse heterophilic graphs. Following
their evaluation protocol, the datasets can be grouped according to the reported
metric as follows:

\begin{enumerate}
    \item \emph{Accuracy datasets: \textsc{Roman-Empire} and
    \textsc{Amazon-Ratings}.}
    \textsc{Roman-Empire} is constructed from the English Wikipedia article on
    the Roman Empire. Nodes correspond to words, while edges connect consecutive
    words or words linked by syntactic dependencies. Node features are word
    embeddings, and the task is to predict the syntactic role of each word.
    The graph is sparse, chain-like, and contains long-range structural
    dependencies. \textsc{Amazon-Ratings} is derived from an Amazon product
    co-purchasing network. Nodes are products, edges connect frequently
    co-purchased products, node features are obtained from product descriptions,
    and labels correspond to discretised product ratings.

    \item \emph{ROC-AUC datasets: \textsc{Minesweeper}, \textsc{Tolokers}, and
    \textsc{Questions}.}
    \textsc{Minesweeper} is a synthetic grid-based graph inspired by the
    Minesweeper game. Nodes are grid cells, edges connect neighbouring cells,
    and the task is to predict whether a cell contains a mine from partially
    masked local-count features. \textsc{Tolokers} is built from a
    crowdsourcing platform: nodes represent workers, edges connect workers who
    participated in the same task, node features describe worker profiles and
    task-performance statistics, and the goal is to predict whether a worker was
    banned in one of the projects. \textsc{Questions} is derived from a
    question-answering platform. Nodes are users, edges connect users when one
    answered another user's question, node features are based on user
    descriptions, and the task is to predict whether a user remains active on
    the platform.
\end{enumerate}

\autoref{tab:dataset-stats} and \autoref{tab:new_heterophilous_dataset_stats} reports the basic statistics used throughout our
experiments: homophily level $h$ measured as the fraction of edges whose
endpoints share the same label, number of nodes, number of edges, and number of
classes. For the standard sheaf-learning benchmarks, we adopt the per-class
$48\%/32\%/20\%$ train/validation/test protocol and average results over the same
$10$ fixed splits as in prior work. For the additional Platonov et al.
heterophily benchmarks, we follow their standard evaluation protocol unless
otherwise specified.

\input{tables/dataset_stats}
\input{tables/new_datasets_stats}

\subsubsection{Synthetic Benchmarks}
\label{app:synthetic-datasets}

For the controlled stress tests, we adopt the synthetic benchmark of \cite{caralt2024joint}, which explicitly decouples (i) feature complexity, (ii) graph connectivity, and (iii) label structure. Each graph instance is specified by: $N$ (number of nodes), $K$ (even base degree of a regular ring lattice before rewiring), $n_c$ (number of balanced classes), $het \in [0,1]$ (heterophily coefficient controlling inter-/intra-class edges) and $\sigma \ge 0$ (feature noise level). We briefly recall the feature and graph generation steps below.

\emph{Feature Generation: non-linear Class Manifolds with Shared Mean.} We construct class-specific, non-linearly separable features while enforcing a
common mean across classes. This makes naive averaging ineffective while
preserving class structure under non-trivial aggregation. Let $n_{\text{data}}$ be the feature dimension, $n_h$ an auxiliary latent dimension, and $n_c$ the number of classes. The process is detailed as follows:

\begin{enumerate}
  \item Draw class prototypes
  $v_k \sim \mathcal{U}([0,1]^{\,n_{\text{data}}})$ for $k=1,\dots,n_c$.
  \item Define a fixed non-linear map
  $f : \mathbb{R}^{n_h} \to \mathbb{R}^{n_{\text{data}}}$ as the
  sine–cosine embedding of the $(n_h-1)$-sphere.
  We sample angles
  $\theta_0,\dots,\theta_{n_h-3} \sim \mathcal{U}([0,\pi])$ and
  $\theta_{n_h-2} \sim \mathcal{U}([0,2\pi))$,
  construct $s \in \mathbb{S}^{n_h-1}$ via standard hyperspherical
  coordinates, and set $z = f(\theta)$ after truncating/tiling to length
  $n_{\text{data}}$.
  \item For a node of class $k$, we sample $\theta$ as above and set the raw
  feature $x_{\text{raw}} \;=\; v_k \odot f(\theta)$, where $\odot$ denotes element-wise multiplication, yielding an ellipsoidal shell per class.
  \item We centre features to enforce a shared expected value across classes $\mu \;=\; \frac{1}{N}\sum_{i=1}^N x_{\text{raw},i}, \tilde{x}_i \;=\; x_{\text{raw},i} - \mu$.
  \item Finally, we add i.i.d.\ Gaussian noise
  $\epsilon_i \sim \mathcal{N}(0,\sigma^2 I)$ and optionally rescale $x_i \;=\; \tilde{x}_i + \epsilon_i$. 
\end{enumerate}

\emph{Graph Generation: Ring–rewire with Class-aware Mixing.} Connectivity is generated independently of features and controlled by the
heterophily index $het$. The graph generation process consists in:

\begin{enumerate}
  \item Assign classes almost uniformly so that $|C_k| \approx N/n_c$.
  \item Build a regular ring lattice of degree $K$, being $E \;=\; \bigl\{(i,j)\;:\; 0 < \min(|i{-}j|,\, N{-}|i{-}j|)\le K/2 \bigr\}$.
  \item Define the class–transition matrix $R^{c} \;=\; (1- het)\, I_{n_c}
      \;+\; \frac{het}{n_c-1}\,(\mathbf 1\mathbf 1^\top - I_{n_c})$ so that $\Pr(\text{same class})=1-het$ and each different class has
  probability $het/(n_c-1)$. 
  \item For each node $i$ and each of its $K/2$ ``rightmost'' edges, rewire
  with probability $p$, we sample a target class $c' \sim \text{Categorical}\bigl((R^{c})_{c_i,:}\bigr)$, where $c_i$ is the class of node $i$ and choose $j$ uniformly among nodes in class $c'$ that are not $i$ and not already adjacent to $i$, and replace the endpoint with $j$.
  We then de-duplicate multi-edges and remove self-loops. The average degree
  remains $K$, while the expected fraction of inter-class edges equals $het$.
\end{enumerate}

We consider two synthetic regimes (\textsc{RiSNN} and \textsc{Diff}) that
differ only in feature dimensionality and a few graph knobs, following
\cite{caralt2024joint}.

\input{tables/hyperparams_synthetic}

\subsection{Stalk Dimension vs.~Accuracy}
\label{app:stalk-sweep-real}
In this subsection here we detail the results obtained in \autoref{subsec:hetero-real-main} about the stalk dimensionality impact on the accuracy. This experiment aims to isolate the effect of stalk dimension on real-world performance. As previously seen, the stalk dimension $d$ controls the size of the local sheaf fibres and therefore the dimensionality of the features transported along edges. In prior sheaf models, relatively large stalks, with $d\approx 4$, were often used by default and occasionally increased further in search of improved expressivity. To quantify how much PolyNSD depends on stalk size, we sweep $d \in \{2,3,4,5\}$ on a representative subset of datasets spanning different homophily regimes: \textsc{Cora}, \textsc{PubMed} (homophilous citation graphs) and \textsc{Texas}, \textsc{Film} (strongly heterophilous graphs). For each dataset, we run all three PolyNSD variants (Diagonal, Bundle, and General transports). Within each dataset–variant pair, we fix all architectural hyperparameters to the default PolyNSD configuration (depth $L=2$, hidden width $32$, Chebyshev parameterisation, same regularisation and optimiser settings) and vary only the stalk dimension $d$, training on the standard $10$ fixed splits. We then report mean$\pm$std test accuracy across the $10$ runs.
\begin{figure}[h!]
  \centering
  \includegraphics[width=0.7\linewidth]{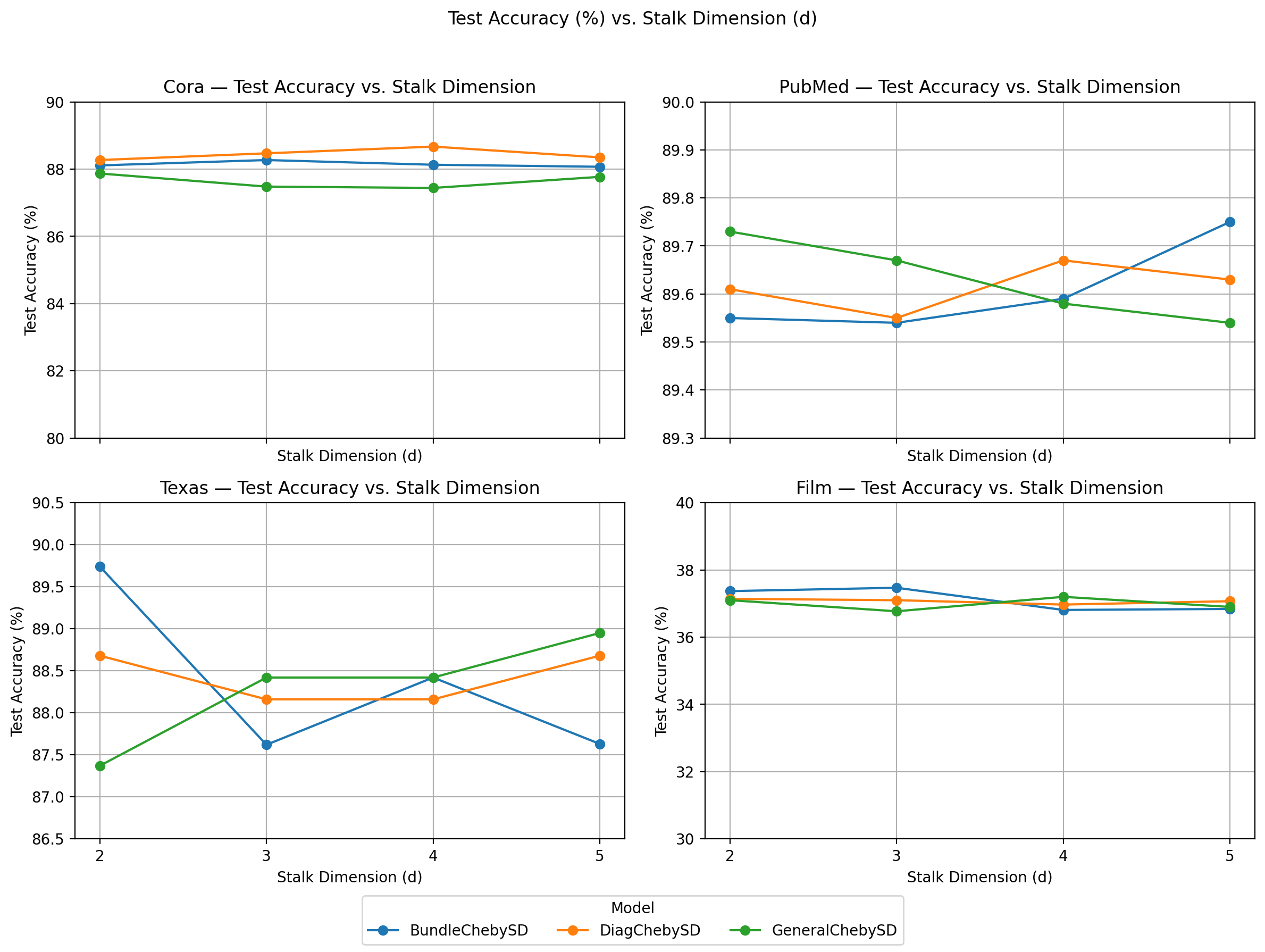}
  \caption{\textit{Test accuracy vs.~stalk dimension $d\in\{2,3,4,5\}$.}
  We sweep $d$ on four real-world datasets over the three versions, keeping fixed the other hyperparameters.}
  \label{fig:stalk-sweep-real}
\end{figure}

\paragraph{Findings: PolyNSD maintains SOTA results across depth.}
The results we get are summarised in \autoref{fig:stalk-sweep-real} and show that on \textsc{PubMed} and \textsc{Film}, accuracies are essentially flat within error bars as $d$ varies, indicating that increasing stalk dimension beyond $d=2$ brings no systematic benefit. On \textsc{Cora} and \textsc{Texas} we observe mild, non-monotonic trends: performance can slightly increase when moving from $d=2$ to $d=3$, but often degrades again for $d\ge 4$. The strongest configurations on all four datasets typically occur at \emph{small stalk dimensions} ($d\in\{2,3\}$). Larger stalks do not lead to clear improvements and can even hurt performance through over-parameterisation and increased optimisation difficulty. This contrasts with earlier sheaf architecture settings, where $d=4$ emerged as a de facto default and larger stalks were sometimes required to approach state-of-the-art performance. This is because polynomial diffusion on the sheaf Laplacian compensates for the reduced stalk dimensionality, where expressive spectral filters and explicit parallel transport allows us to achieve near state-of-the-art accuracy with compact stalks, motivating the choice of low $d$ in the main experiments.

\subsection{Depth Robustness and Oversmoothing}
\label{app:oversmoothing}

\input{tables/oversmoothing_polynsd}

In this subsection, we will better investigate how PolyNSD behaves as depth increases, compared to classical GNNs and other sheaf models. We consider four representative datasets: \textsc{Cora}, \textsc{Citeseer} (homophilous citation graphs), and \textsc{Cornell}, \textsc{Chameleon} (heterophilous benchmarks). For each model we sweep the number of layers $L \in \{2,4,8,16,32\}$, keeping all other hyperparameters fixed. For each model–depth pair, we report test accuracy mean$\pm$std over the $10$ fixed
splits, and we additionally highlight, for each model, the depth at which its
best test accuracy is achieved. The complete results are given in \autoref{tab:oversmoothing}. This makes explicit the qualitative patterns summarised in the main text: GCN and Geom-GCN deteriorate quickly with depth (often approaching random performance on heterophilous graphs), GAT runs into numerical instability at high depth, GCNII and SAN/ANSD provide strong deep baselines, and PolyNSD variants remain competitive and stable up to $L=32$ layers, especially on the heterophilous benchmarks.

\paragraph{Dirichlet-energy Diagnostics on Heterophily.}
To complement the depth sweep in \autoref{tab:oversmoothing}, we add a \emph{representation-level} diagnostic that probes how the intermediate signals evolve as depth increases on heterophilous graphs. In particular, we focus on the two canonical heterophily benchmarks \textsc{Chameleon} and \textsc{Squirrel} and track the (normalized) Dirichlet energy of intermediate representations produced by NSD and PolyNSD. Accuracy as a function of depth (as in \autoref{tab:oversmoothing}) reveals whether a method degrades as layers increase, but it does not reveal \emph{why} a degradation might occur. In deep message passing, poor depth scaling can arise either from oversmoothing (collapse to an overly low-frequency subspace, which is reconducible to have Dirichlet Energy = 0) or from the opposite failure mode, a.k.a. \emph{Energy Amplification}, where intermediate representations become increasingly ``rough'' with respect to the learned propagation operator, yielding unstable depth-wise dynamics.
The goal here is to quantify whether a model exhibits \emph{controlled propagation} as depth increases, i.e., whether the representation variation induced by the learned transport remains bounded and predictable.

\paragraph{Measuring channel-averaged Normalised Dirichlet-Energy}
Let $x_\ell \in \mathbb{R}^{(Nd)\times C}$ be the intermediate sheaf representation after the $\ell$-th propagation (diffusion/filter) step, and let $L$ denote the sparse Laplacian-like operator applied at the probed point in the layer.
We compute the channel-averaged normalised Dirichlet energy as $E_{\mathrm{norm}}(x_\ell) \;=\; \frac{\langle x_\ell,\, Lx_\ell\rangle}{\langle x_\ell,\, x_\ell\rangle}$, and report the depth-wise trajectory $\{E_{\mathrm{norm}}(x_\ell)\}_{\ell=1}^{L}$. Here, larger values indicate higher variation with respect to the learned transport, while smaller values indicate increased alignment with the low-energy modes of $L$. We train NSD and PolyNSD with identical routines and log $E_{\mathrm{norm}}(x_\ell)$ at the best validation checkpoint. We repeat over a controlled grid of settings: $\text{seeds}\in\{0,1,2\}$, stalk dimensions $d\in\{2,3,4\}$, and depths $L\in\{2,3,4\}$, for each transport class (Diagonal, Bundle, General). For each (dataset, transport class, method) group we aggregate all runs (total $n=27$) and plot mean$\pm$std per layer.

\paragraph{Findings: PolyNSD Controls Propagation Energy.}
Across both heterophilous datasets and all transport classes, we observe a consistent qualitative separation in which \emph{NSD shows high and increasing energy with depth}, while \emph{PolyNSD produces substantially lower energy and more stable trajectories.} More specifically, we have that NSD trajectories typically grow monotonically with $\ell$, indicating that intermediate signals become progressively more varying with respect to the learned transport operator. This is consistent with an \emph{energy amplification} regime, where deeper propagation accumulates higher-frequency components in the representation space. On the other hand, PolyNSD yields markedly smaller $E_{\mathrm{norm}}$ values and a smoother depth-wise evolution, suggesting that the polynomial parameterisation learns a propagation rule with \emph{better-conditioned} spectral behaviour. This result, in relation to depth robustness, provides an explanatory lens: compared to NSD, PolyNSD induces \emph{tighter control} over the variation of intermediate representations with depth, preventing the progressive energy growth observed in NSD. 

\begin{figure*}[t]
    \centering
    \begin{subfigure}[t]{0.48\textwidth}
        \centering
        \includegraphics[width=\linewidth]{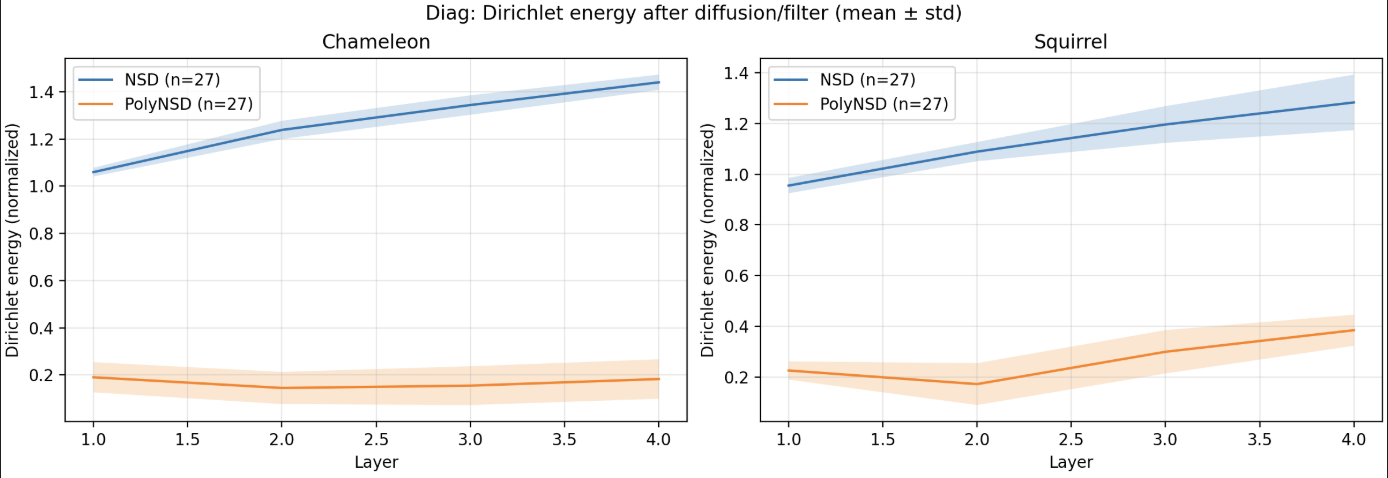}
        \caption{\textbf{Diagonal}}
        \label{fig:dirichlet-diag}
    \end{subfigure}\hfill
    \begin{subfigure}[t]{0.48\textwidth}
        \centering
        \includegraphics[width=\linewidth]{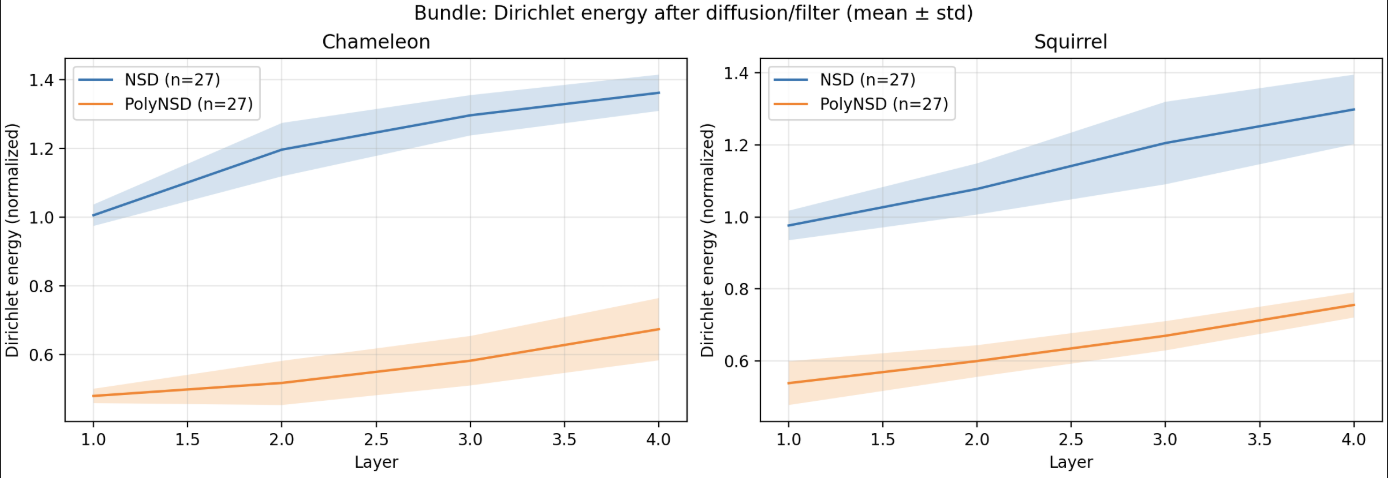}
        \caption{\textbf{Bundle}}
        \label{fig:dirichlet-bundle}
    \end{subfigure}

    \vspace{0.6em}

    \begin{subfigure}[t]{0.48\textwidth}
        \centering
        \includegraphics[width=\linewidth]{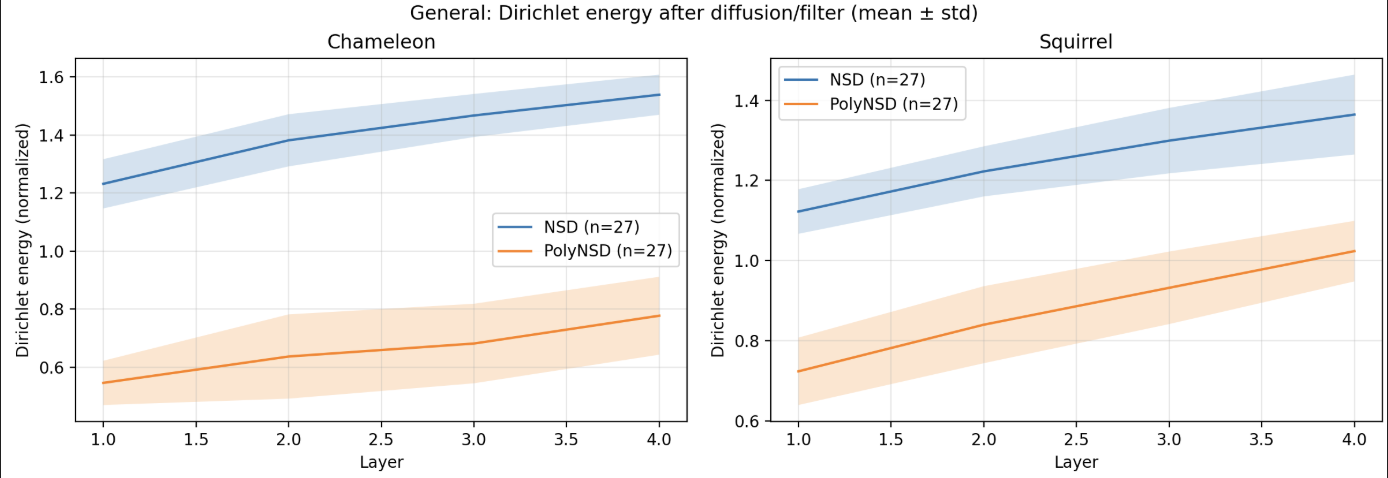}
        \caption{\textbf{General}}
        \label{fig:dirichlet-general}
    \end{subfigure}

    \caption{\emph{Dirichlet-energy diagnostics on heterophily.} Normalized Dirichlet energy trajectories on \textsc{Chameleon} and \textsc{Squirrel} for NSD vs.\ PolyNSD across transport classes. }
    \label{fig:dirichlet-hetero}
\end{figure*}

\subsection{Long-Range Influence Decay Diagnostics}
\label{app:long_range_influence}
The depth-robustness study in \autoref{app:oversmoothing} examines how performance changes as the number of propagation steps increases. However, accuracy alone does not reveal how strongly a model uses information from distant nodes. Since PolyNSD introduces an explicit polynomial spectral component that enables multi-hop mixing within each layer, we complement the depth and energy analyses with a gradient-based \emph{long-range influence} diagnostic. This diagnostic measures how the sensitivity of a target node prediction to source-node features changes as a function of graph distance. More formally, let $x \in \mathbb{R}^{N\times F}$ denote the input node features and let $s_v(x)$ be a scalar score associated with a target node $v$; in practice, we use the logit or log-probability of the ground-truth class at $v$. For a source node $u$, we measure how much $u$ can affect the prediction at $v$ through the gradient magnitude:
\begin{equation}
    G(u \!\to\! v)
    =
    \left\lVert \frac{\partial s_v}{\partial x_u} \right\rVert_2 
\end{equation}
We then aggregate this quantity by graph distance. Let $\operatorname{dist}(u,v)$ denote the shortest-path distance on the graph. For each hop $d \in \{0,\dots,D\}$, we define the distance-binned influence as:
\begin{equation}
    I(d)
    =
    \mathbb{E}_{v \in \mathcal{T}}
    \left[
    \frac{1}{|\{u:\operatorname{dist}(u,v)=d\}|}
    \sum_{\substack{u:\\ \operatorname{dist}(u,v)=d}}
    \left\lVert \frac{\partial s_v}{\partial x_u} \right\rVert_2
    \right]
\end{equation}
where $\mathcal{T}$ is a set of sampled target nodes from the evaluation split. Finally, to focus on the relative decay with distance, we normalise by the zero-hop influence:
\begin{equation}
    \widetilde{I}(d)
    =
    \frac{I(d)}{I(0)} 
\end{equation}
A steep decay of $\widetilde{I}(d)$ indicates that the prediction is mostly sensitive to nearby nodes, whereas a slower decay indicates that information from farther nodes remains influential.

\paragraph{Experimental protocol.}
We evaluate this diagnostic on three recent heterophily benchmarks with different structural regimes: \textsc{Minesweeper}, \textsc{Roman-Empire}, and \textsc{Amazon-Ratings}. For each dataset, we compare the three transport classes, Diagonal, Bundle, and General, under both diffusion families, NSD and PolyNSD, yielding six curves per dataset. NSD is trained with $10$ layers, while PolyNSD is trained with $6$ layers and polynomial degree $K=10$. Thus, NSD relies on a deeper stack of repeated first-order local updates, whereas PolyNSD combines moderate depth with higher-order spectral propagation inside each layer. For each trained model, we sample target nodes from the evaluation split, compute $\nabla_x s_v$ by backpropagation, and aggregate $\left\lVert \partial s_v / \partial x_u \right\rVert_2$ into hop-distance bins up to $D=10$. Hop distances are computed by BFS on the undirected input graph. We report $\widetilde{I}(d)$ on a log scale for readability, clamping extremely small values only for plotting stability.

\paragraph{Findings: PolyNSD Preserves Long-Range Influence.}
\autoref{fig:long-range-influence} shows a clear separation between NSD and PolyNSD across \textsc{Minesweeper}, \textsc{Roman-Empire}, and \textsc{Amazon-Ratings}. We observe that:
\begin{itemize}
    \item \emph{NSD influence decays sharply with distance.}
    Across datasets and transport classes, NSD curves generally drop by several orders of magnitude as the hop distance increases, indicating that predictions become rapidly less sensitive to distant nodes under repeated first-order diffusion.

    \item \emph{PolyNSD maintains stronger medium- and long-range influence.}
    PolyNSD variants preserve larger values of $\widetilde{I}(d)$ at intermediate and large distances. This is especially visible after $d\geq 5$, where PolyNSD curves often remain several orders of magnitude above their NSD counterparts.

    \item \emph{Polynomial propagation benefits all transport classes.}
    The long-range advantage is visible for Diagonal, Bundle, and General variants, suggesting that the gain is not only due to more expressive restriction maps, but also to the higher-order spectral filtering introduced by PolyNSD.
\end{itemize}

Overall, this diagnostic indicates that PolyNSD improves long-range sensitivity compared to first-order NSD. While NSD propagates information through repeated local updates, PolyNSD combines $6$ layers with degree-$10$ polynomial filtering, allowing each layer to perform multi-hop transport-aware mixing. This leads to a slower decay of influence with graph distance and suggests that PolyNSD can retain useful signals from farther regions of the graph.

\begin{figure*}[t]
    \centering
    \begin{subfigure}[t]{0.32\textwidth}
        \centering
        \includegraphics[width=\linewidth]{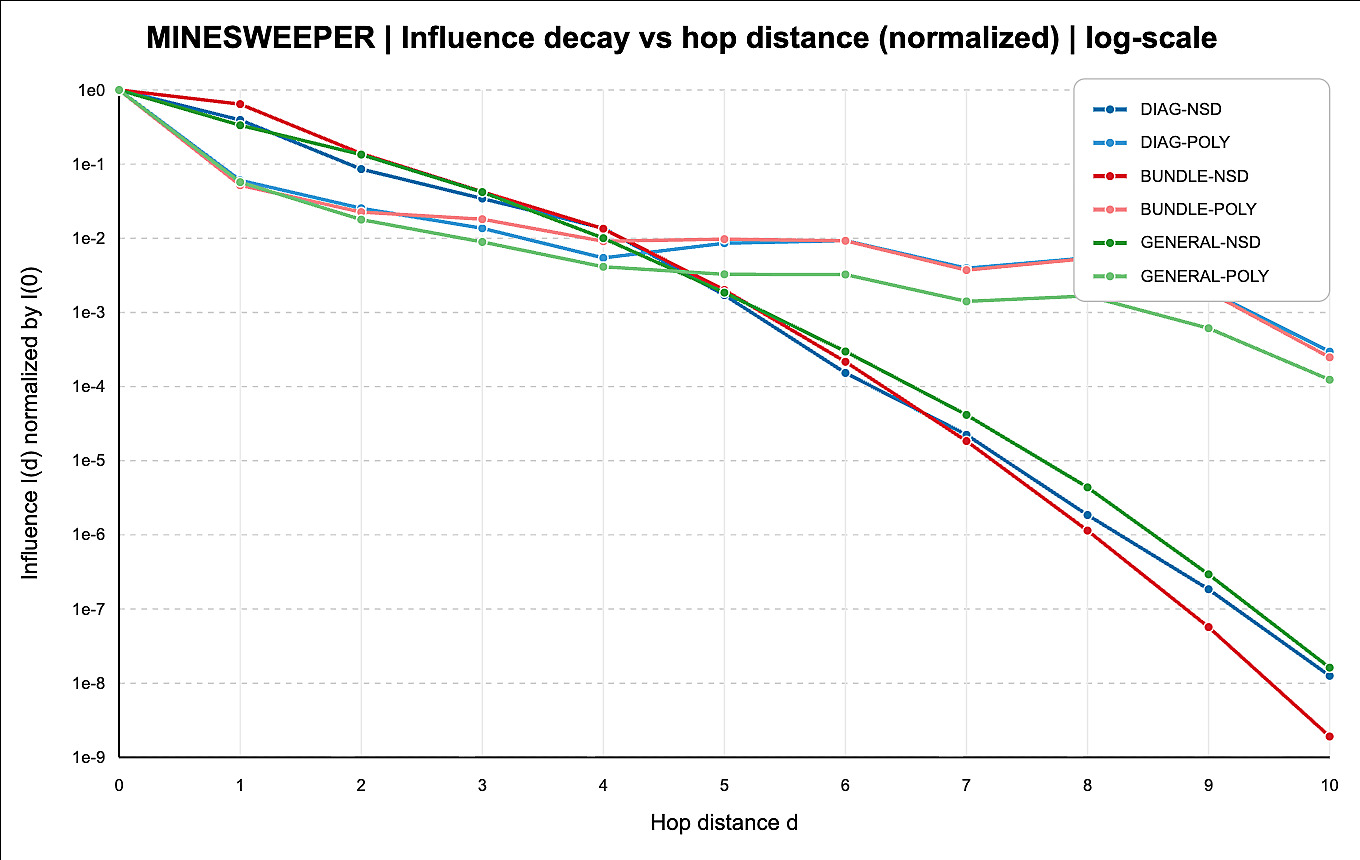}
        \caption{\textbf{\textsc{Minesweeper}}}
        \label{fig:long-range-minesweeper}
    \end{subfigure}\hfill
    \begin{subfigure}[t]{0.32\textwidth}
        \centering
        \includegraphics[width=\linewidth]{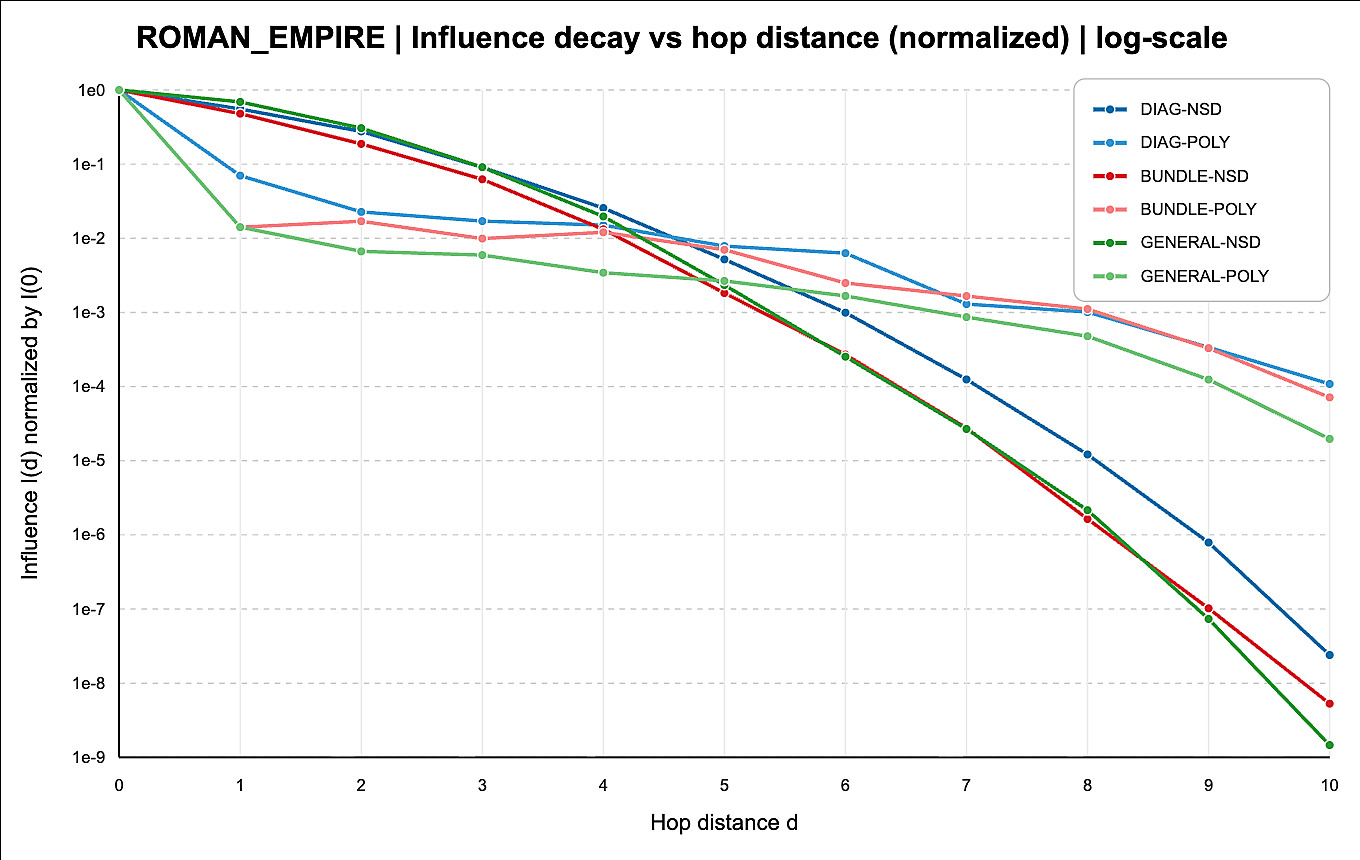}
        \caption{\textbf{\textsc{Roman-Empire}}}
        \label{fig:long-range-roman-empire}
    \end{subfigure}\hfill
    \begin{subfigure}[t]{0.32\textwidth}
        \centering
        \includegraphics[width=\linewidth]{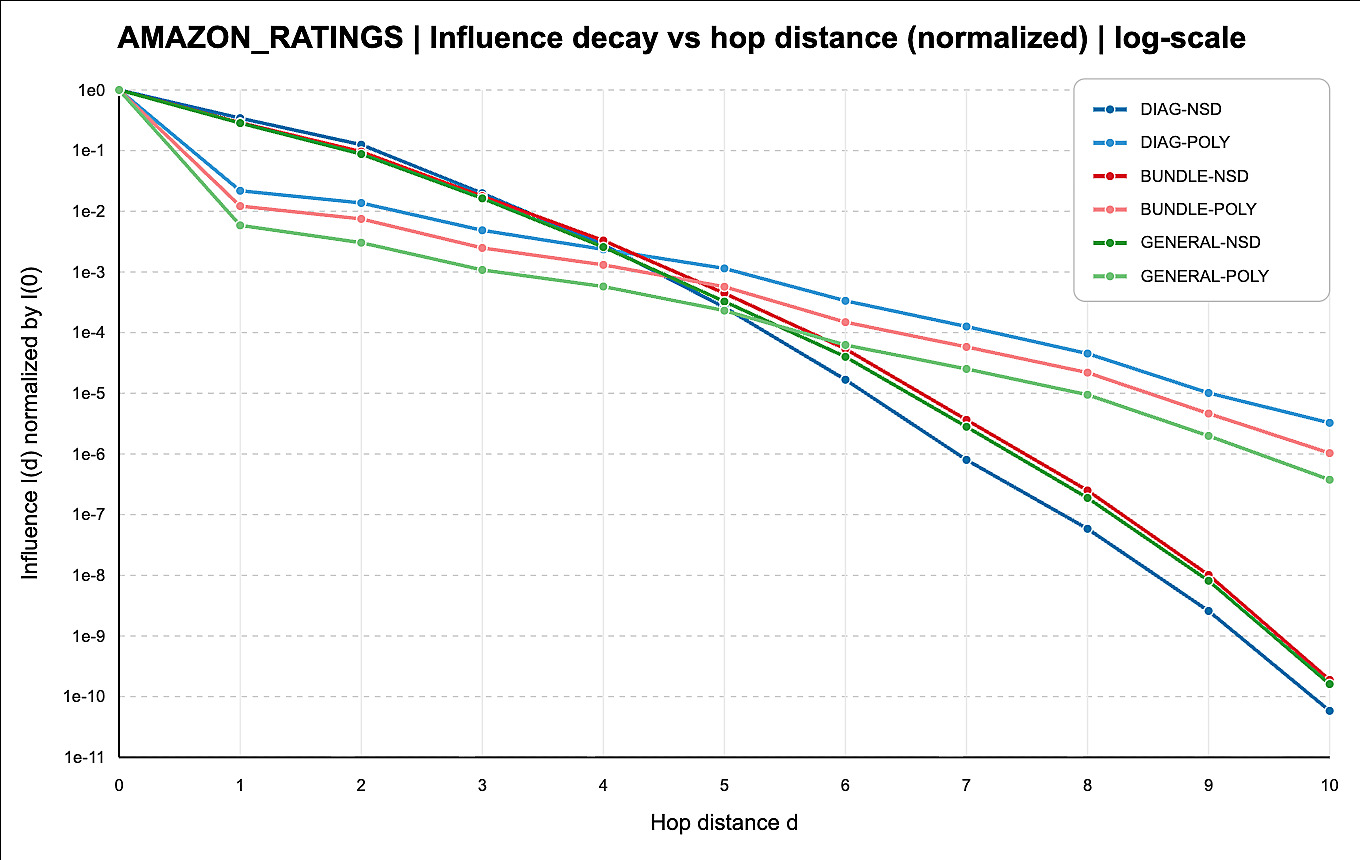}
        \caption{\textbf{\textsc{Amazon-Ratings}}}
        \label{fig:long-range-amazon-ratings}
    \end{subfigure}

    \caption{\emph{Long-range influence decay diagnostics.}
    Normalised influence $\widetilde{I}(d)=I(d)/I(0)$, shown on a log scale, as a function of hop distance $d$ for NSD and PolyNSD across Diagonal, Bundle, and General transport classes. NSD is evaluated with $10$ layers, while PolyNSD uses $6$ layers with polynomial degree $K=10$. Across \textsc{Minesweeper}, \textsc{Roman-Empire}, and \textsc{Amazon-Ratings}, PolyNSD generally preserves stronger medium- and long-range influence, indicating higher sensitivity to distant nodes than first-order NSD.}
    \label{fig:long-range-influence}
\end{figure*}

\subsection{Chebyshev Order \texorpdfstring{$K$}     -- Sweep}
\label{app:chebK-sweep}
This section investigates how the polynomial order $K$ of PolyNSD affects performance and how this interacts with the choice of spectral scaling $\lambda_{\max}$. The goal is to understand when higher-order filters are beneficial and to what extent they strictly improve over the $K{=}1$ (NSD-equivalent) case. For these ablations we use a controlled configuration shared across all datasets: stalk dimension $d=4$, number of diffusion layers $L=2$, hidden channels $16$, and within each dataset and PolyNSD variant, we sweep the polynomial order $K \in \{2,4,8,16\}$ and the spectral scaling strategy: \emph{analytic} vs \emph{iterative} estimate of $\lambda_{\max}$, as described in \autoref{subsec:polysd-poly-filter}. We report mean$\pm$std test accuracy over the
$10$ splits. 

\begin{figure}[h]
  \centering
  \includegraphics[width=0.75\linewidth]{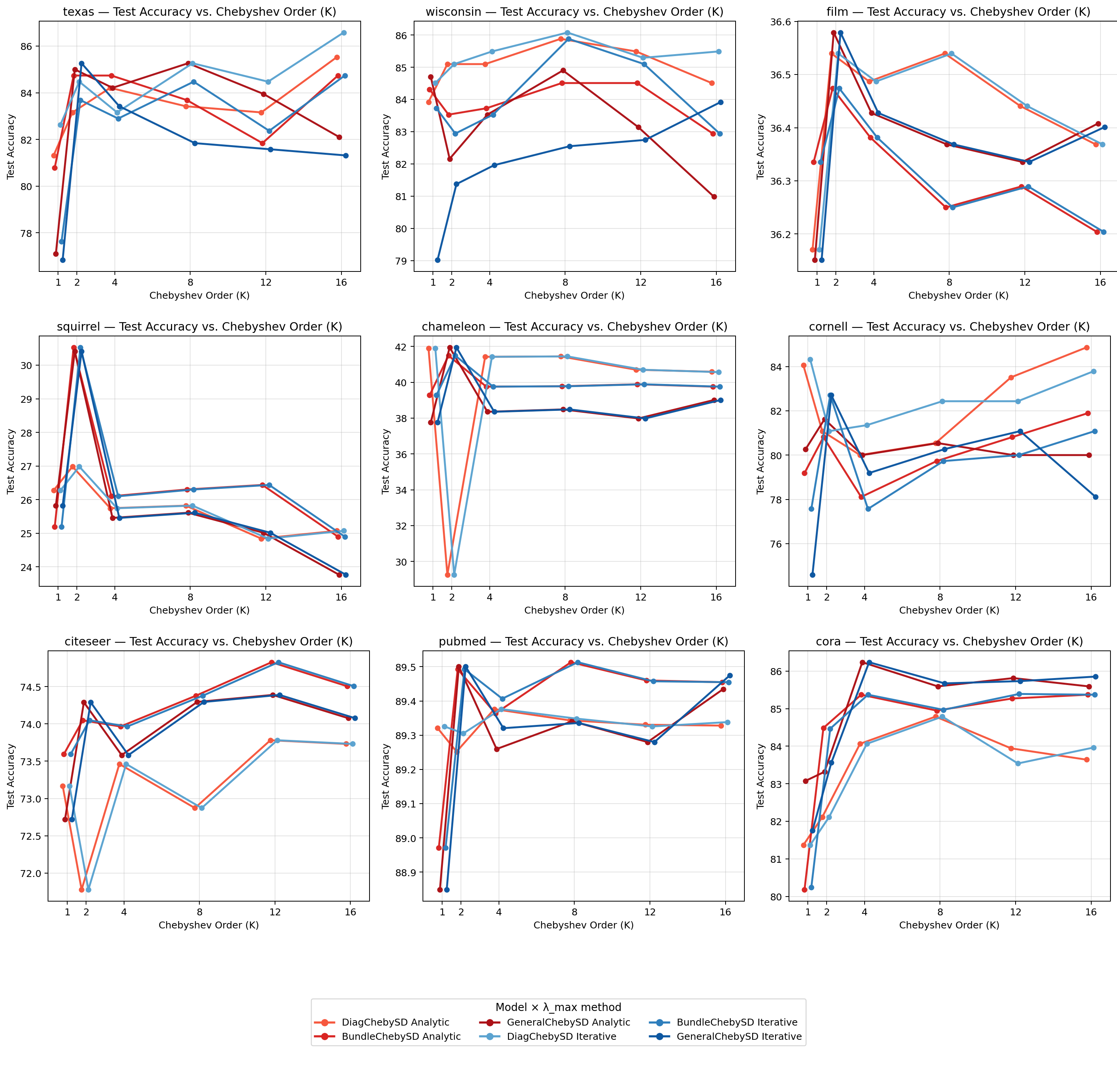}
  \caption{\textit{Chebyshev order $K$ sweep.}
  Test accuracy vs.~$K$ for the nine real-world benchmarks.
  Each panel corresponds to one dataset and overlays the six configurations
  given by the three PolyNSD variants crossed with analytic vs.~iterative
  estimates of $\lambda_{\max}$. Error bars denote mean$\pm$std over the
  $10$ fixed splits.}
  \label{fig:chebK-sweep-3x3}
\end{figure}

\paragraph{Homophily and Heterophily Order \(K\)'s Impact.} The results in \autoref{fig:chebK-sweep-3x3} shows test accuracy as a function of $K$ for the nine real-world benchmarks, with one panel per dataset. Each panel overlays the six configurations obtained by crossing the three transport classes with the two choices of $\lambda_{\max}$ estimation. Across datasets : (i) On homophilous graphs such as \textsc{Cora}, \textsc{Citeseer} and \textsc{PubMed}, moderate orders ($K\approx 4\text{--}8$) are often optimal. Increasing $K$ beyond this range yields diminishing returns. (ii) On heterophilous graphs such as \textsc{Squirrel} and \textsc{Chameleon}, the optimum shifts towards larger values ($K\in\{8,16\}$), consistent with the need for longer-range, multi-frequency propagation to stabilise learning under label inconsistency. (iii) For all three transport classes and across all datasets, the best configuration always satisfies $K>1$, meaning that higher-order polynomial filters strictly improve over the NSD-like $K=1$ baseline at fixed depth and width.

\input{tables/polysd_ksweep}

\paragraph{Numeric Results: Degree $K>1$ is always beneficial.} To provide a more compact view of the effect of $K$ in the context of the accuracy–efficiency trade-offs, \autoref{tab:polysd_ksweep} reports PolyNSD performance on three datasets with distinct homophily levels: \textsc{PubMed} (high homophily), \textsc{Chameleon}, and \textsc{Squirrel} (heterophilous). We fix the depth and width and sweep $K\in\{2,4,6,8,16\}$, reporting for each entry both the test accuracy and the corresponding parameter count. On all three datasets, the best configuration occurs at $K>1$, and increasing $K$ allows PolyNSD to emulate a wide range of effective propagation depths without changing the number of layers. Here, higher-order sheaf polynomials are shown to strictly enlarge the class of linear operators that can be realised at fixed depth.

\subsection{PolyNSD vs NSD: Detailed Accuracy–Efficiency Analysis}
\label{app:polynsd-vs-nsd-exp}

This subsection quantifies how much accuracy can be retained or gained when replacing NSD layers with PolyNSD layers at comparable or reduced parameter counts. We focus on three large benchmarks that cover both homophilous and heterophilous regimes: \textsc{PubMed} (high homophily), \textsc{Chameleon} and \textsc{Squirrel} (low homophily). For each dataset and
for each transport class, we consider two complementary sweeps: (i) \emph{Depth sweep (NSD)}, where hidden width and stalk dimension are fixed, while the number of NSD layers $L$ is varied, and (ii) \emph{Width sweep (NSD)}, where depth and stalk dimension are fixed, while the hidden width (number of channels) is varied. For each configuration we record the test accuracy (mean$\pm$std over the $10$ fixed splits) and the total number of trainable parameters. This allows us to draw iso-accuracy and iso-parameter comparisons between PolyNSD and NSD.

\input{tables/nsd_layers_plus_polysd_summary}

\subsubsection{PolyNSD VS NSD: Depth Sweep}
\label{app:layers-sweep}

In the first comparison we fixed PolyNSD depth to $L=2$ and vary its polynomial
order $K$, while sweeping NSD depth $L\in\{2,4,8,16,32\}$ at fixed hidden
width $16$ and stalk dimension $4$. Once chosen the best setting from \autoref{tab:polysd_ksweep}, we compare it against the NSD depth sweep. \autoref{tab:nsd_layers_plus_polysd_summary} summarises the results by reporting, for each dataset and transport class: (i) the best NSD configuration over the depth sweep, (ii) the best PolyNSD configuration over the $K$ sweep at depth $L=2$ and (iii) the accuracy difference and parameter ratio between PolyNSD and the best NSD configuration. 

\paragraph{Findings: PolyNSD improves performances w.r.t. NSD and it does this with fewer layers.} On \textsc{PubMed} (homophilic setting), PolyNSD matches or slightly improves over the best NSD configuration with comparable or smaller parameter counts. On \textsc{Chameleon} and \textsc{Squirrel} (heterophilic settings), PolyNSD yields substantial gains (up to  $+6\%\text{--}+13\%$ accuracy) over NSD at similar or smaller parameter budgets, particularly in the diagonal and bundle variants, highlighting the advantage of higher-order spectral control in heterophilous regimes. An important comment to make is that, if we were to compare PolyNSD and NSD with the same power, we would need to look at the very first column of the layers table. From this, we can notice further improvements: with the same computational power, PolyNSD is extremely stronger than NSD, reaching peaks of $+20\%$ in accuracy in heterophilic settings.

\subsubsection{PolyNSD VS NSD: Width Sweep}
\label{app:width-sweep}

We now consider the complementary scenario w.r.t. the previous experiment: we fix NSD depth to $L=2$ and sweep the hidden width $\text{Hidden} \in \{16,32,64,128,256\}$, keeping stalk dimension $4$ constant. \autoref{tab:nsd_width_plus_polysd_summary} reports the corresponding test
accuracies and parameter counts, together with the relative improvement of
PolyNSD is the best NSD configuration across the width sweep. 

\paragraph{Findings: PolyNSD improves performances w.r.t. NSD and it does this with fewer hidden channels.} On \textsc{PubMed} and \textsc{Chameleon}, PolyNSD matches or exceeds the best NSD performance while using fewer parameters dramatically (often only $1\%\text{--}20\%$ of the parameters of the widest NSD models). Increasing NSD width beyond a certain point yields only marginal gains at a substantial parameter cost. On \textsc{Squirrel}, NSD can sometimes surpass PolyNSD by combining wide layers with large depth, resulting in models with between $10^6$ and $10^7$ parameters. However, at equal or smaller parameter budgets, PolyNSD consistently attains higher accuracy, indicating a more favourable accuracy–efficiency frontier. 

\input{tables/nsd_width_plus_polysd_summary}

\subsection{Spectral-Response Diagnostics Across Homophilic and Heterophilic Regimes}
\label{app:spectral-response-diagnostics}

In this appendix subsection, we make some \emph{model diagnostic} and we try to ask to the question \emph{what kind of frequency response} the learned PolyNSD filters implement, and how this behaviour changes across transport classes, network depth $L$, stalk dimension $d$, and graph regimes (homophilic vs.\ heterophilic). PolyNSD replaces an affine-in-spectrum diffusion core with a learnable degree-$K$ polynomial in the (sheaf) Laplacian. Raw accuracy tables we have alone do not reveal \emph{how} the model uses its additional spectral expressivity, so the goal here is therefore to provide evidence that PolyNSD learns \emph{non-trivial, dataset-dependent spectral shaping}, and to characterise systematic differences between homophilic and heterophilic datasets.

More formally, for a (sheaf) Laplacian eigenpair $Lu=\lambda u$, we saw that a polynomial filter $p(L)$ acts diagonally in the eigenbasis: $p(L)u = p(\lambda)\,u$. In our implementation, the polynomial is evaluated on the rescaled eigenvalue $\xi(\lambda)=2\lambda/\lambda_{\max}-1\in[-1,1]$. We have also seen that our discrete PolyNSD layers include a high-pass reinjection term controlled by $\alpha_{\mathrm{hp}}$, which, under a one-step linearisation, contributes an approximately linear correction proportional to $(1-\lambda/\lambda_{\max})$. For interpretability, we therefore visualise three curves:
\begin{align}
p(\xi(\lambda)) &\quad \text{(polynomial component)},\\
\alpha_{\mathrm{hp}}\Big(1-\frac{\lambda}{\lambda_{\max}}\Big) &\quad \text{(HP correction, approximate)},\\
m(\lambda) := p(\xi(\lambda)) + \alpha_{\mathrm{hp}}\Big(1-\frac{\lambda}{\lambda_{\max}}\Big)
&\quad \text{(combined response, approximate)}.
\end{align}

\paragraph{Experimental Protocol and Statistics.}
We train PolyNSD and then plot the learned spectral response at the best validation checkpoint. We do this for four datasets: two heterophilic (\textsc{Chameleon} and \textsc{Squirrel}) and two homophilic (\textsc{CiteSeer}and \textsc{PubMed}). We run these experiments across multiple seeds and hyperparameters settings. To make the grid comparable at scale, we compute from the combined response $m(\lambda)$ some statistics: $G_{\text{low}}$, being the mean of $m(\lambda)$ over the lowest $20\%$ of the spectrum (``low frequencies''). $G_{\text{high}}$, being the mean of $m(\lambda)$ over the highest $20\%$ of the spectrum (``high frequencies''), and finally $\Delta G =  G_{\text{high}}-G_{\text{low}}$, being a separation measure between high and low spectral regions. We also report \texttt{Non-monotone}, being the number of sign changes in $\frac{d}{d\lambda}m(\lambda)$, a simple indicator of non-monotonic / band-pass-like behaviour. We report the diagnostic statistics in \autoref{tab:spectral-response-summary}, aggregated per dataset over the entire grid. In \autoref{fig:spectral-response-examples}, we show representative learned responses from our grid of 4 datasets.

\paragraph{Findings: Heterophilic vs.\ Homophilic Spectral Patterns.}
Across the full grid (all models, $L$, $d$, and seeds), we observe systematic differences between heterophilic and homophilic datasets. 
More in particular: 

\textit{(1) Homophilic datasets exhibit Stronger Low--High Spectral Separation in the learned response.}
Using $\Delta G$ as a robust separation proxy, homophilic graphs show substantially larger separation. Indeed, in average, we have that \textit{Heterophilic} dataset spectral separation is $\Delta G = 0.45\pm 0.14$, while \textit{Homophilic} one is: $\Delta G =  1.05\pm 0.34$. The reason for that is that the learned filter in homophilic regimes tends to implement a more pronounced ``contrast'' between different parts of the spectrum. In heterophilic regimes, the separation is milder, consistent with the need to balance information at multiple neighbourhood radii rather than primarily emphasising a single spectral extreme.

\textit{(2) Low-frequency Gain differs sharply between Regimes.}
A striking pattern is that $G_{\text{low}}$ is \emph{never} positive on \textsc{PubMed}/\textsc{CiteSeer} in our grid, while it is frequently positive on heterophilic graphs (the fraction of runs with $G_{\text{low}}>0$ is indeed \;\;38.8\% (heterophilic) vs.\ 0.0\% (homophilic)). This suggests that \text{heterophilic} datasets more often induce filters that \emph{retain or amplify} a portion of low-frequency content, whereas homophilic datasets consistently learn a response that suppresses the lowest end of the spectrum and relies more heavily on the interaction between the polynomial component and the HP reinjection.

\textit{(3) High-pass Reinjection is markedly Stronger on Homophilic datasets.} The learned $\alpha_{\mathrm{hp}}$ is consistently more negative (larger magnitude) on \textsc{PubMed} and \textsc{CiteSeer} than on \textsc{Chameleon} and \textsc{Squirrel}. Quantitatively, we have: \textit{Heterophilic} being $\alpha_{\mathrm{hp}} = -0.25\pm 0.08$ and \textit{Homophilic} equal to $\alpha_{\mathrm{hp}} = -0.98\pm 0.44$. This indicates that, in homophilic settings, the model uses more the HP skip to counteract oversmoothing and/or to preserve discriminative components that would otherwise be washed out by repeated diffusion.

\textit{(4) Heterophilic Graphs show more Consistently non-Monotone (Band-pass-like) responses.}
The learned curves are frequently non-monotone. The heterophilic runs show near-maximal non-monotonicity (\texttt{nonmonotone}: 2.00 (heterophilic) vs.\ 1.69 (homophilic) on average). Qualitatively, many heterophilic curves resemble mild band-pass shapes, aligning with the common observation that \textit{heterophily benefits from} more \textit{selective, higher-order spectral control} rather than repeated smoothing.

\paragraph{Spectral response evolving with depth $L$: earned responses are less extreme as layers are stacked.}
In \autoref{tab:spectral-response-depth}, we aggregate the
diagnostics per regime as a function of $L$ (averaging across transport classes, $d\in\{2,3,4\}$ and seeds). Here, increasing $L$: (i) makes $\alpha_{\mathrm{hp}}$ \emph{less negative} (smaller
magnitude) in both regimes, and (ii) reduces the low--high
separation $\Delta G$, especially on homophilic datasets. This indicates that depth partly substitutes for explicit HP reinjection and strong spectral separation: as more PolyNSD layers are stacked, the learned response becomes less ``extreme'' at the ends of the spectrum, while heterophilic runs remain non-monotone for all $L$.

\input{tables/spectral_response_layers}

\paragraph{Spectral response evolving with stalk dimension $d$: reduction of magnitude and separation.}
In \autoref{tab:spectral-response-d}, we aggregate the diagnostics as a function of $d$ (averaging across transport classes, depths $L\in\{2,3,4\}$ and seeds). On homophilic datasets, the diagnostic statistics are comparatively \emph{stable} across $d$, while for heterophilic datasets, instead, increasing $d$ tends to reduce both the magnitude of $\alpha_{\mathrm{hp}}$ and the separation $\Delta G$. This result
suggests a flatter (i.e., less contrastive) response when using wider stalks, indicating that, in heterophilic regimes, increasing fiber dimension can
act as an additional degree of freedom that partially compensates for the strong
spectral separation, whereas in homophilic regimes, the HP correction remains a
dominant and robust mechanism irrespective of $d$.

\input{tables/spectral_responsed_d}

\begin{figure}[t]
\centering
\caption{\textit{Spectral-response plots}. Representative spectral response got for the 4 datasets, including the combined response $m(\lambda)$, the polynomial component and the HP correction term.}
\begin{minipage}{0.3\linewidth}
  \centering
  \includegraphics[width=\linewidth]{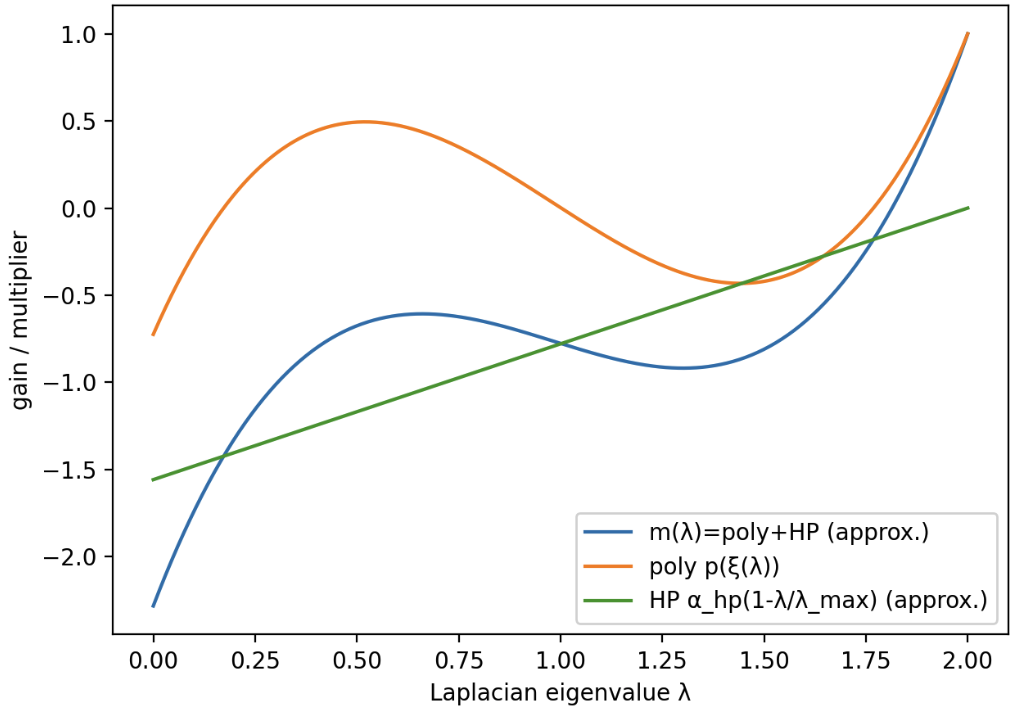}
  \vspace{-1mm}
  {\scriptsize \textsc{PubMed} (homophilic), Bundle, $d=4$, $L=2$.}
\end{minipage}
\begin{minipage}{0.3\linewidth}
  \centering
  \includegraphics[width=\linewidth]{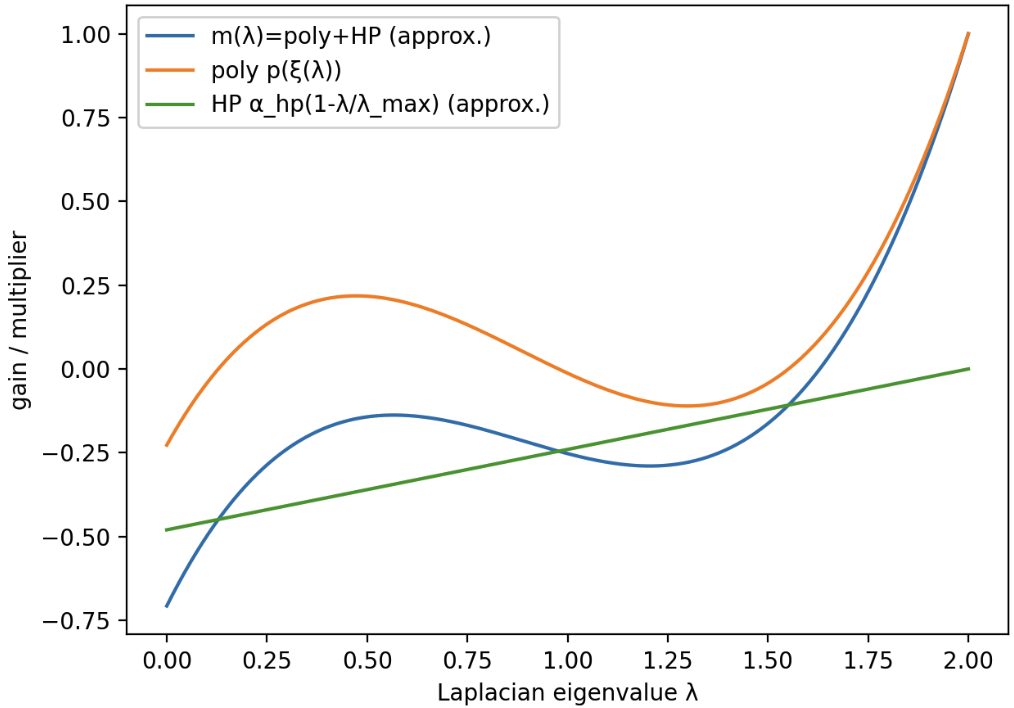}
  \vspace{-1mm}
  {\scriptsize \textsc{CiteSeer} (homophilic), Bundle, $d=2$, $L=4$.}
\end{minipage}

\vspace{2mm}

\begin{minipage}{0.3\linewidth}
  \centering
  \includegraphics[width=\linewidth]{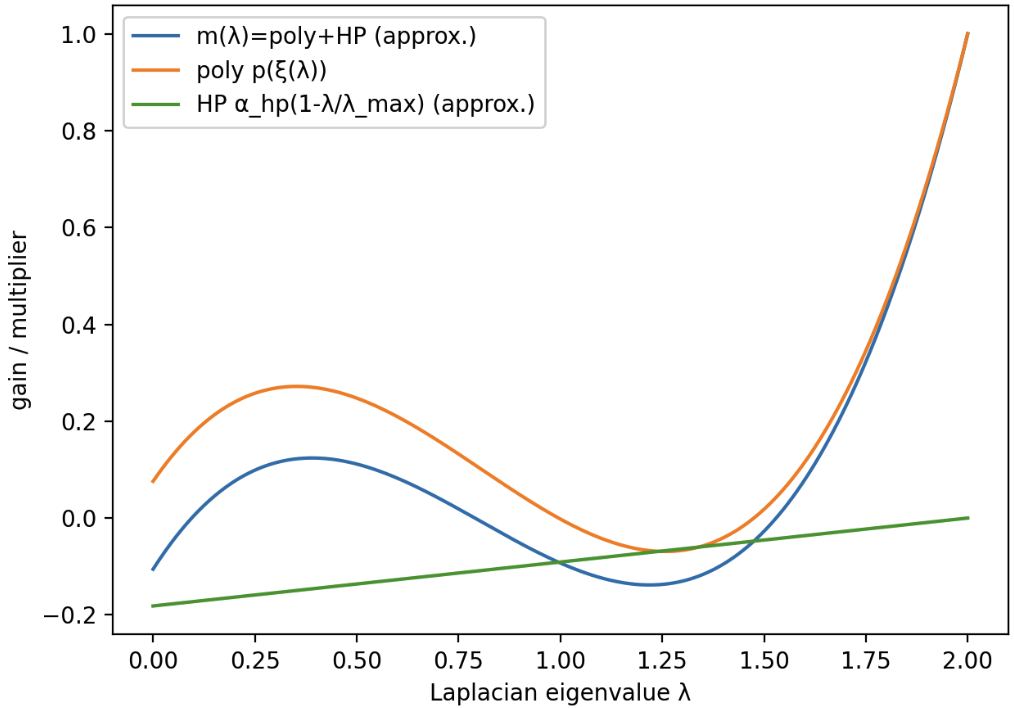}
  \vspace{-1mm}
  {\scriptsize \textsc{Chameleon} (heterophilic), Diagonal, $d=4$, $L=4$.}
\end{minipage}
\begin{minipage}{0.3\linewidth}
  \centering
  \includegraphics[width=\linewidth]{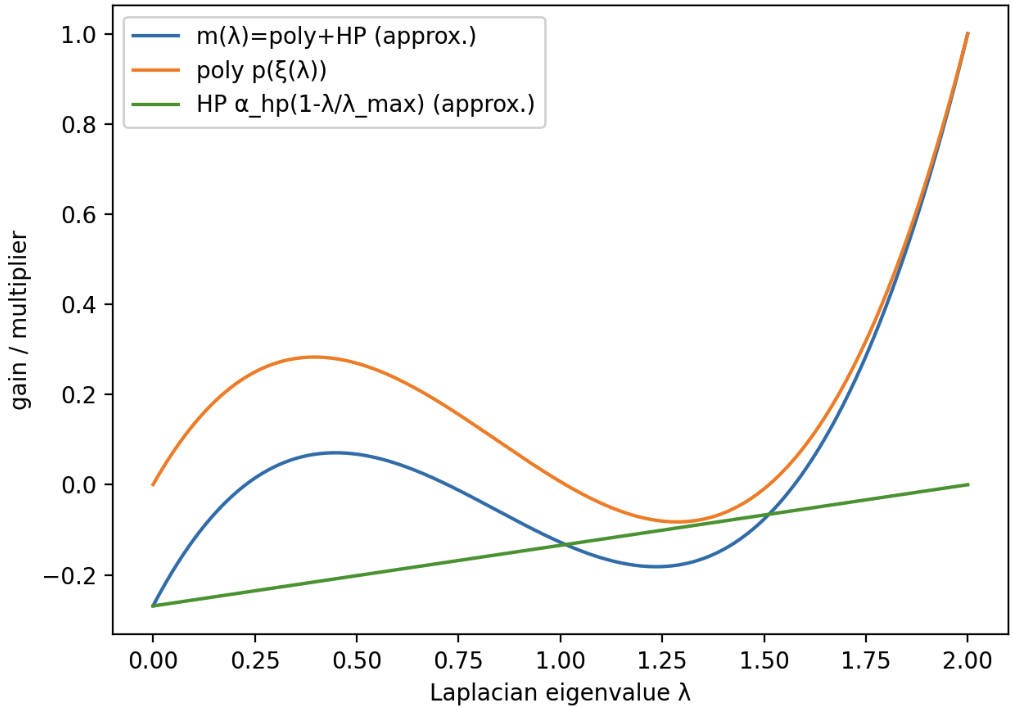}
  \vspace{-1mm}
  {\scriptsize \textsc{Squirrel} (heterophilic), Bundle, $d=2$, $L=3$.}
\end{minipage}
\label{fig:spectral-response-examples}
\end{figure}

\subsection{Continuous-Time PolyNSD via Neural Sheaf ODEs}
\label{app:ode-polynsd}

To confirm the model-agnostic claim, we extend the PolyNSD from discrete-time layers to \emph{continuous-time} formulation in which sheaf diffusion is parameterised as a neural ordinary differential equation (ODE), following the Neural Sheaf Diffusion (NSD) model of \cite{bodnar2022neural}. At the level of node features $X(t) \in \mathbb{R}^{nd \times f}$, NSD defines a time-evolving sheaf $(G,\mathcal{F}(t))$ and a diffusion-type ODE of the form: $\dot{X}(t)
  \;=\;
  -\,\sigma\!\big(
    \Delta_{\mathcal{F}(t)} \, (I_n \otimes W_1)\, X(t)\, W_2
  \big)$, 
where $\Delta_{\mathcal{F}(t)}$ is the (normalised) sheaf Laplacian at time $t$, $W_1, W_2$ are trainable weight matrices, and $\sigma$ is a (typically $1$-Lipschitz) nonlinearity. The restriction/transport maps and, hence, $\Delta_{\mathcal{F}(t)}$ are obtained from the current features via a learnable parametrisation $(G,\mathcal{F}(t)) = g(G,X(t);\theta)$, so that both the geometry and the diffusion field evolve over time.

\paragraph{From Discrete PolyNSD to PolyNSD ODEs.}
Discrete PolyNSD layers replace the one-step NSD diffusion core $aI + bL$ by a degree-$K$ spectral polynomial $q_\theta(L)$ in a (normalised or unnormalised) sheaf Laplacian $L$, evaluated via a three-term recurrence on a spectrally rescaled operator $L_{\mathrm{e}}$, and augmented with residual and high-pass terms (Section~\ref{sec:polysd}). In the continuous-time setting, we use the same sheaf prediction mechanism and polynomial filter to define the \emph{infinitesimal} vector field of a neural sheaf ODE:
\begin{equation}
  \dot{X}(t)
  \;=\;
  f_\theta\big(G,\mathcal{F}(t), X(t)\big)
  \;:=\;
  -\,\sigma\!\big(
    q_\theta(\Delta_{\mathcal{F}(t)}) \, (I_n \otimes W_1)\, X(t)\, W_2
  \big),
  \label{eq:polynsd-ode}
\end{equation}
where $q_\theta$ is a degree-$K$ polynomial in one variable (e.g.\ a convex combination of Chebyshev polynomials evaluated at $L_{\mathrm{e}}$). For fixed sheaf structure, linear activations and time-independent $q_\theta$, the ODE $\dot{X}(t) \;=\; -\,q_\theta(\Delta_{\mathcal{F}})\,X(t)$ has the closed-form solution $X(T) = \exp(-T\,q_\theta(\Delta_{\mathcal{F}}))\,X(0)$, so that each eigenmode with Laplacian eigenvalue $\lambda$ is multiplied at time $T$ by $\exp(-T\,q_\theta(\lambda))$. In this sense, continuous-time PolyNSD realises a \emph{continuous-depth} sheaf diffusion whose spectral response is given by an exponential of a learnable polynomial in the sheaf spectrum, generalising the NSD case where $q_\theta(\lambda)$ is effectively affine in~$\lambda$.

\paragraph{Continuous ODE Solvers.}
In practice, we work with the same three transport classes as in the discrete setting, yielding the continuous families \emph{Cont DiagChebySD}, \emph{Cont BundleChebySD} and \emph{Cont GeneralChebySD}. At each solver step, we: (i) use the current features $X(t)$ to predict edge-wise restriction maps via a parametric function $\Phi_\theta$, yielding a sheaf $(G,\mathcal{F}(t))$ and its Laplacian $\Delta_{\mathcal{F}(t)}$, (ii) construct the spectrally rescaled Laplacian $L_{\mathrm{e}}(t)$ and evaluate the polynomial $q_\theta(L_{\mathrm{e}}(t))$ via a stable three-term recurrence (with $K$ sparse matrix--vector products), (iii) apply the diffusion field $f_\theta$ in \eqref{eq:polynsd-ode} and advance the state with a fixed-step ODE solver (integrating over a time horizon $t \in [0,T]$ using a simple explicit method like the forward Euler or low-order Runge--Kutta.

\paragraph{Depth, Spectral Control, and the Impact of PolyNSD in Continuous Time.}
From a modelling perspective, the triple $(K, T, N_{\text{steps}})$ jointly controls the effective depth and spectral behaviour of the continuous-time dynamics. The polynomial degree $K$ determines the class of spectral generators $q_\theta(\lambda)$ that can be realised: for $K=1$ we recover NSD-type generators $a + b\lambda$, whereas $K>1$ allows higher-order shaping of the sheaf spectrum. The integration horizon $T$ scales the exponential factor $\exp(-T\,q_\theta(\lambda))$, thus modulating how quickly different frequencies are damped or preserved, and the numerical step size $\Delta t$ trades off solver stability and approximation accuracy.

\paragraph{Findings: PolyNSD continuous-time implementation reaches s.o.t.a. results.}
\autoref{tab:nodecls-homophily} shows that introducing polynomial spectral control in the \emph{continuous-time} setting is consistently beneficial, especially in heterophilic regimes, with the three proposed continuous PolyNSD variants (\textit{Cont DiagChebySD}, \textit{Cont BundleChebySD}, \textit{Cont GeneralChebySD}) achieving \emph{top-three} performance on most benchmarks. Compared to continuous NSD baselines, the improvements can be substantial in the most heterogeneous cases and this supports the interpretation that higher-order spectral generators $q_\theta(\lambda)$ are particularly effective when the task requires mixing information across multiple neighbourhood radii rather than primarily smoothing. As for the discrete-time case, no single transport class dominates uniformly across all datasets, making the \textit{diagonal} variant the preferred one once again, thanks to the important implicit savings it brings. This suggests that polynomial spectral control is complementary to the choice of transport class: the polynomial improves the \emph{spectral expressivity} of the dynamics, while the transport class controls the \emph{geometry} of information propagation.

\subsection{Alternative Orthogonal Polynomial Bases Used in PolyNSD}
\label{app:orthogonal-bases}

Our Polynomial Neural Sheaf Diffusion (PolyNSD) layers apply a spectral filter to a (sheaf) Laplacian \(\mathbf{L}\) via an orthogonal polynomial basis on the rescaled operator $\tilde{\mathbf{L}}$. Given a feature vector \(x\in\mathbb{R}^{N d}\) (stacking fibers) and a family of orthogonal polynomials \(\{\Phi_k\}_{k\ge 0}\) on \([-1,1]\), we define then the filter: $p(\tilde{\mathbf{L}})\,x
  \;=\;
  \sum_{k=0}^{K} \alpha_k \,\Phi_k(\tilde{\mathbf{L}})\,x,  \alpha_k = \mathrm{softmax}(\theta)_k$, where \(\theta\in\mathbb{R}^{K+1}\) are learnable logits and \(\alpha_k\) form a convex mixture of the basis responses.  All bases are implemented via stable three–term recurrences
directly on vectors \(v\mapsto \Phi_k(\tilde{\mathbf{L}})v\). We therefore experiment over multiple solutions in here, as listed in \autoref{tab:poly-bases-recurrences}.

\input{tables/polynomials_table}

\paragraph{Findings: Alternative orthogonal bases yield consistent performance, with Chebyshev as a strong default.}
\autoref{tab:big-leaderboard-enhanced} show that PolyNSD is \emph{robust} to the specific choice of orthogonal polynomial basis: across transport classes (Diagonal / Bundle / General) and datasets spanning the homophily--heterophily spectrum, most bases attain performances that are close to the Chebyshev Type~I default, confirming that the main gains come from \emph{higher-order spectral control} rather than from a single privileged basis. These results confirm that PolyNSD is \emph{basis-agnostic} and the key ingredient is the ability to learn a bounded degree-$K$ spectral multiplier on the rescaled sheaf Laplacian via stable recurrences. Chebyshev Type~I offers the most consistent default, but alternative orthogonal bases provide a viable design space for tailoring approximation properties without changing the PolyNSD layer structure.

\input{tables/polynsd_polytype_summary}

\subsection{Filtered and new Heterophily Benchmarks}
\label{app:filtered_heterophily_benchmarks}
We further evaluate our models on the filtered versions of the heterophilic
\textsc{Chameleon} and \textsc{Squirrel} datasets introduced by
Platonov et al.~\cite{platonov2023critical}. These datasets were proposed after
a re-examination of the commonly used Wikipedia heterophily benchmarks, where
Platonov et al. observed that the original \textsc{Chameleon} and
\textsc{Squirrel} graphs contain a large number of duplicate nodes. We include
these datasets to assess whether the performance of PolyNSD remains stable when the potential leakage present in the original benchmarks is removed. We retain
the $48\%/32\%/20\%$ train/validation/test split protocol used throughout the
Sheaf Neural Network evaluation setting, to keep the comparison
consistent with the rest of our experiments. 

\noindent\textbf{Findings: PolyNSD outperforms baselines in Filtered Wikipedia Datasets.} The results are reported in \Cref{tab:filtered_heterophily_results}. We compare PolyNSD against first-order Neural Sheaf Diffusion variants, standard GNNs, heterophily-specific architectures, and spectral baselines. Across both filtered datasets, PolyNSD variants obtain the strongest results among the considered models. In particular, PolyNSD improves over the corresponding NSD variants on both \textsc{Chameleon-Filtered} and \textsc{Squirrel-Filtered}, showing that the benefit of polynomial sheaf diffusion is preserved in this stricter benchmark. 
\input{tables/filtered_wikipedia_real}

\noindent\textbf{Findings: PolyNSD achieves strong performance on the new heterophily benchmark.}
The results are reported in \Cref{tab:new_heterophilic_results}. Across all five
datasets, PolyNSD consistently improves over first-order sheaf diffusion,
confirming that the benefits of higher-order polynomial propagation extend
beyond the classical small-scale heterophily benchmarks. In particular,
PolyNSD-General obtains the best result on all five datasets, while
PolyNSD-Bundle also reaches the top-three performance on
\textsc{Roman-Empire}, \textsc{Amazon-Ratings}, \textsc{Minesweeper}, and
\textsc{Tolokers}. 

\input{tables/platonov_table}

\subsection{Evaluation on Malignant Heterophily Datasets with 60/20/20 Splits}
\label{app:malignant_heterophily_60_20_20}

We additionally evaluate PolyNSD on the malignant heterophily datasets considered
by Luan et al.~\cite{luan2024re}. This benchmark separates heterophilic datasets
according to whether graph-aware models benefit from, or are harmed by, message
passing over the input graph. In this taxonomy, \textsc{Texas},
\textsc{Wisconsin}, \textsc{Film}, and \textsc{Cornell} are classified as
malignant heterophily datasets because graph-aware baselines such as GCN and
SGC-1 underperforms its graph-agnostic counterparts, namely MLP-2 and MLP-1.
Thus, these datasets represent a stricter evaluation setting for methods designed
to operate under harmful heterophilic structure. Differently from the main Sheaf Neural Network evaluation protocol, this benchmark uses 10 random splits with $60\%/20\%/20\%$ train/validation/test proportions. We therefore report these results separately, so that comparisons are not conflated across different splitting protocols. The goal of this evaluation is to test whether the higher-order spectral propagation introduced by PolyNSD remains beneficial when the graph structure is explicitly identified as harmful for standard message passing.

\noindent\textbf{Findings: PolyNSD remains very competitive under malignant heterophily.}
The results are reported in \Cref{tab:malignant_heterophily_results}. We compare
PolyNSD against first-order NSD variants, graph-aware and graph-agnostic
baselines, and spectral baselines. Across the four datasets, PolyNSD consistently
improves over the corresponding NSD variants. On \textsc{Film}, where the
graph-agnostic MLP baselines are substantially weaker than PolyNSD, the gains
are especially clear. On \textsc{Texas}, \textsc{Wisconsin}, and
\textsc{Cornell}, PolyNSD is competitive with the strongest graph-agnostic
baselines and obtains top-ranked results among the sheaf and spectral models. In some cases, MLP solutions reach the highest accuracy, but they are way more computationally expensive than PolyNSD. These results suggest, therefore, that the polynomial sheaf propagator can mitigate the limitations of first-order diffusion by learning a higher-order spectral response that selectively combines local and non-local information.

\input{tables/60-20-20_benchmarks}

\subsection{Transductive Long-Range Influence on CityNetworks}
\label{app:citynetworks_transductive}

We further evaluate PolyNSD in a transductive long-range graph learning setting using the \textsc{CityNetworks}\cite{liang2026quantifyinglongrangeinteractionsgraph} benchmark, which contains large city-scale graphs for \textsc{Paris}, \textsc{Shanghai}, \textsc{Los Angeles}, and \textsc{London}, and is designed to test whether graph models can propagate useful information across long-range dependencies in realistic spatial networks, stressing the ability of a model to integrate information over many hops in a transductive
setting (which is the main setting sheaf neural networks are currently built upon). We compare PolyNSD against first-order NSD, standard message-passing neural networks, graph transformers, and spectral graph baselines. For the baseline models, we include the best reported results at depth $L=16$ for both \emph{MPNNs}, \emph{GTs} and \emph{Spectral} classes, while for \emph{Sheaf} models, instead we only train for up to $L=6$, and report their results. 

\noindent\textbf{Findings: PolyNSD improves long-range transductive performance.}
The results are reported in \Cref{tab:citynetworks_transductive_results_app}.
PolyNSD-Diag obtains the strongest average performance among the considered
models, reaching $61.64\%$ average accuracy across the four cities. It improves substantially over first-order NSD-Diag, which reaches only $37.77\%$ on average, showing that local sheaf diffusion alone is not sufficient in this long-range transductive regime. PolyNSD also outperforms strong MPNN baselines such as GraphSAGE and ChebNet, as well as graph-transformer baselines such as GraphGPS, Exphormer, and SGFormer. The gains are also further highlighted since the results for \emph{Sheaf} models are obtained using only $L=6$ layers instead of the $L=16$ of the other cases. PolyNSD-Diag improves over the best reported baseline by a sizeable margin. These results suggest that polynomial sheaf propagation is particularly effective when the task requires controlled multi-hop integration over large spatial graphs.

\input{tables/city_network}

\noindent\textbf{Findings: PolyNSD offers the best accuracy-runtime trade-off.}
Accuracy alone does not fully capture the trade-off between expressivity and
efficiency. We therefore also report a runtime-aware comparison in
\Cref{tab:citynetworks_runtime}. Training and inference times are normalised so that PolyNSD-Diag corresponds to $1.00\times$. Although some purely spectral baselines are slightly faster per epoch, they obtain lower accuracy. Conversely, NSD-Diag is both less accurate and slower than PolyNSD-Diag in this setting. Thus, PolyNSD offers the best accuracy--runtime trade-off among the sheaf-based models: it substantially improves long-range predictive performance while remaining computationally competitive.
\input{tables/city_network_runtime}

\subsection{Restriction-Map Geometry Across Layers}
\label{app:restriction-map-geometry}

We inspect the geometry learned by the sheaf restriction maps across
layers. For each layer, we visualise the edge-wise restriction maps in two
complementary ways. First, we project the learned edge transports to two
dimensions using UMAP, colouring each edge by the norm of its transport map.
Second, we plot the raw restriction-map coordinates as a heatmap, where rows
correspond to edges and columns correspond to the restriction dimensions. This allows us to inspect both the global geometry of the learned edge transports and the coordinate-wise structure of the learned sheaf maps.

\paragraph{Layer-wise UMAP evolution and Growth of transport magnitude with depth.}
The UMAP projections show a clear progression across layers. In the first
layers, the edge transports form a relatively dense and continuous manifold:
edges are broadly distributed, and the transport norms vary smoothly across the embedding. This suggests that early sheaf layers learn a coarse transport geometry, where many edges still share similar geometric roles. As depth increases, the UMAP structure becomes more fragmented and filament-like. By layers $3$ and $4$, the projected edge transports concentrate along separated curves and branches, with several high-norm regions appearing at the periphery of the embedding. This indicates that deeper layers no longer treat edges as a single homogeneous population, but instead separate them into more specialised transport regimes. The colour scale of the UMAP plots also reveals a systematic increase in transport norm across layers. Early layers exhibit moderate edge-transport magnitudes, whereas deeper layers contain many edges with substantially larger norms. In particular, layers $3$ and $4$ show several high-norm regions, suggesting that the model progressively amplifies a selected subset of edge-wise transports. This behaviour is consistent with the role of sheaf restriction maps as anisotropic transport operators: rather than uniformly smoothing over all edges, the model learns which relations should carry stronger or weaker information.

\paragraph{Coordinate-wise structure of restriction maps.}
The heatmaps provide a more direct view of the learned restriction coordinates. In the first layers, the restriction dimensions exhibit relatively smooth and dimension-specific patterns. Some coordinates are mostly positive, while others are mostly negative, indicating that different stalk dimensions already acquire distinct transport roles. As the layers become deeper, the heatmaps become more structured and higher contrast. In particular, later layers show sparse high-magnitude positive bands against a background of mostly negative values. This suggests that the model learns a selective gating-like behaviour: most edges are attenuated in certain dimensions, while a smaller subset of edges is
strongly activated.

\paragraph{Findings: deeper layers specialise edge transports.}
The visualisations show that restriction maps become increasingly structured
with depth. UMAP projections evolve from compact, continuous clouds to more
separated and filamentary geometries, while heatmaps reveal stronger
dimension-wise specialisation and higher contrast between active and suppressed edge transports. This indicates that deeper PolyNSD layers learn more selective transport patterns, allowing the model to distinguish edges that should preserve, attenuate, or amplify information. These results provide an additional interpretability perspective on the empirical gains of PolyNSD: polynomial spectral propagation supplies higher-order mixing, while the learned sheaf restriction maps organise this mixing through edge-specific and dimension-specific transport geometry.

\begin{figure*}[t]
    \centering
    \begin{subfigure}[t]{0.32\textwidth}
        \centering
        \includegraphics[width=\linewidth]{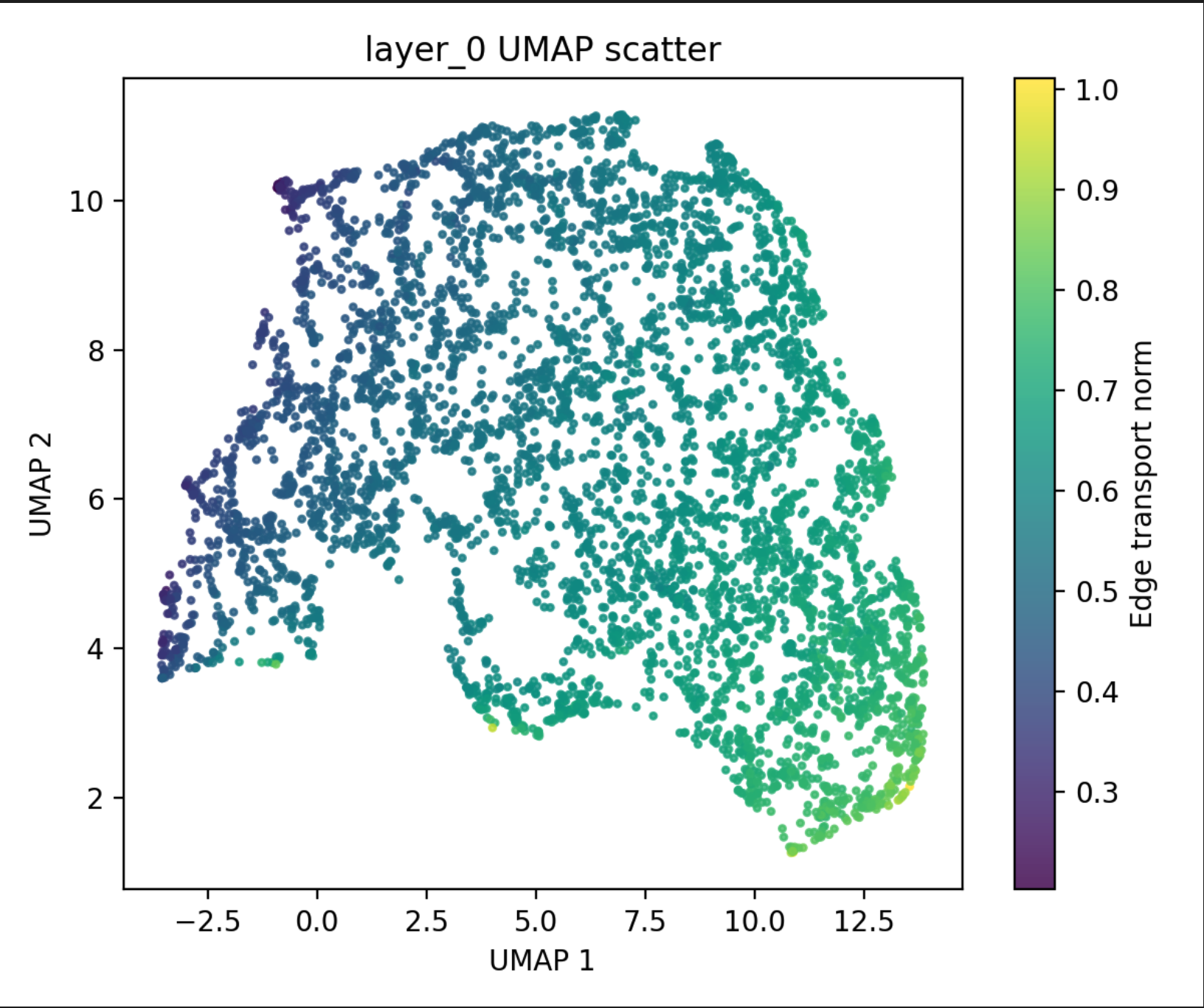}
        \caption{Layer 0 UMAP}
    \end{subfigure}
    \hfill
    \begin{subfigure}[t]{0.32\textwidth}
        \centering
        \includegraphics[width=\linewidth]{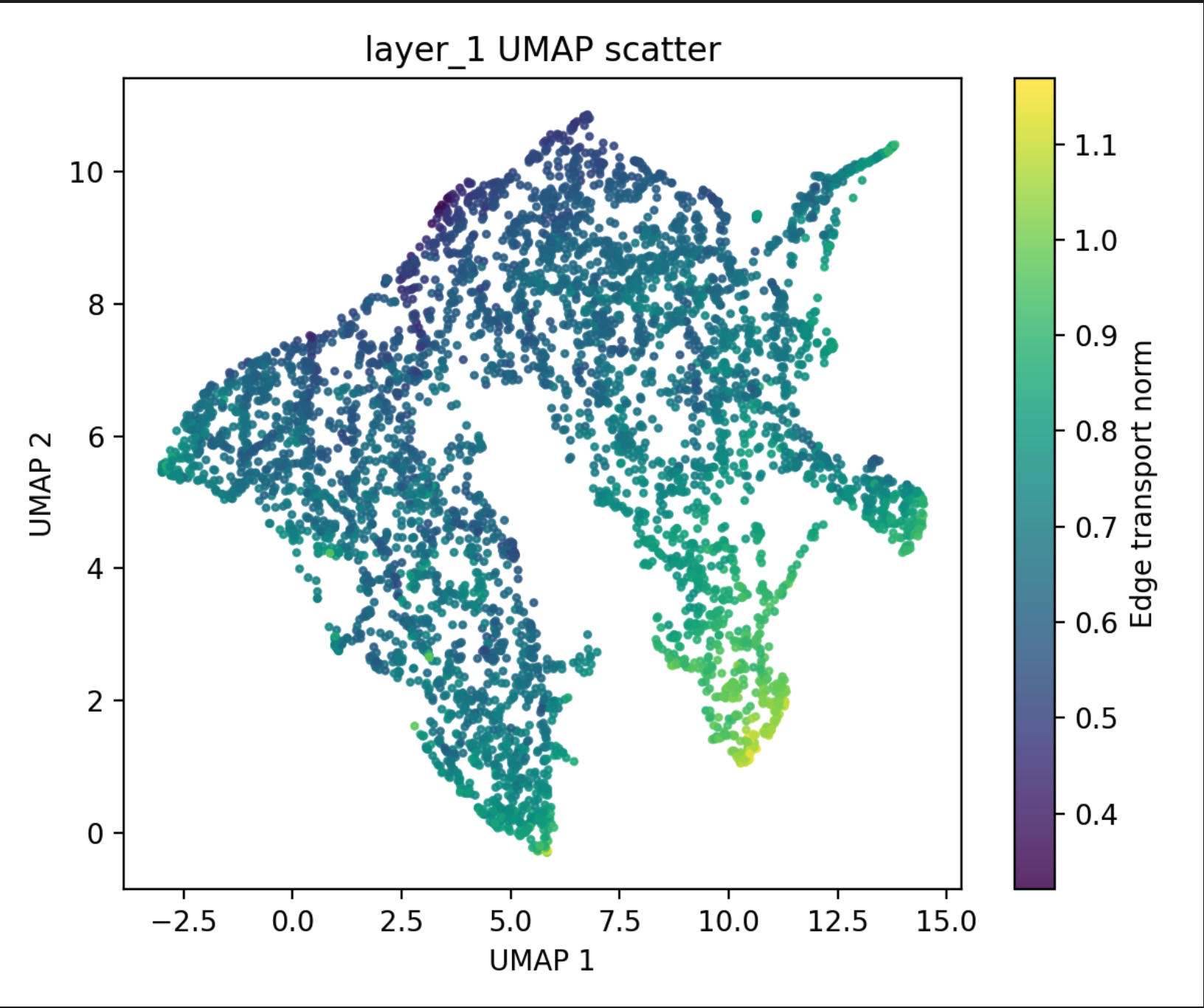}
        \caption{Layer 1 UMAP}
    \end{subfigure}
    \hfill
    \begin{subfigure}[t]{0.32\textwidth}
        \centering
        \includegraphics[width=\linewidth]{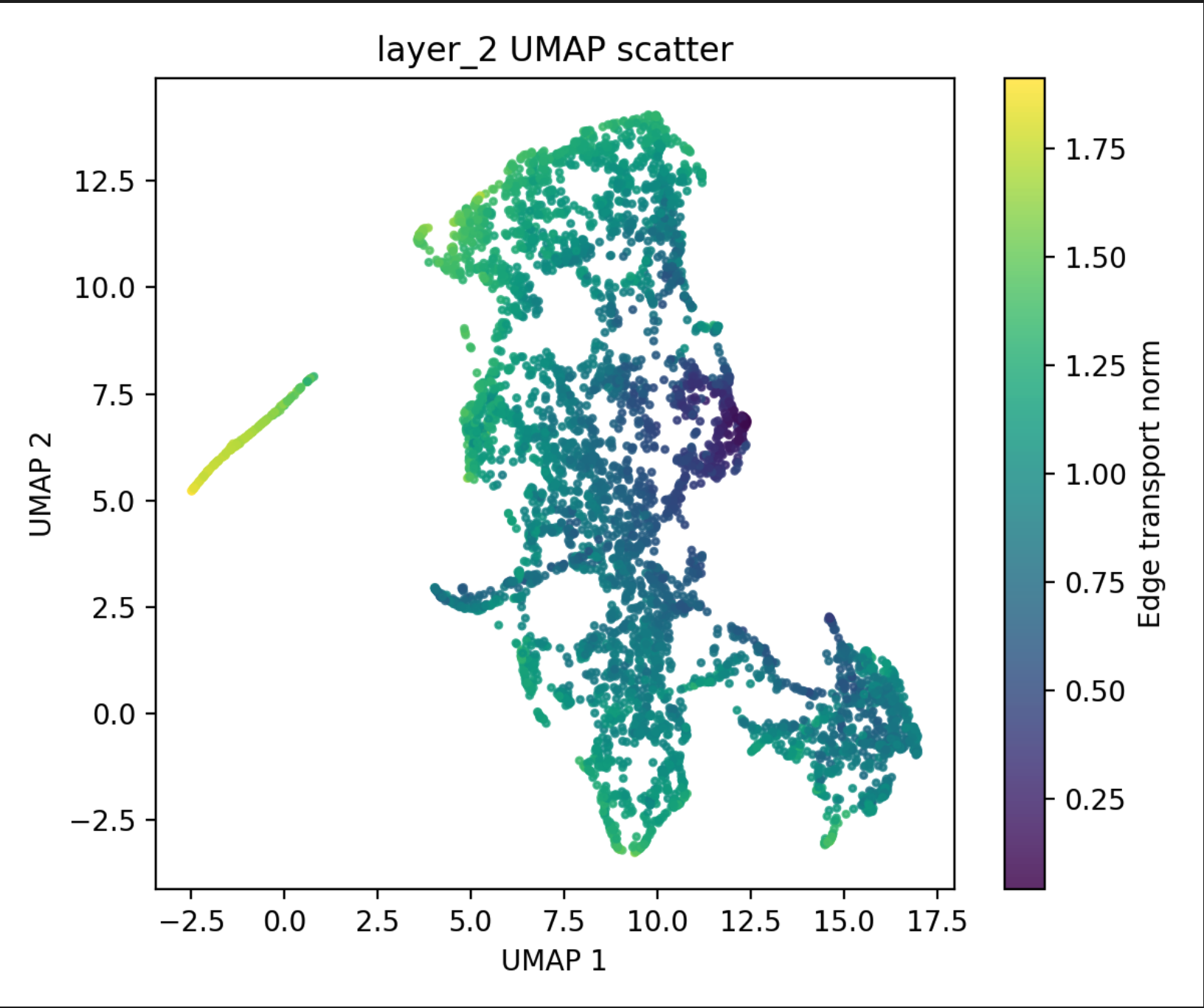}
        \caption{Layer 2 UMAP}
    \end{subfigure}

    \vspace{0.5em}
    \centering
    \begin{subfigure}[t]{0.32\textwidth}
        \centering
        \includegraphics[width=\linewidth]{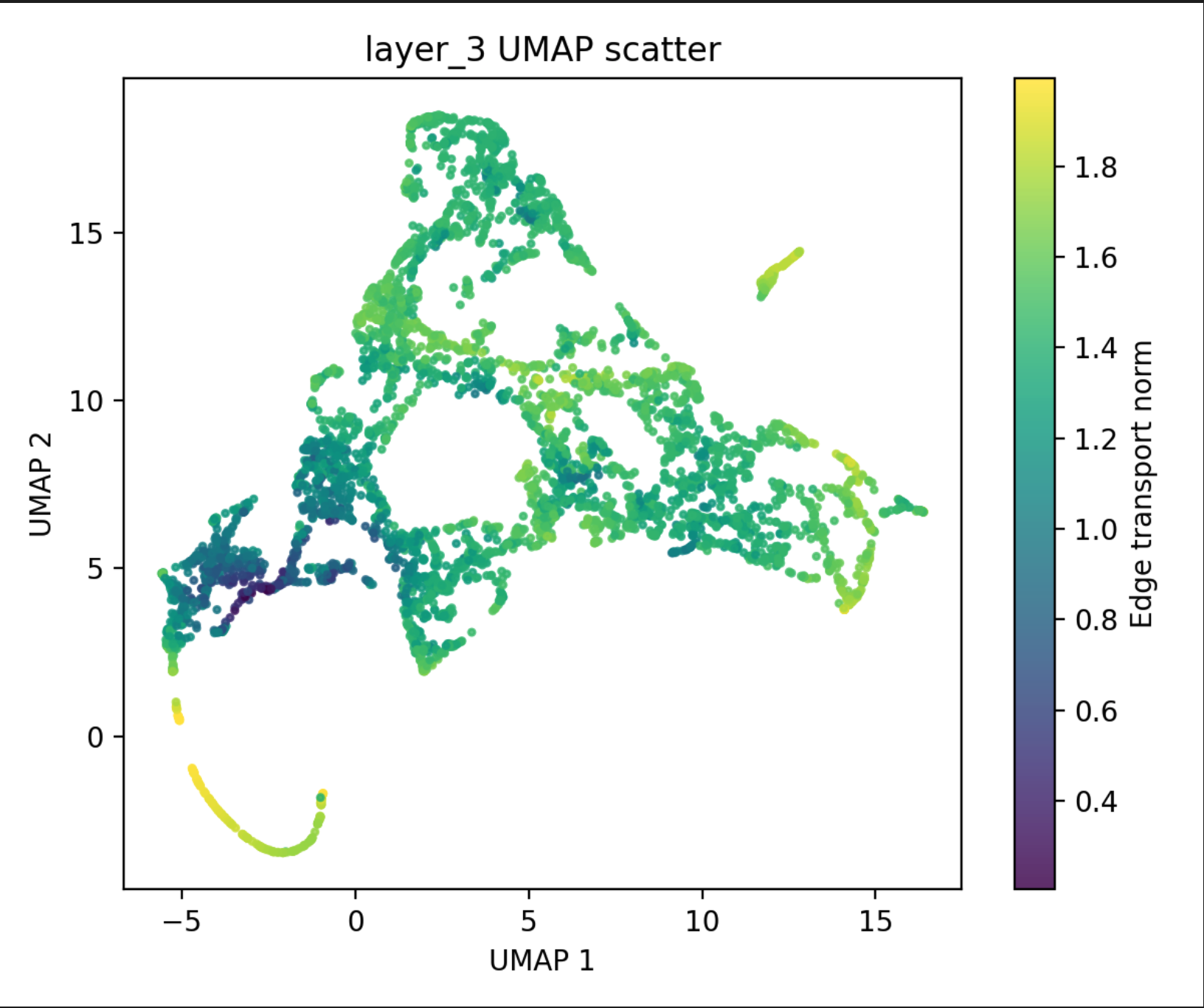}
        \caption{Layer 3 UMAP}
    \end{subfigure}
    \hfill
    \begin{subfigure}[t]{0.32\textwidth}
        \centering
        \includegraphics[width=\linewidth]{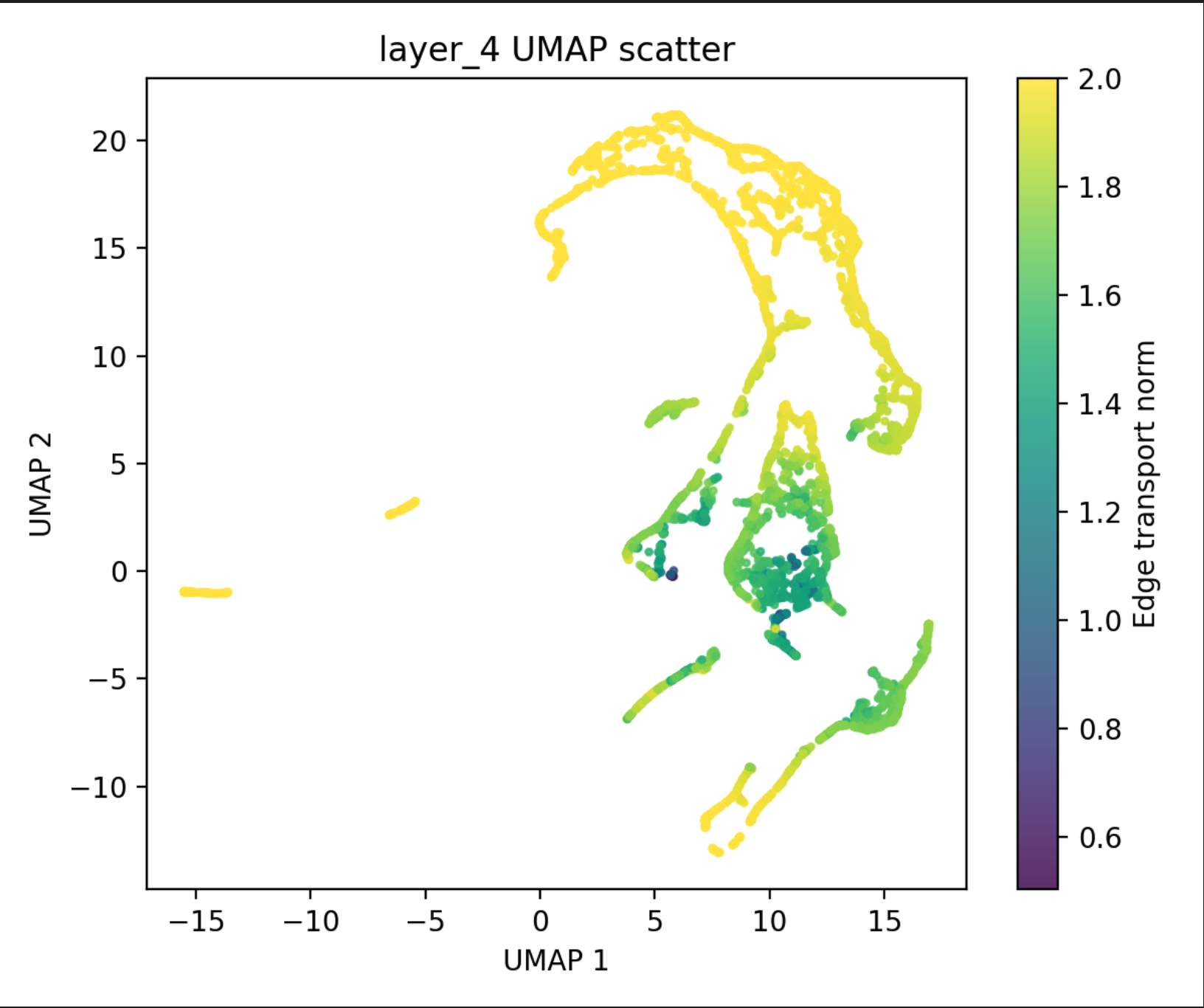}
        \caption{Layer 4 UMAP}
    \end{subfigure}

    \caption{\emph{UMAP projections of learned restriction maps across layers.}
    Each point corresponds to an edge-wise restriction map, projected to two
    dimensions with UMAP and coloured by the norm of the corresponding transport
    map. Deeper layers exhibit more separated and filament-like geometries,
    indicating increasing edge-transport specialisation.}
    \label{fig:restriction-map-umap}
\end{figure*}

\begin{figure*}[t]
    \centering
    \begin{subfigure}[t]{0.48\textwidth}
        \centering
        \includegraphics[width=\linewidth]{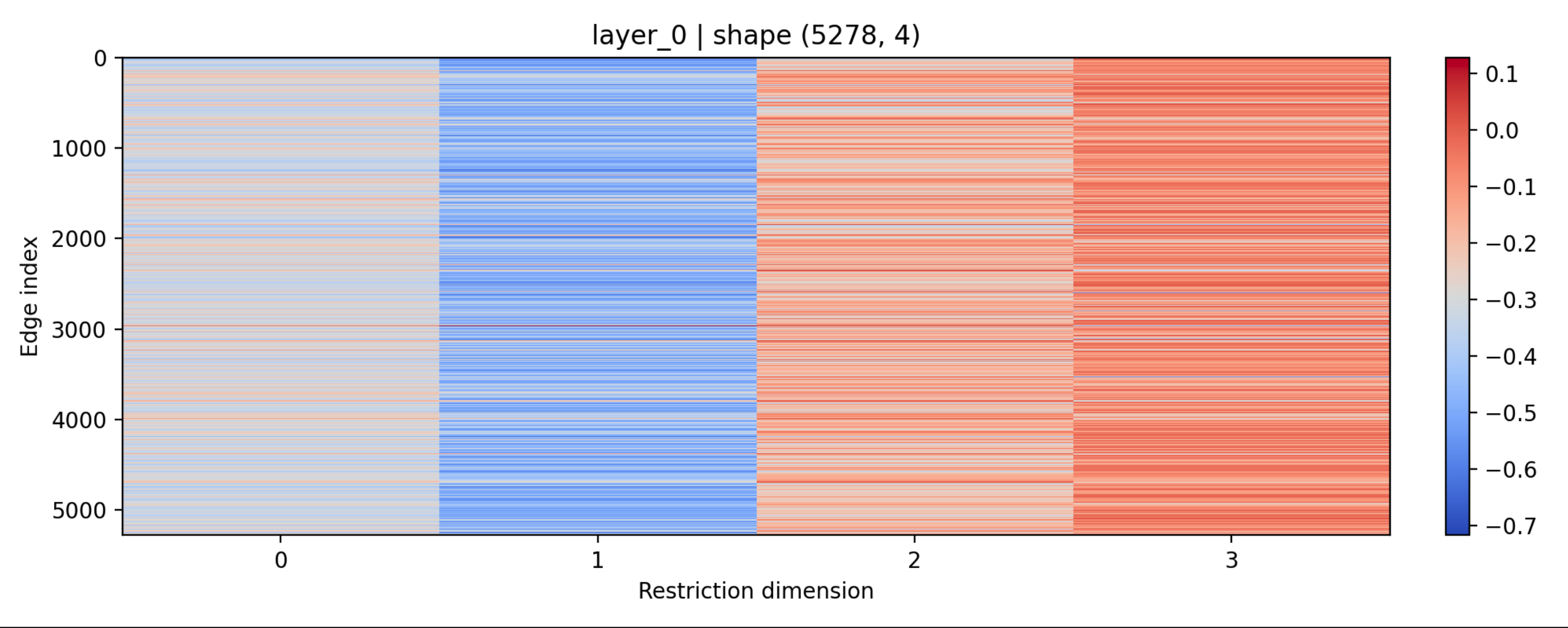}
        \caption{Layer 0}
    \end{subfigure}
    \hfill
    \begin{subfigure}[t]{0.48\textwidth}
        \centering
        \includegraphics[width=\linewidth]{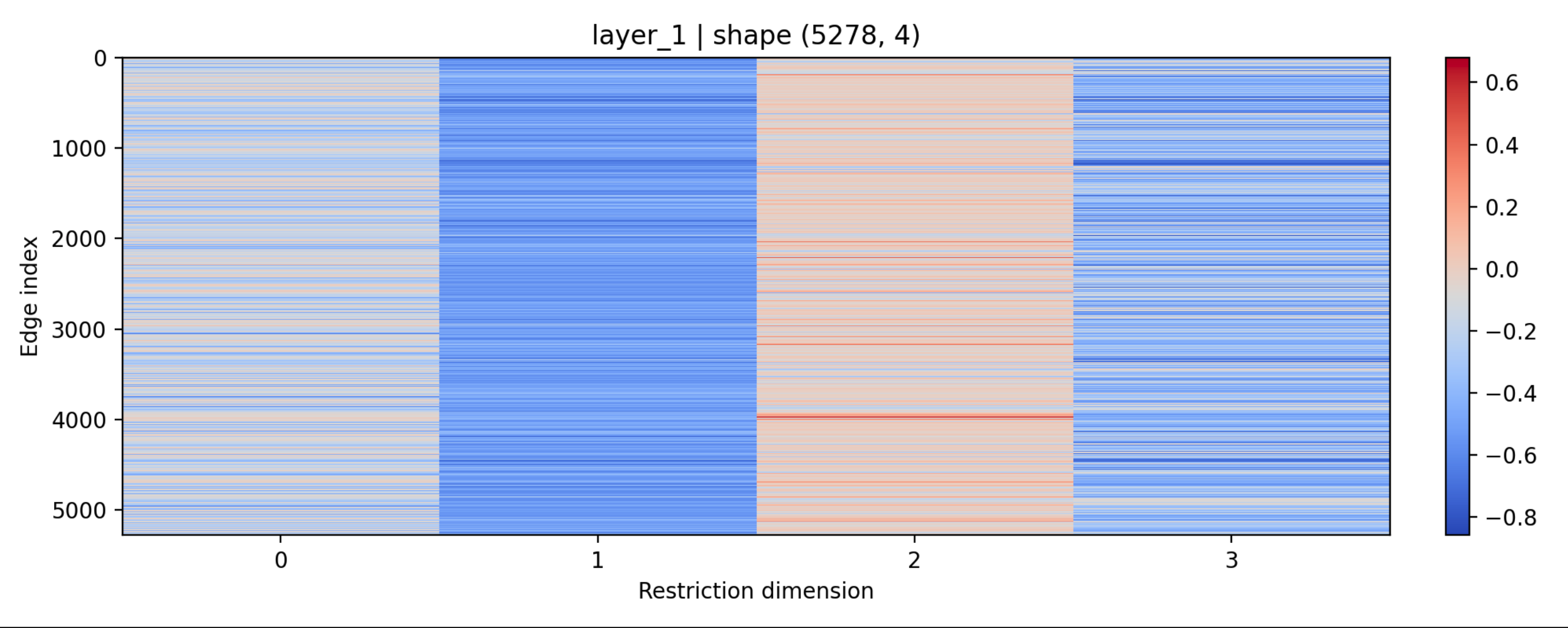}
        \caption{Layer 1}
    \end{subfigure}

    \vspace{0.5em}

    \begin{subfigure}[t]{0.48\textwidth}
        \centering
        \includegraphics[width=\linewidth]{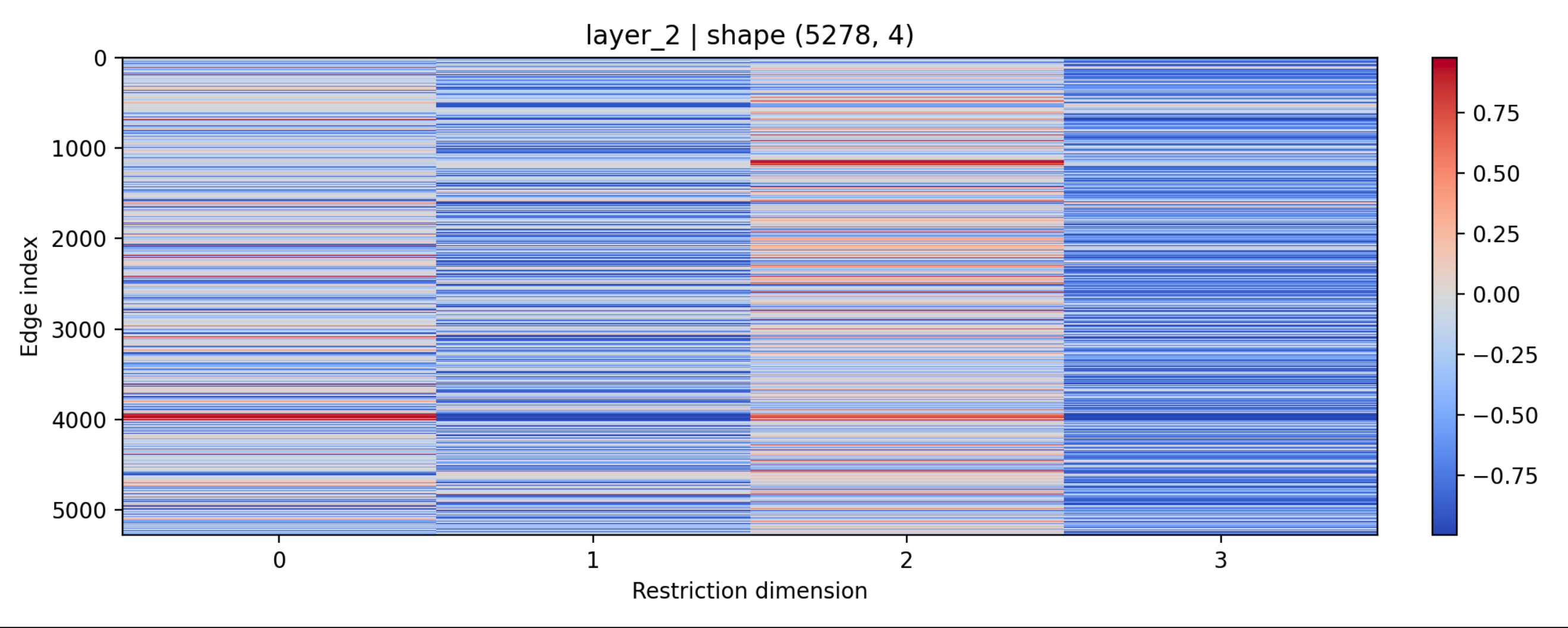}
        \caption{Layer 2}
    \end{subfigure}
    \hfill
    \begin{subfigure}[t]{0.48\textwidth}
        \centering
        \includegraphics[width=\linewidth]{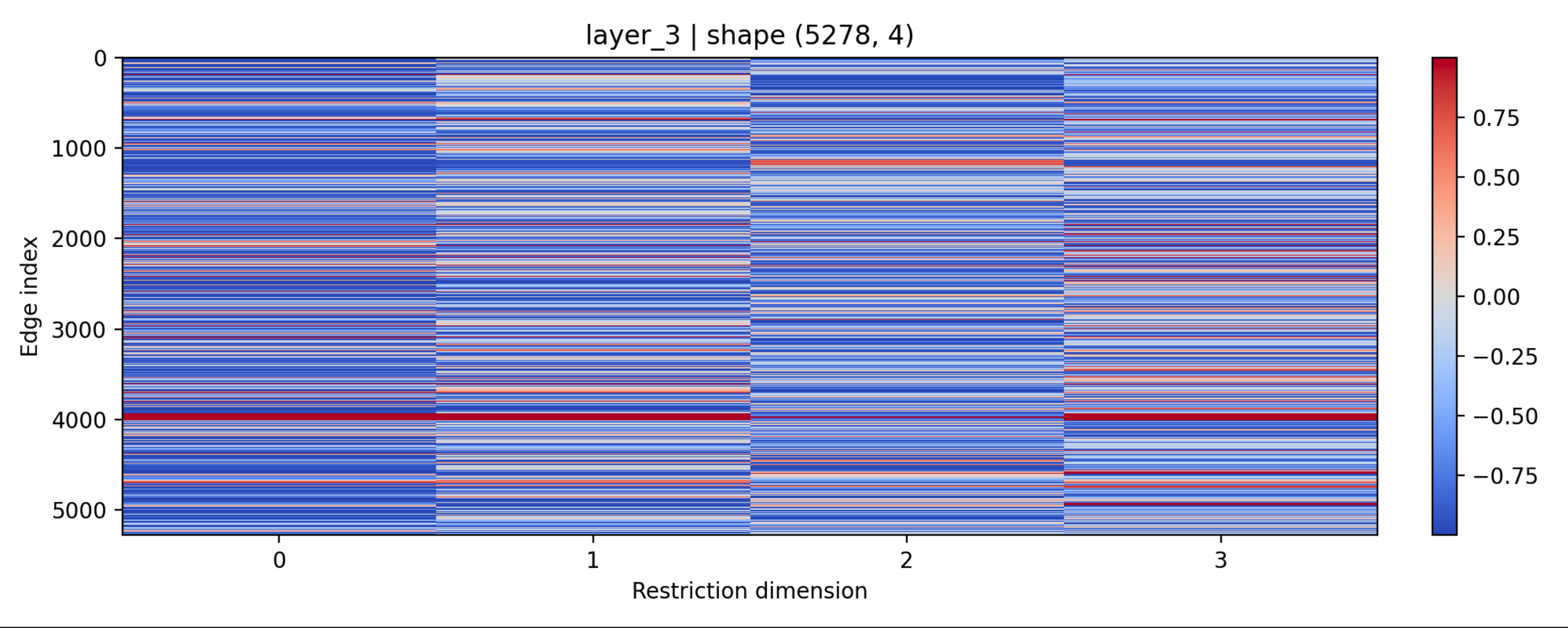}
        \caption{Layer 3}
    \end{subfigure}

    \vspace{0.5em}

    \begin{subfigure}[t]{0.48\textwidth}
        \centering
        \includegraphics[width=\linewidth]{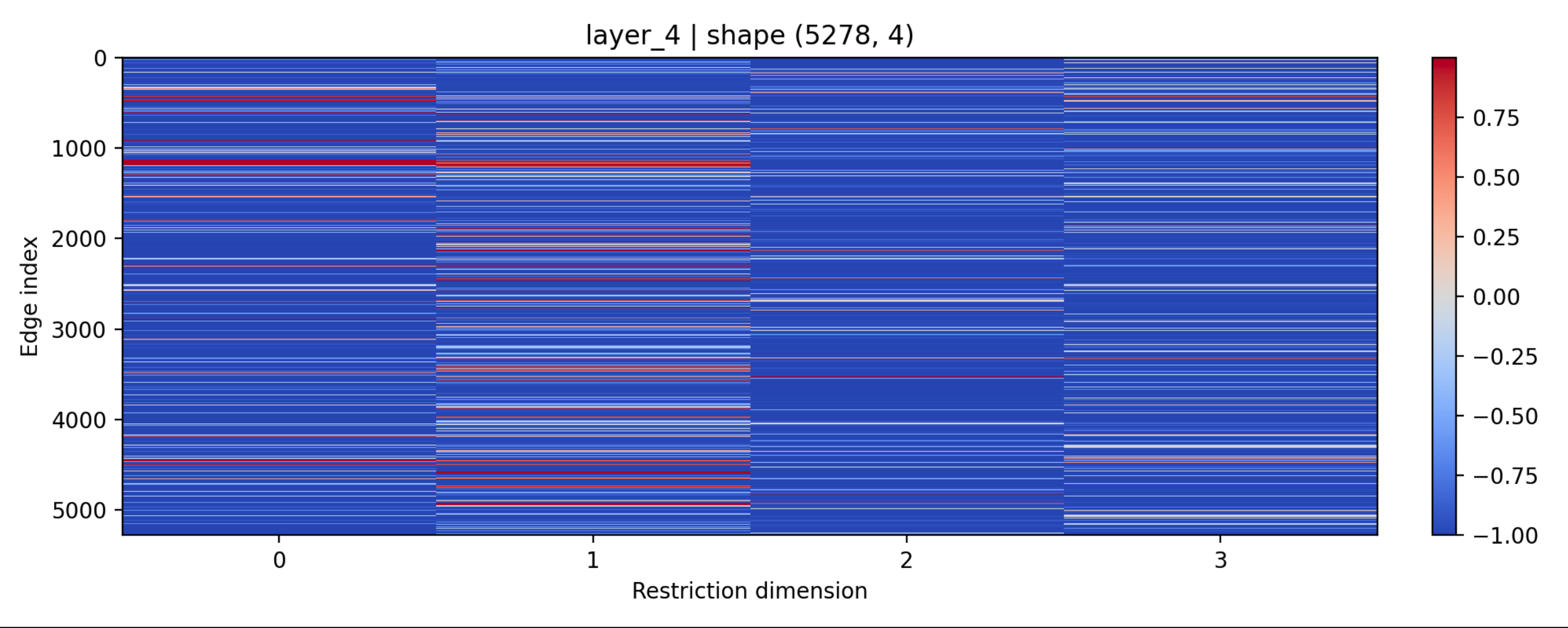}
        \caption{Layer 4}
    \end{subfigure}

    \caption{\emph{Heatmaps of learned restriction-map coordinates across layers.}
    Rows correspond to edges and columns to restriction dimensions. The heatmaps
    show increasing coordinate-wise specialisation with depth, with later layers
    exhibiting stronger high-contrast activation and suppression patterns.}
    \label{fig:restriction-map-heatmaps}
\end{figure*}

\subsection{Federated Causal Sheaf Learning on \textsc{Roman-Empire}}
\label{app:federated-causal-sheaf}

We further evaluate PolyNSD in a federated graph learning setting based on the FedATH causal federated training pipeline~\cite{fu2025less}. The goal of this experiment is to test whether polynomial sheaf-based propagation remains effective when the graph is distributed across multiple clients, and the model must be trained through local updates followed by server-side aggregation. We focus on the \textsc{Roman-Empire} dataset from the new heterophily benchmark suite\cite{platonov2023critical}, since it is a large, sparse, sequence-like graph with long-range syntactic dependencies and low homophily. This makes it a natural stress test for federated heterophilic graph learning.

\paragraph{Federated Setting, Causal and Biased branches.}
We partition the dataset into $K=10$ clients. Each client owns a local subgraph, with its own node features, labels, train/test indices, and edge structure. At each communication round, the current global model is broadcast to all clients. Each client then performs local training on its subgraph for a fixed number of epochs. After local training, the server aggregates the client models by weighted averaging, where the weight of each client is proportional to the number of nodes in its local subgraph. The resulting server model is then used as the global model for the next communication round. Following FedATH~\cite{fu2025less}, each client maintains two predictive branches: a causal branch and a biased branch. The causal branch is trained to predict labels from the edge subset considered causally useful, while the biased branch is trained on the complementary signal. The two branches are coupled by a dependence penalty, encouraging separation between causal and biased representations. The causal branch minimises a supervised cross-entropy loss plus a dependence regularisation term, while the biased branch minimises a negative-entropy objective plus the complementary dependence regularisation. As concerns backbones models for FedATH, we compare PolyNSD and NSD-based models with conventional graph backbones. The code implements these terms through \texttt{loss\_dependence} and \texttt{negative\_entropy}, and logs the cross-entropy, entropy, and dependence components during training. 

\paragraph{RGSheafMask: Restriction-map guided Masking Variant.}
To adapt FedATH to sheaf neural networks, we use learned sheaf restriction maps as edge-level structural features for the causal mask, creating the \emph{RGSheafMask} variant. In the \textsc{LoadRMaps} setting, restriction maps are first loaded from client-level checkpoints. A small mask network then maps the restriction maps to edge scores, which are expanded to the directed edge list and used to weight the causal branch. The biased branch receives the complementary mask $1-s_e$. This allows the causal/bias decomposition to depend not only on raw graph connectivity, but also on the learned transport geometry of the sheaf. 

\paragraph{FullSheafMask: End-to-end Sheaf-mask Variant.}
We also implement an end-to-end variant in which each client maintains three
local branches: a sheaf branch, a causal branch, and a biased branch. The sheaf branch is trained with supervised cross-entropy and produces restriction maps during training. These maps are flattened into edge-level features and passed to the mask network, which generates the causal edge scores used by the causal branch. The total local objective combines the supervised sheaf loss, the causal loss, and the biased loss. In this \emph{FullSheafMask}version, the restriction maps are therefore not only precomputed artifacts, but part of the local learning process. The server still aggregates the causal branch by weighted model averaging. 

\paragraph{Findings: Polynomial sheaf backbones strongly improve federated heterophilic learning.}
The results are reported in \Cref{tab:fedath_roman_empire_k10}. Standard
federated learning baselines remain close to chance-level performance for this setting: FedAvg, FedProx, MOON, FedOPT, and FedProto all obtain approximately $34$--$35\%$ accuracy. Federated graph learning baselines improve on this range, with FedSage+ reaching $41.59\%$ and the strongest conventional FedATH backbone, FedATH-GCN, reaching $48.18\%$. However, replacing conventional graph backbones with sheaf neural backbones yields a substantial improvement. In the \textsc{RGSheafMask} setting, FedCausalSheaf-General reaches $72.31\%$, already far above FedATH-GCN, and introducing polynomial sheaf propagation further raises the best configuration to $76.19\%$ with FedCausalSheaf-GeneralPoly. The strongest results are obtained in the \textsc{FullSheafMask} setting, where restriction maps are explicitly reused in the causal masking pipeline. Even the non-sheaf LoadRMaps-FedATH-GCN baseline improves to $60.63\%$, suggesting that
restriction-map-informed masking provides useful edge-level structure. When this masking mechanism is combined with sheaf backbones, performance increases further: LoadRMaps-FedCausalSheaf-General obtains $76.51\%$, while
LoadRMaps-FedCausalSheaf-BundlePoly reaches $78.00\%$. The best overall model is LoadRMaps-FedCausalSheaf-GeneralPoly, which achieves $80.24\%$ accuracy. These results indicate that learned sheaf restriction maps provide informative causal edge-level signals for federated graph learning, and that higher-order polynomial sheaf propagation further improves the ability of both local and global models to exploit long-range heterophilic dependencies.

\input{tables/federated_learning}

\subsection{Synthetic Benchmarks: Heterophily, Scalability, and Noise}
\label{app:synthetic}

In this subsection, we detail synthetic experiments that
underpin the synthetic stress tests discussed in \autoref{subsec:hetero-real-main}. All experiments use the synthetic data-generation procedure described in \autoref{app:synthetic-datasets}.

\emph{Heterophily Sweeps.}
\label{app:heterophily-synthetic}
To probe performance under controlled heterophily, we vary the heterophily
coefficient $ het \in \{0, 0.25, 0.5, 0.75, 1\} $, while keeping all other graph and feature parameters fixed. We consider both synthetic regimes introduced in \cite{caralt2024joint}: (i) \textsc{RiSNN}-style: higher-dimensional features and a configuration mirroring the RiSNN experiments, and (ii) \textsc{Diff}-style: lower-dimensional features and a configuration mirroring the diffusion-style experiments. For each regime we evaluate multiple class counts and average results over $5$ synthetic graph realisations per configuration. We compare PolyNSD variants against MLP, GCN, VanillaSheaf, simple sheaf baselines, and the RiSNN/JdSNN architectures. \autoref{fig:hetero-synth} summarises the results as a $4\times 2$ grid (rows correspond to different numbers of classes, columns to the two synthetic regimes). Within each subplot, we plot test accuracy vs.~$het$ for all models. As heterophily increases, homophily-biased message passing (e.g., GCN) rapidly degrades towards MLP performance, confirming that edges become uninformative or adversarial. Sheaf-based models are more robust, and PolyNSD variants consistently occupy the top of the accuracy curves across all $het$ values, especially in the bundle and general transport classes. This confirm the real-world findings in \autoref{subsec:hetero-real-main} under controlled conditions.

\begin{figure}[htbp]
  \centering
  \includegraphics[width=0.9\textwidth,height=\textheight,keepaspectratio]{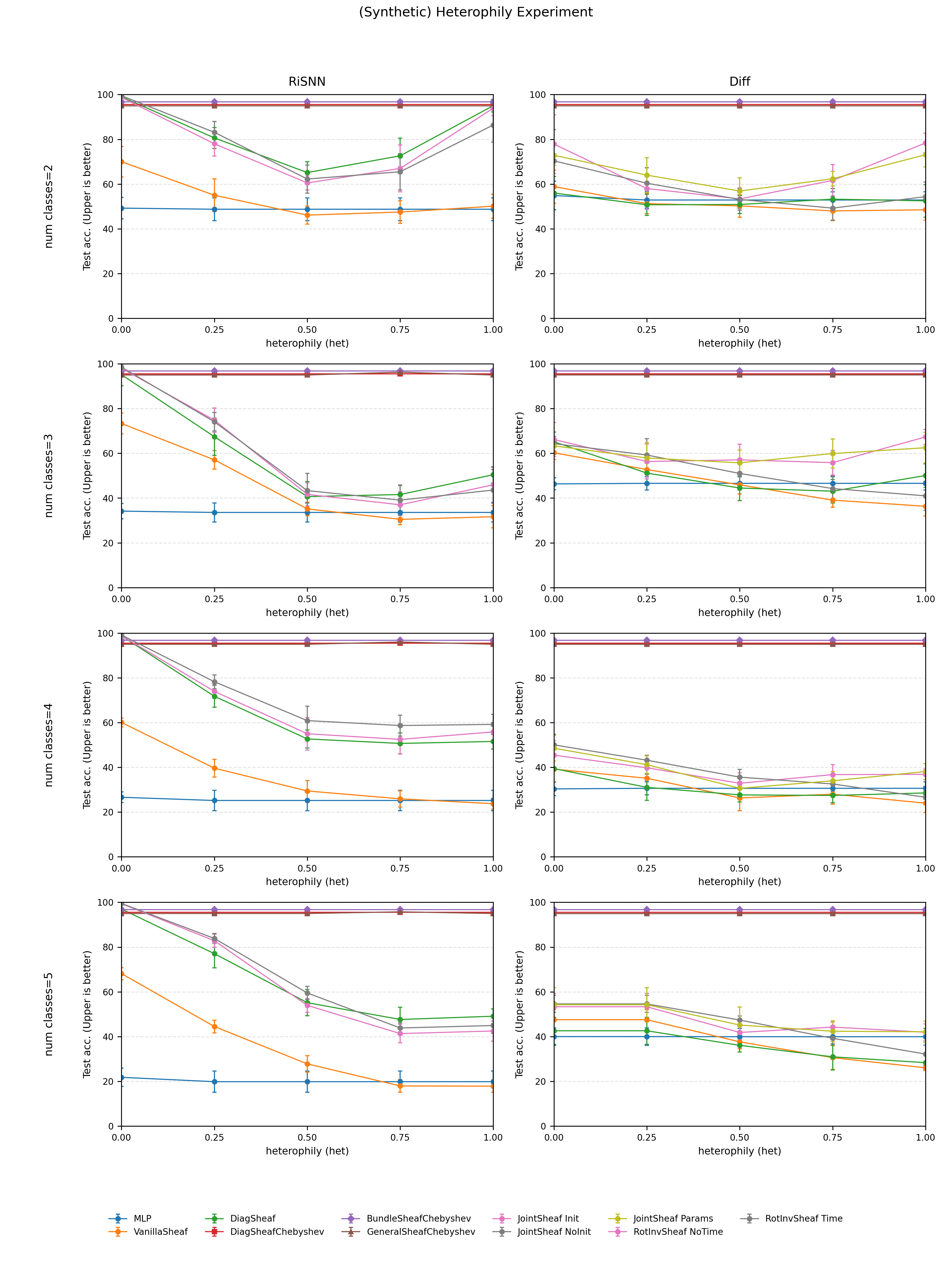}
  \caption{\textit{Synthetic heterophily sweeps.}
  Each row corresponds to a different number of classes; columns distinguish
  the \textsc{RiSNN} and \textsc{Diff} regimes.
  We sweep $het\in\{0,0.25,0.5,0.75,1.0\}$; error bars show mean$\pm$std over
  multiple random graph realisations.}
  \label{fig:hetero-synth}
\end{figure}

\emph{Data Scalability.}
To study how PolyNSD scales with graph size and degree, we jointly vary the number of nodes $N \in \{100, 500, 1000\}$ and the base degree $K \in \{2, 6, 10\} $, while keeping heterophily fixed at a high value ($het = 0.9$). This setting acts as a stress test where edges are mostly cross-class and graphs become denser as $K$ grows. For each $(N,K)$ pair we generate a new synthetic graph in both the \textsc{RiSNN} and \textsc{Diff} regimes and evaluate all models using the same training protocol as before. \autoref{fig:scalability-synth-fig} depicts the results as a grid with rows indexed by $N$ and columns by $K$. Each subplot reports test accuracy for PolyNSD variants and baselines. Across both regimes, PolyNSD variants maintain near-saturated performance (often close to $98\%$) across all scales, whereas baseline methods improve more slowly or plateau at lower accuracies as $N$ and $K$ increase, indicating that PolyNSD scales favourably with both the number of nodes and edge density, and that its polynomial filters remain effective even as graphs become larger and more connected.

\begin{figure}[H]
  \centering
  \includegraphics[width=0.6\textwidth]{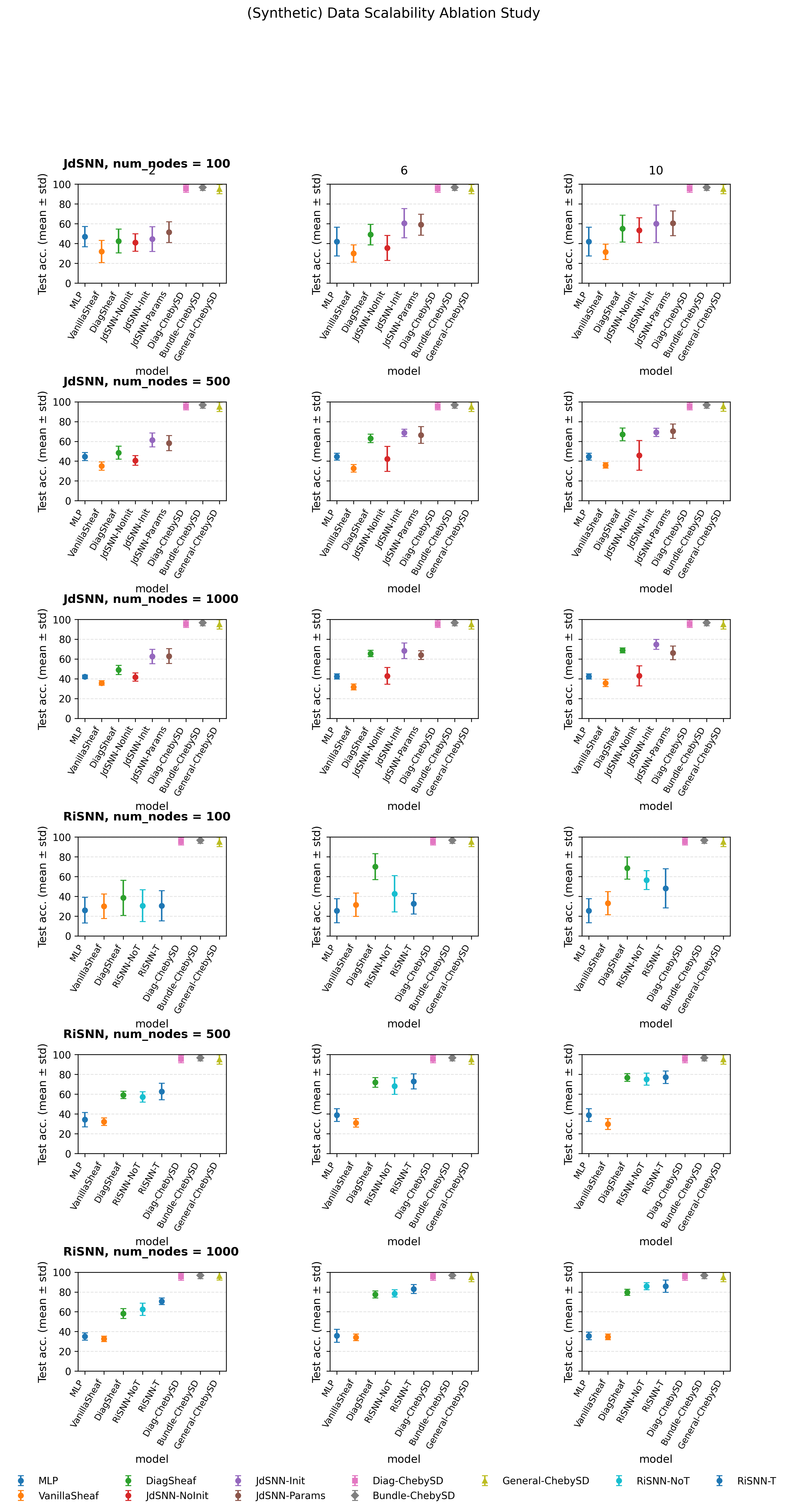}
  \caption{\textit{Data scalability ablation.}
  Rows correspond to increasing number of nodes $N\in\{100,500,1000\}$;
  columns to degree $K\in\{2,6,10\}$ at fixed high heterophily $het=0.9$.
  Top block: \textsc{Diff} setup; bottom block: \textsc{RiSNN} setup.
  PolyNSD maintains near-saturated performance across scales, while baselines
  plateau at lower accuracies.}
  \label{fig:scalability-synth-fig}
\end{figure}

\emph{Effect of Feature Noise.}
Finally, we examine robustness to feature corruption by injecting i.i.d.\ Gaussian noise into node features while keeping the underlying graphs maximally heterophilous ($het = 1$). This setting isolates the effect of covariate noise from that of connectivity. We consider two noise sweeps: for the \textsc{Diff}-style setup, we vary $\texttt{feat\_noise}\in\{0.0, 0.2, 0.4, 0.6, 0.8, 1.0\}$, while for the \textsc{RiSNN}-style setup, we use finer-grained noise levels $\texttt{feat\_noise}\in\{0.00, 0.05, 0.10, 0.15, 0.20, 0.25\}$. For each noise level we generate multiple synthetic instances and average test accuracy over these realisations. \autoref{fig:effect-noise-synth} summarises the results. As noise increases, GCN and other homophily-based message-passing models degrade rapidly, eventually approaching the performance of an MLP that ignores the graph. Sheaf-based models are more robust, and PolyNSD variants are consistently among the best-performing methods across all noise levels. The bundle and general transport classes, combined with spectral control, show the strongest robustness, retaining significant accuracy even at the highest noise levels. These findings support the interpretation of PolyNSD as a \emph{structure-aware denoiser}: the sheaf transports align features in local fibres before comparison, while the spectral polynomial can attenuate high-frequency noise modes and preserve informative low- and mid-frequency components.

\begin{figure}[htbp]
  \centering
  \includegraphics[width=0.7\textwidth]{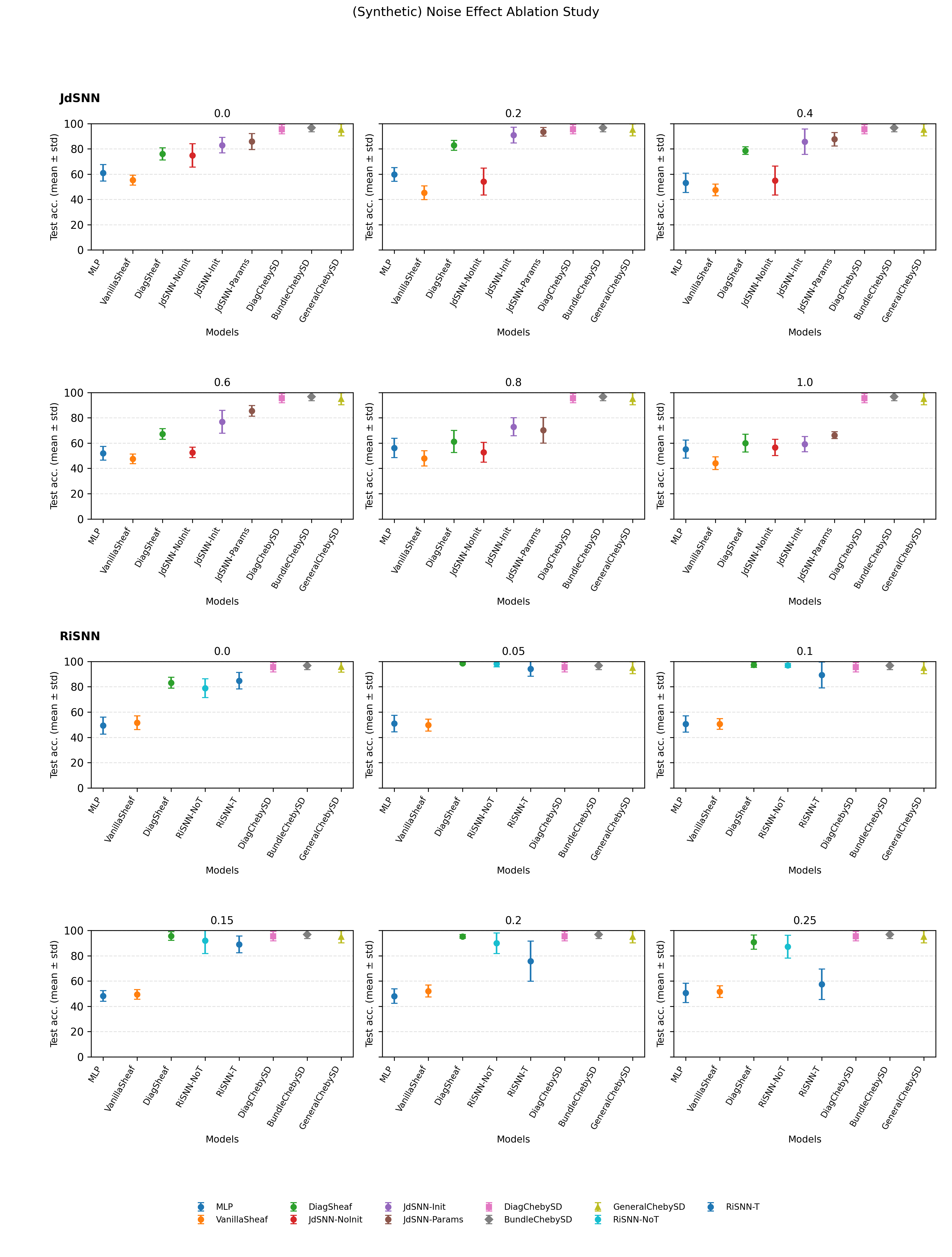}
  \caption{\textit{Effect of feature noise on synthetic tasks.}
  Top two rows: \textsc{Diff}-style setup with noise levels
  $\{0.0, 0.2, 0.4, 0.6, 0.8, 1.0\}$.
  Bottom two rows: \textsc{RiSNN}-style setup with noise levels
  $\{0.00, 0.05, 0.10, 0.15, 0.20, 0.25\}$.
  PolyNSD variants remain among the most robust models, degrading more slowly
  and exhibiting smaller variance than baselines.}
  \label{fig:effect-noise-synth}
\end{figure}

%% file: tables/dataset_stats.tex

\begin{table*}[t]
  \centering

  \begin{minipage}[t]{0.44\textwidth}
    \centering
    \caption{Real-World dataset statistics.}
    \label{tab:dataset-stats}
    \sffamily
    \vspace{0pt}
    \begin{adjustbox}{width=\linewidth}
      {\small
      \setlength{\tabcolsep}{6pt}
      \renewcommand{\arraystretch}{1.05}
      \begin{tabular}{lcccc}
        \toprule
        Dataset & Homophily $h$ & \#Nodes & \#Edges & \#Classes \\
        \midrule
        Texas     & 0.11 & 183    & 295     & 5 \\
        Wisconsin & 0.21 & 251    & 466     & 5 \\
        Film      & 0.22 & 7,600  & 26,752  & 5 \\
        Squirrel  & 0.22 & 5,201  & 198,493 & 5 \\
        Chameleon & 0.23 & 2,277  & 31,421  & 5 \\
        Cornell   & 0.30 & 183    & 280     & 5 \\
        Citeseer  & 0.74 & 3,327  & 4,676   & 7 \\
        Pubmed    & 0.80 & 18,717 & 44,327  & 3 \\
        Cora      & 0.81 & 2,708  & 5,278   & 6 \\
        \bottomrule
      \end{tabular}}
    \end{adjustbox}
  \end{minipage}
  \hfill
  \begin{minipage}[t]{0.53\textwidth}
    \centering
    \caption{Hyper-parameter search space.}
    \label{tab:hyper-table}
    \sffamily
    \vspace{0pt}
    \begin{adjustbox}{width=\linewidth}
      \begin{tabular}{@{} ll @{}}
        \toprule
        \textbf{Hyper-parameter} & \textbf{Searchable values / notes} \\
        \midrule
        Hidden channels & \(\{8,16,32\}\) (WebKB) and \(\{8,16,32,64\}\) (others) \\
        Stalk width \(d\) & \(\{1,2,\dots,5\}\) \\
        Layers & \(\{1,2,\dots,8\}\) \\
        Poly degree \(K\) & \(\{2,3,4,5,8,12,16\}\) (for PolySD / Chebyshev) \\
        Learning rate & \(0.02\) (WebKB) and \(0.01\) (others) \\
        Activations & ELU \\
        Weight decay (model) & Log-uniform (exponent range e.g. \([-4.5,11.0]\)) \\
        Sheaf decay & Log-uniform (exponent range e.g. \([-4.5,11.0]\)) \\
        Input dropout & categorical \(\{0.0,0.1,\dots,0.9\}\) \\
        Layer dropout & Uniform \([0.0,0.9]\) \\
        Patience (epochs) & 100 (Wiki) and 200 (others) \\
        Max training epochs & 1000 (Wiki) and 500 (others) \\
        Optimiser & Adam (\(\beta_1=0.9,\ \beta_2=0.999\)) \\
        Misc model flags & left/right weights, normalised, use\_act, edge\_weights, etc. \\
        \bottomrule
      \end{tabular}
    \end{adjustbox}
  \end{minipage}

\end{table*}

%% file: tables/new_datasets_stats.tex
\begin{table}[t]
    \centering
    \scriptsize
    \caption{Statistics of the new heterophilous datasets.}
    \label{tab:new_heterophilous_dataset_stats}
    \resizebox{\textwidth}{!}{%
    \begin{tabular}{lccccc}
        \toprule
        & \textsc{Roman-Empire} & \textsc{Amazon-Ratings} & \textsc{Minesweeper} & \textsc{Tolokers} & \textsc{Questions} \\
        \midrule
        nodes                   & $22662$ & $24492$ & $10000$ & $11758$  & $48921$ \\
        edges                   & $32927$ & $93050$ & $39402$ & $519000$ & $153540$ \\
        avg degree              & $2.91$  & $7.60$  & $7.88$  & $88.28$  & $6.28$ \\
        global clustering       & $0.29$  & $0.32$  & $0.43$  & $0.23$   & $0.02$ \\
        avg local clustering    & $0.39$  & $0.58$  & $0.44$  & $0.53$   & $0.03$ \\
        diameter                & $6824$  & $46$    & $99$    & $11$     & $16$ \\
        node features           & $300$   & $300$   & $7$     & $10$     & $301$ \\
        classes                 & $18$    & $5$     & $2$     & $2$      & $2$ \\
        edge homophily          & $0.05$  & $0.38$  & $0.68$  & $0.59$   & $0.84$ \\
        adjusted homophily      & $-0.05$ & $0.14$  & $0.01$  & $0.09$   & $0.02$ \\
        label informativeness   & $0.11$  & $0.04$  & $0.00$  & $0.01$   & $0.00$ \\
        \bottomrule
    \end{tabular}%
    }
\end{table}

%% file: tables/hyperparams_synthetic.tex
\begin{table}[ht]
  \centering
  \caption{Synthetic experiments common training hyper-parameters.}
  \label{tab:synthetic-hparams-common}
  \sffamily
  \begin{adjustbox}{width=0.7\linewidth}
  \begin{tabular}{@{} ll @{}}
    \toprule
    \textbf{Hyper-parameter} & \textbf{Value / note} \\
    \midrule
    Models & DiagSheafChebyshev,\; BundleSheafChebyshev,\; GeneralSheafChebyshev \\
    Optimiser & Adam (learning rate $0.01$) \\
    Weight decay & $5\times10^{-4}$ \\
    Sheaf decay & $5\times10^{-4}$ \\
    Max epochs / Patience & $1500$ / $200$ \\
    Seed & $43$ \\
    Device & single NVIDIA A100–SXM4–80GB \\
    Dataset flag & \texttt{dataset} $=$ \texttt{synthetic\_exp}, \; \texttt{ellipsoids} $=$ \texttt{false}, \; \texttt{edge\_noise} $=0.0$ \\
    \bottomrule
  \end{tabular}
  \end{adjustbox}
\end{table}

\begin{table}[ht]
  \centering
  \caption{Synthetic experiments panel–specific grids.}
  \label{tab:synthetic-hparams-panels}
  \sffamily
  \begin{adjustbox}{width=\textwidth}
  \begin{tabular}{@{} llllllll @{}}
    \toprule
    \textbf{Family (Fig)} & \textbf{Column} & \textbf{\#Classes} & \textbf{\#Feats} & \textbf{\#Nodes} & \textbf{Degree} & \textbf{Feat noise} & \textbf{Heterophily ($het$)} \\
    \midrule
    Heterophily & Diff   & $\{2,3,4,5\}$ & $4$   & $\{400,600,800,1000\}$ & $4$  & $0.0$ & $\{0.0,0.25,0.50,0.75,1.0\}$ \\
    Heterophily  & RiSNN  & $\{2,3,4,5\}$ & $15$  & $\{400,600,800,1000\}$ & $4$  & $0.0$ & $\{0.0,0.25,0.50,0.75,1.0\}$ \\
    \midrule
    Feature Noise  & Diff  & $2$           & $3$   & $400$                   & $2$  & $\{0.0,0.2,0.4,0.6,0.8,1.0\}$ & $1.0$ \\
    Feature Noise  & RiSNN & $2$           & $15$  & $500$                   & $5$  & $\{0.0,0.05,0.1,0.15,0.2,0.25,0.30\}$ & $1.0$ \\
    \midrule
    Amount of Data  & Diff & $3$           & $3$   & $\{100,500,1000\}$      & $\{2,6,10\}$ & $0.0$ & $0.9$ \\
    Amount of Data  & RiSNN& $3$           & $15$  & $\{100,500,1000\}$      & $\{2,6,10\}$ & $0.0$ & $0.9$ \\
    \bottomrule
  \end{tabular}
  \end{adjustbox}
\end{table}

%% file: tables/oversmoothing_polynsd.tex
\begin{table*}[t]
\centering
\scriptsize
\setlength{\tabcolsep}{2.5pt}
\caption{\textit{Depth ablation study: studying oversmoothing presence.} We report accuracy $\pm$ stdev for real-world datasets used to study the
oversmoothing effect. The “best” result corresponds to the number of layers
with the highest accuracy. “OOM” denotes out-of-memory and “INS” numerical
instability. The top three models for each dataset and layer are coloured by
\textcolor{Top1}{\textbf{First}}, \textcolor{Top2}{\textbf{Second}} and
\textcolor{Top3}{\textbf{Third}}, respectively.}
\label{tab:oversmoothing}

\begin{subtable}[t]{0.48\textwidth}
\centering
\begin{adjustbox}{width=\linewidth}
\begin{tabular}{lcccccc}
\toprule
\textbf{Layers} & \textbf{2} & \textbf{4} & \textbf{8} & \textbf{16} & \textbf{32} & \textbf{Best} \\
\midrule
\multicolumn{7}{c}{\textbf{Cora ($h{=}0.81$)}}\\
\addlinespace
\textbf{DiagChebySD}    & \textcolor{Top1}{\textbf{\pmstd{88.67}{1.29}}} & \textcolor{Top1}{\textbf{\pmstd{88.23}{1.36}}} & \textcolor{Top1}{\textbf{\pmstd{88.11}{1.08}}} & \textcolor{Top2}{\textbf{\pmstd{87.73}{1.64}}} & \pmstd{87.22}{1.64} & 2  \\
\textbf{BundleChebySD}  & \textcolor{Top2}{\textbf{\pmstd{87.95}{1.31}}} & \pmstd{87.75}{1.49} & \textcolor{Top3}{\textbf{\pmstd{87.87}{1.90}}} & \textcolor{Top1}{\textbf{\pmstd{88.01}{1.46}}} & \pmstd{87.36}{1.49} & 16  \\
\textbf{GeneralChebySD} & \pmstd{87.44}{1.09} & \textcolor{Top3}{\textbf{\pmstd{87.77}{1.24}}} & \pmstd{87.48}{1.25} & \pmstd{86.92}{1.18} & \pmstd{86.90}{1.48} & 4  \\
SAN         & 86.90$\pm$1.31 & 86.84$\pm$0.97 & 86.68$\pm$1.13 & 86.54$\pm$0.89 & 86.62$\pm$1.39  & 2 \\
ANSD        & 86.98$\pm$1.07 & 87.08$\pm$1.26 & 86.80$\pm$1.15 & 86.84$\pm$1.24 & \textcolor{Top3}{\textbf{86.56$\pm$0.75}} & 4 \\
GGCN              & 87.00$\pm$1.15 & 87.48$\pm$1.32 & 87.63$\pm$1.33 & \textcolor{Top3}{\textbf{87.51$\pm$1.19}} & \textcolor{Top1}{\textbf{87.95$\pm$1.05}}  & 32 \\
GPRGNN            & \textcolor{Top3}{\textbf{87.93$\pm$1.11}} & \textcolor{Top2}{\textbf{87.95$\pm$1.18}} & \textcolor{Top2}{\textbf{87.87$\pm$1.41}} & 87.26$\pm$1.51 & 87.18$\pm$1.29  & 4 \\
H2GCN            & 87.87$\pm$1.20 & 86.10$\pm$1.51 & 86.18$\pm$2.10 & OOM & OOM  & 2 \\
GCNII            & 85.35$\pm$1.56 & 85.35$\pm$1.48 & 86.38$\pm$0.98 & 87.12$\pm$1.11 & \textcolor{Top2}{\textbf{87.95$\pm$1.23}}  & 64 \\
PairNorm          & 85.79$\pm$1.01 & 85.07$\pm$0.91 & 84.65$\pm$1.09 & 82.21$\pm$2.84 & 60.32$\pm$8.28  & 2 \\
Geom-GCN         & 85.35$\pm$1.57 & 21.01$\pm$2.61 & 13.98$\pm$1.48 & 13.98$\pm$1.48 & 13.98$\pm$1.48  & 2 \\
GCN               & 86.98$\pm$1.27 & 83.24$\pm$1.56 & 31.03$\pm$3.08 & 31.05$\pm$2.36 & 30.76$\pm$3.43  & 2 \\
GAT               & 87.30$\pm$1.10 & 86.50$\pm$1.20 & 84.97$\pm$1.24 & INS & INS & 2 \\
\bottomrule
\end{tabular}
\end{adjustbox}
\end{subtable}\hfill
\begin{subtable}[t]{0.48\textwidth}
\centering
\begin{adjustbox}{width=\linewidth}
\begin{tabular}{lcccccc}
\toprule
\textbf{Layers} & \textbf{2} & \textbf{4} & \textbf{8} & \textbf{16} & \textbf{32} & \textbf{Best} \\
\midrule
\multicolumn{7}{c}{\textbf{Citeseer ($h{=}0.74$)}}\\
\addlinespace
\textbf{DiagChebySD}    & \textcolor{Top3}{\textbf{\pmstd{77.19}{1.25}}} & \textcolor{Top2}{\textbf{\pmstd{77.12}{1.61}}} & \pmstd{76.54}{1.93} & \pmstd{76.71}{2.50} & \pmstd{76.34}{1.24} & 2  \\
\textbf{BundleChebySD}  & \pmstd{76.98}{1.76} & \textcolor{Top1}{\textbf{\pmstd{77.57}{1.55}}} & \textcolor{Top2}{\textbf{\pmstd{76.97}{1.90}}} & \textcolor{Top3}{\textbf{\pmstd{76.87}{1.51}}} & \textcolor{Top2}{\textbf{\pmstd{77.26}{1.99}}} & 4  \\
\textbf{GeneralChebySD} & \textcolor{Top2}{\textbf{\pmstd{77.82}{1.68}}} & \pmstd{76.75}{1.59} & \pmstd{76.62}{1.04} & \pmstd{75.87}{1.69} & \pmstd{75.11}{1.86} & 2  \\
SAN         & 76.27$\pm$1.76 & 76.30$\pm$1.80 & 76.62$\pm$1.70 & 76.18$\pm$1.47 & 76.07$\pm$2.18  & 8 \\
ANSD        & 76.99$\pm$1.74 & 76.86$\pm$1.71 & 76.61$\pm$1.51 & 76.69$\pm$1.56 & 76.22$\pm$1.47  & 2 \\
GGCN              & 76.83$\pm$1.82 & 76.77$\pm$1.48 & \textcolor{Top3}{\textbf{76.91$\pm$1.56}} & \textcolor{Top2}{\textbf{76.88$\pm$1.56}} & \textcolor{Top3}{\textbf{76.97$\pm$1.52}} & 10 \\
GPRGNN            & 77.13$\pm$1.67 & \textcolor{Top3}{\textbf{77.05$\pm$1.43}} & \textcolor{Top1}{\textbf{77.09$\pm$1.62}} & 76.00$\pm$1.64 & 74.97$\pm$1.47  & 2 \\
H2GCN            & 76.90$\pm$1.80 & 76.09$\pm$1.54 & 74.10$\pm$1.83 & OOM & OOM  & 1 \\
GCNII            & 75.42$\pm$1.78 & 75.29$\pm$1.90 & 76.00$\pm$1.66 & \textcolor{Top1}{\textbf{76.96$\pm$1.38}} & \textcolor{Top1}{\textbf{77.33$\pm$1.48}}  & 32 \\
PairNorm          & 73.59$\pm$1.47 & 72.62$\pm$1.97 & 72.32$\pm$1.58 & 59.71$\pm$15.97 & 27.21$\pm$10.95  & 2 \\
Geom-GCN         & \textcolor{Top1}{\textbf{78.02$\pm$1.15}} & 23.01$\pm$1.95 & 7.23$\pm$0.87 & 7.23$\pm$0.87 & 7.23$\pm$0.87  & 2 \\
GCN               & 76.50$\pm$1.36 & 64.33$\pm$8.27 & 24.18$\pm$1.71 & 23.07$\pm$2.95 & 25.3$\pm$1.77  & 2 \\
GAT               & 76.55$\pm$1.23 & 75.33$\pm$1.39 & 66.57$\pm$5.08 & INS & INS & 2 \\
\bottomrule
\end{tabular}
\end{adjustbox}
\end{subtable}

\vspace{0.8em}

\begin{subtable}[t]{0.49\textwidth}
\centering
\begin{adjustbox}{width=\linewidth}
\begin{tabular}{lcccccc}
\toprule
\textbf{Layers} & \textbf{2} & \textbf{4} & \textbf{8} & \textbf{16} & \textbf{32} & \textbf{Best} \\
\midrule
\multicolumn{7}{c}{\textbf{Cornell ($h{=}0.3$)}}\\
\addlinespace
\textbf{DiagChebySD}    & \textcolor{Top1}{\textbf{\pmstd{85.40}{5.16}}} & \textcolor{Top1}{\textbf{\pmstd{85.13}{5.57}}} & \textcolor{Top2}{\textbf{\pmstd{85.68}{7.36}}} & \pmstd{84.05}{8.41} & \pmstd{81.62}{6.92} & 8  \\
\textbf{BundleChebySD}  & \textcolor{Top2}{\textbf{\pmstd{85.40}{7.95}}} & \textcolor{Top2}{\textbf{\pmstd{84.86}{5.82}}} & \pmstd{84.59}{7.05} & \textcolor{Top1}{\textbf{\pmstd{85.13}{8.31}}} & \textcolor{Top1}{\textbf{\pmstd{84.59}{7.93}}} & 2  \\
\textbf{GeneralChebySD} & \textcolor{Top3}{\textbf{\pmstd{85.13}{6.19}}} & \textcolor{Top3}{\textbf{\pmstd{84.86}{6.42}}} & \pmstd{84.32}{6.49} & \pmstd{81.62}{7.53} & \pmstd{81.35}{6.22} & 2  \\
SAN        & 82.70$\pm$6.64 & 84.59$\pm$4.69 & \textcolor{Top1}{\textbf{85.68$\pm$4.53}} & \textcolor{Top3}{\textbf{84.32$\pm$6.82}} & \textcolor{Top3}{\textbf{83.51$\pm$7.20}}  & 8 \\
ANSD       & 84.86$\pm$6.07 & 84.32$\pm$5.10 & 84.86$\pm$5.95 & \textcolor{Top2}{\textbf{84.59$\pm$6.51}} & 83.24$\pm$3.97  & 8 \\
GGCN              & 83.78$\pm$6.73 & 83.78$\pm$6.16 & \textcolor{Top3}{\textbf{84.86$\pm$5.69}} & 83.78$\pm$6.73 & \textcolor{Top2}{\textbf{83.78$\pm$6.51}}  & 6 \\
GPRGNN            & 76.76$\pm$8.22 & 77.57$\pm$7.46 & 80.27$\pm$8.11 & 78.38$\pm$6.04 & 74.59$\pm$7.66  & 8 \\
H2GCN            & 81.89$\pm$5.98 & 82.70$\pm$6.27 & 80.27$\pm$6.63 & OOM & OOM & 1 \\
GCNII            & 67.57$\pm$11.34 & 64.59$\pm$9.63 & 73.24$\pm$5.91 & 77.84$\pm$3.97 & 75.41$\pm$5.47  & 16 \\
PairNorm          & 50.27$\pm$7.17 & 53.51$\pm$8.00 & 58.38$\pm$5.01 & 58.38$\pm$3.01 & 58.92$\pm$3.15  & 32 \\
Geom-GCN         & 60.54$\pm$3.67 & 23.78$\pm$11.64 & 12.97$\pm$2.91 & 12.97$\pm$2.91 & 12.97$\pm$2.91  & 2 \\
GCN               & 60.54$\pm$5.30 & 59.19$\pm$3.30 & 58.92$\pm$3.15 & 58.92$\pm$3.15 & 58.92$\pm$3.15  & 2 \\
GAT               & 61.89$\pm$5.05 & 58.38$\pm$4.05 & 58.38$\pm$3.86 & INS & INS  & 2 \\
\bottomrule
\end{tabular}
\end{adjustbox}
\end{subtable}\hfill
\begin{subtable}[t]{0.49\textwidth}
\centering
\begin{adjustbox}{width=\linewidth}
\begin{tabular}{lcccccc}
\toprule
\textbf{Layers} & \textbf{2} & \textbf{4} & \textbf{8} & \textbf{16} & \textbf{32} & \textbf{Best} \\
\midrule
\multicolumn{7}{c}{\textbf{Chameleon ($h{=}0.23$)}}\\
\addlinespace
\textbf{DiagChebySD}    & \textcolor{Top2}{\textbf{\pmstd{69.50}{1.81}}} & \textcolor{Top1}{\textbf{\pmstd{70.39}{1.66}}} & \textcolor{Top2}{\textbf{\pmstd{70.04}{2.87}}} & \textcolor{Top2}{\textbf{\pmstd{68.99}{4.13}}} & \pmstd{66.78}{3.99} & 4  \\
\textbf{BundleChebySD}  & \textcolor{Top3}{\textbf{\pmstd{67.13}{1.63}}} & \textcolor{Top3}{\textbf{\pmstd{67.13}{3.59}}} & \pmstd{66.03}{1.70} & \pmstd{66.51}{2.11} & \pmstd{66.47}{2.79} & 2  \\
\textbf{GeneralChebySD} & \pmstd{63.29}{2.47} & \pmstd{65.22}{1.72} & \pmstd{58.55}{3.43} & \pmstd{61.80}{4.64} & \pmstd{65.17}{5.19} & 4  \\
SAN       & 65.88$\pm$2.10 & 65.99$\pm$1.28 & \textcolor{Top3}{\textbf{68.16$\pm$2.18}} & \textcolor{Top3}{\textbf{68.62$\pm$2.81}} & \textcolor{Top2}{\textbf{67.61$\pm$2.80}}  & 16 \\
ANSD       & 65.35$\pm$1.26 & \textcolor{Top2}{\textbf{67.11$\pm$1.88}} & 66.69$\pm$2.30 & 67.39$\pm$1.84 & \textcolor{Top3}{\textbf{66.91$\pm$1.61}}  & 16 \\
GGCN              & \textcolor{Top1}{\textbf{70.77$\pm$1.42}} & 69.58$\pm$2.68 & \textcolor{Top1}{\textbf{70.33$\pm$1.70}} & \textcolor{Top1}{\textbf{70.44$\pm$1.82}} & \textcolor{Top1}{\textbf{70.29$\pm$1.62}}  & 5 \\
GPRGNN            & 46.58$\pm$1.77 & 45.72$\pm$3.45 & 41.16$\pm$5.79 & 39.58$\pm$7.85 & 35.42$\pm$8.52  & 2 \\
H2GCN            & 59.06$\pm$1.85 & 60.11$\pm$2.15 & OOM & OOM & OOM  & 4 \\
GCNII            & 61.07$\pm$4.10 & 63.86$\pm$3.04 & 62.89$\pm$1.18 & 60.20$\pm$2.10 & 56.97$\pm$1.81  & 4 \\
PairNorm          & 62.74$\pm$2.82 & 59.01$\pm$2.80 & 54.12$\pm$2.24 & 46.38$\pm$2.23 & 46.78$\pm$2.26  & 2 \\
Geom-GCN         & 60.00$\pm$2.81 & 19.17$\pm$1.66 & 19.58$\pm$1.73 & 19.58$\pm$1.73 & 19.58$\pm$1.73  & 2 \\
GCN               & 64.82$\pm$2.24 & 53.11$\pm$4.44 & 35.15$\pm$3.14 & 35.39$\pm$3.23 & 35.20$\pm$3.25  & 2 \\
GAT               & 60.26$\pm$2.50 & 48.71$\pm$2.96 & 35.09$\pm$3.55 & INS & INS  & 2 \\
\bottomrule
\end{tabular}
\end{adjustbox}
\end{subtable}

\end{table*}

%% file: tables/polysd_ksweep.tex
\begingroup
\setlength{\tabcolsep}{2pt} 
\renewcommand{\pmstd}[2]{\footnotesize #1{\scriptsize$\pm$}#2} 

\begin{table}[H]
\centering
\small
\caption{\textit{PolyNSD (Chebyshev) $K$-sweep}. The three PolySD variants are held fixed with  \emph{Layers=2}, \emph{StalkDim=4} and \emph{Hidden=16}, and columns sweep \(K\in\{2,4,8,16\}\) are reported, with the associated number of parameters. Top-1 per dataset and model (per-row) are coloured as \textcolor{Top1}{\textbf{First}}.}
\label{tab:polysd_ksweep}
\resizebox{0.5\linewidth}{!}{
\begin{tabular}{l@{\hspace{2pt}}c@{\hspace{2pt}}c@{\hspace{2pt}}c@{\hspace{2pt}}c@{\hspace{2pt}}c@{\hspace{2pt}}c}
\toprule
\textbf{$K$} & \textbf{2} & \textbf{4} & \textbf{6} & \textbf{8} & \textbf{16} & \textbf{Best} \\
\midrule\addlinespace[6pt]

\multicolumn{7}{c}{\footnotesize\emph{\textbf{PubMed} ($h{=}0.80,\ \#\text{N}=18{,}707,\ \#\text{E}=44{,}327,\ \#\text{C}=3$)}}\\
\addlinespace[4pt]\midrule\addlinespace[6pt]
\textbf{DiagChebySD}    & \pmstd{87.37}{0.45} & \pmstd{87.95}{0.60} & \pmstd{88.01}{0.39} & \textcolor{Top1}{\textbf{\pmstd{88.18}{0.40}}} & \pmstd{87.65}{0.42} & 8 \\
\paramcells{48,659}{48,661}{48,663}{48,665}{48,673}{48,665}
\textbf{BundleChebySD}  & \pmstd{87.86}{0.47} & \textcolor{Top1}{\textbf{\pmstd{87.91}{0.32}}} & \pmstd{87.52}{0.45} & \pmstd{87.63}{0.37} & \pmstd{87.51}{0.74} & 4 \\
\paramcells{49,619}{49,621}{49,623}{49,625}{49,633}{49,621}
\textbf{GeneralChebySD} & \pmstd{87.74}{0.49} & \pmstd{87.80}{0.25} & \pmstd{87.62}{0.43} & \textcolor{Top1}{\textbf{\pmstd{87.87}{0.49}}} & \pmstd{87.82}{0.49} & 8 \\
\paramcells{52,499}{52,501}{52,503}{52,505}{52,513}{52,505}

\midrule\addlinespace[6pt]
\multicolumn{7}{c}{\footnotesize\emph{\textbf{Chameleon} ($h{=}0.23,\ \#\text{N}=2{,}277,\ \#\text{E}=31{,}421,\ \#\text{C}=5$)}}\\
\addlinespace[4pt]\midrule
\textbf{DiagChebySD}    & \pmstd{36.27}{2.48} & \pmstd{61.21}{10.09} & \pmstd{59.36}{10.55} & \pmstd{68.55}{2.32} & \textcolor{Top1}{\textbf{\pmstd{70.57}{1.32}}} & 16 \\
\paramcells{194,821}{194,823}{194,825}{194,827}{194,835}{194,835}
\textbf{BundleChebySD}  & \pmstd{59.85}{7.93} & \pmstd{61.97}{7.18} & \pmstd{66.18}{2.34} & \textcolor{Top1}{\textbf{\pmstd{66.64}{2.35}}} & \pmstd{66.58}{9.92} & 8 \\
\paramcells{195,781}{195,783}{195,785}{195,787}{195,795}{195,787}
\textbf{GeneralChebySD} & \pmstd{62.39}{3.16} & \pmstd{66.05}{1.65} & \pmstd{65.94}{1.64} & \pmstd{62.08}{3.27} & \textcolor{Top1}{\textbf{\pmstd{67.39}{2.50}}} & 16 \\
\paramcells{198,661}{198,663}{198,665}{198,667}{198,675}{198,675}

\midrule\addlinespace[6pt]
\multicolumn{7}{c}{\footnotesize\emph{\textbf{Squirrel} ($h{=}0.22,\ \#\text{N}=5{,}201,\ \#\text{E}=198{,}493,\ \#\text{C}=5$)}}\\
\addlinespace[4pt]\midrule
\textbf{DiagChebySD}    & \pmstd{35.34}{2.76} & \pmstd{46.65}{1.76} & \pmstd{43.92}{2.83} & \pmstd{47.18}{1.18} & \textcolor{Top1}{\textbf{\pmstd{47.72}{1.20}}} & 16 \\
\paramcells{174,661}{174,663}{174,665}{174,667}{174,675}{174,675}
\textbf{BundleChebySD}  & \pmstd{44.12}{1.65} & \pmstd{43.58}{2.26} & \pmstd{41.69}{5.56} & \pmstd{41.29}{2.55} & \textcolor{Top1}{\textbf{\pmstd{44.48}{3.22}}} & 16 \\
\paramcells{176,901}{174,983}{176,903}{176,905}{176,913}{176,913}
\textbf{GeneralChebySD} & \pmstd{41.66}{1.42} & \pmstd{40.27}{2.69} & \textcolor{Top1}{\textbf{\pmstd{42.02}{2.50}}} & \pmstd{40.70}{2.01} & \pmstd{41.82}{3.24} & 6 \\
\paramcells{179,781}{179,783}{179,785}{179,787}{179,795}{179,785}

\bottomrule
\end{tabular}%
} 
\end{table}
\endgroup

%% file: tables/nsd_layers_plus_polysd_summary.tex
\begin{table}[H]
\centering
\small
\setlength{\tabcolsep}{4pt}
\caption{\textit{PolySD vs NSD: layers sweep.} “OOM” stands for out of memory, whilst “N/A” means that the parameter count is not available. We report the best results, for each dataset, for the three PolySD variant, with the associated parameter count. The last column contains the summary of improvements of the PolySD model variant, showing the PolySD \good{\textbf{Improvement}} or \bad{\textbf{Deterioration}} w.r.t. the NSD Best setting. Top-1 per dataset and model (per-row) are coloured as \textcolor{Top1}{\textbf{First}}.  }

\label{tab:nsd_layers_plus_polysd_summary}
\resizebox{0.7\linewidth}{!}{%
\begin{tabular}{lccccccc}
\toprule
\textbf{Layers} & \textbf{2} & \textbf{4} & \textbf{8} & \textbf{16} & \textbf{32} & \textbf{Best} & \textbf{PolySD Improvement}\\
\midrule\addlinespace[6pt]

\multicolumn{8}{c}{\emph{\textbf{PubMed} ($h{=}0.80,\ \#\text{N}=18{,}707,\ \#\text{E}=44{,}327,\ \#\text{C}=3$)}}\\\addlinespace[4pt]
\multicolumn{8}{c}{\textbf{DiagChebySD}: \pmstd{88.18}{0.40} (K=8), \#params=48,665}\\
\multicolumn{8}{c}{\textbf{BundleChebySD}: \pmstd{87.91}{0.32} (K=4), \#params=49,621}\\
\multicolumn{8}{c}{\textbf{GeneralChebySD}: \pmstd{87.87}{0.49} (K=8), \#params=52,505}\\[2pt]

\midrule\addlinespace[6pt]
Diag-NSD    & \pmstd{87.82}{0.55} & \textcolor{Top1}{\textbf{\pmstd{87.92}{0.51}}} & \pmstd{87.92}{0.52} & \pmstd{65.92}{20.39} & \pmstd{39.49}{1.60} & 4 & \good{\textbf{+0.26\%}}\\
\paramcellsimp{48,655}{50,507}{54,211}{61,619}{76,435}{50,507}{\good{\textbf{-3,65\% (-1,842)}}}
Bundle-NSD  & \pmstd{87.70}{0.56} & \pmstd{87.85}{0.42} & \textcolor{Top1}{\textbf{\pmstd{87.94}{0.36}}} & \pmstd{87.63}{0.47} & \pmstd{37.03}{4.72} & 8 & \bad{\textbf{-0.03\%}}\\
\paramcellsimp{49,615}{52,427}{58,051}{69,299}{91,795}{58,051}{\good{\textbf{-14,52\%(-8,430)}}}
General-NSD & \pmstd{87.48}{0.64} & \pmstd{87.62}{0.36} & \textcolor{Top1}{\textbf{\pmstd{87.72}{0.68}}} & \pmstd{39.94}{1.03} & \pmstd{39.33}{2.25} & 8 & \good{\textbf{+0.15\%}}\\
\paramcellsimp{52,495}{58,187}{69,571}{92,339}{137,875}{69,571}{\good{\textbf{-24,53\%(-17,066)}}}

\midrule\addlinespace[6pt]
\multicolumn{8}{c}{\emph{\textbf{Chameleon} ($h{=}0.23,\ \#\text{N}=2{,}277,\ \#\text{E}=31{,}421,\ \#\text{C}=5$)}}\\\addlinespace[4pt]
\multicolumn{8}{c}{\textbf{DiagChebySD}: \pmstd{70.57}{1.32} (K=16), \#params=194,835}\\
\multicolumn{8}{c}{\textbf{BundleChebySD}: \pmstd{66.64}{2.35} (K=8), \#params=195,787}\\
\multicolumn{8}{c}{\textbf{GeneralChebySD}: \pmstd{67.39}{2.50} (K=16), \#params=198,675}\\[2pt]

\midrule\addlinespace[6pt]
Diag-NSD    & \textcolor{Top1}{\textbf{\pmstd{64.43}{2.06}}} & \pmstd{61.27}{5.14} & \pmstd{57.34}{6.10} & \pmstd{22.92}{1.42} & \pmstd{22.89}{2.59} & 2 & \good{\textbf{+6.15\%}}\\
\paramcellsimp{194,817}{196,669}{200,373}{207,781}{222,597}{194,817}{\bad{\textbf{+0.009\%(+18)}}}
Bundle-NSD  & \pmstd{47.10}{10.14} & \textcolor{Top1}{\textbf{\pmstd{54.08}{6.30}}} & \pmstd{50.57}{3.53} & \pmstd{24.91}{2.87} & \pmstd{23.05}{2.46} & 4 & \good{\textbf{+12.56\%}}\\
\paramcellsimp{195,777}{198,589}{204,213}{215,461}{237,957}{198,589}{\good{\textbf{-1,41\%(-2,802)}}}
General-NSD & \textcolor{Top1}{\textbf{\pmstd{59.60}{4.53}}} & \pmstd{58.18}{3.56} & \pmstd{26.05}{3.41} & \pmstd{23.79}{3.41} & \pmstd{20.07}{3.74} & 2 & \good{\textbf{+7.79\%}}\\
\paramcellsimp{198,657}{204,349}{215,733}{238,501}{284,037}{198,657}{\good{\textbf{+0\%(+0)}}}

\midrule\addlinespace[6pt]
\multicolumn{7}{c}{\emph{\textbf{Squirrel} ($h{=}0.22,\ \#\text{N}=5{,}201,\ \#\text{E}=198{,}493,\ \#\text{C}=5$)}}\\\addlinespace[4pt]
\multicolumn{8}{c}{\textbf{DiagChebySD}: \pmstd{47.72}{1.20} (K=16), \#params=174,675}\\
\multicolumn{8}{c}{\textbf{BundleChebySD}: \pmstd{44.48}{3.22} (K=16), \#params=176,913}\\
\multicolumn{8}{c}{\textbf{GeneralChebySD}: \pmstd{42.02}{2.50} (K=6), \#params=179,785}\\[2pt]

\midrule\addlinespace[6pt]
Diag-NSD    & \pmstd{41.98}{1.17} & \textcolor{Top1}{\textbf{\pmstd{42.60}{1.83}}} & \pmstd{42.05}{1.54} & \pmstd{20.80}{1.33} & OOM & 4 & \good{\textbf{+5.12\%}} \\
\paramcellsimp{175,937}{177,789}{181,493}{188,901}{N/A}{177,789}{\good{\textbf{-1,75\%(-3,114)}}}
Bundle-NSD  & \textcolor{Top1}{\textbf{\pmstd{42.46}{1.45}}} & \pmstd{42.30}{1.83} & \pmstd{37.99}{2.61} & \pmstd{22.86}{3.03} & OOM & 2 & \good{\textbf{+2.02\%}} \\
\paramcellsimp{176,897}{179,709}{185,333}{196,581}{N/A}{176,897}{\bad{\textbf{+0.009\%(+16)}}}
General-NSD & \pmstd{39.11}{1.96} & \textcolor{Top1}{\textbf{\pmstd{39.96}{1.76}}} & \pmstd{33.69}{1.32} & OOM & OOM & 4 & \good{\textbf{+2.06\%}} \\
\paramcellsimp{179,777}{185,469}{196,853}{N/A}{N/A}{185,469}{\good{\textbf{-3.06\%(-5,684)}}}

\bottomrule
\end{tabular}%
}
\end{table}

%% file: tables/nsd_width_plus_polysd_summary.tex
\begin{table}[H]
\centering
\small
\setlength{\tabcolsep}{4pt}
\caption{\textit{PolySD vs NSD: hidden-channels sweep.} “OOM” stands for out of memory, whilst “N/A” means that the parameter count is not available. We also report the best results, for each dataset, for the three PolySD variant, with the associated parameter count. The last column contains the summary of improvements of the PolySD model variant, showing the PolyNSD \good{\textbf{Improvement}} or \bad{\textbf{Deterioration}} w.r.t. the NSD Best setting. Top-1 per dataset and model (per-row) are coloured as \textcolor{Top1}{\textbf{First}}.  }
\label{tab:nsd_width_plus_polysd_summary}
\resizebox{0.7\linewidth}{!}{%
\begin{tabular}{lccccccc}
\toprule
\textbf{Hidden Channels} & \textbf{16} & \textbf{32} & \textbf{64} & \textbf{128} & \textbf{256} & \textbf{Best} & \textbf{PolySD Improvement}\\
\midrule\addlinespace[6pt]

\multicolumn{8}{c}{\emph{\textbf{PubMed} ($h{=}0.80,\ \#\text{N}=18{,}707,\ \#\text{E}=44{,}327,\ \#\text{C}=3$)}}\\\addlinespace[4pt]
\multicolumn{8}{c}{\textbf{DiagChebySD}: \pmstd{88.18}{0.40} (K=8), \#params=48,665}\\
\multicolumn{8}{c}{\textbf{BundleChebySD}: \pmstd{87.91}{0.32} (K=4), \#params=49,621}\\
\multicolumn{8}{c}{\textbf{GeneralChebySD}: \pmstd{87.87}{0.49} (K=8), \#params=52,505}\\[2pt]

\midrule\addlinespace[6pt]
Diag-NSD    & \pmstd{87.82}{0.55} & \pmstd{87.86}{0.44} & \pmstd{88.02}{0.53} & \textcolor{Top1}{\textbf{\pmstd{88.05}{0.46}}} & \pmstd{88.03}{0.50} & 128 & \good{\textbf{+0.13\%}}\\
\paramcellsimp{48,655}{111,071}{277,375}{775,871}{2,436,415}{775,871}{\good{\textbf{-93,73\%(-727,206)}}}
Bundle-NSD  & \pmstd{87.70}{0.56} & \pmstd{87.83}{0.51} & \pmstd{87.84}{0.38} & \pmstd{87.87}{0.56} & \textcolor{Top1}{\textbf{\pmstd{87.92}{0.44}}} & 256 & \bad{\textbf{-0.01}}\\
\paramcellsimp{49,615}{112,991}{281,215}{783,551}{2,451,775}{2,451,775}{\good{\textbf{-97,98\% (-2,402,154)}}}
General-NSD & \pmstd{87.48}{0.64} & \pmstd{87.69}{0.53} & \pmstd{87.71}{0.66} & \pmstd{87.64}{0.55} & \textcolor{Top1}{\textbf{\pmstd{87.91}{0.37}}} & 256 & \bad{\textbf{-0.04\%}}\\
\paramcellsimp{52,495}{118,751}{292,735}{806,591}{2,497,855}{2,497,855}{\good{\textbf{-97,90\%(-2,445,350)}}}

\midrule\addlinespace[6pt]
\multicolumn{7}{c}{\emph{\textbf{Chameleon} ($h{=}0.23,\ \#\text{N}=2{,}277,\ \#\text{E}=31{,}421,\ \#\text{C}=5$)}}\\\addlinespace[4pt]
\multicolumn{7}{c}{\textbf{DiagChebySD}: \pmstd{70.57}{1.32} (K=16), \#params=194,835}\\
\multicolumn{7}{c}{\textbf{BundleChebySD}: \pmstd{66.64}{2.35} (K=8), \#params=195,787}\\
\multicolumn{7}{c}{\textbf{GeneralChebySD}: \pmstd{67.39}{2.50} (K=16), \#params=198,675}\\[2pt]

\midrule\addlinespace[6pt]
Diag-NSD    & \pmstd{64.43}{2.06} & \pmstd{65.33}{1.78} & \textcolor{Top1}{\textbf{\pmstd{65.66}{1.61}}} & \pmstd{65.42}{1.09} & \pmstd{64.78}{1.99} & 64 & \good{\textbf{+4.91\%}}\\
\paramcellsimp{194,817}{403,393}{862,017}{1,945,153}{4,774,977}{862,017}{\good{\textbf{-77,40\%(-667,182)}}}
Bundle-NSD  & \pmstd{46.75}{10.70} & \pmstd{61.91}{8.76} & \pmstd{61.53}{7.50} & \textcolor{Top1}{\textbf{\pmstd{63.68}{1.60}}} & \pmstd{62.10}{2.11} & 128 & \good{\textbf{+2,96\%}}\\
\paramcellsimp{195,777}{405,313}{865,857}{1,952,833}{4,790,337}{1,952,833}{\good{\textbf{-89,97\%(-1,757,046)}}}
General-NSD & \pmstd{61.12}{2.27} & \pmstd{64.28}{1.58} & \pmstd{64.74}{1.46} & \textcolor{Top1}{\textbf{\pmstd{65.04}{2.13}}} & \pmstd{63.92}{1.75} & 128 & \good{\textbf{+2.35\%}}\\
\paramcellsimp{198,657}{411,073}{877,377}{1,975,873}{4,836,417}{1,975,873}{\good{\textbf{-89,94\%(-1,777,198)}}}

\midrule\addlinespace[6pt]
\multicolumn{7}{c}{\emph{\textbf{Squirrel} ($h{=}0.22,\ \#\text{N}=5{,}201,\ \#\text{E}=198{,}493,\ \#\text{C}=5$)}}\\\addlinespace[4pt]
\multicolumn{7}{c}{\textbf{DiagChebySD}: \pmstd{47.72}{1.20} (K=16), \#params=174,675}\\
\multicolumn{7}{c}{\textbf{BundleChebySD}: \pmstd{44.48}{3.22} (K=16), \#params=176,913}\\
\multicolumn{7}{c}{\textbf{GeneralChebySD}: \pmstd{42.02}{2.50} (K=6), \#params=179,785}\\[2pt]

\midrule\addlinespace[6pt]
Diag-NSD    & \pmstd{41.98}{1.17} & \pmstd{49.21}{1.87} & \pmstd{49.35}{1.53} & \textcolor{Top1}{\textbf{\pmstd{49.59}{1.10}}} & \pmstd{49.51}{1.89} & 128 & \bad{\textbf{-1.87\%}}\\
\paramcellsimp{175,937}{365,633}{786,497}{1,794,113}{4,472,897}{1,794,113}{\good{\textbf{-90,26\%(-1,619,438)}}}
Bundle-NSD  & \pmstd{42.46}{1.45} & \pmstd{47.48}{3.47} & \textcolor{Top1}{\textbf{\pmstd{49.81}{1.47}}} & \pmstd{48.87}{2.70} & OOM & 64 & \bad{\textbf{-5.33\%}}\\
\paramcellsimp{176,897}{367,553}{790,337}{1,801,793}{N/A}{790,337}{\good{\textbf{-77,62\%(-613,424)}}}
General-NSD & \pmstd{39.11}{1.96} & \pmstd{39.22}{1.97} & \textcolor{Top1}{\textbf{\pmstd{42.58}{3.66}}} & OOM & OOM & 64 & \bad{\textbf{-0.56\%}}\\
\paramcellsimp{179,777}{373,313}{801,857}{N/A}{N/A}{801,857}{\good{\textbf{-77,58\%(-622,072)}}}

\bottomrule
\end{tabular}%
}
\end{table}

%% file: tables/spectral_response_layers.tex
\begin{table}[t]
\centering
\small
\setlength{\tabcolsep}{6pt}
\caption{\textit{Network depth $L$ effect on the learned spectral diagnostics.}
We report mean$\pm$std across all models, stalk dimensions $d\in\{2,3,4\}$ and seeds (3)
for heterophilic (\textsc{Chameleon}, \textsc{Squirrel}) vs.\ homophilic
(\textsc{CiteSeer}, \textsc{PubMed}) datasets. }
\label{tab:spectral-response-depth}
\begin{tabular}{c|cc|cc|cc}
\toprule
 & \multicolumn{2}{c|}{$\alpha_{\mathrm{hp}}$} & \multicolumn{2}{c|}{$\Delta G$} & \multicolumn{2}{c}{$\#\,$sign-changes}\\
$L$ & Hetero & Homo & Hetero & Homo & Hetero & Homo\\
\midrule
2 & -0.29$\pm$0.09 & -1.48$\pm$0.17 & 0.50$\pm$0.15 & 1.41$\pm$0.30 & 2.00$\pm$0.00 & 1.49$\pm$0.57\\
3 & -0.26$\pm$0.09 & -0.95$\pm$0.32 & 0.46$\pm$0.15 & 0.96$\pm$0.31 & 2.00$\pm$0.00 & 1.74$\pm$0.51\\
4 & -0.22$\pm$0.09 & -0.59$\pm$0.34 & 0.36$\pm$0.10 & 0.77$\pm$0.28 & 2.00$\pm$0.00 & 1.67$\pm$0.62\\
\bottomrule
\end{tabular}
\end{table}

%% file: tables/spectral_responsed_d.tex
\begin{table}[t]
\centering
\small
\setlength{\tabcolsep}{6pt}
\caption{\textit{Stalk dimension $d$ effect on the learned spectral diagnostics.}
We report mean$\pm$std across all models, depths $L\in\{2,3,4\}$ and seeds (3).}
\label{tab:spectral-response-d}
\begin{tabular}{c|cc|cc|cc}
\toprule
 & \multicolumn{2}{c|}{$\alpha_{\mathrm{hp}}$} & \multicolumn{2}{c|}{$\Delta G$} & \multicolumn{2}{c}{$\#\,$sign-changes}\\
$d$ & Hetero & Homo & Hetero & Homo & Hetero & Homo\\
\midrule
2 & -0.30$\pm$0.10 & -1.03$\pm$0.49 & 0.49$\pm$0.15 & 1.06$\pm$0.33 & 2.00$\pm$0.00 & 1.76$\pm$0.44\\
3 & -0.27$\pm$0.09 & -0.96$\pm$0.38 & 0.50$\pm$0.16 & 1.04$\pm$0.35 & 2.00$\pm$0.00 & 1.57$\pm$0.63\\
4 & -0.20$\pm$0.07 & -1.03$\pm$0.46 & 0.33$\pm$0.09 & 1.05$\pm$0.37 & 2.00$\pm$0.00 & 1.56$\pm$0.65\\
\bottomrule
\end{tabular}
\end{table}

%% file: tables/polynomials_table.tex
\begin{table}[htb]
\centering
\caption{\emph{Polynomial bases recurrences evaluated on the rescaled Laplacian $\tilde{\mathbf{L}}$.}}
\label{tab:poly-bases-recurrences}

\resizebox{0.80\linewidth}{!}{%
\begin{tabular}{p{3.6cm} p{5.3cm} p{7.2cm}}
\hline
\textbf{Basis} & \textbf{Initialisation} & \textbf{Three--term recurrence ($k\ge 1$)} \\
\hline

Chebyshev Type I $T_k$ &
$T_0(\tilde{\mathbf{L}})x = x,\quad T_1(\tilde{\mathbf{L}})x = \tilde{\mathbf{L}}x$ &
$T_{k+1}(\tilde{\mathbf{L}})x
= 2\,\tilde{\mathbf{L}}\,T_k(\tilde{\mathbf{L}})x
- T_{k-1}(\tilde{\mathbf{L}})x$ \\

Chebyshev Type II $U_k$ &
$U_0(\tilde{\mathbf{L}})x = x,\quad U_1(\tilde{\mathbf{L}})x = 2\,\tilde{\mathbf{L}}x$ &
$U_{k+1}(\tilde{\mathbf{L}})x
= 2\,\tilde{\mathbf{L}}\,U_k(\tilde{\mathbf{L}})x
- U_{k-1}(\tilde{\mathbf{L}})x$ \\

Chebyshev Type III $V_k$ &
$V_0(\tilde{\mathbf{L}})x = x,\quad
V_1(\tilde{\mathbf{L}})x = (2\,\tilde{\mathbf{L}}-\mathbf{I})x$ &
$V_{k+1}(\tilde{\mathbf{L}})x
= 2\,\tilde{\mathbf{L}}\,V_k(\tilde{\mathbf{L}})x
- V_{k-1}(\tilde{\mathbf{L}})x$ \\

Chebyshev Type IV $W_k$ &
$W_0(\tilde{\mathbf{L}})x = x,\quad
W_1(\tilde{\mathbf{L}})x = (2\,\tilde{\mathbf{L}}+\mathbf{I})x$ &
$W_{k+1}(\tilde{\mathbf{L}})x
= 2\,\tilde{\mathbf{L}}\,W_k(\tilde{\mathbf{L}})x
- W_{k-1}(\tilde{\mathbf{L}})x$ \\

Legendre $P_k$ &
$P_0(\tilde{\mathbf{L}})x = x,\quad P_1(\tilde{\mathbf{L}})x = \tilde{\mathbf{L}}x$ &
$P_{k+1}(\tilde{\mathbf{L}})x
= \frac{2k+1}{k+1}\,\tilde{\mathbf{L}}\,P_k(\tilde{\mathbf{L}})x
- \frac{k}{k+1}\,P_{k-1}(\tilde{\mathbf{L}})x$ \\

Gegenbauer $C_k^{(\lambda)}$ ($\lambda>0$) &
$C_0^{(\lambda)}(\tilde{\mathbf{L}})x = x,\quad
C_1^{(\lambda)}(\tilde{\mathbf{L}})x = 2\lambda\,\tilde{\mathbf{L}}x$ &
$C_{k+1}^{(\lambda)}(\tilde{\mathbf{L}})x
= \frac{2(k+\lambda)}{k+1}\,\tilde{\mathbf{L}}\,C_k^{(\lambda)}(\tilde{\mathbf{L}})x
- \frac{k+2\lambda-1}{k+1}\,C_{k-1}^{(\lambda)}(\tilde{\mathbf{L}})x$ \\

Jacobi $P_k^{(\alpha,\beta)}$ ($\alpha,\beta>-1$) &
$\begin{aligned}[t]
P_0^{(\alpha,\beta)}(\tilde{\mathbf{L}})x &= x,\\
P_1^{(\alpha,\beta)}(\tilde{\mathbf{L}})x &= c_1\,\tilde{\mathbf{L}}x + c_0\,x
\end{aligned}$ &
$\begin{aligned}[t]
P_{k+1}^{(\alpha,\beta)}(\tilde{\mathbf{L}})x
&= A_k\,\tilde{\mathbf{L}}\,P_k^{(\alpha,\beta)}(\tilde{\mathbf{L}})x
 + B_k\,P_k^{(\alpha,\beta)}(\tilde{\mathbf{L}})x \\
&\qquad - C_k\,P_{k-1}^{(\alpha,\beta)}(\tilde{\mathbf{L}})x,
\quad \text{with:}\\[0.2em]
&\hspace{-1.5em} A_k = \frac{2(k+1)(k+\alpha+\beta+1)}
{(2k+\alpha+\beta+1)(2k+\alpha+\beta+2)}\\
&\hspace{-1.5em} B_k = \frac{\beta^2-\alpha^2}
{(2k+\alpha+\beta)(2k+\alpha+\beta+2)}\\
&\hspace{-1.5em} C_k = \frac{2(k+\alpha)(k+\beta)}
{(2k+\alpha+\beta)(2k+\alpha+\beta+1)}
\end{aligned}$\\
\\
\hline
\end{tabular}}
\end{table}

%% file: tables/polynsd_polytype_summary.tex
\begin{table}[H]
  \centering
  \caption{\textit{Extended PolyNSD discrete node classification benchmark results.} We report the accuracy $\pm$ stdev on node classification datasets, ordered by increasing homophily. Our techniques are denoted in \textbf{bold}. The first section includes Sheaf Neural Networks models, while the second includes other GNN models. The top three models for each dataset are coloured by \textcolor{Top1}{\textbf{First}}, \textcolor{Top2}{\textbf{Second}} and \textcolor{Top3}{\textbf{Third}}, respectively. 
}
  \label{tab:big-leaderboard-enhanced}
  \setlength{\tabcolsep}{4pt}
  \renewcommand{\arraystretch}{1.12}
  \begin{adjustbox}{max width=0.8\textwidth}
  \begin{tabular}{lccccccccc}
    \toprule
     & \textbf{Texas} & \textbf{Wisconsin} & \textbf{Film} & \textbf{Squirrel} & \textbf{Chameleon} & \textbf{Cornell} & \textbf{Citeseer} & \textbf{Pubmed} & \textbf{Cora}\\
    \midrule
    \textit{Homophily level} & \textbf{0.11} & \textbf{0.21} & \textbf{0.22} & \textbf{0.22} & \textbf{0.23} & \textbf{0.30} & \textbf{0.74} & \textbf{0.80} & \textbf{0.81} \\
    \#Nodes  & 183    & 251  & 7{,}600 & 5{,}201 & 2{,}277 & 183   & 3{,}327 & 18{,}717 & 2{,}708 \\
    \#Edges  & 295    & 466  & 26{,}752 & 198{,}493 & 31{,}421 & 280   & 4{,}676 & 44{,}327 & 5{,}278 \\
    \#Classes& 5      & 5    & 5     & 5      & 5      & 5     & 7     & 3      & 6 \\
    \midrule
    \textbf{DiagChebyT1SD}   &\pmstd{88.68}{3.12}& \pmstd{87.84}{4.45} & \pmstd{37.14}{1.26} & \textcolor{Top1}{\textbf{\pmstd{56.61}{2.06}}} & \pmstd{70.41}{2.47} & \textcolor{Top2}{\textbf{\pmstd{86.49}{5.54}}} & \textcolor{Top2}{\textbf{\pmstd{77.74}{1.26}}} & \pmstd{89.67}{0.34} & \textcolor{Top2}{\textbf{\pmstd{88.67}{1.29}}} \\
    \textbf{BundleChebyT1SD} & \textcolor{Top3}{\textbf{\pmstd{89.74}{5.32}}} & \pmstd{87.65}{3.29} & \pmstd{37.47}{0.86} & \pmstd{54.33}{2.67} & \pmstd{69.29}{1.88} & \pmstd{85.40}{7.94} & \pmstd{77.57}{1.55} & \textcolor{Top3}{\textbf{\pmstd{89.75}{0.34}}} & \pmstd{88.12}{1.35} \\
    \textbf{GeneralChebyT1SD}& \pmstd{88.94}{4.53} & \pmstd{88.23}{4.56} & \pmstd{37.20}{0.77} & \pmstd{53.88}{1.65} & \pmstd{67.34}{2.45} & \textcolor{Top3}{\textbf{\pmstd{86.49}{5.80}}} & \pmstd{77.10}{1.30} & \pmstd{89.73}{0.41} & \pmstd{88.47}{1.19} \\
    \textbf{DiagChebyT2SD}   & \pmstd{88.68}{5.27} & \pmstd{88.23}{3.40} & \pmstd{36.98}{0.71} & \pmstd{51.20}{3.16} & \pmstd{71.01}{2.57} & \pmstd{85.95}{5.38} & \pmstd{77.42}{1.80} & \pmstd{89.60}{0.44} & \pmstd{88.49}{1.34}\\
    \textbf{BundleChebyT2SD}   & \pmstd{88.42}{6.25} & \pmstd{88.23}{3.28} & \pmstd{37.14}{1.17} & \pmstd{52.30}{2.01} & \pmstd{67.67}{1.80} & \pmstd{86.22}{7.69} & \pmstd{77.56}{1.66} & \pmstd{89.57}{0.54} & \pmstd{88.17}{1.20}\\ \textbf{GeneralChebyT2SD}   & \pmstd{88.42}{4.74} & \pmstd{87.84}{3.90} & \pmstd{37.00}{1.03} & \pmstd{51.34}{1.35} & \pmstd{67.00}{2.38} & \pmstd{85.68}{4.53} & \pmstd{77.21}{1.58} & \pmstd{89.62}{0.28} & \pmstd{88.03}{1.11} \\
    \textbf{DiagChebyT3SD}   & \pmstd{88.68}{7.07} & \pmstd{87.45}{4.74} & \pmstd{36.99}{1.13} & \pmstd{53.28}{1.71} & \pmstd{70.09}{2.75} & \pmstd{85.40}{5.16} & \pmstd{77.43}{1.54} & \pmstd{89.67}{0.37} & \pmstd{88.37}{1.40}\\
    \textbf{BundleChebyT3SD}   & \pmstd{88.95}{4.37} & \pmstd{88.43}{4.51} & \pmstd{37.31}{0.93} & \pmstd{55.48}{2.84} & \pmstd{67.70}{1.21} & \pmstd{85.95}{6.71} & \pmstd{77.31}{1.23} & \pmstd{89.62}{0.34} & \pmstd{88.05}{1.49}\\ \textbf{GeneralChebyT3SD}   & \pmstd{89.21}{5.05} & \pmstd{88.23}{3.16} & \pmstd{37.00}{1.06} & \textcolor{Top3}{\textbf{\pmstd{55.79}{2.52}}} & \pmstd{66.49}{1.94} & \pmstd{84.59}{5.80} & \pmstd{77.03}{1.33} & \pmstd{89.51}{0.32} & \pmstd{87.99}{1.56}\\
    \textbf{DiagChebyT4SD}   & \textcolor{Top1}{\textbf{\pmstd{90.00}{4.68}}} & \pmstd{88.04}{2.70} & \pmstd{37.31}{0.98} & \pmstd{54.27}{1.99} & \textcolor{Top1}{\textbf{\pmstd{71.45}{2.03}}} & \pmstd{85.68}{5.54} & \pmstd{77.50}{1.54} & \pmstd{89.64}{0.34} & \textcolor{Top1}{\textbf{\pmstd{88.79}{1.13}}}\\
    \textbf{BundleChebyT4SD}   & \pmstd{89.21}{4.15} & \pmstd{88.43}{2.97} & \pmstd{37.05}{1.12} & \pmstd{53.23}{1.52} & \pmstd{67.34}{1.49} & \pmstd{85.40}{6.42} & \pmstd{77.09}{1.76} & \pmstd{89.62}{0.35} & \pmstd{87.99}{1.21}\\ \textbf{GeneralChebyT4SD}   & \pmstd{88.16}{5.30} & \pmstd{88.04}{4.92} & \pmstd{37.31}{0.98} & \pmstd{53.60}{2.92} & \pmstd{69.62}{1.85} & \pmstd{84.86}{6.86} & \pmstd{76.58}{1.72} & \pmstd{89.63}{0.29} & \pmstd{87.54}{1.84}\\
    \textbf{DiagChebyInterpSD}   & \textcolor{Top2}{\textbf{\pmstd{89.74}{4.05}}} & \pmstd{87.45}{3.74} & \pmstd{36.89}{0.70} & \pmstd{54.45}{3.47} & \pmstd{68.46}{3.11} & \pmstd{85.67}{5.54} & \textcolor{Top3}{\textbf{\pmstd{77.60}{1.75}}} & \pmstd{89.70}{0.32} & \pmstd{88.63}{1.09}\\
    \textbf{BundleChebyInterpSD}   & \pmstd{88.95}{4.04} & \pmstd{88.04}{3.56} & \pmstd{36.92}{0.72} & \pmstd{55.02}{2.32} & \pmstd{70.02}{2.11} & \pmstd{86.21}{7.09} & \pmstd{77.34}{1.72} & \pmstd{89.66}{0.41} & \pmstd{88.33}{1.34}\\ \textbf{GeneralChebyInterpSD}   & \pmstd{88.43}{3.76} & \pmstd{87.84}{4.45} & \pmstd{37.23}{0.59} & \pmstd{52.75}{6.60} & \pmstd{67.37}{2.65} & \pmstd{85.95}{6.14} & \pmstd{77.03}{2.21} & \pmstd{89.72}{0.26} & \pmstd{88.11}{1.23}\\
    \textbf{DiagLegendreSD}   & \pmstd{89.47}{4.99} & \pmstd{88.43}{4.15} & \pmstd{37.20}{0.55} & \pmstd{51.88}{2.61} & \pmstd{70.35}{2.30} & \pmstd{85.40}{6.86} & \pmstd{77.53}{1.57} & \pmstd{89.70}{0.48} & \textcolor{Top3}{\textbf{\pmstd{88.65}{1.17}}} \\
    \textbf{BundleLegendreSD}   & \pmstd{88.95}{4.04} & \textcolor{Top1}{\textbf{\pmstd{89.41}{4.04}}} & \pmstd{37.29}{0.84} & \pmstd{51.65}{1.83} & \textcolor{Top2}{\textbf{\pmstd{71.18}{1.46}}} & \pmstd{84.86}{4.22} & \pmstd{77.49}{1.46} & \pmstd{89.61}{0.32} & \pmstd{88.27}{1.38} \\ \textbf{GeneralLegendreSD}   & \pmstd{89.21}{5.05} & \pmstd{88.63}{2.88} & \pmstd{37.34}{1.13} & \pmstd{51.97}{1.98} & \pmstd{68.95}{2.37} & \pmstd{85.40}{6.96} & \pmstd{76.85}{1.69} & \pmstd{87.72}{0.32} & \pmstd{87.89}{1.12} \\
    \textbf{DiagGegenbauerSD}   & \pmstd{89.21}{4.77} & \pmstd{87.65}{3.72} & \pmstd{37.06}{1.16} & \pmstd{52.97}{4.69} & \pmstd{70.39}{4.31} & \pmstd{85.68}{5.54} & \pmstd{77.23}{1.83} & \pmstd{89.61}{0.32} & \pmstd{88.63}{1.81}\\
    \textbf{BundleGegenbauerSD}   & \pmstd{88.95}{5.10} & \pmstd{87.65}{3.93} & \pmstd{36.99}{1.12} & \pmstd{49.81}{1.35} & \pmstd{69.87}{2.02} & \textcolor{Top1}{\textbf{\pmstd{86.76}{4.90}}} & \pmstd{77.43}{1.54} & \pmstd{89.61}{0.44} & \pmstd{88.33}{1.44}\\ \textbf{GeneralGegenbauerSD}   & \pmstd{88.68}{3.73} & \pmstd{88.04}{3.22} & \pmstd{37.20}{0.80} & \pmstd{49.47}{1.47} & \pmstd{66.25}{2.29} & \pmstd{85.13}{6.54} & \pmstd{77.18}{1.72} & \pmstd{89.67}{0.35} & \pmstd{87.85}{0.99}\\
    \textbf{DiagJacobiSD}   & \pmstd{88.68}{6.23} & \pmstd{88.63}{3.59} & \pmstd{37.03}{0.72} & \pmstd{52.30}{0.93} & \pmstd{66.14}{1.50} & \pmstd{86.22}{5.05} & \pmstd{77.35}{1.43} & \pmstd{89.09}{0.89} & \pmstd{88.53}{1.02}\\
    \textbf{BundleJacobiSD}   & \pmstd{88.95}{4.96} & \pmstd{88.82}{3.83} & \pmstd{37.18}{0.97} & \pmstd{55.76}{2.02} & \pmstd{68.22}{1.42} & \pmstd{85.68}{5.68} & \pmstd{76.97}{1.74} & \pmstd{89.72}{0.48} & \pmstd{88.13}{1.19}\\ \textbf{GeneralJacobiSD}   & \pmstd{88.68}{5.14} & \pmstd{88.82}{4.89} & \pmstd{37.03}{0.95} & \pmstd{52.30}{1.30} & \pmstd{66.62}{1.07} & \pmstd{85.68}{3.83} & \pmstd{76.80}{1.68} & \pmstd{89.53}{0.46} & \pmstd{87.66}{1.26}\\
    \textbf{PolySpectralGNN}& \pmstd{64.59}{6.40} & \pmstd{58.62}{6.04} & \pmstd{25.50}{0.85} & \pmstd{51.72}{1.76} & \pmstd{62.65}{3.03} & \pmstd{54.59}{6.26} & \pmstd{60.4}{0.89} & \pmstd{78.7}{0.56} & \pmstd{76.2}{0.67} \\
    \midrule
    ChebNet               &\pmstd{54.73}{7.04} & \pmstd{52.65}{4.43} & \pmstd{26.32}{6.46} & \pmstd{39.71}{1.25} & \pmstd{41.36}{2.45} & \pmstd{55.00}{6.45} & \pmstd{50.65}{2.09} & \pmstd{82.75}{4.65} & \pmstd{82.93}{3.76}\\
    ChebNetII             &\pmstd{68.42}{6.96} & \pmstd{64.76}{6.93} & \pmstd{32.39}{4.45} & \pmstd{47.06}{1.74} & \pmstd{63.23}{5.83} & \pmstd{80.40}{5.56} & \pmstd{72.00}{1.66} & \pmstd{85.87}{3.34} & \pmstd{85.35}{6.75} \\
    BernNet               &\pmstd{73.68}{6.00} & \pmstd{63.65}{6.76} & \pmstd{30.23}{6.34} & \pmstd{46.47}{1.47} & \pmstd{61.00}{2.34} & \pmstd{76.87}{5.26} & \pmstd{65.29}{5.55} & \pmstd{83.94}{3.83} & \pmstd{84.76}{4.75}\\
    \midrule
    RiSNN\,-\,NoT & \pmstd{87.89}{4.28} & \pmstd{88.04}{2.39} & \placeholder{N/A} & \pmstd{51.24}{1.71} & \pmstd{66.58}{1.81} & \pmstd{82.97}{6.17} & \pmstd{75.07}{1.56} & \pmstd{87.91}{0.55} & \pmstd{85.86}{1.31} \\
    RiSNN         & \pmstd{86.84}{3.72} & \pmstd{87.84}{2.60} & \placeholder{N/A} & \pmstd{53.30}{3.30} & \pmstd{65.15}{2.40} & \pmstd{85.95}{6.14} & \pmstd{76.23}{1.81} & \pmstd{88.00}{0.42} & \pmstd{85.27}{1.11} \\
    JdSNN\,-\,NoW & \pmstd{87.30}{4.53} & \pmstd{88.43}{2.83} & \placeholder{N/A} & \pmstd{51.28}{1.80} & \pmstd{66.45}{3.46} & \pmstd{84.59}{6.95} & \pmstd{75.93}{1.41} & \pmstd{88.09}{0.49} & \pmstd{84.39}{1.47} \\
    JdSNN         & \pmstd{87.37}{5.10} & \textcolor{Top3}{\textbf{\pmstd{89.22}{3.42}}} & \placeholder{N/A} & \pmstd{49.89}{1.71} & \pmstd{66.40}{2.33} & \pmstd{85.41}{4.55} & \pmstd{73.27}{1.86} & \pmstd{88.19}{0.55} & \pmstd{85.43}{1.73} \\
    Conn\,-\,NSD  & \pmstd{86.16}{2.24} & \pmstd{88.73}{4.47} & \textcolor{Top1}{\textbf{\pmstd{37.91}{1.28}}} & \pmstd{45.19}{1.57} & \pmstd{65.21}{2.04} & \pmstd{85.95}{7.72} & \pmstd{75.61}{1.93} & \pmstd{89.28}{0.38} & \pmstd{83.74}{2.19} \\
    SAN   & \pmstd{84.05}{5.33} & \pmstd{86.47}{3.87} & \pmstd{37.09}{1.18} & \pmstd{50.96}{1.40} & \pmstd{67.46}{1.90} & \pmstd{84.32}{5.64} & \pmstd{72.57}{1.50} & \pmstd{87.12}{0.30} & \pmstd{85.90}{1.85} \\
    ANSD  & \pmstd{85.68}{4.69} & \pmstd{87.45}{3.19} & \pmstd{37.66}{1.40} & \pmstd{54.39}{1.76} & \pmstd{68.38}{2.14} & \pmstd{84.59}{5.93} & \pmstd{76.81}{1.82} & \pmstd{89.21}{0.37} & \pmstd{87.20}{1.03} \\
    Diag\,-\,NSD & \pmstd{85.67}{6.95} & \pmstd{88.63}{2.75} & \pmstd{37.79}{1.01} & \pmstd{54.78}{1.81} & \pmstd{68.68}{1.73} & \pmstd{86.49}{7.35} & \pmstd{77.14}{1.85} & \pmstd{89.42}{0.43} & \pmstd{87.14}{1.06} \\
    O($d$)\,-\,NSD & \pmstd{85.95}{5.51} & \textcolor{Top2}{\textbf{\pmstd{89.41}{4.74}}} & \textcolor{Top2}{\textbf{\pmstd{37.81}{1.15}}} & \textcolor{Top2}{\textbf{\pmstd{56.34}{1.32}}} & \pmstd{68.04}{1.58} & \pmstd{84.86}{4.71} & \pmstd{76.70}{1.57} & \pmstd{89.49}{0.40} & \pmstd{86.90}{1.13} \\
    Gen\,-\,NSD  & \pmstd{82.97}{5.13} & \pmstd{89.21}{3.84} & \textcolor{Top3}{\textbf{\pmstd{37.80}{1.22}}} & \pmstd{53.17}{1.31} & \pmstd{67.93}{1.58} & \pmstd{85.68}{6.51} & \pmstd{76.32}{1.65} & \pmstd{89.33}{0.35} & \pmstd{87.30}{1.15} \\
    \specialrule{1pt}{1pt}{1pt} 
    GGCN      & \pmstd{84.86}{4.55} & \pmstd{86.86}{3.29} & \pmstd{37.54}{1.56} & \pmstd{55.17}{1.58} & \textcolor{Top3}{\textbf{\pmstd{71.14}{1.84}}} & \pmstd{85.68}{6.63} & \pmstd{77.14}{1.45} & \pmstd{89.15}{0.37} & \pmstd{87.95}{1.05} \\
    H2GCN     & \pmstd{84.86}{7.23} & \pmstd{87.65}{4.98} & \pmstd{35.70}{1.00} & \pmstd{36.48}{1.86} & \pmstd{60.11}{2.15} & \pmstd{82.70}{5.28} & \pmstd{77.11}{1.57} & \pmstd{89.49}{0.38} & \pmstd{87.87}{1.20} \\
    GPRGNN    & \pmstd{78.38}{4.36} & \pmstd{82.94}{4.21} & \pmstd{34.63}{1.22} & \pmstd{31.61}{1.24} & \pmstd{46.58}{1.71} & \pmstd{80.27}{8.11} & \pmstd{77.13}{1.67} & \pmstd{87.54}{0.38} & \pmstd{87.95}{1.18} \\
    FAGCN     & \pmstd{82.43}{6.89} & \pmstd{82.94}{7.95} & \pmstd{34.87}{1.25} & \pmstd{42.59}{0.79} & \pmstd{55.22}{3.19} & \pmstd{79.19}{9.79} & \placeholder{N/A} & \placeholder{N/A} & \placeholder{N/A} \\
    MixHop    & \pmstd{77.84}{7.73} & \pmstd{75.88}{4.90} & \pmstd{32.22}{2.34} & \pmstd{48.30}{1.48} & \pmstd{60.50}{2.53} & \pmstd{73.51}{6.34} & \pmstd{76.26}{1.33} & \pmstd{85.31}{0.61} & \pmstd{87.61}{0.85} \\
    GCNII     & \pmstd{77.57}{3.83} & \pmstd{80.39}{3.40} & \pmstd{37.44}{1.30} & \pmstd{38.47}{1.58} & \pmstd{63.86}{3.04} & \pmstd{77.86}{3.79} & \pmstd{77.33}{1.48} & \textcolor{Top1}{\textbf{\pmstd{90.15}{0.43}}} & \pmstd{88.37}{1.25} \\
    Geom\,-\,GCN & \pmstd{66.76}{2.72} & \pmstd{64.51}{3.66} & \pmstd{31.59}{1.15} & \pmstd{38.15}{0.92} & \pmstd{60.00}{2.81} & \pmstd{60.54}{3.67} & \textcolor{Top1}{\textbf{\pmstd{78.02}{1.15}}} & \textcolor{Top2}{\textbf{\pmstd{89.95}{0.47}}} & \pmstd{85.35}{1.57} \\
    PairNorm  & \pmstd{60.27}{4.34} & \pmstd{48.43}{6.14} & \pmstd{27.40}{1.24} & \pmstd{50.44}{2.04} & \pmstd{62.74}{2.82} & \pmstd{58.92}{3.15} & \pmstd{73.59}{1.47} & \pmstd{87.53}{0.44} & \pmstd{85.79}{1.01} \\
    GraphSAGE& \pmstd{82.43}{6.14} & \pmstd{81.18}{5.56} & \pmstd{34.23}{0.99} & \pmstd{41.61}{0.74} & \pmstd{58.73}{1.68} & \pmstd{75.95}{5.01} & \pmstd{76.04}{1.30} & \pmstd{88.45}{0.50} & \pmstd{86.90}{1.04} \\
    GCN       & \pmstd{55.14}{5.16} & \pmstd{51.76}{3.06} & \pmstd{27.32}{1.10} & \pmstd{53.43}{2.01} & \pmstd{64.82}{2.24} & \pmstd{60.54}{5.30} & \pmstd{76.50}{1.36} & \pmstd{88.42}{0.50} & \pmstd{86.98}{1.27} \\
    GAT       & \pmstd{52.16}{6.63} & \pmstd{49.41}{4.09} & \pmstd{27.44}{0.89} & \pmstd{40.72}{1.55} & \pmstd{60.26}{2.50} & \pmstd{61.89}{5.05} & \pmstd{76.55}{1.23} & \pmstd{87.30}{1.10} & \pmstd{86.33}{0.48} \\
    MLP       & \pmstd{80.81}{4.75} & \pmstd{85.29}{3.31} & \pmstd{36.53}{0.70} & \pmstd{28.77}{1.56} & \pmstd{46.21}{2.99} & \pmstd{81.89}{6.40} & \pmstd{74.02}{1.90} & \pmstd{75.69}{2.00} & \pmstd{87.16}{0.37} \\
    \bottomrule
  \end{tabular}
  \end{adjustbox}
\end{table}

%% file: tables/filtered_wikipedia_real.tex
\begin{table}[t]
    \centering
    \scriptsize
    \caption{Results on the filtered versions of the heterophilic
    \textsc{Chameleon} and \textsc{Squirrel} datasets by Platonov et al.~\cite{platonov2023critical}. We also include standard GNNs,
    heterophily-specific models, and spectral baselines. Higher is better.
    The top three models for each dataset are coloured by
    \textcolor{Top1}{\textbf{First}},
    \textcolor{Top2}{\textbf{Second}} and
    \textcolor{Top3}{\textbf{Third}}, respectively.}
    \label{tab:filtered_heterophily_results}
    \begin{tabular}{lcc}
        \toprule
        & \texttt{Chameleon-Filtered} & \texttt{Squirrel-Filtered} \\
        \midrule
        ResNet              & $36.73 \pm 4.71$ & $36.55 \pm 1.82$ \\
        ResNet+SGC          & $41.01 \pm 4.54$ & $38.36 \pm 1.97$ \\
        ResNet+adj          & $38.67 \pm 3.87$ & $38.37 \pm 1.99$ \\
        GCN                 & $40.89 \pm 4.12$ & $39.47 \pm 1.47$ \\
        SAGE                & $37.77 \pm 4.14$ & $36.09 \pm 1.99$ \\
        GAT                 & $39.21 \pm 3.08$ & $35.62 \pm 2.06$ \\
        GAT-sep             & $39.26 \pm 2.50$ & $35.46 \pm 3.10$ \\
        GT                  & $38.87 \pm 3.66$ & $36.30 \pm 1.98$ \\
        GT-sep              & $40.31 \pm 3.01$ & $36.66 \pm 1.63$ \\
        \midrule
        H$_2$GCN            & $26.75 \pm 3.64$ & $35.10 \pm 1.15$ \\
        CPGNN               & $33.00 \pm 3.15$ & $30.04 \pm 2.03$ \\
        GPR-GNN             & $39.93 \pm 3.30$ & $38.95 \pm 1.99$ \\
        FSGNN               & $40.61 \pm 2.97$ & $35.92 \pm 1.32$ \\
        GloGNN              & $25.90 \pm 3.58$ & $35.11 \pm 1.24$ \\
        FAGCN               & $41.90 \pm 2.72$ & $41.08 \pm 2.27$ \\
        GBK-GNN             & $39.61 \pm 2.60$ & $35.51 \pm 1.65$ \\
        JacobiConv          & $39.00 \pm 4.20$ & $29.71 \pm 1.66$ \\
        \midrule
        ChebNet             & $24.98 \pm 3.52$ & $35.05 \pm 1.22$ \\
        ChebNetII           & $40.82 \pm 4.96$ & $41.09 \pm 2.41$ \\
        BernNet             & $37.30 \pm 4.17$ & $42.41 \pm 1.97$ \\
        PolySpectralGNN     & $38.25 \pm 5.33$ & $41.24 \pm 1.87$ \\
        \midrule
        NSD-Diag            & $41.05 \pm 4.01$ & $43.11 \pm 2.09$ \\
        NSD-Bundle          & $40.56 \pm 2.67$ & $42.43 \pm 3.23$ \\
        NSD-General         & $42.72 \pm 2.06$ & $43.71 \pm 2.06$ \\
        \textbf{PolyNSD-Diag}    
                            & \textcolor{Top2}{$\mathbf{45.88 \pm 3.60}$} 
                            & \textcolor{Top2}{$\mathbf{48.60 \pm 1.85}$} \\
        \textbf{PolyNSD-Bundle}  
                            & \textcolor{Top3}{$\mathbf{44.64 \pm 3.78}$} 
                            & \textcolor{Top3}{$\mathbf{47.86 \pm 1.79}$} \\
        \textbf{PolyNSD-General} 
                            & \textcolor{Top1}{$\mathbf{45.94 \pm 3.79}$} 
                            & \textcolor{Top1}{$\mathbf{48.82 \pm 1.60}$} \\
        \bottomrule
    \end{tabular}
\end{table}

%% file: tables/platonov_table.tex
\begin{table}[t]
\centering
\scriptsize
\setlength{\tabcolsep}{3.2pt}
\renewcommand{\arraystretch}{0.95}
\caption{Results on the new heterophily benchmark datasets. Accuracy is reported
for \texttt{roman-empire} and \texttt{amazon-ratings}; ROC-AUC is reported for
\texttt{minesweeper}, \texttt{tolokers}, and \texttt{questions}. We report
PolyNSD results under a stricter hyper-parameter space. Higher is better.
The top three models for each dataset are coloured by
\textcolor{Top1}{\textbf{First}},
\textcolor{Top2}{\textbf{Second}} and
\textcolor{Top3}{\textbf{Third}}, respectively.}
\label{tab:new_heterophilic_results}
\begin{tabular}{lccccc}
\toprule
 & \texttt{roman-empire} 
 & \texttt{amazon-ratings} 
 & \texttt{minesweeper} 
 & \texttt{tolokers} 
 & \texttt{questions} \\
\midrule
ResNet     
    & $65.88 \pm 0.38$ 
    & $45.90 \pm 0.52$ 
    & $50.89 \pm 1.39$ 
    & $72.95 \pm 1.06$ 
    & $70.34 \pm 0.76$ \\
ResNet+SGC 
    & $73.90 \pm 0.51$ 
    & $50.66 \pm 0.48$ 
    & $70.88 \pm 0.90$ 
    & $80.70 \pm 0.97$ 
    & $75.81 \pm 0.96$ \\
ResNet+adj 
    & $52.25 \pm 0.40$ 
    & $51.83 \pm 0.57$ 
    & $50.42 \pm 0.83$ 
    & $78.78 \pm 1.11$ 
    & $75.77 \pm 1.24$ \\
\midrule
GCN     
    & $73.69 \pm 0.74$ 
    & $48.70 \pm 0.63$ 
    & $89.75 \pm 0.52$ 
    & $83.64 \pm 0.67$ 
    & $76.09 \pm 1.27$ \\
SAGE    
    & $85.74 \pm 0.67$ 
    & \textcolor{Top3}{$\mathbf{53.63 \pm 0.39}$} 
    & $93.51 \pm 0.57$
    & $82.43 \pm 0.44$ 
    & $76.44 \pm 0.62$ \\
GAT     
    & $80.87 \pm 0.30$ 
    & $49.09 \pm 0.63$ 
    & $92.01 \pm 0.68$ 
    & $83.70 \pm 0.47$ 
    & $77.43 \pm 1.20$ \\
GAT-sep 
    & \textcolor{Top2}{$\mathbf{88.75 \pm 0.41}$}
    & $52.70 \pm 0.62$
    & $93.91 \pm 0.35$
    & \textcolor{Top3}{$\mathbf{83.78 \pm 0.43}$}
    & $76.79 \pm 0.71$ \\
GT      
    & $86.51 \pm 0.73$
    & $51.17 \pm 0.66$ 
    & $91.85 \pm 0.76$ 
    & $83.23 \pm 0.64$ 
    & $77.95 \pm 0.68$ \\
GT-sep  
    & $87.32 \pm 0.39$
    & $52.18 \pm 0.80$ 
    & $92.29 \pm 0.47$ 
    & $82.52 \pm 0.92$ 
    & \textcolor{Top3}{$\mathbf{78.05 \pm 0.93}$} \\
\midrule
H$_2$GCN    
    & $60.11 \pm 0.52$ 
    & $36.47 \pm 0.23$ 
    & $89.71 \pm 0.31$ 
    & $73.35 \pm 1.01$ 
    & $63.59 \pm 1.46$ \\
CPGNN       
    & $63.96 \pm 0.62$ 
    & $39.79 \pm 0.77$ 
    & $52.03 \pm 5.46$ 
    & $73.36 \pm 1.01$ 
    & $65.96 \pm 1.95$ \\
GPR-GNN     
    & $64.85 \pm 0.27$ 
    & $44.88 \pm 0.34$ 
    & $86.24 \pm 0.61$ 
    & $72.94 \pm 0.97$ 
    & $55.48 \pm 0.91$ \\
FSGNN       
    & $79.92 \pm 0.56$ 
    & $52.74 \pm 0.83$
    & $90.08 \pm 0.70$ 
    & $82.76 \pm 0.61$ 
    & \textcolor{Top2}{$\mathbf{78.86 \pm 0.92}$} \\
GloGNN      
    & $59.63 \pm 0.69$ 
    & $36.89 \pm 0.14$ 
    & $51.08 \pm 1.23$ 
    & $73.39 \pm 1.17$ 
    & $65.74 \pm 1.19$ \\
FAGCN       
    & $65.22 \pm 0.56$ 
    & $44.12 \pm 0.30$ 
    & $88.17 \pm 0.73$ 
    & $77.75 \pm 1.05$ 
    & $77.24 \pm 1.26$ \\
GBK-GNN     
    & $74.57 \pm 0.47$ 
    & $45.98 \pm 0.71$ 
    & $90.85 \pm 0.58$ 
    & $81.01 \pm 0.67$ 
    & $74.47 \pm 0.86$ \\
JacobiConv  
    & $71.14 \pm 0.42$ 
    & $43.55 \pm 0.48$ 
    & $89.66 \pm 0.40$ 
    & $68.66 \pm 0.65$ 
    & $73.88 \pm 1.16$ \\
\midrule
NSD-Diag            
    & $76.82 \pm 2.49$ 
    & $38.79 \pm 0.39$ 
    & $86.34 \pm 0.69$ 
    & $75.31 \pm 1.14$ 
    & $69.94 \pm 2.04$ \\
NSD-Bundle            
    & $78.75 \pm 2.23$ 
    & $40.23 \pm 0.23$ 
    & $87.43 \pm 0.96$ 
    & $78.65 \pm 2.94$ 
    & $66.23 \pm 1.49$ \\
NSD-General           
    & $80.41 \pm 0.72$ 
    & $42.76 \pm 0.54$ 
    & $92.15 \pm 0.84$ 
    & $80.31 \pm 0.76$ 
    & $69.69 \pm 1.46$ \\
\textbf{PolyNSD-Diag}        
    & $87.49 \pm 2.21$ 
    & $52.87 \pm 0.24$ 
    & \textcolor{Top3}{$\mathbf{96.98 \pm 0.55}$} 
    & $82.27 \pm 0.61$ 
    & $75.37 \pm 1.08$ \\
\textbf{PolyNSD-Bundle}        
    & \textcolor{Top3}{$\mathbf{88.54 \pm 1.65}$} 
    & \textcolor{Top2}{$\mathbf{53.87 \pm 0.28}$} 
    & \textcolor{Top2}{$\mathbf{97.23 \pm 0.55}$} 
    & \textcolor{Top3}{$\mathbf{83.92 \pm 0.61}$} 
    & $76.37 \pm 1.08$ \\
\textbf{PolyNSD-General}        
    & \textcolor{Top1}{$\mathbf{89.87 \pm 1.36}$} 
    & \textcolor{Top1}{$\mathbf{54.34 \pm 0.34}$} 
    & \textcolor{Top1}{$\mathbf{98.86 \pm 0.55}$} 
    & \textcolor{Top1}{$\mathbf{84.75 \pm 0.61}$} 
    & \textcolor{Top1}{$\mathbf{79.37 \pm 1.08}$} \\
PolySpectralGNN     
    & $59.16 \pm 1.64$ 
    & $36.98 \pm 0.64$ 
    & $86.59 \pm 0.60$ 
    & $78.20 \pm 0.61$ 
    & $69.40 \pm 1.74$ \\
\midrule
ChebNet             
    & $53.96 \pm 0.32$ 
    & $36.79 \pm 0.07$ 
    & $70.00 \pm 0.45$ 
    & $70.00 \pm 0.43$ 
    & $60.00 \pm 0.40$ \\
ChebNetII           
    & $63.86 \pm 0.43$ 
    & $38.82 \pm 0.07$ 
    & $74.73 \pm 1.51$ 
    & $75.09 \pm 0.87$ 
    & $68.50 \pm 1.55$ \\
BernNet             
    & $67.46 \pm 4.38$ 
    & $36.79 \pm 0.05$ 
    & $80.94 \pm 1.16$ 
    & $76.19 \pm 0.85$ 
    & $66.35 \pm 1.23$ \\
\bottomrule
\end{tabular}
\end{table}

%% file: tables/60-20-20_benchmarks.tex
\begin{table}[t]
    \centering
    \scriptsize
    \caption{Results on the malignant heterophily datasets \textsc{Texas},
    \textsc{Wisconsin}, \textsc{Film}, and \textsc{Cornell}. We follow the
    protocol of Luan et al.~\cite{luan2024re}, using 10 random splits with
    $60\%/20\%/20\%$ train/validation/test proportions. We include the
    graph-aware and graph-agnostic baselines reported in their benchmark,
    together with NSD, PolyNSD, and spectral baselines. Higher is better.
    The top three models for each dataset are coloured by
    \textcolor{Top1}{\textbf{First}},
    \textcolor{Top2}{\textbf{Second}} and
    \textcolor{Top3}{\textbf{Third}}, respectively.}
    \label{tab:malignant_heterophily_results}
    \resizebox{0.8\columnwidth}{!}{%
    \begin{tabular}{lcccc}
        \toprule
        & \texttt{Texas} & \texttt{Wisconsin} & \texttt{Film} & \texttt{Cornell} \\
        \midrule
        GCN        
            & $83.11 \pm 3.20$ 
            & $75.50 \pm 2.92$ 
            & $35.51 \pm 0.99$ 
            & $82.46 \pm 3.11$ \\
        MLP-2      
            & $92.26 \pm 0.71$ 
            & \textcolor{Top3}{$\mathbf{93.87 \pm 3.33}$} 
            & $38.58 \pm 0.25$ 
            & $91.30 \pm 0.70$ \\
        SGC-1      
            & $83.28 \pm 5.43$ 
            & $70.38 \pm 2.85$ 
            & $25.26 \pm 1.18$ 
            & $70.98 \pm 8.39$ \\
        MLP-1      
            & \textcolor{Top1}{$\mathbf{93.77 \pm 3.34}$} 
            & \textcolor{Top3}{$\mathbf{93.87 \pm 3.33}$} 
            & $34.53 \pm 1.48$ 
            & \textcolor{Top1}{$\mathbf{93.77 \pm 3.34}$} \\
        \midrule
        NSD-Diag              
            & $86.84 \pm 4.40$ 
            & $88.82 \pm 3.04$ 
            & $37.47 \pm 1.18$ 
            & $85.79 \pm 5.02$ \\
        NSD-Bundle            
            & $87.89 \pm 3.93$ 
            & $89.21 \pm 4.49$ 
            & $37.58 \pm 1.15$ 
            & $84.74 \pm 6.53$ \\
        NSD-General           
            & $87.37 \pm 4.37$ 
            & $88.63 \pm 2.74$ 
            & $36.93 \pm 1.11$ 
            & $85.53 \pm 4.60$ \\
        \textbf{PolyNSD-Diag}    
            & \textcolor{Top2}{$\mathbf{92.89 \pm 4.59}$} 
            & \textcolor{Top2}{$\mathbf{94.04 \pm 4.42}$} 
            & \textcolor{Top1}{$\mathbf{41.48 \pm 1.24}$} 
            & \textcolor{Top2}{$\mathbf{92.84 \pm 4.08}$} \\
        \textbf{PolyNSD-Bundle}  
            & $92.37 \pm 4.96$ 
            & $93.65 \pm 4.30$ 
            & \textcolor{Top3}{$\mathbf{40.43 \pm 0.97}$} 
            & $91.26 \pm 4.43$ \\
        \textbf{PolyNSD-General} 
            & \textcolor{Top3}{$\mathbf{92.74 \pm 3.21}$} 
            & \textcolor{Top1}{$\mathbf{94.62 \pm 3.59}$} 
            & \textcolor{Top2}{$\mathbf{41.33 \pm 0.98}$} 
            & \textcolor{Top3}{$\mathbf{92.26 \pm 5.02}$} \\
        PolySpectralGNN       
            & $84.47 \pm 5.19$ 
            & $80.20 \pm 4.34$ 
            & $36.68 \pm 1.12$ 
            & $75.79 \pm 6.20$ \\
        \midrule
        ChebNet               
            & $80.79 \pm 7.81$ 
            & $65.25 \pm 6.98$ 
            & $25.51 \pm 1.00$ 
            & $68.68 \pm 5.89$ \\
        ChebNetII             
            & $82.37 \pm 8.07$ 
            & $85.88 \pm 3.01$ 
            & $36.53 \pm 1.11$ 
            & $83.42 \pm 4.71$ \\
        BernNet               
            & $80.00 \pm 9.57$ 
            & $80.58 \pm 5.55$ 
            & $36.55 \pm 1.32$ 
            & $72.89 \pm 9.45$ \\
        \bottomrule
    \end{tabular}%
    }
\end{table}

%% file: tables/city_network.tex
\begin{table*}[t]
\centering
\scriptsize
\caption{\textit{Transductive long-range results on \textsc{CityNetworks}.}
We report accuracy (\%) on \textsc{Paris}, \textsc{Shanghai},
\textsc{Los Angeles}, and \textsc{London}. Baseline results are the best
reported configurations at depth $L=16$ for all models, except for the \emph{Sheaf} models that are reported at depth $L=6$. Higher is better. The top three models
for each dataset are coloured by
\textcolor{Top1}{\textbf{First}},
\textcolor{Top2}{\textbf{Second}} and
\textcolor{Top3}{\textbf{Third}}, respectively.}
\label{tab:citynetworks_transductive_results_full}
\setlength{\tabcolsep}{4.5pt}
\renewcommand{\arraystretch}{1.08}
\begin{adjustbox}{max width=\textwidth}
\begin{tabular}{llccccc}
\toprule
\textbf{Family} & \textbf{Method}
& \textbf{Paris}
& \textbf{Shanghai}
& \textbf{Los Angeles}
& \textbf{London}
& \textbf{Avg.} \\
\midrule
\multirow{7}{*}{MPNNs}
& MLP
& $25.50 \pm 0.40$
& $28.40 \pm 0.60$
& $24.10 \pm 0.50$
& $27.90 \pm 0.10$
& $26.48$ \\

& ChebNet
& $54.10 \pm 0.20$
& $66.50 \pm 0.10$
& $61.40 \pm 0.40$
& $54.70 \pm 0.20$
& $59.18$ \\

& GCN
& $53.20 \pm 0.30$
& $62.10 \pm 0.20$
& $58.30 \pm 0.30$
& $50.10 \pm 0.70$
& $55.93$ \\

& GraphSAGE
& $54.60 \pm 0.20$
& \textcolor{Top3}{\textbf{$68.30 \pm 0.50$}}
& $61.40 \pm 0.30$
& $55.40 \pm 0.20$
& $59.93$ \\

& GAT
& $51.10 \pm 0.30$
& $68.00 \pm 0.50$
& $59.50 \pm 0.30$
& $52.00 \pm 0.30$
& $57.65$ \\

& GCNII
& $51.30 \pm 0.20$
& $61.50 \pm 0.40$
& $56.00 \pm 0.30$
& $48.20 \pm 0.30$
& $54.25$ \\

& DropEdge
& $48.20 \pm 0.20$
& $60.80 \pm 0.40$
& $55.50 \pm 0.30$
& $45.00 \pm 0.30$
& $52.38$ \\
\midrule
\multirow{3}{*}{GTs}
& GraphGPS
& $52.10 \pm 0.60$
& $63.00 \pm 0.50$
& $59.80 \pm 0.50$
& OOM
& -- \\

& Exphormer
& $55.10 \pm 0.80$
& \textcolor{Top1}{\textbf{$70.20 \pm 0.40$}}
& \textcolor{Top1}{\textbf{$63.80 \pm 0.60$}}
& $49.50 \pm 0.40$
& $59.65$ \\

& SGFormer
& $52.00 \pm 0.80$
& $64.10 \pm 0.30$
& $60.10 \pm 0.70$
& $48.30 \pm 0.30$
& $56.13$ \\
\midrule
\multirow{3}{*}{Spectral}
& ChebNet
& $41.92 \pm 7.65$
& $59.34 \pm 0.00$
& $53.34 \pm 5.52$
& $51.25 \pm 4.42$
& $51.46$ \\

& ChebNetII
& \textcolor{Top3}{\textbf{$49.89 \pm 2.33$}}
& $66.83 \pm 0.00$
& $58.34 \pm 4.72$
& \textcolor{Top2}{\textbf{$57.37 \pm 5.45$}}
& $58.11$ \\

& BernNet
& $47.43 \pm 7.23$
& $65.23 \pm 0.00$
& $57.77 \pm 4.45$
& \textcolor{Top3}{\textbf{$56.75 \pm 4.91$}}
& $56.80$ \\
\midrule
\multirow{3}{*}{Sheaf}
& NSD-Diag
& $34.55 \pm 4.76$
& $43.85 \pm 0.78$
& $34.54 \pm 3.22$
& $38.14 \pm 2.67$
& $37.77$ \\

& \textbf{PolyNSD-Diag}
& \textcolor{Top1}{\textbf{$57.43 \pm 2.28$}}
& \textcolor{Top2}{\textbf{$68.57 \pm 0.34$}}
& \textcolor{Top2}{\textbf{$61.82 \pm 4.23$}}
& \textcolor{Top1}{\textbf{$58.72 \pm 3.33$}}
& \textcolor{Top1}{\textbf{$61.64$}} \\

& PolySpectralGNN
& $45.72 \pm 3.36$
& $57.78 \pm 0.00$
& $55.67 \pm 3.45$
& $51.25 \pm 2.04$
& $52.61$ \\
\bottomrule
\end{tabular}
\end{adjustbox}
\end{table*}

%% file: tables/city_network_runtime.tex
\begin{table}[t]
\centering
\scriptsize
\caption{\textit{Runtime-aware comparison on \textsc{CityNetworks}.}
Training and inference time per epoch are normalised so that
\textsc{PolyNSD-Diag}$=1.00\times$. Lower is better for runtime, while higher
is better for accuracy.}
\label{tab:citynetworks_runtime}
\setlength{\tabcolsep}{5pt}
\renewcommand{\arraystretch}{1.08}
\begin{tabular}{lccc}
\toprule
\textbf{Method}
& \textbf{Avg. Acc.}
& \textbf{Train / epoch}
& \textbf{Infer / epoch} \\
\midrule
NSD-Diag
& $37.77$
& $\sim 1.60\times$
& $\sim 1.45\times$ \\

\textbf{PolyNSD-Diag}
& \textbf{61.64}
& \textbf{$1.00\times$}
& \textbf{$1.00\times$} \\

PolySpectralGNN
& $52.61$
& $\sim 0.85\times$
& $\sim 0.82\times$ \\
\midrule
ChebNet
& $51.46$
& $\sim 0.83\times$
& $\sim 0.80\times$ \\

ChebNetII
& $58.11$
& $\sim 0.90\times$
& $\sim 0.87\times$ \\

BernNet
& $56.80$
& $\sim 0.92\times$
& $\sim 0.89\times$ \\
\bottomrule
\end{tabular}
\end{table}

%% file: tables/federated_learning.tex
\begin{table}[t]
      \centering
      \scriptsize
      \caption{\textit{Accuracy (\%) on \textsc{Roman-Empire} with ten clients
      ($K=10$).} FedATH~\cite{fu2025less} is a causal federated graph learning
      pipeline. We compare standard federated learning baselines, federated
      graph learning baselines, FedATH with conventional graph backbones, and
      FedATH with sheaf/PolyNSD backbones. Higher is better. The top three
      models are coloured by \textcolor{Top1}{\textbf{First}},
      \textcolor{Top2}{\textbf{Second}} and
      \textcolor{Top3}{\textbf{Third}}, respectively.}
      \label{tab:fedath_roman_empire_k10}
      \setlength{\tabcolsep}{4pt}
      \renewcommand{\arraystretch}{1.08}
      \begin{tabular}{llc}
        \toprule
        \textbf{Type} & \textbf{Method} & \textbf{\textsc{Roman-Empire} ($K=10$)} \\
        \midrule
        BL & FedAvg & 34.41 \\
        \midrule
        \multirow{4}{*}{FL}
          & FedProx & 34.30 \\
          & MOON & 33.97 \\
          & FedOPT & 34.35 \\
          & FedProto & 34.93 \\
        \midrule
        \multirow{8}{*}{FGL}
          & FedSage+ & 41.59 \\
          & FGSSL & 36.96 \\
          & FedPUB & 37.31 \\
          & FedTAD & 39.01 \\
          & FedATH-GCN & 48.18 \\
          & FedATH-GATv2 & 14.24 \\
          & FedATH-GIN & 25.69 \\
          & FedATH-Transformer & 13.47 \\
        \midrule
        \multirow{6}{*}{RGSheafMask}
          & FedCausalSheaf-Diag & 68.44 \\
          & FedCausalSheaf-Bundle & 70.07 \\
          & FedCausalSheaf-General & 72.31 \\
          & FedCausalSheaf-DiagPoly & 71.07 \\
          & FedCausalSheaf-BundlePoly & 72.05 \\
          & FedCausalSheaf-GeneralPoly & 76.19 \\
        \midrule
        \multirow{6}{*}{FullSheafMask}
          & LoadRMaps-FedCausalSheaf-Diag & 70.84 \\
          & LoadRMaps-FedCausalSheaf-Bundle & 75.27 \\
          & \textbf{LoadRMaps-FedCausalSheaf-General} & \textcolor{Top3}{\textbf{76.51}} \\
          & LoadRMaps-FedCausalSheaf-DiagPoly & 73.66 \\
          & \textbf{LoadRMaps-FedCausalSheaf-BundlePoly} & 
            \textcolor{Top2}{\textbf{78.00}} \\
          & \textbf{LoadRMaps-FedCausalSheaf-GeneralPoly} & 
            \textcolor{Top1}{\textbf{80.24}} \\
        \bottomrule
      \end{tabular}
\end{table}